\documentclass[11pt]{article}
\usepackage[top=1in, bottom=1in, left=1in, right=1in]{geometry}

\usepackage{sidecap}
\usepackage{microtype}
\usepackage{graphicx}
\usepackage{subfigure}
\usepackage{booktabs}

\usepackage{hyperref}

\usepackage{amsmath}
\usepackage{amssymb}
\usepackage{mathtools}
\usepackage{amsthm}

\usepackage[capitalize,noabbrev]{cleveref}

\theoremstyle{plain}
\newtheorem{theorem}{Theorem}[section]
\newtheorem{proposition}[theorem]{Proposition}
\newtheorem{lemma}[theorem]{Lemma}

\theoremstyle{definition}
\newtheorem{definition}[theorem]{Definition}

\theoremstyle{remark}
\newtheorem{remark}[theorem]{Remark}

\usepackage[textsize=tiny]{todonotes}

\usepackage[utf8]{inputenc}
\usepackage{caption}
\usepackage{tcolorbox}

\usepackage{savesym}
\savesymbol{Bbbk}
\usepackage{amsthm,amsmath,amssymb,amsfonts}
\usepackage{tikz}
\usepackage{mathrsfs}
\usepackage{nicefrac,bbm}
\usepackage{pifont,enumitem}
\usepackage[framemethod=TikZ]{mdframed}
\usepackage{tabularx}
\usepackage{algorithm}
\usepackage{algpseudocode}[1]

\usetikzlibrary{math}
\usepackage{dsfont}

\usepackage{wrapfig}

\newtheorem{fact}[theorem]{Fact}

\newtheorem{problem}[theorem]{Problem}
\crefname{section}{Section}{Sections}

\crefname{theorem}{Theorem}{Theorems}
\crefname{assumption}{Assumption}{Assumptions}
\crefname{lemma}{Lemma}{Lemmas}
\crefname{definition}{Definition}{Definitions}
\crefname{conjecture}{Conjecture}{Conjectures}
\crefname{corollary}{Corollary}{Corollaries}
\crefname{construction}{Construction}{Constructions}
\crefname{claim}{Claim}{Claims}
\crefname{observation}{Observation}{Observations}
\crefname{proposition}{Proposition}{Propositions}
\crefname{fact}{Fact}{Facts}
\crefname{question}{Question}{Questions}
\crefname{problem}{Problem}{Problems}
\crefname{remark}{Remark}{Remarks}
\crefname{example}{Example}{Examples}
\crefname{equation}{Equation}{Equations}
\crefname{appendix}{Appendix}{Appendices}

\newcommand{\yesnum}{\addtocounter{equation}{1}\tag{\theequation}}
\newcommand{\tagnum}[1]{\addtocounter{equation}{1}{\tag{#1)\ \ (\theequation}}}
\makeatletter
\newcommand{\customlabel}[2]{%
\protected@write \@auxout {}{\string \newlabel {#1}{{#2}{\thepage}{#2}{#1}{}} }%
\hypertarget{#1}{}
}
\makeatother

\mdfdefinestyle{FrameBox}{ linecolor=white, outerlinewidth=0pt, roundcorner=5pt,
innertopmargin=5pt, innerbottommargin=5pt,
innerrightmargin=10pt, innerleftmargin=10pt, backgroundcolor=white}

\mdfdefinestyle{FrameBox2}{ linecolor=black, outerlinewidth=0pt, roundcorner=5pt,
innertopmargin=5pt, innerbottommargin=5pt,
innerrightmargin=10pt, innerleftmargin=10pt, backgroundcolor=white}

\mdfdefinestyle{FrameQuestion2}{linecolor=white, outerlinewidth=0pt,
roundcorner=5pt,
innertopmargin=5pt, innerbottommargin=0pt,
innerrightmargin=16pt, innerleftmargin=16pt, backgroundcolor=white}

\newcommand{\white}[1]{\textcolor{white}{#1}}

\newcommand{\N}{\mathbb{N}}

\newcommand{\R}{\mathbb{R}}
\newcommand{\Z}{\mathbb{Z}}

\newcommand{\cC}{\mathcal{C}}

\newcommand{\cR}{\mathcal{R}}

\newcommand{\evE}{\ensuremath{\mathscr{E}}}
\newcommand{\evF}{\ensuremath{\mathscr{F}}}

\newcommand{\wt}{\widetilde}

\newcommand{\st}{\mathrm{s.t.}}

\newcommand{\conv}{\mathrm{conv}}

\newcommand{\eps}{\varepsilon}
\renewcommand{\epsilon}{\varepsilon}

\newcommand{\argmax}{\operatornamewithlimits{argmax}}
\newcommand{\Ex}{\operatornamewithlimits{\mathbb{E}}}

\newcommand{\poly}{\mathop{\mbox{\rm poly}}}

\def\abs#1{\left| #1 \right|}
\def\sabs#1{| #1 |}

\newcommand{\sinparen}[1]{(#1)}

\newcommand{\inparen}[1]{\left(#1\right)}
\newcommand{\inbrace}[1]{\left\{#1\right\}}
\newcommand{\insquare}[1]{\left[#1\right]}
\newcommand{\inangle}[1]{\left\langle#1\right\rangle}
\newcommand{\floor}[1]{\left\lfloor#1\right\rfloor}

\newcommand{\ceil}[1]{\left\lceil#1\right\rceil}

\newcommand{\zo}{\{0,1\}}
\newcommand{\np}{{\bf NP}}

\newcommand{\wh}[1]{\widehat{#1}}

\newcommand{\hP}{\smash{\widehat{P}}}
\newcommand{\hG}{\smash{\widehat{G}}}

\newcommand{\Eqref}[1]{Equation~\eqref{#1}}

\newcommand{\prog}[1]{Program~\eqref{#1}}

\newcommand{\Stackrel}[2]{\stackrel{\mathmakebox[\widthof{\ensuremath{#2}}]{#1}}{#2}}

\newcommand{\folder}{./figures/}

\newcommand{\uncons}{{\bf Uncons}}
\newcommand{\csv}{{\bf CSV}}
\newcommand{\sj}{{\bf SJ}}
\newcommand{\mc}{{\bf MC}}
\newcommand{\detgreedy}{{\bf GAK}}
\newcommand{\ouralgo}{{\bf NResilient}}
\newcommand{\rd}{{\rm RD}}
\newcommand{\prd}{Prop-\rd{}}
\newcommand{\selift}{{\rm SL}}

\newif\ifconf
\conftrue

\newcommand{\leftmarginINTERNAL}{18pt}
\ifconf
\renewcommand{\leftmarginINTERNAL}{18pt}
\else
\renewcommand{\leftmarginINTERNAL}{\leftmargin}
\fi

\ifconf

\else

\fi

\usepackage{etoolbox}

\title{\bf Fair Ranking with Noisy Protected Attributes}

\author{Anay Mehrotra \\ Yale University \and Nisheeth K. Vishnoi \\ Yale University}

\begin{document}

\maketitle

\begin{abstract}
  The fair-ranking problem, which asks to rank a given set of items to maximize utility subject to group fairness constraints, has received attention in the fairness, information retrieval, and machine learning literature. Recent works, however, observe that errors in socially-salient (including protected) attributes of items can significantly undermine fairness guarantees of existing fair-ranking algorithms and raise the problem of mitigating the effect of such errors. We study the fair-ranking problem under a model where socially-salient attributes of items are randomly and independently perturbed. We present a fair-ranking framework that incorporates group fairness requirements along with probabilistic information about perturbations in socially-salient attributes. We provide provable guarantees on the fairness and utility attainable by our framework and show that it is information-theoretically impossible to significantly beat these guarantees. Our framework works for multiple non-disjoint  attributes and a general class of fairness constraints that includes proportional and equal representation. Empirically, we observe that, compared to baselines, our algorithm outputs rankings with higher fairness, and has a similar or better fairness-utility trade-off compared to baselines.
\end{abstract}

\newpage

\tableofcontents
\addtocontents{toc}{\protect\setcounter{tocdepth}{2}}

\newpage

\section{Introduction}\label{sec:intro}
Given a query and a set of $m$  items, ranking problems require one to output an ordering of a small subset of  items in decreasing order of {\em relevance} to the query.
Such ranking problems have been extensively studied in the information retrieval \cite{IRbook} and the machine learning \cite{liu2011learning} literature, and algorithms for them are used in applications such as search engines, personalized feed generators, and online recruiting platforms \cite{liu2010personalized,burges2010ranknet,googleLTR}.
Several studies have observed that when the outputs of ranking algorithms are consumed by end-users, e.g., image results for occupation-related queries, articles with different political leanings, and job applicants in online recruiting, {the outputs} can mislead or alter their perceptions about  socially-salient groups \cite{KayMM15}, polarize their opinions \cite{Epstein2015,polarizationWSJ2020}, and affect economic opportunities available to individuals \cite{hannak2017bias}. %
A reason is that  {relevance} (or utilities) input to ranking algorithms may be influenced by  human or societal biases, leading to  output rankings that skew  representations of
{socially-salient, and often legally-protected, groups such as women and Black people \cite{Noble2018}.}

A growing number of works  aim to make the output of ranking algorithms {\em fair} with respect to socially-salient attributes \cite{fair_ranking_survey1,fair_ranking_survey2,overviewFairRanking}.
As for notions of fairness, in the case when  each  item belongs to one of two socially-salient  groups ($G_1$ or $G_2$),
equal representation  requires that, for every $k$, (roughly) $\frac{k}{2}$  items from each of $G_1$ and $G_2$ appear in the first $k$ positions of the output ranking.
Proportional representation  requires that at most $k\cdot\frac{\abs{G_\ell}}{m}$ items from  each $G_\ell$ appear in the first $k$ positions.
Fairness criteria that generalize proportional representation and involve $p\geq 2$ groups $G_1,\ldots,G_p$, where each item may belong to multiple groups, have also been considered:
Given values $U_{k\ell}$, {they} require that at most $U_{k\ell}$ items from $G_\ell$  appear in the first $k$ positions of the output ranking
\cite{fairExposureAshudeep,celis2018ranking}.
One set of works in the fair-ranking  literature tries to improve fairness in utility-estimation \cite{YangS17,policyLearningAshudeep,ReducingDisparateExposureZehlike,MorikSHJ20}.
Such approaches have the benefit that no changes to the existing ranking algorithm are necessary
but they may be unable to guarantee that the output ranking satisfies the required  fairness criteria \cite{linkedin_ranking_paper}.
Another set of works use the given utilities as-it-is and change the ranking algorithm to output the ranking with the highest utility subject to satisfying the specified fairness criteria by including them as  {\em fairness constraints}
\cite{fairExposureAshudeep,AmortizedFairnessBiega2018,celis2018ranking,linkedin_ranking_paper,GorantlaUnderranking21}.
While these latter approaches can guarantee fairness, they require coming up with new algorithms to solve the arising constrained ranking problems.
Both approaches, however,  rely on knowledge of  the socially-salient attributes of the items {\cite{criticalReviewFairRanking22}.} %

Assuming precise access to socially-salient attributes is  reasonable in some contexts and has led to successful deployment of fair-ranking frameworks; see  \cite{linkedin_ranking_paper}.
However, in several contexts, socially-salient attributes can be erroneous, missing, or known only probabilistically.
For instance, errors can arise due to misreporting, which is a common concern with self-reported attributes \cite{Andrus2021WhatWeCantMeasure}.
Attributes can also be missing, as is the case with images in web-search or in settings where it is illegal to collect certain socially-salient attributes \cite{ChenKMSU19}.
Often attributes are predicted using ML-classifiers, but such prediction has inaccuracies \cite{BuolamwiniG18}.
In such cases, one can calibrate the confidence scores of classifiers to derive (aggregate) probabilistic information about the true attributes \cite{jung2020multicalibration}.
Moreover, probabilistic information about socially-salient (protected) attributes can be sometimes computed from other attributes.
For instance, name and location of an individual, combined with aggregate census data may be used to get a conditional distribution of their race
\cite{elliott2009UsingCencusSurnameList,KallusMZ20,ChenKMSU19}.
Even accurate attributes may be randomly and independently flipped to preserve user privacy{, and the distribution of flipped attributes is determined by {public parameters of, e.g., the randomized response mechanism \cite{KasiviswanathanLNRS11,YangZ18}.}}

Several  models  of inaccuracies in data have been proposed \cite{ManwaniS13,FrenayV14}.
We consider one such model (due to \cite{AngluinL87}) to capture inaccuracies in socially-salient attributes.
Each item $i$ belongs to the $\ell$-th group with a known probability $P_{i\ell}$.
For each item $i$, the distribution corresponding to $P_{i\ell}$s over groups is assumed to be independent of corresponding distributions of other items.
This model can be used in cases where these probabilities are available or can be derived, as in some of the aforementioned examples {(see \cref{sec:empirical_results} and \cref{sec:dis_noise_model})}.
In other cases, e.g., when errors are strategic or adversarial, other  models are needed.
This model and its variants have also been used by works on designing fair algorithms in the presence of inaccuracies, for problems including  classification \cite{LamyZ19,wang2020robust,wang2021label,celis2021fairclassification}, subset {selection {\cite{MehrotraC21}, and clustering \cite{prob_fair_clustering} {(\cref{sec:related_work} briefly discusses these works).}}}

In this noise model, while  socially-salient attributes are not explicitly specified, one could still use existing fair-ranking algorithms by first sampling groups for items from the given probabilities.
Indeed, \cite{GhoshDW21} evaluate existing fair-ranking algorithms on attributes obtained from the probabilities derived from  ML classifiers.
They find that ``errors in [socially-salient attributes] can dramatically undermine fair-ranking algorithms'' and can cause ``[non-disadvantaged groups] to become disadvantaged after a `fair' re-ranking.''
We confirm this observation on a synthetic dataset when the goal is to finding a ranking that satisfies equal representation (\cref{sec:simulation_syn_disp_err}).
We assigned each item the socially-salient group that is  most likely and find that when existing fair-ranking algorithms (for equal representation) are run with this group information, they output rankings that significantly violate the equal representation criteria (\cref{fig:different_fdr}).
Further, we mathematically analyze two natural methods to sample groups from probabilities and give examples where taking such information as input, existing fair-ranking algorithms output rankings which provably violate the equal representation criteria (\cref{sec:rounding_methods}).
Thus, new ideas are needed to design fair-ranking frameworks that can guarantee given fairness criteria under this noise model.

\paragraph{Our Contributions.}

We present a fair-ranking framework that guarantees given fairness criteria when the  socially-salient attributes are assumed to follow the probabilistic noise model mentioned above.
In particular, it finds a utility maximizing ranking subject to  a class of constraints that only rely on given probability distributions (\prog{prog:noisy_fair_boxed}).
These constraints relax the given fairness criteria by a carefully chosen factor: for equal representation, the  relaxation is by roughly a  $1+\frac{1}{\sqrt{k}}$ multiplicative factor for position $k$ {for any $k$}.
Moreover, instead of sampling the attribute values and applying constraints on them,
these constraints apply the  relaxed-fairness criteria to the expected number of items from each group that appear in the first $k$ positions.
We show that these constraints ensure that any ranking {approximately satisfying} the given fairness criteria is feasible for them and any ranking feasible for them approximately satisfies the given fairness criteria (\cref{thm:ub}).
Our fair-ranking framework works for the general class of fairness criteria introduced earlier, which involve multiple overlapping groups $G_1,\dots,G_p$ and upper bound $U_{k\ell}$ for the $\ell$-th group and $k$-th position (\cref{thm:ub}), and for their position-weighted versions (\cref{thm:ub_exnt_wt}).

We show that our fair-ranking framework, besides nearly satisfying the given fairness criteria, has a provably high utility (\cref{thm:ub}).
Complementing \cref{thm:ub}, we prove near-tightness of the fairness guarantee (\cref{thm:lb}):
for equal representation fairness criteria, this results shows that it is information theoretically impossible to output a ranking that violates this criteria by less than a multiplicative factor of $1+\wt{O}\inparen{\frac{1}{
\sqrt{k}}}$ at the $k$-th position for any $k$.
Finally, we give a polynomial-time  algorithm to approximately solve \prog{prog:noisy_fair_boxed}  (\cref{thm:algo}).

Empirically, we evaluate our framework on both synthetic and real-world data against standard metrics like weighted-risk difference (RD) that measure deviation from specific fairness criteria (\cref{sec:empirical_results}).
We compare its performance to key baselines \cite{celis2018ranking,fairExposureAshudeep,linkedin_ranking_paper,MehrotraC21}  {on both single and multiple attributes.}
In all simulations, compared to baselines, our framework has a higher {maximum} fairness (2-10\% for RD{; \cref{fig:different_fdr,fig:simulation_image,fig:simulation_intersectional}}) and a similar/better fairness-utility trade-off {(\cref{fig:fig1bskjkj,fig:simulation_image,fig:fig6,fig:fig7,fig:fig8,fig:fig9})}.

\newcommand{\negvspace}{\vspace{0mm}}

\section{Related Work}\label{sec:related_work}

\vspace{-1.25mm}

\paragraph{Relevance Estimation in Information Retrieval.} Work on automated information retrieval dates back to 1940s \cite{LiddyAutomatic05,cleverdon1991significance}.
Since then the IR literature has devoted a significant effort in  measuring relevance of
items to specific queries across different tasks: including, web search \cite{bar2008random}, personalization \cite{jeh2003scaling}, and product rating \cite{dave2003mining}; we also refer the reader to \cite{IRbook} and the references therein.
In the last three decades, works in the ML literature have also made significant contributions to relevance-estimation \cite{liu2011learning}, by { proposing methods that:~(1)} supplement traditional IR approaches, e.g., by automatically tuning {their--previously hard to tune--}parameters \cite{taylor2006optimisation} and by improving their efficiency through clustering-based techniques \cite{singitham2004efficiency,altingovde2008incremental},
and (2) {substitute traditional IR approaches by neural-network based models to predict item relevance \cite{burges2010ranknet,burges2005learning,weston2010large,googleLTR}.}
\vspace{-1.25mm}

\paragraph{Fair Ranking.}
Existing works on the fair-ranking problem take diverse approaches:
Among works that de-bias utilities, different approaches include, post-processing the utilities so that the post-processed utilities satisfy some fairness requirement \cite{causal2021yang},
introducing a ``fairness penalty'' in the objective function used to train learning-to-rank models \cite{policyLearningAshudeep,ReducingDisparateExposureZehlike,robustFairLTR2020}, and
modifying feature representations generated by up-stream algorithms so that the utilities learned from the modified representations satisfy some fairness requirements \cite{YangS17}.
Works that alter the ranking algorithms can also be further categorized into those which satisfy the constraints for each ranking \cite{celis2018ranking,BalancedRankingYang2019,linkedin_ranking_paper,GorantlaUnderranking21} and those that satisfy the constraints in aggregate over multiple rankings \cite{fairExposureAshudeep,AmortizedFairnessBiega2018}.
{Among aforementioned works, \cite{robustFairLTR2020} uses a version of adversarial training to make (fair) learning-to-rank models robust to outliers but, unlike this work, they require socially-salient attributes of items to be accurately known to specify the ``fairness penalty.''}
All of the {other} aforementioned works also need access to the socially-salient attributes of items.
When protected attributes are inaccurate, these works can fail to satisfy their fairness and/or utility guarantees \cite{GhoshDW21}.
\vspace{-1.25mm}

\paragraph{Effect of Inaccuracies on Fair-Ranking Algorithms.}
Some recent works have considered assessing fairness of rankings and ranking algorithms with missing or inaccurate protected attributes.
\cite{Kirnap0BECY21} analyze the setting where all protected attributes are missing, but can be purchased at a fixed cost per item.
They give statistical-techniques to estimate the fairness-value of a given ranking at a small cost.
\cite{GhoshDW21} use ML-classifiers to infer protected attributes from real-world data and study performance of the fair-ranking algorithm by \cite{linkedin_recuiter_algorithm} when given inferred attributes as input.
While these works underscore the need for fair-ranking algorithms to be robust to inaccuracies in protected attributes, they only assess fairness in the presence of noisy protected attributes.
\vspace{-1.25mm}

\paragraph{Fair Algorithmic Decision Making with Inaccuracies in Protected Attributes.}
Several recent works develop fair algorithms for tasks different from ranking that are robust to inaccuracies in the socially-salient attributes \cite{LamyZ19,awasthi2020equalized,MozannarOS20,prob_fair_clustering,wang2020robust,MehrotraC21,wang2021label,celis2021fairclassification,celis2021adversarial,prob_fair_clustering}.
In particular, several works study classification and clustering \cite{LamyZ19,awasthi2020equalized,MozannarOS20,prob_fair_clustering,wang2020robust,MehrotraC21,wang2021label,celis2021fairclassification,celis2021adversarial,prob_fair_clustering}, and develop fair algorithms robust to inaccuracies in protected attributes.
Many of these works consider the same random error model as us (or one of its variants) \cite{LamyZ19,awasthi2020equalized,wang2020robust,MehrotraC21,celis2021fairclassification,prob_fair_clustering}, but some very recent works have also considered adversarial noises in protected attributes \cite{wang2020robust,konstantinov2021fairness,celis2021adversarial}.
However, because the underlying algorithmic tasks are fundamentally different from the variant of the ranking problem we study it is not clear how to adapt their approaches to our setting.
\cite{MehrotraC21} studies the problem of fair subset selection under the same noise model.
In subset selection, given $m$ items the goal is to output an {\em unordered} subset of $n\leq m$ items with the highest utility.
They develop an optimization framework outputs a subset satisfying the fairness constraint up to a small multiplicative error with high probability but
leave the problem of ranking open.
We compare against an adaptation of their approach to ranking in our empirical results.

\section{Model of Fair Ranking with Noisy Attributes}

{\bf Ranking problem.}
In ranking problems, given $m$ items, one has to select a subset of $n$ items and output a permutation of the selected items.
This permutation is said to be a {\em ranking}.
There is a large body of work on estimating the relevance of items and personalizing these estimates to specific users/queries \cite{IRbook,liu2011learning}.
We consider a ranking problem where the relevance of items are known.
Abstracting relevance estimation, in this problem, one is given an $m\times n$ matrix $W$, such that placing the $i$-th item at the $j$-th position generates {\em utility} $W_{ij}$.
The utility of a ranking is the sum of utilities generated by each item in its assigned position.
The algorithmic task in the ranking problem is to output a ranking with the highest utility.
We denote rankings by assignment matrices $R\in \zo^{m\times n}$, where $R_{ij}=1$ indicates that item $i$ appears in position $j$, and $R_{ij}=0$ indicates otherwise.
In this notation, the utility of a ranking is $$\inangle{R,W}\coloneqq \sum\nolimits_{i=1}^m\sum\nolimits_{j=1}^n R_{ij} W_{ij}.$$
Then this ranking problem is to solve: $\max\nolimits_{R\in \cR} \inangle{R, W}.$
Where $\cR$ is the set of all assignment matrices denoting a ranking: %
\begin{align*}
  \hspace{0mm} {\cR \coloneqq \inbrace{ {X\in \zo^{m\times n}}
  :\hspace{0mm}
  {\forall {i \in [m]}, \sum\nolimits_{j=1}^n X_{ij}\leq 1}, \ \ {\forall {j \in [n]}, \sum\nolimits_{i=1}^m X_{ij} = 1}
  }}.
  \yesnum\label{eq:assignment_matrices}
\end{align*}
Here, the constraint $\sum_{i=1}^m X_{ij} = 1$ ensures position $j$ has exactly one item
and the constraint $\sum_{j=1}^n X_{ij}\leq 1$ ensures that item $i$ occupies at most one position. %
{While this model captures a variety of applications, in some cases, the entries of $W$ may be skewed by an unknown amount \cite{KleinbergR18,celis2020interventions} or not known accurately \cite{AshudeepUncertainty2021} and the utility of the ranking may not be linear in the entries of $W$ \cite{microsoft_diverse}. These are interesting directions but are not studied in this work.}

\paragraph{Fair-Ranking Problem.}
There are several versions of the fair-ranking problem.
We consider {a version with} $p\geq 2$ {\em socially-salient groups} $G_1,G_2,\dots,G_p \subseteq [m]$ (e.g., the group of all women or all Black people) which are often protected by law.
Each of the $m$ items belongs to {\em one or more} of these socially-salient groups (henceforth referred to as just groups).
{This} fair-ranking problem is to output the ranking {with maximum utility subject to satisfying certain fairness criteria} with respect to these groups.
The appropriate notion of fairness is context dependent, and to capture different fairness criteria numerous {\em fairness constraints} have been proposed. %
{We consider a class of general fairness constraints.}

\begin{definition}[\textbf{Fairness Constraints}]\label{def:lu_fairness_constraints}
  Given a matrix $U\in \Z_{+}^{n\times p}$, a ranking $R$ satisfies the upper bound constraint if $\sum\nolimits_{i\in G_\ell}\sum\nolimits_{j=1}^k R_{ij} \leq U_{k\ell},$ for all  $\ell\in [p]$ and  $k\in [n]$.
\end{definition}
\noindent Existing works consider similar constraints and show that they can encapsulate a variety of fairness criteria \cite{fairExposureAshudeep}.
For instance, when groups are disjoint, to capture equal and proportional representation, one can choose  ${U_{k\ell} {\coloneqq} {\ceil{\frac{k}p}}}$ and  ${U_{k\ell}} {\coloneqq} {\ceil{k\cdot \frac{|G_\ell|}{m}}}$
for all $k$ and $\ell$ respectively.
{(That said, they do not capture qualitative differences among groups, such as, misrepresentation of demographics in image results \cite{KayMM15,Noble2018}, which could arise even when rankings has sufficient individuals from each group.)}
As a running example, we consider the fair-ranking {problem with equal representation  with two disjoint groups, i.e.,}
\begin{align*}
   \hspace{-3mm} {\max}_{R\in \cR}& \inangle{R, W}
   \quad \st\quad  ~\forall k\in [n]\ \forall \ell\in [2],\ \ \
  \sum\nolimits_{i\in G_\ell}\sum\nolimits_{j=1}^k R_{ij}\leq \ceil{\frac{k}2}.
  \yesnum\label{eq:equal_rep_with_known_groups}
\end{align*}
{To ease readability, we omit ceilings-operators henceforth.}

\paragraph{Noise Model.}
If the socially-salient attributes of items are known accurately, then one can solve the fair-ranking problem.
However, as discussed, in many contexts, attributes are inaccurate, missing, or only probabilistically known.
Several models have been proposed to capture different errors in attributes.
Here, we consider a model (due to \cite{AngluinL87}) which has also appeared in \cite{prob_fair_clustering,LamyZ19,MehrotraC21}.
\begin{definition}[\textbf{Noise Model}]\label{def:noise_model}
  Let $P\in [0,1]^{m\times p}$ be a known matrix.
  The groups $G_1,\dots,G_p\subseteq [m]$ are random variables, such that, for each $i\in [m]$ and $\ell\in [p]$, $$\Pr[G_\ell\ni i]=P_{i\ell}.$$
  Moreover, for different items $i\neq j$ the events $G_\ell\ni i$ and $G_k\ni j$ are {\em independent} for all $\ell,k\in [p]$.
\end{definition}

\noindent \cref{def:noise_model} makes two key assumptions: the matrix $P$ is known and for each item $i$, the events $G_\ell\ni i$ over groups $\ell$ are independent of the corresponding events for other items.
{Both of these assumptions hold when attributes are flipped to preserve local differential privacy (\cref{rem:discussion_of_noise_model}).
In other settings, $P$'s estimate can be inaccurate and above events may be correlated.}
These can adversely affect the performance of our framework.
We empirically study {this in simulations where $P$ is estimated using confidence scores of off-the-shelf {classifiers and is {\em miscalibrated} (\cref{fig:simulation_image,fig:simulation_intersectional}).}}
{{\cref{sec:dis_noise_model} shows how \cref{def:noise_model} captures both disjoint and overlapping groups.}}

\paragraph{Fairness Constraint with Noisy Attributes.}

Most existing fairness {constraints} assume that the groups are deterministic.
Hence, it is not clear how to {impose them when groups are random variables, as in \cref{def:noise_model}.}
One definition is to require the constraints to be approximately satisfied with high probability.
Consider the instantiation of this definition for equal representation:
A ranking $R$ satisfies $(\rho,\delta)$-equal representation, if with probability $1-\delta$, at most $\frac{k}{2}(1+\rho)$ items from $G_\ell$ appear in the first $k$ positions in $R$ places for all $k\in [n]$ and $\ell\in[p]$.
Naturally, one would like to satisfy this definition for small $\delta,\rho$.
However, it turns out to be too stringent and is infeasible for {any} small $\delta,\rho$.
\begin{proposition}\label{lem:eps_depends_on_k}
  {
    No ranking satisfies $(\rho,\delta)$-equal representation for $\rho<1$, $\delta\leq \frac12$, and $P=\insquare{\frac{1}{2}}_{m\times p}$.}
\end{proposition}

\noindent The proof of \cref{lem:eps_depends_on_k} shows that any ranking $R$ violates the equal-representation constraint at the 2nd position by a multiplicative factor of $2$ with probability $\frac12$.
The issue is that the same relaxation parameter $\rho$ is used for each position {$k$ (whereas the information theoretically best-achievable relaxation parameter at $k$ improves as $k$ increases, this, e.g., follows by \cref{thm:ub}.)}
{Motivated by this observation, we consider the following alternate version of upper bound constraints.}
\begin{definition}[\textbf{$(\eps,\delta)$-constraint}]\label{def:eps_delta_const}
  For any $\eps\in \R_{\geq 0}^n$ and $\delta\in (0,1]$, a ranking $R$ is said to satisfy $(\eps,\delta)$-constraint if with probability at least $1-\delta$ over the draw of $G_1,\dots,G_p$
  \begin{align*}
    \forall k\in [n]~\forall \ell\in [p],~~ \sum\nolimits_{i\in G_\ell}\sum\nolimits_{j=1}^k R_{ij}\leq U_{k\ell}  (1+\eps_k).
    \yesnum\label{eq:eps_delta_const_gen}
  \end{align*}
\end{definition}
\noindent We would like to output a ranking that satisfies \cref{def:eps_delta_const} for small $\delta$ and {small} $\eps_1,\eps_2,\dots,\eps_n$.

\begin{problem}[\textbf{Ranking Problem with Noisy Attributes}]\label{prob:2}
  Given matrices
  $P$, $U$, and $W$,
  find the ranking $R$ maximizing utility $\inangle{R,W}$ subject to satisfying \mbox{$(\eps,\delta)$-constraint for some small $\eps$ and $\delta$.}
\end{problem}

\subsection{{Challenges in Solving \cref{prob:2}}}\label{sec:challenges}

In this section we discuss potential approaches for solving \cref{prob:2}. {In other words, solving:}
\begin{align*}
  \hspace{0mm}\text{$\max\nolimits_{R\in \cR}\inangle{R,W}$,\ \  s.t.,\ \ $R$ satisfies $(\eps,\delta)$-constraint.}
  \yesnum
  \label{eq:exact_prob}
\end{align*}
{Even for two disjoint groups, given $V$, it is \np{}-hard to decide if the value of \prog{eq:exact_prob} is at least $V$ (\cref{thm:np_hardness_of_exact_const}).
To bypass this hardness, one can consider approximation algorithms.
\prog{eq:exact_prob} is an integer program (IP) because the entries of the matrix $R$ are required to be integers (\cref{eq:assignment_matrices}).
A standard approach to (approximately) solve IPs is to: (1) consider their continuous relaxation that drops the integrality constraints, (2) compute the optimal solution  $R_c$ of the relaxed problem, and then (3) ``round'' $R_c$ to satisfy integrality constraints while ``retaining'' its utility and fairness properties.
To take this approach, we first need an efficient algorithm to find $R_c$.
However, not just \prog{eq:exact_prob}, but even its continuous relaxation is non-convex. {{Hence,} it is unclear how to solve it to find $R_c$.}}

{Due to the independence assumption in \cref{def:noise_model}, the number of items from $G_\ell$ appearing in the first $k$ positions of a ranking is concentrated around its expectation (for large $k$).
This implies that if, in expectation, less that $U_{k\ell}$ items from $G_\ell$ appear in the top $k$ positions then, with high probability, the number of items from $G_\ell$ in the top $k$ positions is not much larger than $U_{k\ell}$.
Using this one can show that a ranking satisfying the following constraints
\begin{align*}
  \forall k\in [n]~\forall \ell\in [p],~~ \Ex\insquare{\sum\nolimits_{i\in G_\ell}\sum\nolimits_{j=1}^k R_{ij}}\leq U_{k\ell}
  \yesnum
  \label{eq:exp_const}
\end{align*}
also satisfies $(\eps,\delta)$-constraint for small $\eps$ and $\delta$.
One idea is to find the ranking maximizing utility subject to satisfying Constraint~\eqref{eq:exp_const}.
A feature of Constraint~\eqref{eq:exp_const} is that it is linear in $R$ as $${\Ex\insquare{\sum\nolimits_{i\in G_\ell}\sum\nolimits_{j=1}^k R_{ij}}}{=} {\sum\nolimits_{i=1}^m \sum\nolimits_{j=1}^k P_{i\ell}R_{ij}}$$ and, hence, one may hope to find the ranking with the maximum utility subject to satisfying Constraint~\eqref{eq:exp_const}.
However, the issue is that there are examples where any ranking satisfying Constraint~\eqref{eq:exp_const} has 0 utility and there are rankings that satisfy $(\eps,\delta)$-constraint and have a large positive utility (Lemma~\ref{lem:exp_const_only_suff}).
Hence, this approach can output rankings whose utility is significantly smaller than the utility of the solution to \cref{prob:2}.
To overcome this, we relax Constraint~\eqref{eq:exp_const} by a carefully chosen position-dependent factor, such that, any ranking satisfying the $(\eps,\delta)$-constraint (for appropriate $\eps$ and $\delta$) is also feasible for our framework.}

\section{Theoretical Results}\label{sec:thy_results}

In this section we present our optimization framework and its fairness and utility guarantees.

 \begin{minipage}{0.47\linewidth}
 \small
 \hspace{-5mm}

 \begin{tcolorbox}[grow to left by=0.55cm,colback=white,left=3pt,right=3pt,top=2pt,bottom=0pt]
   {\em Input:} Matrices $P\in{[0,1]}^{m\times p}\hspace{-4mm},\hspace{3mm} W\in {\R}^{m\times n}\hspace{-3mm},\hspace{2mm}U\in  {\R}^{n\times p}$\\
   \vspace{-0.25mm}

   \noindent {\em Parameters:} Constant $c>1$, failure probability $\delta\in (0,1]$, and $k\in [n]$, relaxation parameter
   \begin{align*}
     \gamma_k\coloneqq 12\cdot \log\inparen{\frac{2np}\delta}\cdot  \max_{\ell\in [p]} \sqrt{\frac{1}{U_{k\ell }}}.
     \yesnum\label{eq:def_gamma}
   \end{align*}
 \end{tcolorbox}
 \end{minipage}
 \begin{minipage}{0.49\linewidth}
 \small
 \vspace{-2.25mm}
 \begin{tcolorbox}[enlarge top by=0.5cm,enlarge bottom by=0.25cm,colback=white,left=3pt,right=3pt,top=2pt,bottom=0pt]
   {\em Our Fair-Ranking Program}
   \hrule
   \vspace{-0.75mm}
   \begin{align*}
     &\hspace{-5.5mm}\max\nolimits_{R\in \cR} \inangle{R,W},
     \tagnum{Noise Resilient}\customlabel{prog:noisy_fair_boxed}{\theequation}\\
     \st &\hspace{-2mm} ~~~~ \forall {\ell \in [p]}~~\forall {k\in [n]}\\
     &\hspace{-6mm}\sum\nolimits_{\substack{i\in [m],\\ j\in [k]}} P_{i\ell} R_{ij} \leq U_{k\ell} \inparen{1+\inparen{1-\frac{1}{2\sqrt{c}}}\gamma_k}.\hspace{-1mm}
     \yesnum\label{eq:const_of_noisy_fair_boxed}
   \end{align*}
   \vspace{0mm}
   \vspace{-1.75mm}
   \end{tcolorbox}
 \end{minipage}

  \smallskip

  \noindent
  The above program modifies the program for fair ranking with accurate groups:
  It has the same objective but different constraints.
  {Instead of sampling the attribute values and applying constraints on the sampled values, {Constraint~\eqref{eq:const_of_noisy_fair_boxed} applies} upper bounds on the expected number of items in the first $k$ positions from group $\ell$ (\cref{sec:challenges}).
  Further, {Constraint~\eqref{eq:const_of_noisy_fair_boxed}} relaxes upper bounds $U_{k\ell}$ by a small position-dependent factor.
  Like for Constraint~\eqref{eq:exp_const},  one can show that any ranking satisfying Constraint~\eqref{eq:const_of_noisy_fair_boxed} also satisfies $(\eps,\delta)$-constraint (for small $\eps_1,\dots,\eps_n$ and $\delta$).
  But unlike Constraint~\eqref{eq:exp_const}, and somewhat surprisingly,
  any ranking that satisfies ${\sinparen{\eps,\delta}}$-constraint (for appropriate $\eps_1,\dots,$ $\eps_n$~and $\delta$) must {also satisfy Constraint~\eqref{eq:const_of_noisy_fair_boxed}.
  {(In fact, $\gamma_k$ is chosen to be the smallest, up to logarithmic factors, value such that this is true.)}
  We use this to prove \cref{thm:ub}'s utility guarantee.}}

  Our first result bounds the fairness and utility of the optimal solution of \prog{prog:noisy_fair_boxed}.
  \begin{theorem}\label{thm:ub}
    Let $\gamma\in \R^n$ be as defined in \cref{eq:def_gamma}.
    There is an optimization program (Program~\eqref{prog:noisy_fair_boxed}),
    parameterized by a constant $c$ and failure probability $\delta$,
    such that for any $c>1$ and $\delta\in (0,\frac{1}{2}]$ its optimal solution satisfies $(c\gamma,\delta)$-constraint and
    has a utility at least as large as the utility of any ranking satisfying $\inparen{(c-\sqrt{c})\gamma,\delta}$-constraint.
  \end{theorem}
  \noindent  For equal representation, $\gamma_k$ is ${\wt{O}{\inparen{\frac{1}{\sqrt{k}}}}}$.
  Thus, \cref{thm:ub} guarantees that, with high probability, the optimal solution of \prog{prog:noisy_fair_boxed} multiplicatively violates equal representation at the $k$-th position by at most $1+{\wt{O}{\inparen{\frac{1}{\sqrt{k}}}}}$.
  Further, this solution's utility is higher than the utility of any ranking satisfying a slight relaxation of this fairness guarantee.
  {\cref{thm:ub} can be extended to position-weighted versions of fairness constraints (\cref{thm:ub_exnt_wt}),
  where the fairness constraint is $$\forall \ell,\forall k,\quad {\sum\nolimits_{i\in G_\ell} \sum\nolimits_{j\in [k]}} {v_j R_{ij}} \leq  {U_{k\ell}}$$
  for specified discount factors $v_1\geq \dots\geq v_n$
  such as  NDCG~\cite{DCG}}.
  If we are also guaranteed {$U_{k\ell}\geq \psi k$} for some constant {$\psi>0$} and all $k$ and $\ell$, then we can improve $\gamma_k$'s dependence on $\delta$ from ${\log\frac{1}{\delta}}$ to ${\sqrt{\log\frac{1}{\delta}}}$ {(\cref{sec:imp_depen_delta:thm:ub}).
  The proof of \cref{thm:ub} appears in \cref{sec:proofof:thm:ub:mb}.}

  \paragraph{Lower Bound on Fairness Guarantee.}
  Our next result complements \cref{thm:ub}'s fairness guarantee.
  \begin{theorem}\label{thm:lb}
    There is a family of matrices $U\in \Z_{+}^{n\times p}$ such that for any $U$ in the family and
    any parameters $\delta\in[0,1)$ and $\eps_1,\dots,\eps_n\geq 0$,
    if for any position $k\in [n]$,
    \begin{align*}
        \eps_k\leq 1
        \quad\text{and}\quad
        \eps_k < \max_{\ell\in [p]}\sqrt{\frac{1}{2U_{k\ell}} \log\frac{1}{4\delta}},
    \end{align*}
    then there exists a matrix $P\in [0,1]^{m\times p}$, %
    such that it is information theoretically impossible to output a ranking that satisfies $(\eps,\delta)$-constraint.
    This family, {in particular, contains the matrices} $U$ {corresponding to equal representation {and proportional representation} constraints.}
  \end{theorem}

  \noindent
  Since $\gamma_k$ is ${O\inparen{\log\sinparen{\frac{np}{\delta}}\cdot\max_{\ell}\sqrt{\frac{1}{U_{k\ell}}}}}$,
  \cref{thm:lb} shows that {\cref{thm:ub}'s fairness guarantee is optimal up to log-factors.}
  \cref{sec:proofof:thm:lb} proves \cref{thm:lb}.

  \paragraph{An Efficient Algorithm.}
  As for solving our optimization program, it is \np{}-hard to check its feasibility (\cref{thm:hardness_results_main}).
  {However, because Constraint~\eqref{eq:const_of_noisy_fair_boxed} is linear in $R$,  the continuous relaxation of Program~\eqref{prog:noisy_fair_boxed} is a standard linear program and can be solved efficiently.}
  Our algorithm (\cref{algo}) solves the {standard} linear programming relaxation of \prog{prog:noisy_fair_boxed} to find a solution {$R_c$} and then uses a dependent-rounding algorithm by \cite{ChekuriVZ11} to convert {$R_c$} to a ranking. {(See \cref{sec:proofof:thm:algo} for brief discussion of why straightforward rounding approaches are insufficient.)}
  \begin{theorem}\label{thm:algo}
    There is a randomized algorithm {(\cref{algo})} that given
    constants $d > 2$, a failure probability $0<\delta\leq 1$, and matrices $P\in [0,1]^{m\times p}$ and $W\in [0,1]^{m\times n}$,
    outputs a ranking satisfying $\inparen{O(d\gamma),\delta}$-constraint and
    with probability at least $1-\delta$, and has a utility at least $$\inparen{1-\frac{1}{d}}\cdot V-\wt{O}\inparen{\sqrt{dn}},$$
    where $V$ is the utility of
    {any ranking satisfying $\sinparen{(d-\sqrt{d})\gamma,\delta}$-constraint.}
    The algorithm runs in polynomial time in $d$ and the bit complexity\footnote{{The bit complexity of the inputs is the number if bits required to encode the input using the standard binary encoding (which, e.g., maps integers to their binary representation, rational numbers as pair of integers, and vectors/matrices as a tuple of their entries) \cite[Section 1.3]{grotschel2012geometric}.}} of the input.
    \end{theorem}

    \noindent
    The tension in setting $d$ is that
    decreasing $d$ improves the fairness guarantee and the utility guarantee's second term, but worsens the first term in the utility guarantee.
    Under the mild assumption that $V=\Omega(n)$, {increasing $d$ improves the utility guarantee because the first term in the {utility guarantee} dominates the second term.
    In this case, the utility guarantee improves to $(1-\frac{1}{d}-o(1))\cdot V$.}
    Finally, while \cref{thm:algo} requires utilities to be between 0 and 1, it can be extended to any non-negative and bounded utilities by {scaling.
    The proof of \cref{thm:algo} appears in \cref{sec:proofof:thm:algo}.}

    \renewcommand{\folder}{./figs/main-body}
    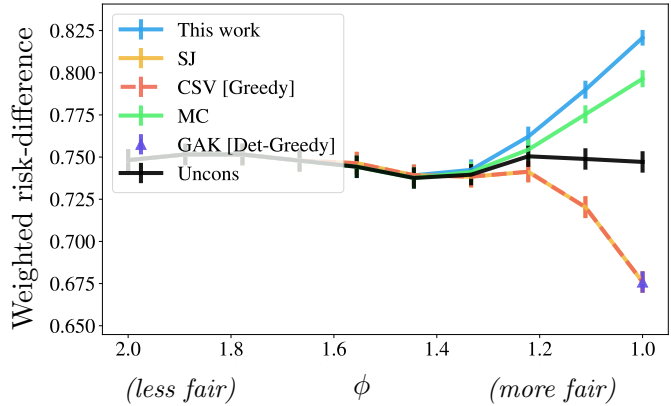
\begin{figure}[t!]
      \centering
      \vspace{-10mm}
      {\begin{tikzpicture}
        \tikzmath{\s = 1.1;}
        \tikzmath{\mvx = 0;}
        \tikzmath{\mvy = 1.5;}
        \node (image) at (0,-0.17*\s+1.9) {{\includegraphics[width=0.50\linewidth, trim={1.25cm 0cm 2.55cm 1cm},clip]{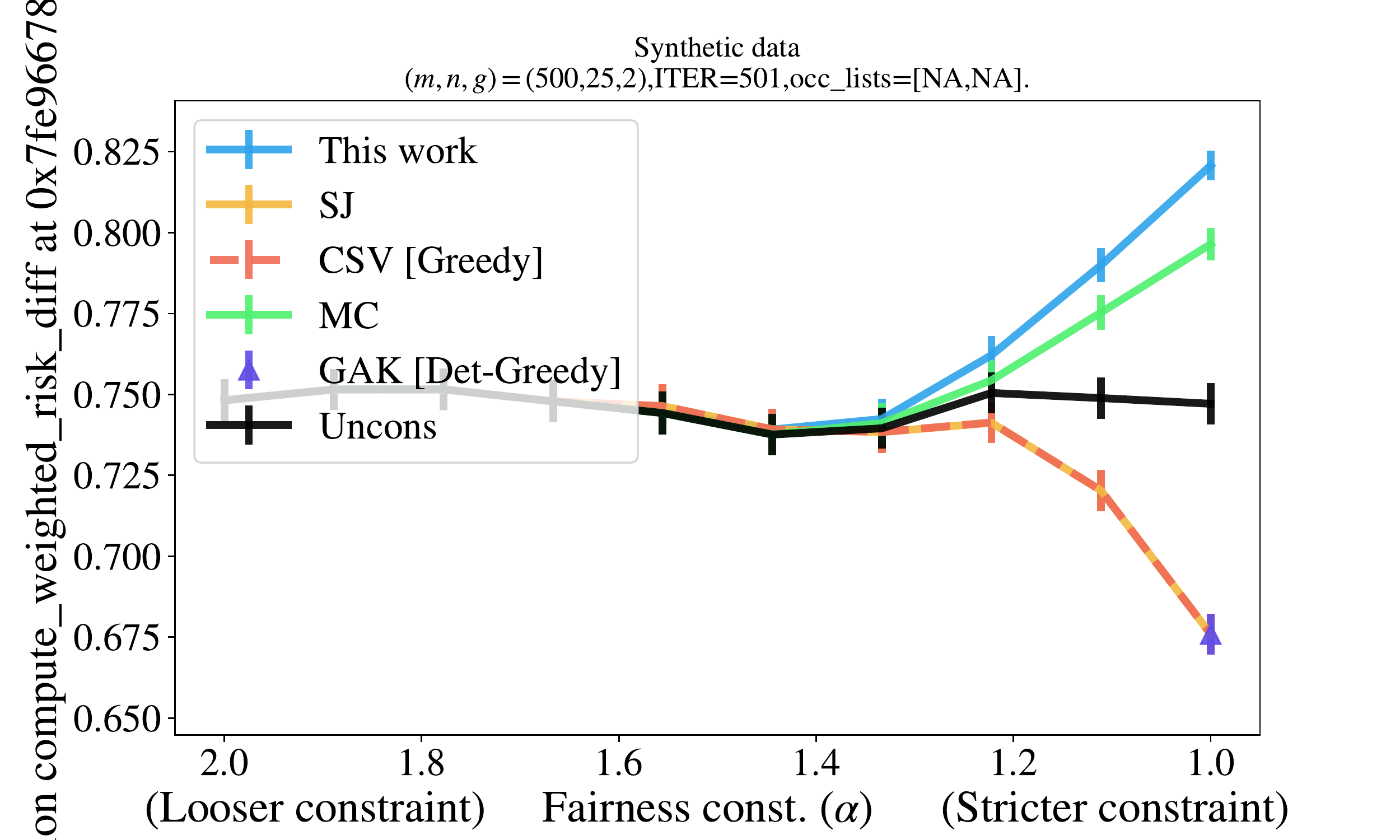}}};
        \node[rotate=90, fill=white] at (+\mvx-3.29375*\s-0.8,0.3+\mvy) {\white{......}Weighted risk-difference\white{......}};
        \draw[draw=white, fill=white] (+\mvx-2*\s+0.35,1.2325*\s+\mvy+1.3) rectangle ++(4*\s,0.15*\s+0.14);
        \node[rotate=0, fill=white] at (+\mvx+0.05*\s, -3.3*5/8*\s+0.05*\s-0.255+\mvy)  {\white{||||||||||}$\phi$\white{||||||||||}};
        \node[rotate=0,fill=white] at (+\mvx+2.35*\s,-1.8625*\s-0.475+\mvy) {\textit{\small(more fair)}};
        \node[rotate=0,fill=white] at (+\mvx-2.14375*\s,-1.8625*\s-0.475+\mvy) {\textit{\small(less fair)}};
      \end{tikzpicture}}
      \caption{
      {\em {Synthetic Data: Nonuniform Error Rate.}}
      We consider synthetic data where imputed socially-salient attributes have a higher false-discovery rate on the minority group.
      We vary the fairness constraint ($\phi$) and observe the weighted risk-difference (RD) of algorithms.
      The $y$-axis plots RD and $x$-axis plots $\phi$. ({\em Note that the $x$-axis decreases toward the right}).
      We observe that \ouralgo{} achieves the most fair RD, while obtaining a similar utility for all $\phi$ (\cref{fig:fig1bskjkj}).
      Error-bars denote the error of the mean.
      \vspace{0mm}
      }
      \vspace{0mm}
      \label{fig:different_fdr}
    \end{figure}

    \section{Empirical Results}\label{sec:empirical_results}

    In this section we evaluate our framework's performance {on} synthetic and real-world data.\footnote{The code for simulations is available at \url{https://github.com/AnayMehrotra/FairRankingWithNoisyAttributes}} %

    \paragraph{Baselines and Metrics.}
    The correct choice of fairness metric is context-dependent and beyond the scope of this work \cite{Selbst:2019}.
    To illustrate our results, we arbitrarily fix the fairness metric as weighted risk-difference (RD).
    {This is a position-weighted version of the standard risk-difference metric \cite{calders2010three} and measures the extent to which a ranking violates equal representation.}
    {The RD of a ranking $R$ is:}
    \begin{align*}
       { 1-\frac1Z\sum\nolimits_{ k=5,10,\dots} \frac{1}{\log{k}}
      \max\nolimits_{\ell,q\in [p]}\abs{{ \sum\nolimits_{\substack{i\in G_\ell, j\in [k]\white{,}}}   R_{ij}  - \sum\nolimits_{\substack{i\in G_q, j\in [k]\white{,}}}   R_{ij}}},}
    \end{align*}
    {Where $G$ denotes the ground-truth protected groups and $Z$ is a constant so that $\rd$ has range $[0,1]$.}
    Here, $\rd=1$ is most fair and $\rd=0$ is least fair.
    We compare our framework, \ouralgo{}, against state-of-the-art fair-ranking algorithms:
    \csv{} (``greedy'' in \cite{celis2018ranking}), \sj{} \cite{fairExposureAshudeep}, and \detgreedy{} (``DetGreedy'' in \cite{linkedin_ranking_paper}).
    {We also compare against \mc{}, which ranks the items, in the subset output by \cite{MehrotraC21}'s algorithm, to maximize utility.}
    {Finally, we compare against the baseline, \uncons{}, which outputs the utility maximizing ranking without fairness considerations.}
    We present additional discussion and variations of the empirical results presented here in \cref{sec:empirical_results_extended}.
    The variations in \cref{sec:empirical_results_extended} repeat the simulations in this section with different parameters and amounts of noise, other fairness constraints, different fairness metrics, and larger number of groups.
    \vspace{1mm}

    \paragraph{Setup.}
    We consider the DCG model of utilities~\cite{DCG} and a  relaxation of equal representation constraints:
    (1) Given an intrinsic value $w_1,w_2,\dots,w_m\geq 0$, for each item $i$ and position $j$, we set $$W_{ij}\coloneqq \frac{w_i}{\log{(j+1)}},$$
    (2) Given a parameter $\phi\in [1,p]$, we set the following upper bounds:
    for each $k\in [n]$ and $\ell\in [p]$
    $$U_{k\ell}\coloneqq \smash{\frac{\phi}{p}}\cdot k.$$
    In simulations, we set $m=500$, $n=25$, and vary $\phi$ from $p$ to 1.
    {To gain some intuition about the relevant values of $\phi$, note that satisfying the 80\% rule requires $\phi\leq \frac{5p}{5p-1}$, i.e., $\phi\leq 1.11$ for $p=2$ and $\phi\leq 1.05$ for $p=4$.}
    For each $\phi$, we draw $m$ items uniformly without replacement and compute an estimate $\hP$ of the matrix $P$ {(from \cref{def:noise_model}) using, e.g., off-the-shelf ML classifiers or public APIs (see the paragraphs ``Estimating $\hP$'' in simulation with image data and ``Setup'' in simulation on name data).}
    We infer socially-salient groups {$\hG_1,\dots,\hG_p$} via $\hP$ by assigning each item to its most-likely group.
    Finally, we run all algorithms using $\hP$ or {$\hG_1,\dots,\hG_p$} as discussed next.

    \vspace{1mm}

    \paragraph{Implementation Details.}
    \ouralgo{} and \mc{} take probabilistic information about socially-salient attributes as input and are given $\hP$.
    \csv{}, \sj{}, and \detgreedy{} require access to socially-salient groups and are given $\hG_1,\dots,\hG_p$.
    \ouralgo{}, \sj{}, and \csv{} use fairness constraints from \cref{def:lu_fairness_constraints} and are given:
    for each $k$ and $\ell$, $U_{k\ell} = \smash{\frac{\phi}{p}}\cdot k$.
    \mc{} requires, for each $\ell\in[p]$, an upper bound on the number of items from $G_\ell$ that can appear in top-$n$ positions.
    It is given $ \smash{\frac{\phi}{p}}\cdot n$ for each $\ell\in[p]$.
    \detgreedy{} requires the desired proportion $\alpha_\ell$ for each group $G_\ell$ and, roughly, satisfies the constraint $U_{k\ell} = \alpha_\ell\cdot k$ for each $k\in [n]$ and $\ell\in [p]$.
    {It is given $\alpha_\ell=\smash{\frac{1}{p}}$ for each $\ell\in[p]$,}
    {this corresponds to $\phi=1$ (hence,  figures only  plot the \detgreedy{} at $\phi=1$).}
    As a heuristic, we set $${\gamma_k =} {\frac{1}{20}\max_{\ell\in [p]}\sqrt{\frac{1}{U_{k\ell}}}}$$ in all simulations.
    We find that this suffices and expect a more refined approach to improve the performance of \ouralgo{}.

\renewcommand{\folder}{./figs/main-body}
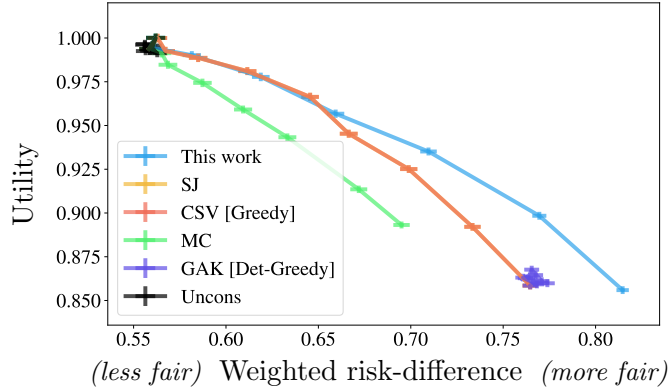
\begin{figure}[t!]
  \centering
  \vspace{-5mm}
  {\begin{tikzpicture}
    \tikzmath{\s = 1;}
    \tikzmath{\mvx = 0;}
    \tikzmath{\mvy = 0;}
    \node (image) at (0,-0.17*\s+0.11) {\includegraphics[width=0.5\linewidth, trim={1.35cm 0cm 1.9cm 1.675cm},clip]{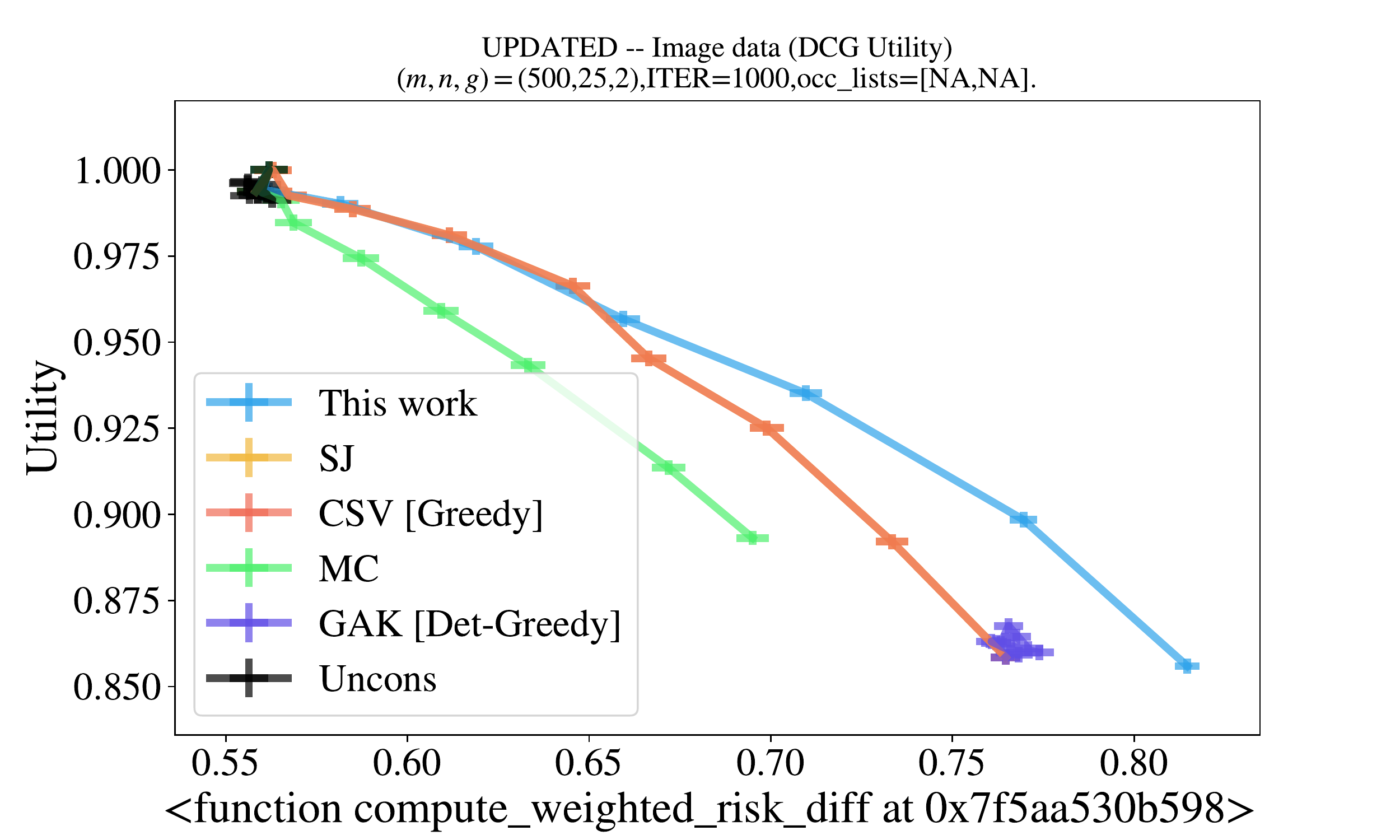}};
    \node[rotate=90, fill=white] at (-4.54375,0) {\white{......}Utility\white{......}};
    \draw[draw=white, fill=white] (+\mvx-2*\s+0.35,1.2325*\s+\mvy+1.3) rectangle ++(4*\s,0.15*\s+0.14);
    \node[rotate=0, fill=white] at (+\mvx+0.05*\s, -3.3*5/8*\s+0.05*\s-0.255+\mvy-0.2-0.05)  {\white{||||}Weighted risk-difference\white{||||}};
    \node[rotate=0,fill=white] at (+\mvx+2.35*\s+0.8,-1.8625*\s-0.425+\mvy-0.18-0.05) {\textit{\small(more fair)}};
    \node[rotate=0,fill=white] at (+\mvx-2.14375*\s-0.8,-1.8625*\s-0.425+\mvy-0.18-0.05) {\textit{\small(less fair)}};
  \end{tikzpicture}}
  \caption{
    {\em Real-World Image Data.}
    In this simulation,
    given {\em non-gender labeled} images and their utilities, our goal is to generate a high-utility gender-balanced ranking.
    We estimate $P$ using an off-the-shelf ML-classifier and vary $\phi$ from $p=2$ (less fair) to $1$ (more fair).
    The $y$-axis plots the utility of algorithms and the $x$-axis plots RD.
    We observe that \ouralgo{} has the most fair RD and the best fairness-utility trade-off.
    Error bars show the error of the mean.
  }
  \label{fig:simulation_image}
  \vspace{-2.5mm}
\end{figure}

    \subsection{Simulation on Synthetic Data}
    \label{sec:simulation_syn_disp_err}

    {We show that on synthetic data, where error-rates of given socially-salient attributes vary over groups, {existing fair-ranking algorithms have worse RD than \uncons{}.}}

    \paragraph{Data.}
    {We generate $w$ and $P$ for two groups using code by \cite{MehrotraC21} and fix $\hP=P$.}
    For all items $i$, $w_i$ is i.i.d. from the uniform distribution over $[0,1]$.
    $\hP$ is constructed such that attributes inferred from $\hP$ have a higher false-discovery rate for \mbox{the minority group compared to the majority ($40\%$ vs $10\%$).\footnote{This 30\% difference in false-discovery rates is comparable to the 34\% difference in the false-discovery rates of dark-skinned females and light-skinned men observed by \cite{BuolamwiniG18} for a commercial classifier.}}

    \paragraph{Results.}
    {See \cref{fig:different_fdr} for the observed \rd{} averaged over 500 iterations.
    We observe that \ouralgo{} achieves best \rd{} ($\approx$0.81), while not {losing} significant utility ($\geq 98\%$ of max.; see \cref{fig:fig1bskjkj}).
    \mc{} achieves the {next} best \rd{} ($\approx$0.79).
    In contrast, \csv{}, \sj{}, and \detgreedy{}, which do not account for noise in the socially-salient attributes, achieve a worse \rd{} ($\leq$0.68) for $\phi\leq 1.2$ than \uncons{} ($\approx$0.75).} %
    {This range of $\phi$ can be desired in practice: e.g., a platform must set $\phi\leq 1.1$ to guarantee the 80\% rule is satisfied two groups.}
    {Thus, we observe that existing fair-ranking algorithms may achieve a worse RD than \uncons{}.}

    \subsection{{Simulation on Real-World Image Data}}\label{sec:simulation_image}

    In this simulation, given {\em non-gender labeled} images-search results %
    and their utilities, our goal is to generate a high-utility and gender-balanced ranking.

    \paragraph{Data.}
    We use the Occupations dataset \cite{celis2020cscw} which contains the top 100 Google Image results for 96 occupation-related queries.
    For each image, the data has its position in search results,
    gender (coded as male/female) of the individual depicted in the image, collected via MTurk.
    We use the (true) gender labels in the data to compute RD and to estimate $\hP$, but do not provide them to algorithms.

    \paragraph{Setup.}
    For each image $i$, with rank $r_i$, we define $$w_{i}\coloneqq \frac{1}{\log\inparen{1+r_i}}.$$
    We say an occupation is gender-stereotypical if more than 80\% of images for this occupation have the same gender label (41/96 occupations).
    An image is said to be stereotypical if {it is} in a gender-stereotypical occupation and its gender label is the majority label for its occupation.
    {We define the socially-salient groups as the sets of stereotypical and non-stereotypical images in gender-stereotypical occupations.}

\paragraph{Estimating $\hP$.}
{After pre-processing, we use a CNN-based gender-classifier $f$ \cite{imdb_wiki_code} to predict the (apparent) gender of the person depicted in each image. We calibrate the confidence scores output by $f$ by binning and use these to estimate $\hP$ (see \cref{sec:empirical_results_extended} for more details).}
We perform this calibration once and on all occupations and, then, use it for gender-stereotypical occupations.
{{Because of this $\hP$ is miscalibrated (and hence, inaccurate).}
For instance, among samples $i$ for which $0.25 \leq \hP_{i, \text{male}} \leq 0.5$, more than 75\% are labeled as ‘man’ (instead of some percentage between 25\% and 50\%). {\em This violates the assumption that $P$ is accurately known.}}

\paragraph{Results.}
See \cref{fig:simulation_image} for RD and utilities (NDCG) averaged over 1000 iterations.
We observe that \ouralgo{} achieves the best RD ($\approx$0.81) and has a better RD-utility trade-off than the other baselines.
In contrast, \csv{}, \sj{}, and \detgreedy{}, achieve a worse RD ($\leq$0.77).
\mc{} achieves the worst RD ($\leq$0.70) and a worst RD-utility trade-off.
{In particular, \ouralgo{}'s RD-utility trade-off strictly dominates all baselines for $\rd\geq 0.66$.
This value of $\rd$ can arise in practice:
\cref{fig:fig2b} plots the $\rd$ vs $\phi$ for this simulation and shows that if $\phi\leq 1.1$ (as required to, e.g., guarantee satisfaction of the 80\% rule), then all baselines have $\rd$ at least 0.66.}
{We further evaluate the robustness of \ouralgo{} to varying levels of noise on the Occupations dataset in \cref{sec:vary_noise} and observe \ouralgo{} has a better or similar \rd{} than each baseline at all noise levels.}

\subsection{{Simulation on Real-World Name Data}} \label{sec:simulation_name_intersectional}

{We consider gender and race (encoded as binary) as socially-salient attributes.
Our goal is to ensure equal representation across the four disjoint groups formed by combinations of these: non-White non-men, White non-men, non-White men, and White men.}

\paragraph{Data.}
{We consider the chess ranking data \cite{GhoshDW21} which has of {3,251} chess players.
For each player, among other attributes, the data has their full-name, self-identified gender (coded as male/female), FIDE rating, and race (Asian, Black, Hispanic, White) collected via MTurk.}
We use the (true) gender and race labels in the data to evaluate RD, but do not provide them to algorithms.

\paragraph{Setup.}
We partition the races into White (81.66\%) and non-White (18.34\%).
For each player $i$,
we query Genderize and EthniColr\footnote{\url{gender-api.com} and \url{github.com/appeler/ethnicolr} respectively} with $i$'s full-name to obtain the ``probabilities''  $p_f(i)$ and $p_{nw}(i)$ that player $i$ is labeled as a women and non-white respectively.
We assume that these probabilities are correct and that the gender and race of players are drawn independently.
Hence, e.g., we set the probability that $i$ is a non-white women as $$\hP_{i,nw+f} = p_{nw}(i) p_{f}(i).$$
Similarly, we set
$$\hP_{i,w+f} = (1-p_{nw}(i) )p_{f}(i),\ \
\hP_{i,nw+m} = p_{nw}(i) (1-p_{f}(i)),\ \ \text{and}\ \
\hP_{i,w+m} = (1-p_{nw}(i)) (1-p_{f}(i)).$$
%

\noindent {Notably, we do not calibrate $\hP$ on this data.
We verify that, like the previous simulation, $\hP$ is miscalibrated in this simulation.
E.g., only 31\% of the samples $i$ for which $\hP_{i, nw+m} > 0.75$ are labeled as `Non-white man' (instead of 75\%). {\em Hence, the assumption that $P$ is accurately known is violated in this simulation.}}
We expect calibration to improve \ouralgo{}'s performance.

\renewcommand{\folder}{./figs/main-body}

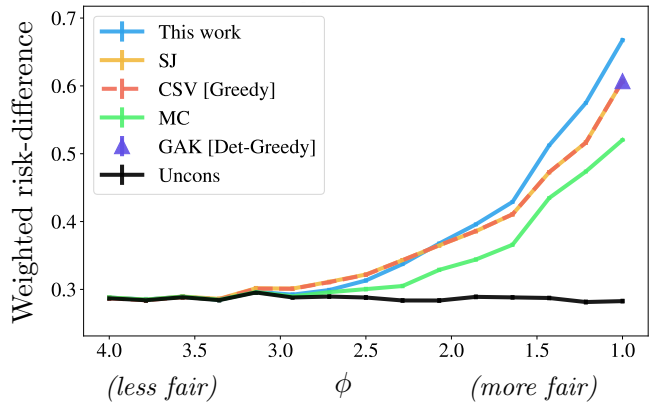
\begin{figure}[t!]
      \centering
      {\begin{tikzpicture}
      \tikzmath{\s = 1.1;}
      \tikzmath{\mvx = 0;}
      \tikzmath{\mvy = -0.3;}
      \node (image) at (0,-0.17*\s) {{\includegraphics[width=0.50\linewidth, trim={1.25cm 0cm 2.5cm 1cm},clip]{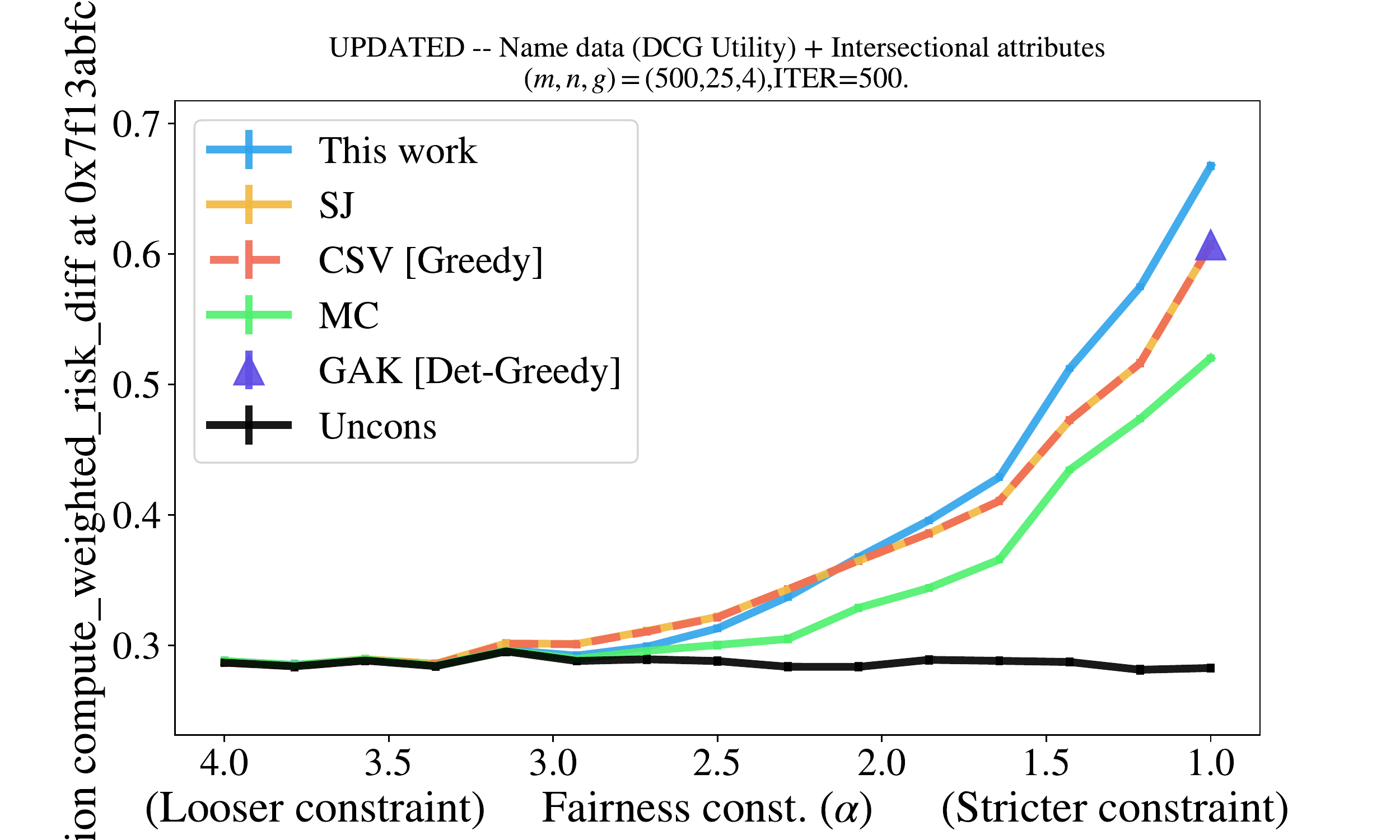}}};
      \node[rotate=90, fill=white] at (+\mvx-3.29375*\s-0.55,0) {\white{......}Weighted risk-difference\white{......}};
      \draw[draw=white, fill=white] (+\mvx-2*\s+0.35-0.5,1.2325*\s-0.3+1.2) rectangle ++(4*\s+1,0.15*\s+0.4);
      \node[rotate=0, fill=white] at (+\mvx+0.05*\s, -3.3*5/8*\s+0.05*\s-0.305+\mvy)  {\white{||||||||||||}$\phi$\white{||||||||||||}};
      \node[rotate=0,fill=white] at (+\mvx+2.35*\s,-1.8625*\s-0.475+\mvy-0.05) {\textit{\small(more fair)}};
      \node[rotate=0,fill=white] at (+\mvx-2.14375*\s,-1.8625*\s-0.475+\mvy-0.05) {\textit{\small(less fair)}};
    \end{tikzpicture}}
    \caption{
    \protect\rule{0ex}{5ex}
    {\em Real-World Name Data: Multiple Attributes.}
    In this simulation, the goal is to ensure equal representation {across four disjoint groups formed by combinations of two attributes (non-White non-men, White non-men, non-White men, and White men).}
    We estimate $P$ by querying public APIs and libraries with names in the data.
    The $y$-axis plots RD and $x$-axis plots $\phi$.
    ({\em Note that the values decrease toward the right}).
    We observe that all algorithms have a better RD than \uncons{} and \ouralgo{} has the best RD compared to all other baselines.
    Error bars represent the error of the mean.
    }
    \label{fig:simulation_intersectional}
\end{figure}

\paragraph{Results.}
See \cref{fig:simulation_intersectional} for RD averaged over 500 iterations.
We observe that all algorithms (\ouralgo{}, \csv{}, \detgreedy{}, \sj{}, and \mc{}) have better RD  than \uncons{}.
Among these, \ouralgo{} achieves the best RD  ($\approx$0.67), next \csv{}, \detgreedy{}, and \sj{} obtain RD ($\approx$0.61), and \mc{} achieves RD ($\leq 0.53$).
{More specifically, for all $\phi\leq
1.75$, \ouralgo{} has a strictly better \rd{} than all baselines (this range of $\phi$ subsumes, e.g., the range $\phi\leq 1.05$--required guarantee satisfaction of the 80\% rule with four groups.)}
Further, in \cref{fig:fig6}, we observe that all algorithms have a similar fairness-utility trade-off.

\section{Proofs of Main Theoretical Results}\label{sec:proofs}

    \subsection{Proof of \cref{thm:ub}}\label{sec:proofof:thm:ub:mb}

    In this section we prove \cref{thm:ub}. The proof is divided into the following two propositions.
    \begin{proposition}\label{prop:xNF_is_fair:main_body}
      For any $\delta\in(0,1]$, any ranking feasible for Prog.~\eqref{prog:noisy_fair_boxed} satisfies $(c\gamma,\delta)$-constraint.
    \end{proposition}

    \begin{proposition}\label{prop:xOPT_delta_is_feasible:main_body}
      For any $\delta\in(0,\frac12)$ and $c>1$,
      any ranking satisfying the $\inparen{(c-\sqrt{c}){\gamma},\delta}$-constraint is feasible for \prog{prog:noisy_fair_boxed}.
    \end{proposition}

    \begin{proof}[Proof of \cref{thm:ub} assuming \cref{prop:xNF_is_fair:main_body,prop:xOPT_delta_is_feasible:main_body}]
      Let $R^\star$ be the optimal solution of \prog{prog:noisy_fair_boxed}.
      Since $R^\star$ is feasible by definition, \cref{prop:xNF_is_fair:main_body} implies that $R^\star$ satisfies the $(c\gamma,\delta)$-constraint.
      Pick any $R'$ that satisfies the $\inparen{(c-\sqrt{c}){\gamma},\delta}$-constraint.
      \cref{prop:xOPT_delta_is_feasible:main_body} implies that $R'$ is feasible for \prog{prog:noisy_fair_boxed}.
      Since $R^\star$ is an optimal solution of \prog{prog:noisy_fair_boxed}, $R^\star$'s utility is at least as large as the utility of $R'$.
    \end{proof}

    \paragraph{Notation.}
    For each item $i$ and group $\ell$,
    let $Z_{i\ell}\in \zo$ be the indicator random variable
    $$Z_i\coloneqq \mathds{I}[G_\ell\ni i].$$
    \noindent By \cref{def:noise_model},  $\Pr[Z_{i\ell}]=P_{i\ell}$ and $Z_{i\ell}$ and $Z_{j\ell}$ are independent for any $i\neq j$.
    Given ranking $R\in \cR$, group $\ell\in [p]$, and position $k\in [n]$, let $Z_{\#}(R,\ell,k)$ be the number of items from $G_\ell$ in the top $k$ positions of $R$ and let $$P_{\#}(R,\ell,k)\coloneqq \Ex[Z_{\#}(R,\ell,k)].$$
    \noindent From the above, we get: %
    \begin{align*}
      P_{\#}(R,\ell,k) =\Ex\insquare{Z_{\#}(R,\ell,k)} =\sum_{i\in [m]}\sum_{j\in [k]} P_{i\ell} R_{ij}.
    \end{align*}
    We use the following concentration result (proved in \cref{sec:conc_ineq}) in the proof.
    \begin{lemma}\label{lem:conc_bound:main_body}
      For any position $k\in [n]$, group $\ell\in[p]$, parameters $\eps\geq 0$ and $L,U\in \R$, and ranking $R\in \cR$,
      where $R$ is possibly a random variable independent of $\inbrace{Z_{i\ell}}_{i,\ell}$,
      {if $P_{\#}(R,\ell,k) \leq U$ or $P_{\#}(R,\ell,k) \geq L$ then the following equations hold respectively}
      \begin{align*}
          \Pr\insquare{Z_{\#}(R,\ell,k) < \inparen{1+\eps} U} \geq 1-{e^{-\frac{U \eps^2}{2+\eps}}}
          \quad\text{and}\quad
          \Pr\insquare{Z_{\#}(R,\ell,k) > \inparen{1-\eps} L} \geq 1-{e^{-\frac{L\eps^2}{2(1-\eps)}}}.
      \end{align*}
    \end{lemma}

    \begin{proof}[Proof of \cref{prop:xNF_is_fair:main_body}]
      Fix any $k$ and $\ell$.
      Let
      \begin{align*}
        {\phi  \hspace{0.3mm}{\coloneqq}\hspace{0.3mm} 1-\frac{1}{2\sqrt{c}},}\quad
        {U'  \hspace{0.3mm}{\coloneqq}\hspace{0.3mm} U_{k\ell} \inparen{1+\phi\gamma_k}},
        \quad \text{and}\quad
        {\zeta   \hspace{0.3mm}{\coloneqq}\hspace{0.3mm} \frac{(1-\phi)\gamma_k}{1+\phi\gamma_k}}.
        \yesnum
        \label{eq:proof1:params_val:main_body}
      \end{align*}
      Here, $U'$ and $\zeta$ satisfy $U'(1+\zeta) = U_{k\ell} (1+c\gamma_k).$
      Fix any ranking $R$ that is feasible for \prog{prog:noisy_fair_boxed}.
      Since $R$ is feasible, it satisfies that
      \begin{align*}
         \forall \ell\in [p],\ k\in [n],\ \ \
        P_{\#}(R,\ell,k)
        \leq U_{\ell k} \inparen{1+\phi\gamma_k}.
        \yesnum\label{eq:proof1:ub:main_body}
      \end{align*}
      \noindent Using $U'(1+\zeta) = U_{k\ell} (1+c\gamma_k)$, \cref{eq:proof1:ub:main_body}, and \cref{lem:conc_bound:main_body}, we get that
      \begin{align*}
        \Pr\insquare{{Z_{\#}(R,\ell,k) \geq U'
        (1+\zeta)}}
         \ \ \ &\Stackrel{}{\leq }\ \ \  \exp\inparen{-\frac{2U^\prime\zeta^2}{2+\zeta}}\\
         \ \ \ &
         \Stackrel{\eqref{eq:proof1:params_val:main_body}}{=}\ \ \  \exp\inparen{-\frac{(1-\phi)^2 c^2\gamma_k^2  U_{k\ell} }{2 + (1+\phi)c\gamma_k}}\\
         \ \ \ &\Stackrel{(\phi\leq 1)}{=} \ \ \  \exp\inparen{-\frac{(1-\phi)^2 c^2 \gamma_k^2 U_{k\ell} }{2(1 + c\gamma_k)}}.
        \yesnum\label{eq:proof1:main_body}
      \end{align*}
      \begin{fact}\label{fact:quadratic_roots:main_body}
        For all $x,y\geq 0$, if $x\geq y+\sqrt{y}$, then $$\frac{x^2}{1+x}\geq y.$$
      \end{fact}
      \noindent Using \cref{fact:quadratic_roots:main_body} and \cref{eq:def_gamma},
      we can show that for each $k$, $${\frac{c^2\gamma_k^2}{1 + c\gamma_k} \geq \frac{2}{(1-\phi)^2 U_{k\ell}}\cdot \log\frac{2np}{\delta}}.$$
      (This uses ${\delta<\frac12}$ and ${U_{k\ell},n\geq 1}$.)
      Substituting this in \cref{eq:proof1:main_body} we get:
      \begin{align*}
        \Pr\insquare{{Z_{\#}(R,\ell,k) \geq U_{\ell k}(1+c\gamma_k)}}
        \leq \frac{\delta}{2np}.
        \yesnum\label{eq:proof1:proved:main_body}
      \end{align*}
      {Taking the union bound over all positions $k$ and $\ell$, we get (as desired) that
      with probability at least $1-\delta$, for all
      $k\in [n]$ and $\ell \in [p]$, $$Z_{\#}(R, \ell, k) \leq U_{\ell k} (1+c\gamma_{k}).$$}
    \end{proof}

    \begin{proof}[Proof of \cref{prop:xOPT_delta_is_feasible:main_body}]
              Let $${\phi\coloneqq 1-\frac{1}{2\sqrt{c}}}.$$
              Towards a contradiction, suppose that
              $R'$ satisfies $\inparen{(c-\sqrt{c}){\gamma},\delta}$-constraint but is not feasible for \prog{prog:noisy_fair_boxed}.
              Then there exists $\ell$ and $k$ such that $$P_{\#}(R',k,\ell) > U_{k\ell}\cdot\inparen{1+\phi\gamma_k}.$$
              Fix any $k$ and $\ell$ satisfying this.
              Let
              \begin{align*}
                \hspace{-2mm}
                b\hspace{0.3mm}{\coloneqq}\hspace{0.3mm} 1-\frac{1}{\sqrt{c}},\ \
                L'\hspace{0.3mm}{\coloneqq}\hspace{0.3mm} U_{k\ell}\inparen{1+\phi\gamma_k}
                \ \ \text{and}\ \
                \zeta \hspace{0.3mm}{\coloneqq}\hspace{0.3mm} \frac{(1+b)\gamma_k}{1+\phi\gamma_k}.
                \yesnum
                \label{eq:proof2:params_val2:main_body}
              \end{align*}
              It holds that ${L'(1-\zeta) = U_{k\ell}(1+b\gamma_k)}$ and, hence, we get
              \begin{align*}
                {\Pr\insquare{{Z_{\#}(R',k,\ell) \leq L^\prime (1-\zeta)}}}
                \qquad &\Stackrel{\eqref{eq:proof2:params_val2:main_body}, \ \rm Lem.\ref{lem:conc_bound:main_body}}{\leq}\qquad
                \exp\inparen{-\frac{L^\prime\zeta^2}{2(1-\zeta)}}
                \yesnum\label{eq:int_step_2}\\
                \qquad &\Stackrel{\eqref{eq:proof2:params_val2:main_body}}{=}\qquad  \exp\inparen{\frac{-(c-b)^2  U_{k\ell}\gamma_{k}}{2(1+b)}}\\
                \qquad &\Stackrel{}{\leq}\qquad  \exp\inparen{\frac{-U_{k\ell}c\gamma_{k}}{4(2\sqrt{c}-1)\sqrt{c}}}\\
                \qquad &\Stackrel{(c>0)}{=}\qquad \exp\inparen{\frac{-U_{k\ell}\gamma_{k}}{8}}.
              \end{align*}
              Since ${\gamma_k\geq 8\log{\frac{np}{\delta}}\cdot \max_{\ell}\sqrt{\frac{1}{U_{k\ell}}}}$, ${\delta<\frac{1}{2}}$, and ${U\geq 1}$, we have
              $$\Pr\insquare{{Z_{\#}(R',k,\ell)\leq U_{k\ell}}} \leq \frac{\delta}{np} < 1 - \delta.$$
              Since $R'$ satisfies $\inparen{(c-\sqrt{c}){\gamma},\delta}$-constraint we have a contradiction, hence $R'$ must be feasible. %
            \end{proof}

        \subsubsection{Proof of \cref{lem:conc_bound:main_body}}\label{sec:conc_ineq}
    In this section, we prove certain concentration inequalities which are used in the proof of  \cref{thm:ub}.
    We divide the proof of \cref{lem:conc_bound:main_body} into two parts: \cref{lem:lowerbound_main,lem:upperbound_main_2}

    For each item $i\in [m]$ and protected attribute $\ell\in [p]$, let $Z_{i\ell}\in \zo$ be the indicator random variable that the $i$-th item is in the $\ell$-th protected group, i.e.,
    if $i\in G_\ell$, then $Z_i=1$, and other $Z_i=0$.
    Using \cref{def:noise_model}, it follows that:
    \begin{align*}
      \forall i\in [m],\ \ell \in [p],\quad &\Pr[Z_{i\ell}]=P_{i\ell},
      \yesnum\label{eq:prob_Z}\\
      \forall i,j\in [m],\ \ell \in [p],\quad\st, i\neq j,\quad &\text{$Z_{i\ell}$ and $Z_{j\ell}$ are independent.}
      \yesnum\label{eq:indep_Z}
    \end{align*}

    \noindent To simplify the notation, given a ranking $R\in \cR$, a protected attribute $\ell\in [p]$, and a position $k\in [n]$, let $Z_{\#}(R,\ell,k)\in \Z$ be the random variable equal to the number of items from $G_\ell$ in the top $k$ positions of $R$ and let $P_{\#}(R,\ell,k)\in \R$  be the expectation of $Z_{\#}(R,\ell,k)$, i.e.,
    \begin{align*}
      Z_{\#}(R,\ell,k)\coloneqq \sum_{i\in [m]}\sum_{j\in [k]} Z_{i\ell} R_{ij}
      \quad\text{and}\quad
      P_{\#}(R,\ell,k)\coloneqq \Ex\insquare{Z_{\#}(R,\ell,k)}.
    \end{align*}
    Using \cref{eq:prob_Z} and linearity of expectation it follows that
    \begin{align*}
      P_{\#}(R,\ell,k) = \sum_{i\in [m]}\sum_{j\in [k]} P_{i\ell} R_{ij}.
    \end{align*}

    \begin{lemma}\label{lem:lowerbound_main}
      For any position $k\in [n]$, attribute $\ell\in[p]$, parameters $\eps\geq 0$ and $L\in \R$, and ranking $R\in \cR$,
      where $R$ is possibly a random variable and is independent of $\inbrace{Z_{i\ell}}_{i,\ell}$,
      if $P_{\#}(R,\ell,k) \geq L$
      then with probability at least $1-\exp\inparen{-\frac{L\eps^2}{2(1-\eps)}}$, it holds that $Z_{\#}(R,\ell,k) > L\inparen{1-\eps}$.
    \end{lemma}
    \begin{proof}
      Since $\ell$, $k$, and $R$ are fixed, we use $Z_{\#}$ and $P_{\#}$ to denote $Z_{\#}(R,\ell,k)$ and $P_{\#}(R,\ell,k)$ respectively.
      Since $R$ and $\inbrace{Z_{i\ell}}_{i,\ell}$ are independent, we can bound the required probability as follows
      \begin{align*}
        \Pr\insquare{Z_{\#} \leq L (1-\eps)}\ \
        &=\ \ \Pr\insquare{Z_{\#} \leq
        P_{\#}\cdot \inparen{1-\frac{P_{\#}-L(1-\eps)}{P_{\#}}}}\\
        &\leq\ \ \exp\inparen{-\frac{P_{\#}}{2}\cdot \inparen{\frac{P_{\#}-L(1-\eps)}{P_{\#}}}^2 }\tag{Chernoff bound, see \cite{motwani1995randomized}}\\
        &=\ \ \exp\inparen{-\frac{1}{2}\cdot \frac{\inparen{P_{\#}-L(1-\eps)}^2}{P_{\#}} }.\yesnum\label{eq:proof_1}
      \end{align*}
      To bound the right-hand side of \cref{eq:proof_1}, we will use the following fact.
      \begin{fact}\label{fact:algebra_lowerbound}
        For all $L,\eps>0$, $\frac{(x-L(1-\eps))^2}{x}$ attains its minima at $L$ over the domain $[L,\infty)$.
      \end{fact}
      \noindent Since $P_{\#}\geq L$, from \cref{fact:algebra_lowerbound} it follows that the right-hand side of \cref{eq:proof_1} attains its maxima at $P_{\#} = L$.
      Substituting $P_{\#} = L$ in \cref{eq:proof_1}, we get:
      \begin{align*}
        \Pr\insquare{Z_{\#} \leq L (1-\eps)}
        &\leq \exp\inparen{-\frac{1}{2}\cdot \frac{\inparen{L\eps}^2}{L(1-\eps) }}
        = \exp\inparen{\frac{-L\eps^2}{2(1-\eps)}}.
      \end{align*}
    \end{proof}

    \begin{lemma}\label{lem:upperbound_main}
      For any position $k\in [n]$, attribute $\ell\in[p]$, parameters $\eps\geq 0$ and $U\in \R$, and ranking $R\in \cR$,
      where $R$ is possibly a random variable and is independent of $\inbrace{Z_{i\ell}}_{i,\ell}$,
      if $R$ satisfies that $P_{\#}(R,\ell,k) \leq U$
      then with probability at least $1-\exp\inparen{-\frac{U \eps^2}{2+\eps}}$, it holds that $Z_{\#}(R,\ell,k) < \inparen{1+\eps}\cdot U$.
    \end{lemma}
    \begin{proof}
      Since $\ell$, $k$, and $R$ are fixed, we use $Z_{\#}$ and $P_{\#}$ to denote $Z_{\#}(R,\ell,k)$ and $P_{\#}(R,\ell,k)$ respectively.
      Since $R$ and $\inbrace{Z_{i\ell}}_{i,\ell}$ are independent, we can bound the required probability as follows
      \begin{align*}
        \Pr\insquare{Z_{\#} \geq U (1+\eps)}\ \
        &=\ \ \Pr\insquare{Z_{\#} \leq
        P_{\#}\cdot \inparen{1 + \frac{U(1+\eps) - P_{\#}}{P_{\#}}}}\\
        &\leq\ \ \exp\inparen{P_{\#} \cdot\inparen{\frac{U(1+\eps) - P_{\#}}{P_{\#}}}^2 \cdot \frac{1}{2+\frac{U(1+\eps) - P_{\#}}{P_{\#}}}}.
        \intertext{Where we used the fact that: For any $\delta>0$ and independent 0/1 random variables $Y_1,Y_2,\dots,Y_n$, $\Pr\insquare{\sum_{i}Y_i>(1+\delta)\mu}<\exp\inparen{\frac{\mu\delta^2}{2+\delta}}$, where $\mu\coloneqq \Ex[\sum_i Y_i]$ (see\cite{motwani1995randomized}).
        Simplifying the right-hand side of the above equation, we get:}
        \Pr\insquare{Z_{\#} \geq U (1+\eps)}\ \
        &=\ \ \exp\inparen{-\frac{\inparen{U(1+\eps) - P_{\#}}^2}{U(1+\eps) + P_{\#}}}.\yesnum\label{eq:proof_2}
      \end{align*}
      To bound the right-hand side of \cref{eq:proof_2}, we will use the following fact.
      \begin{fact}\label{fact:algebra_upperbound}
        For all $U,\eps>0$, $\frac{(U(1+\eps)-x)^2}{U(1+\eps)+x}$ attains its minima at $U$ over the domain $[0,U]$.
      \end{fact}
      \noindent Since $P_{\#} \leq U$, from \cref{fact:algebra_upperbound} it follows that the right-hand side of \cref{eq:proof_2} attains its maxima at $P_{\#} = U$.
      Substituting $P_{\#} = U$ in \cref{eq:proof_2}, we get:
      \begin{align*}
        \Pr\insquare{Z_{\#} \geq U (1+\eps)}
        \ \ &\leq\ \ \exp\inparen{\frac{-U\eps^2}{2+\eps}}.\yesnum\label{eq:proof_2_2}
      \end{align*}
    \end{proof}

    \subsubsection{Improved Dependence of \cref{thm:ub} on $\gamma$ on $\delta$}
        \label{sec:imp_depen_delta:thm:ub}
        In this section, we show that given a constant $\psi>0$, if $U$ satisfies that
        \begin{align*}
          \forall\ell\in[p],\forall k\in[n],\quad
          U_{k\ell} \geq \psi k,
        \end{align*}
        then we can improve the dependence of $\gamma$ (from \cref{eq:def_gamma}) on $\log\frac{2np}{\delta}$ and $\alpha$.
        Concretely, \cref{thm:ub} holds for the following $\gamma$:
        \begin{align*}
          \forall k\in [n],\quad
          \gamma_k\coloneqq \max_{\ell\in [p]}\sqrt{\frac{1}{2\psi}\cdot \log\inparen{\frac{2np}\delta}\cdot \frac{1}{U_{k\ell }}}.
          \yesnum\label{eq:def_gamma:proof:imp}
        \end{align*}

        \noindent The proof of this relies on analogous of \cref{lem:lowerbound_main,lem:upperbound_main}: \cref{lem:lowerbound_main_2,lem:upperbound_main_2}.
        \begin{lemma}\label{lem:lowerbound_main_2}
          For any position $k\in [n]$, attribute $\ell\in[p]$, parameter $\eps\geq 0$, and lower bound constraint $L\in \Z_{\geq 0}^{n\times p}$, and ranking $x\in \cR$,
          if $x$ satisfies that $P_{\#}(R,\ell,k) \geq L$
          then with probability at least $1-\exp\inparen{-2L^2\eps^2 k^{-1}},$ it holds that $Z_{\#}(R,\ell,k) > L\inparen{1-\eps}.$
        \end{lemma}
        \begin{lemma}\label{lem:upperbound_main_2}
          For any position $k\in [n]$, attribute $\ell\in[p]$, parameters $\eps\geq 0$ and $U\in \R$, and ranking $R\in \cR$,
          where $R$ is possibly a random variable and is independent of $\inbrace{Z_{i\ell}}_{i,\ell}$,
          if $R$ satisfies that $P_{\#}(R,\ell,k) \leq U$
          then with probability at least $1-\exp\inparen{- \frac{2U^2 \eps^2}{k}},$ it holds that $Z_{\#}(R,\ell,k) < U\inparen{1+\eps}.$
        \end{lemma}
        \noindent To prove the improved dependence of $\gamma$, it suffices to prove \cref{prop:xNF_is_fair:main_body,prop:xOPT_delta_is_feasible:main_body}.
        For the new value of $\gamma$, their proofs change as follows:

        \paragraph{Proof of \cref{prop:xNF_is_fair:main_body}.}
        The parameters in \cref{eq:proof1:params_val:main_body} remain the same.
        Hence, following the same argument, \cref{eq:proof1:ub:main_body} holds.
        Now, we can prove \cref{eq:proof1:proved:main_body} as follows:
        \begin{align*}
          \Pr\insquare{Z_{\#}(R,\ell, k) \geq U_{\ell k}(1+\phi\gamma_k)}
           \
          &\Stackrel{}{=} \   \Pr\insquare{Z_{\#}(R,\ell,k) \geq U'
          (1+\zeta)}
          \tag{Using that $U'(1+\zeta) = U_{k\ell} (1+\phi\gamma_k)$}\\
          &\leq \   \exp\inparen{-\frac{2\inparen{U^\prime}^2\zeta^2}{k}}
          \tag{Using \cref{lem:upperbound_main_2}}\\
          &= \   \exp\inparen{-\frac{2(1-\phi)^2U_{\ell k}^2 \gamma_{k}^2}{k}}
          \tag{Using \cref{eq:proof1:params_val:main_body}}\\
          &\Stackrel{}{\leq} \   \exp\inparen{-2\psi(1-\phi)^2U_{\ell k} \gamma_{k}^2}
          \tag{Using that $U_{k\ell}\geq \psi k$}\\
          &\Stackrel{}{\leq} \  \frac{\delta}{2np}.
          \tagnum{Using \cref{eq:def_gamma:proof:imp}}
          \customlabel{eq:delta_bound_on_upperbound}{\theequation}
        \end{align*}
        \cref{prop:xNF_is_fair:main_body} follows by replacing \cref{eq:proof1:proved:main_body} by \Eqref{eq:delta_bound_on_upperbound} in the rest of its proof.

        \paragraph{Proof of \cref{prop:xOPT_delta_is_feasible:main_body}.}
        The parameters in \cref{eq:proof2:params_val2:main_body} remain the same.
        Now, we can prove $\Pr\insquare{{Z_{\#}(R',k,\ell)\leq U_{k\ell}}}< 1 - \delta$ as follows:
        \begin{align*}
          \Pr\insquare{Z_{\#}(R',k,\ell) \leq U_{k\ell}}\ \
          &=\ \Pr\insquare{Z_{\#}(R',k,\ell) \leq L^\prime\cdot (1-\zeta)}
          \tag{Using that $L'(1-\zeta) = U_{k\ell}(1+b\gamma_k)$}\\
          &\leq\ \exp\inparen{-\frac{2\inparen{L^\prime}^2\zeta^2}{k}}
          \tag{Using \cref{lem:lowerbound_main_2}}\\
          &=\ \exp\inparen{-\frac{2(\phi-b)^2\gamma_k^2U_{k\ell}^2}{k}}
          \tag{Using \cref{eq:proof2:params_val2:main_body}}\\
          &\leq\ \exp\inparen{-2\psi(\phi-b)^2\gamma_k^2U_{k\ell}}
          \tag{Using that $U_{k\ell}\geq \psi k$}\\
          &\Stackrel{}{<}\ \frac{\delta}{2np}
          \tagnum{Using \cref{eq:def_gamma:proof:imp} and \cref{eq:proof2:params_val2:main_body}}\\
          &\Stackrel{}{<}\ 1-\delta.
          \tagnum{Using that $\delta<\frac12$ and $n\geq 1$}
          \customlabel{eq:proof2:proved:alt}{\theequation}
        \end{align*}
        The rest of the proof is identical.

        \begin{proof}[Proof of \cref{lem:lowerbound_main_2}]
          First, note that since $x$ is not a function of the outcomes of the random variables $Z_{i\ell}$, $x$ is independent of the random variables $\inbrace{Z_{i\ell}}_{i,\ell}$.
          Since $\ell$, $k$, and $x$ are fixed, we use $Z_{\#}$ and $P_{\#}$ to denote $Z_{\#}(R,\ell,k)$ and $P_{\#}(R,\ell,k)$ respectively.
          Now, we can bound the required probability as follows
          \begin{align*}
            \Pr\insquare{Z_{\#} \leq L (1-\eps)}\ \
            &=\ \ \Pr\insquare{Z_{\#} \leq
            P_{\#}\cdot \inparen{1-\frac{P_{\#}-L(1-\eps)}{P_{\#}}}}\\
            &\leq\ \ \exp\inparen{-\frac{2}{k}\cdot P_{\#}^2 \cdot \inparen{\frac{P_{\#}-L(1-\eps)}{P_{\#}}}^2 }.
            \intertext{Where we used the fact that: For any $\delta>0$ and bounded random variables $Y_1,Y_2,\dots,Y_n\in [0,1]$, $$\Pr\insquare{\sum\nolimits_{i}Y_i<(1-\delta)\mu}<\exp\inparen{-2\mu^2\delta^2n^{-1}}.$$ Here, $\mu\coloneqq \Ex[\sum_i Y_i]$.}
            \Pr\insquare{Z_{\#} \leq L (1-\eps)} \ \ &=\ \ \exp\inparen{-\frac{2}{k} \cdot \inparen{P_{\#}-L(1-\eps)}^2 }\\
            &\leq\ \ \exp\inparen{-2L^2 \eps^2 k^{-1}}.
          \end{align*}
        \end{proof}

        \begin{proof}[Proof of \cref{lem:upperbound_main_2}]
          Since $\ell$, $k$, and $R$ are fixed, we use $Z_{\#}$ and $P_{\#}$ to denote $Z_{\#}(R,\ell,k)$ and $P_{\#}(R,\ell,k)$ respectively.
          Since $R$ and $\inbrace{Z_{i\ell}}_{i,\ell}$ are independent, we can bound the required probability as follows
          \begin{align*}
            \Pr\insquare{Z_{\#} \geq U (1+\eps)}\ \
            &=\ \ \Pr\insquare{Z_{\#} \leq
            P_{\#}\cdot \inparen{1 + \frac{U(1+\eps) - P_{\#}}{P_{\#}}}}\\
            &\leq\ \ \exp\inparen{-\frac{2}k \cdot P_{\#}^2 \cdot\inparen{\frac{U(1+\eps) - P_{\#}}{P_{\#}}}^2}.
            \intertext{Where we used the fact that: For any $\delta>0$ and bounded random variables $Y_1,Y_2,\dots,Y_n\in [0,1]$, $$\Pr\insquare{\sum_{i}Y_i>(1+\delta)\mu}<\exp\inparen{-2\mu^2\delta^2n^{-1}},$$ where $\mu\coloneqq \Ex[\sum_i Y_i]$ (\cite{motwani1995randomized}).
            Simplifying the right-hand side of the above equation, we get }
            \Pr\insquare{Z_{\#} \geq U (1+\eps)}\ \
            &\leq\ \ \exp\inparen{-\frac{2}k\inparen{U(1+\eps) - P_{\#}}^2}\\
            &\leq\ \ \exp\inparen{-\frac{2U^2\eps^2}k}. \tag{Using that $P_{\#}\leq U$}
          \end{align*}
        \end{proof}

    \subsection{Proof of \cref{thm:lb}}\label{sec:proofof:thm:lb}
        We consider the family of matrices $U\in \R^{n\times p}$ that satisfy the following condition:
        For each position $k\in [n]$, there exists an attribute $\ell$ such that
        \begin{align*}
          U_{k\ell}\leq \frac{k}{4}.
        \end{align*}
        Notably, equal representation constraints satisfy this condition for any $p\geq 4$.
        We will use \cref{fact:lb_chernoff} to prove \cref{thm:lb}.
        \begin{fact}[Theorem 2 in \protect{\cite{mousavi2010tight}}]\label{fact:lb_chernoff}
          For all $p\in (0,\frac14]$, $0\leq \eps\leq \frac{1}{p}(1-p)$, and $s\in \N$ independent 0/1 random variables $Z_1,Z_2,\dots,Z_s\in \zo$, such that for all $i\in [s]$, $\Pr[Z_i=1]=p$,
          \begin{align*}
            \Pr\insquare{\sum\nolimits_{i\in [s]} Z_i \geq (1+\eps)ps }\geq \frac14\exp\inparen{-2\eps^2 ps}.
          \end{align*}
        \end{fact}
        \begin{proof}[Proof of \cref{thm:lb}]
          Fix the $k$ to the value specified in the theorem.
          Let $\ell\in [n]$, be any attribute such that $U_{k\ell}\leq \frac{k}{4}$.
          Such a $\ell$ exists because of the family of constraints we chose.
          Without loss of generality suppose $\ell\neq 1$.
          Fix any $n,m\geq k$.
          For each item $i\in [m]$, set
          \begin{align*}
            P_{i\ell} &\coloneqq
            {\frac{U_{k\ell}}{k}}
            \quad\text{and}\quad
            P_{i1} \coloneqq
            {1-\frac{U_{k1}}{k}}.
            \yesnum
            \label{eq:proof3:construction}
          \end{align*}
          Further, for all $k\in [p]$, $k\neq p$ and $k\neq 1$, let $P_{ik} \coloneqq 0$.

          Suppose, toward a contradiction, that there is a ranking $R\in \cR$ that satisfies the $(\eps,\delta)$-constraint.
          $R$ must satisfy the following equation:
          \begin{align*}
            \Pr\insquare{Z_{\#}(R,k,\ell) \leq U_{k\ell}\cdot (1+\eps_k)}\geq 1-\delta.
            \yesnum
            \label{eq:proof3:ub}
          \end{align*}
          For each position $j\in [n]$, let $Z_j\in \zo$ be the indicator random variable that the item placed in the $j$-th place in the ranking $R$
          is in the protected group $G_\ell$.
          From \cref{eq:proof3:construction} and \cref{def:noise_model}, it follows that:
          \begin{align*}
            \forall j\in [n],\quad &\Pr[Z_{j}]=\frac{U_{k\ell}}{k},
            \yesnum\label{eq:proof3:prob_Z}\\
            \forall u,v\in [n],\quad\st, u\neq v,\quad &\text{$Z_{u}$ and $Z_{v}$ are independent.}
            \yesnum\label{eq:proof3:indep_Z}
          \end{align*}
          Using linearity of expectation and \cref{eq:proof3:prob_Z}, we get that:
          \begin{align*}
            \Pr\insquare{Z_{\#}(R,k,\ell) \leq  (1+\eps_k)\cdot U_{k\ell}}
            \ \
            &=\ \  \Pr\insquare{\sum\nolimits_{j\in [k]} Z_j \geq (1+\eps_k)\cdot \Ex\insquare{\sum\nolimits_{j=1}^k Z_j}}.
            \yesnum\label{eq:proof3:eq1}
          \end{align*}
          Since $0\leq\eps_k\leq 1$ and $\frac{1}{k}\Ex\insquare{\sum\nolimits_{j=1}^k Z_j}\leq \frac{1}{4}$, we can use \cref{fact:lb_chernoff} with $\eps\coloneqq \eps_k$, $p\coloneqq \frac{1}{k}\Ex\insquare{\sum\nolimits_{j=1}^k Z_j}\leq \frac{1}{4}$, $s\coloneqq k$, and for all $j\in [n]$, $Z_j = Z_j$.
          Using this, we get that
          \begin{align*}
            \Pr\insquare{\sum\nolimits_{j\in [k]} Z_j \geq (1+\eps_k)\cdot \Ex\insquare{\sum\nolimits_{j=1}^k Z_j}}
            \ \
            &\Stackrel{}{\leq}\ \  1-\frac14\exp\inparen{-2\eps_k^2 \cdot \Ex\insquare{\sum\nolimits_{j=1}^k Z_j} }\\
            \ \
            &\Stackrel{}{\leq}\ \  1-\frac14\exp\inparen{-2\eps_k^2 U_{k\ell} }.
            \tagnum{Using \cref{eq:proof3:prob_Z}}\customlabel{eq:proof3:eq2}{\theequation}
          \end{align*}
          Chaining Equations~\eqref{eq:proof3:ub}, \eqref{eq:proof3:eq1}, and \eqref{eq:proof3:eq2}, we get that
          \begin{align*}
            1-\frac14\exp\inparen{-2\eps_k^2 U_{k\ell} } \geq 1-\delta.
          \end{align*}
          Hence,
          \begin{align*}
            \eps_k \geq  \sqrt{\frac{1}{2U_{k\ell}} \log\frac{1}{4\delta}}.
          \end{align*}
          This is a contradiction since $\eps_k$ is specified to be less than $\sqrt{\frac{1}{2U_{k\ell}} \log\frac{1}{4\delta}}$.
          Thus, no ranking $R$ satisfies the $(\eps,\delta)$-constraint for any $U$ in the chosen family chosen.
        \end{proof}

    \subsection{Proof of \cref{thm:algo}}\label{sec:proofof:thm:algo}

    In this section, we prove \cref{thm:algo}.
    Our algorithm uses the dependent-rounding algorithm of \cite{ChekuriVZ11} as a subroutine.

    \paragraph{\textbf{\em Remark.} Desirable Properties and Potential Approaches for Rounding.}
        \textit{At a high level, the goal of this dependent-rounding algorithm is the following:
        Given a feasible solution $R_c$ of the standard linear programming relaxation of Program (7) output a ranking $R$ such that
        for any matrix $A$ with nonnegative entries, $\inangle{R,A}$ is approximately equal to $\inangle{R_c, A}$.
        This property guarantees that, with high probability, $R$ approximately satisfies the fairness constraints and has a similar utility as $R_c$.}

        \textit{A naive approach to achieve this property is to do independent rounding:
        For each $i$ and $j$, set $R_{ij}=1$ with probability $(R_c)_{ij}$.
        The desired concentration property then follows from, e.g., the Chernoff bound.
        However, the resulting $R$ may not be a valid ranking because it could set $R_{ij}=R_{ik}=1$ for $j\neq k$, hence requiring $i$ to appear at two different positions (which is not possible).
        Similarly, it could also place more than one items at one position (which also violates the constraints).}

        \textit{Another approach is (1) to express $R_c$ as a convex combination of rankings $\sum_{i}\alpha_i R_i$ ($\alpha_i\geq 0$) (e.g., using the Birkhoff von Neumann decomposition) and (2) set $R\coloneqq R_i$ with probability $\propto \alpha_i$.
        Since each $R_i$ is a ranking this guarantees that $R$ is a ranking, but it may violate fairness constraints significantly.
        For example, consider the fractional assignment in which the $k$-th best female (respectively male)  appears in the $k$-th position with weight 0.5 for all positions $k$.
        This can be decomposed into two rankings: 1) females are ranked in decreasing order of utility, and 2) males are ranked in decreasing order of utility.
        The fractional solution satisfies equal representation, but both rankings  violate equal representation significantly.}

        \textit{The dependent-rounding algorithm of \cite{ChekuriVZ11}, which we use, also expresses $R_c$ as a convex combination of rankings $\sum_{i}\alpha_i R_i$ ($\alpha_i\geq 0$).
        But it does not output $R_i$ for any $i$.
        Instead, it initially, sets $R\coloneqq R_1$.
        Then it iteratively ``merges'' $R$ with $R_2$, then $R_3$, and so on.}

    \medskip

    \noindent \cite{ChekuriVZ11}'s algorithm satisfies the following guarantees.

    \begin{theorem}[\textbf{Theorem 1.1 from \cite{ChekuriVZ11}}]\label{thm:rounding}
      Let $P\subseteq [0,1]^N$ be either a matroid intersection polytope or a (non-bipartite graph) matching polytope.
      For any fixed $0 < \alpha \leq \frac12$, there is an efficient randomized rounding procedure, such that given a (fractional) point $R_F\in P$, it outputs a random feasible solution $R$ corresponding to a (integer) vertex of $P$ such that $$\Ex[1_R]=(1-\alpha)\cdot R_F.$$
      In addition, for any linear function $w(R)\coloneqq \sum_{i\in R}w_i,$ where $w_i\in [0,1]$ it holds that
      \begin{enumerate}[itemsep=0pt]
        \item for any $\delta\in [0,1]$ and $\mu\leq \Ex[1_R]$, $$\Pr\insquare{w(R)\leq (1-\delta)\mu}\leq \exp\inparen{-\frac1{20}\cdot \mu\alpha\delta^2},$$
        \item for any $\delta\in [0,1]$ and $\mu\geq \Ex[1_R]$, $$\Pr\insquare{w(R)\geq (1-\delta)\mu}\leq \exp\inparen{-\frac1{20}\cdot \mu\alpha\delta^2},$$
        \item for any $\Delta\geq 1$ and $\mu\geq \Ex[1_R]$, $$\Pr\insquare{w(R)\geq  \mu(1+\Delta)}\leq \exp\inparen{-\frac1{20}\cdot \mu\alpha(2\Delta-1)}.$$
      \end{enumerate}
      The algorithm runs in time polynomial in the size of the ground set, $N$, and $\frac1\alpha$, and makes at most $\poly(N,d)$ calls to the independence oracles for the underlying matroids.
    \end{theorem}

    \noindent We claim that the following algorithm satisfies the claim in \cref{thm:algo}

    \newcommand{\commentalg}[1]{{\small\em\quad \textcolor{gray}{//{#1}}}\hspace{-3mm}}

    \begin{algorithm}[h!]

      \caption{Algorithm for fair ranking with noisy protected attributes (\ouralgo{})}
      \label{algo}

      {\bf Input:} Matrices $P\in  {[0,1]}^{m\times p}$, $W\in {\R}_{\geq 0}^{m\times n}$, $U\in  {\R}^{n\times p}$

      \vspace{2mm}

      {\bf Parameters:} Constant $d>2$ and $c>1$, a failure probability $\delta\in (0,1]$, and for each $k\in [n]$, a relaxation parameter $\gamma_k$

      \vspace{2mm}

      1. {\bf Initialize} {$R_F\gets$ Solve the linear-programming relaxation of \prog{prog:noisy_fair_boxed} with specified inputs}

      2. {{\bf Round } $R\gets$ Run \cite{ChekuriVZ11}'s rounding algorithm with input $\alpha\coloneqq \frac{1}{d}$ and $P\coloneqq \conv\inparen{\cR}$}

      3. {\bf Return} $R$
    \end{algorithm}

    For each item $i\in [m]$ and protected attribute $\ell\in [p]$, let $Z_{i\ell}\in \zo$ be the indicator random variable that the $i$-th item is in the $\ell$-th protected group, i.e.,
    if $i\in G_\ell$, then $Z_i=1$, and other $Z_i=0$.
    Using \cref{def:noise_model}, it follows that:
    \begin{align*}
      \forall i\in [m],\ \ell \in [p],\quad &\Pr[Z_{i\ell}]=P_{i\ell},
      \yesnum\label{eq:prob_Z2}\\
      \forall i,j\in [m],\ \ell \in [p],\quad\st, i\neq j,\quad &\text{$Z_{i\ell}$ and $Z_{j\ell}$ are independent.}
      \yesnum\label{eq:indep_Z2}
    \end{align*}

    \noindent To simplify the notation, given a ranking $R\in \cR$, a protected attribute $\ell\in [p]$, and a position $k\in [n]$, let $Z_{\#}(R,\ell,k)\in \Z$ be the random variable equal to the number of items from $G_\ell$ in the top $k$ positions of $R$ and let $P_{\#}(R,\ell,k)\in \R$  be the expectation of $Z_{\#}(R,\ell,k)$, i.e.,
    \begin{align*}
      Z_{\#}(R,\ell,k)\coloneqq \sum_{i\in [m]}\sum_{j\in [k]} Z_{i\ell} R_{ij}
      \quad\text{and}\quad
      P_{\#}(R,\ell,k)\coloneqq \Ex\insquare{Z_{\#}(R,\ell,k)}.
    \end{align*}
    Using \cref{eq:prob_Z2} and linearity of expectation it follows that
    \begin{align*}
      P_{\#}(R,\ell,k) = \sum_{i\in [m]}\sum_{j\in [k]} P_{i\ell} R_{ij}.
    \end{align*}

    \begin{proof} The proof is divided into three parts that prove \cref{algo}'s running time guarantee, utility guarantee, and fairness guarantee respectively.

      \paragraph{Running Time.}
      The Step 1 of \cref{algo} runs in polynomial time when implemented with any polynomial-time linear programming solver.
      Observe that $\cR$ corresponds to the bipartite matching polytope, whose bi-partitions have size $n$ and $m$ respectively.
      Since the bipartite matching polytope is a matroid intersection polytope, we can use \cref{thm:rounding}.
      The independence oracle for this polytope can be implemented in $\poly(m)$ time, e.g., using the Birkhoff–von Neumann theorem.
      Finally, since $\alpha=\frac1d$ and $N=O(m^2)$, it follows that Step 2 of \cref{algo} runs in polynomial time in $d$ and the bit complexity of the input (which is at least $m$).

      Let
      $$\phi\coloneqq \frac{2\sqrt{c}-1}{2\sqrt{c}}.$$
      Let $R_F$ and $R$ be the rankings from Steps 1 and 2 of \cref{algo}.
      From \cref{thm:rounding}, we have that $\Ex[1_R]=(1-\alpha)\cdot R_F$.
      Hence, for any weights $V\in \R^{n\times m}$, it holds that
      \begin{align*}
        \Ex\insquare{ \inangle{R, V}  } = (1-\alpha)\cdot \inangle{R_F, V}.
        \yesnum
        \label{eq:exp}
      \end{align*}
      Fix any position $k\in [n]$ and group $\ell\in [p]$.
      Since $\ell$, $k$, and $R$ are fixed,
      {we use $Z_{\#}(R)$ and $Z_{\#}(R')$ and $P_{\#}$ to denote $Z_{\#}(R,\ell,k)$ and $P_{\#}(R,\ell,k)$ respectively.}

      \paragraph{Utility Guarantee.}
      Let $R^\star$ be the solution of \prog{prog:noisy_fair_boxed} for $c=d$.
      Let $$V\coloneqq \inangle{W, R^\star}.$$
      Let $0\leq \Delta\leq V$ be a parameter.
      Since $R_F$ is a solution of the LP-relaxation of \prog{prog:noisy_fair_boxed} and $R^\star$ is a solution of \prog{prog:noisy_fair_boxed}, $R_F$'s utility is at least as large as the utility of $R^\star$.
      From this it follows that
      \begin{align*}
        \Pr\insquare{  \inangle{W, R} \leq \inangle{W, R^\star}\cdot (1-\alpha)  -  \Delta }
        \leq \Pr\insquare{  \inangle{W, R} \leq \inangle{W, R_F}\cdot (1-\alpha)  -  \Delta }.
        \yesnum\label{eq:proofa:util1}
      \end{align*}
      Since $W\in[0,1]^{m\times n}$, we can use \cref{thm:rounding} with $a=W$.
      Using this we get can upper bound the RHS of the above equation.
      \begin{align*}
        \Pr\insquare{  \inangle{W, R} \leq \inangle{W, R_F}\cdot (1-\alpha)  -  \Delta }
        &=  \Pr\insquare{  \inangle{W, R} \leq \Ex\insquare{\inangle{W, R}} -  \Delta }
        \tag{Using \cref{eq:exp}}\\
        &\leq  \exp\inparen{-\frac{\alpha}{20}\cdot \frac{\Delta^2}{\inangle{W, R_F}\cdot (1-\alpha)} }.
      \end{align*}
      Let $$\Delta \coloneqq \sqrt{\frac{20}{\alpha} \cdot \inangle{W, R_F}\cdot  (1-\alpha)\cdot \log\inparen{\frac{2np}{\delta}}}.$$
      Substituting the value of $\Delta$ in the above equation, we have:
      \begin{align*}
        \Pr\insquare{  \inangle{W, R} \leq \Ex\insquare{\inangle{W, R}} -  \Delta }
        &\leq  \frac{\delta}{2np}.
        \yesnum\label{eq:proofa:util2}
      \end{align*}
      Chaining the inequalities in \cref{eq:proofa:util1,eq:proofa:util2}
      \begin{align*}
        \Pr\insquare{  \inangle{W, R} \leq \inangle{W, R^\star}\cdot (1-\alpha)  -  \Delta }\leq \frac{\delta}{2n}.
      \end{align*}
      Since each entry of $W$ is at most 1 and $\sum_{i,j}\inparen{R_F}_{ij}=n$, it follows that $\inangle{W, R_F}\leq n$.
      Using this and that $\alpha=\frac{1}{d}$,
      $$\Delta = O\inparen{\sqrt{dn\cdot \log{\frac{2np}{\delta}}}}.$$
      Thus, the utility guarantee follows.

      \paragraph{Fairness Guarantee.}
      Since $R_F$ is feasible for the LP-relaxation of \prog{prog:noisy_fair_boxed}, it holds that
      \begin{align*}
        P_{\#}(R_F) \leq U_{k\ell} (1+\phi\gamma_k).
        \yesnum\label{eq:pp}
      \end{align*}

      \noindent Let $\eps>0$ be some constant such that
      \begin{align*}
        \eps\geq \phi\gamma_k.
        \yesnum\label{eq:proofa:bbb}
      \end{align*}
      We divide the analysis into two cases depending on the value of $\eps$.

      \paragraph{Case A ($P_{\#}(R)\geq \frac{1}{2}U_{k\ell}(1+\eps)$):}
      Since $P_{\#}(R)\geq \frac{1}{2}\cdot U_{k\ell}(1+\eps)$, we have that
      \begin{align*}
        \frac{U(1+\eps)-P_{\#}(R)}{P_{\#}(R)} \leq 1.
        \yesnum\label{eq:proofa:aaa}
      \end{align*}
      We have that
      \begin{align*}
        \Pr\insquare{  Z_{\#}(R) > U_{k\ell}(1+\eps) }
        &= \Pr\insquare{  Z_{\#}(R) > P_{\#}(R)\cdot \inparen{ 1 + \frac{U_{k\ell}(1+\eps)-P_{\#}(R)}{P_{\#}(R)}  } }.
        \intertext{From \cref{eq:exp} it follows that $P_{\#}(R)=P_{\#}(R_F)(1-\alpha)$. Then from \cref{eq:pp,eq:proofa:bbb} we have that $P_{\#}(R)\leq U_{k\ell}(1+\eps)$.
        Hence, $$\frac{U_{k\ell}(1+\eps)-P_{\#}(R)}{P_{\#}(R)}\geq 0.$$
        Further, from \cref{eq:proofa:aaa} $\frac{U_{k\ell}(1+\eps)-P_{\#}(R)}{P_{\#}(R)}\leq 0$.
        Hence, we can use the second statement of \cref{thm:rounding}.
        Using this we get
        }
        \Pr\insquare{  Z_{\#}(R) > U_{k\ell}(1+\eps) } &\leq \exp\inparen{-\frac{\alpha}{20} \cdot P_{\#}(R)\cdot \inparen{   \frac{U_{k\ell}(1+\eps)-P_{\#}(R)}{P_{\#}(R)} }^2 } \\
        &\leq \exp\inparen{-\frac{\alpha}{20} \cdot P_{\#}(R_F)\cdot \inparen{   \frac{U_{k\ell}(1+\eps)-P_{\#}(R_F)}{P_{\#}(R_F)} }^2 }
        \tag{{\cref{fact:algebra_lowerbound} and that $P_{\#}(R)\leq P_{\#}(R_F)$}}\\
        &\leq \exp\inparen{-\frac{\alpha}{20}\cdot U_{k\ell} \cdot {   \frac{\inparen{\eps-\phi \gamma_k}^2 }{1+\phi \gamma_k} } }.
        \tagnum{{\cref{fact:algebra_lowerbound} and \cref{eq:pp} }}
        \customlabel{eq:proofa:casea}{\theequation}
      \end{align*}

      \paragraph{Case B ($P_{\#}(R) < \frac{1}{2}U_{k\ell}(1+\eps)$):}
      Since $P_{\#}(R) < \frac{1}{2}\cdot U_{k\ell}(1+\eps)$, we have that
      \begin{align*}
        \frac{U_{k\ell}(1+\eps)-P_{\#}(R)}{P_{\#}(R)} \geq 1.
        \yesnum
        \label{eq:proofa:2}
      \end{align*}
      We have that
      \begin{align*}
        \Pr\insquare{  Z_{\#}(R) > U_{k\ell}(1+\eps) }
        &= \Pr\insquare{  Z_{\#}(R) > P_{\#}(R)\cdot \inparen{ 1 + \frac{U_{k\ell}(1+\eps)-P_{\#}(R)}{P_{\#}(R)}  } }\\
        &\leq \exp\inparen{-\frac{\alpha}{20} \cdot P_{\#}(R)\cdot \inparen{2\cdot\frac{U_{k\ell}(1+\eps)-P_{\#}(R)}{P_{\#}(R)} - 1 } }
        \tag{{Using third statement in \cref{thm:rounding} and that \cref{eq:proofa:2}}}\\
        &= \exp\inparen{-\frac{\alpha}{20} \cdot \inparen{2U_{k\ell}(1+\eps)-3P_{\#}(R) } } \\
        &\leq \exp\inparen{-\frac{\alpha}{40} \cdot U_{k\ell}(1+\eps) }.
        \tagnum{Using that $P_{\#}(R) < \frac{1}{2}\cdot U_{k\ell}(1+\eps)$}
        \customlabel{eq:proofa:caseb}{\theequation}
      \end{align*}

      \noindent Combining Equations~\eqref{eq:proofa:casea} and \eqref{eq:proofa:caseb} we get that
      \begin{align*}
        \Pr\insquare{  Z_{\#}(R) > U(1+\eps) }
        \leq
        \max\inbrace{
        \exp\inparen{-\frac{\alpha}{20}  \cdot U_{k\ell} {   \frac{\inparen{\eps-\phi \gamma_k}^2 }{1+\phi \gamma_k} } },
        \exp\inparen{-\frac{\alpha}{40} \cdot U_{k\ell}(1+\eps) }
        }.
        \yesnum\label{eq:tmptmp}
      \end{align*}
      Let
      \begin{align*}
        \eps\coloneqq \frac{40}{\alpha} \cdot \gamma_k.
        \yesnum\label{eq:proofa:def_alpha}
      \end{align*}
      We claim that for this value of $\eps$, it holds that
      \begin{align*}
        \Pr\insquare{  Z_{\#}(R) > U_{k\ell}(1+\eps) } \leq \frac{\delta}{2n}.
        \yesnum\label{eq:proofa:toprove}
      \end{align*}
      Now by taking a union bound over bound over all $\ell\in [n]$ and using that $\alpha \coloneqq \frac{1}{d}$, it follows that $R$ satisfies the fairness guarantee with probability at least $\frac{\delta}{2n}$.

      \noindent We can upper bound the second term in \cref{eq:tmptmp}, as follows
      \begin{align*}
        \exp\inparen{-\frac{\alpha}{40} \cdot U_{k\ell}(1+\eps) }
        &\leq \exp\inparen{-\frac{\alpha}{40} \cdot U_{k\ell}\cdot \eps }\\
        &\leq \exp\inparen{-U_{k\ell}\cdot \gamma_k }\\
        &\leq \frac{\delta}{np}. \tag{Using that $\gamma_k\geq \frac{1}{U_{k\ell}} \cdot \log\frac{2np}{\delta}$; which follows from \cref{eq:def_gamma}, $U_{k\ell}\geq 1$, and $\log\frac{2np}{\delta}\geq 1$}
      \end{align*}
      To upper bound the first term in \cref{eq:tmptmp}, we use \cref{fact:quadratic_rootsb}.
      \begin{fact}\label{fact:quadratic_rootsb}
        For all $x,y\geq 0$, if $x\geq y+\sqrt{y}$, then $\frac{x^2}{1+x}\geq y$.
      \end{fact}
      \begin{proof}
        Since $1+x>0$,  $\frac{x^2}{1+x}\geq y$ holds if and only if $x^2-xy-y\geq 0$.
        The roots of the quadratic $f(x)\coloneqq x^2-xy-y$ are
        \begin{align*}
          \frac{y}{2}-\sqrt{\frac{y^2}{4}+y}
          \quad\text{and}\quad
          \frac{y}{2}+\sqrt{\frac{y^2}{4}+y}.
        \end{align*}
        If $x$ is larger than both roots, then $f(x)\geq 0$ and, hence, $\frac{x^2}{1+x}\geq y$.
        It follows that $x\geq \frac{y}{2}+\sqrt{\frac{y^2}{4}+y}$ suffices.
        Then using that for all $a,b\geq 0$, $\sqrt{a}+\sqrt{b}\geq \sqrt{a+b}$, we get that
        \begin{align*}
          y+\sqrt{y} \geq \frac{y}{2}+\sqrt{\frac{y^2}{4}+y}.
        \end{align*}
        Thus, it suffices $x\geq y+\sqrt{y}$ implies that $\frac{x^2}{1+x}\geq y$.
      \end{proof}
      We have
      \begin{align*}
        \frac{\inparen{\eps-\phi \gamma_k}^2 }{1+\phi \gamma_k}
        &\geq
        \inparen{\frac{39}{\alpha}}^2\cdot \frac{{\gamma_k}^2 }{1+\phi \gamma_k}
        \tag{Using that $0\leq \phi\leq 1$, $\alpha\leq \frac{1}{2}$, and \cref{eq:proofa:def_alpha}}\\
        &\geq
        \inparen{\frac{39}{\alpha}}^2\cdot \frac{{\gamma_k}^2 }{1+\gamma_k}.
        \tag{Using that $0<\phi\leq 1$}
      \end{align*}
      To proof \cref{eq:proofa:toprove}, it suffices to prove that
      \begin{align*}
        \frac{{\gamma_k}^2 }{1+\gamma_k} \geq \frac{1}{U_{k\ell}}\cdot \log\inparen{\frac{n+2}{\delta}}.
        \yesnum\label{eq:proofa:to_prove2}
      \end{align*}
      Further, \cref{fact:quadratic_rootsb} implies that to prove \cref{eq:proofa:to_prove2} it suffices to prove that
      \begin{align*}
        \gamma_k \geq
        y+\sqrt{y},
      \end{align*}
      where $y\coloneqq \frac{1}{U_{k\ell}}\cdot \log\frac{n+2}{\delta}$.
      To prove this, observe that
      \begin{align*}
        \log\frac{np}{\delta}\cdot \frac{1}{U_{k\ell}}
        \ \ \leq \ \
        \log\frac{np}{\delta}\cdot \sqrt{\frac{1}{U_{k\ell}}},
        \tag{Using that $U_{k\ell}\geq 1$}\\
        \sqrt{\log\frac{np}{\delta}\cdot \frac{1}{U_{k\ell}}}
        \ \ \leq \ \
        \log\frac{np}{\delta}\cdot \sqrt{\frac{1}{U_{k\ell}}}.
        \tag{Using that $\log\frac{np}{\delta}\geq \frac{1}{2}$ as $n\geq 1$ and $\delta\leq \frac12$}
      \end{align*}
      Hence, \cref{eq:proofa:to_prove2} follows from \cref{eq:def_gamma}.
    \end{proof}

\section{Proofs of Additional Theoretical Results}\label{sec:extra_proofs}

      \subsection{Proof of \cref{lem:eps_depends_on_k}}

      \begin{proof}[Proof of \cref{lem:eps_depends_on_k}]
        Suppose $R$ is deterministic.
        Suppose it places items $i,j\in [m]$ on the first and second position respectively.
        With probability $p_{i}\cdot p_{j}=\frac14$, both $i$ and $j$ belong to $G_1$, and with probability $p_{i}\cdot p_{j}=\frac14$ both $i$ and $j$ belong to $G_2$.
        Thus, at least one of these events occurs with probability $\frac12$.
        If either of these events hold, then $R$ violates the equal representation constraint on the top-2 positions by a multiplicative factor of $2$.
        The last two statements imply that $R$ violates $(\rho,\delta)$-equal representation for any $\rho<1$ and $\delta<\frac12$.

        If $R$ is a random variable, then any draw $R'$ of $R$ is a deterministic ranking, and hence, by the above argument $R'$ violates the equal representation constraint on the top-2 positions by a multiplicative factor of $2$ with a probability $\frac12$ (over the randomness in $G_1$ and $G_2$).
        Since this holds for all draws of $R$ and $R$ is independent of $G_1$ and $G_2$, it follows that $R$ violates the equal representation constraint on the top-2 positions by a multiplicative factor of $2$ with a probability $\frac12$ (over the randomness in $G_1$ and $G_2$, and $R$).
        Thus, $R$ does not satisfy $(\rho,\delta)$-equal representation for any $\rho<1$ and $\delta<\frac12$.
      \end{proof}

    \subsection{$\np$-Hardness Result}\label{sec:proofof:thm:hardness_results_main}

    \begin{theorem}\label{thm:hardness_results_main}
      Given constants $c>1$ and vector $\gamma\in \R^{n}_{\geq 0}$, , and matrices $P\in  {[0,1]}^{m\times p}$, $W\in {\R}_{\geq 0}^{m\times n}$, $U\in  {\R}^{n\times p}$,
      it is \np{}-hard to decide if \prog{prog:noisy_fair_boxed} is feasible.
    \end{theorem}
    \noindent \cref{thm:hardness_results_main} follows from Theorem 5.2 of \cite{MehrotraC21}, which proves that checking the feasibility of the following program is \np-hard.\footnote{Theorem 5.2 of \cite{MehrotraC21} states an \np-hardness result holds for a generalization of \prog{mc_prog}. However, in their proof they only consider the special case of \prog{mc_prog}. Thus, their proof also implies \np-hardness of \prog{mc_prog}.}
    \begin{mdframed}[style=FrameBox]
      \begin{align*}
        \max_{x\in \zo^m}\ \  &\sum\nolimits_{i=1}^{m}w_{i}^{\circ} x_{i},\yesnum\label{mc_prog}\\
        \st,\quad\  &
        \forall \ \ell\in [p^\circ],\quad
        \sum\nolimits_{i=1}^{m^{\circ}}  q_{i\ell}^{\circ} x_{i} \leq U_{\ell}^{\circ},\
        \yesnum\label{eq:denoised_fair:fairness_constraint}\\
        & \sum\nolimits_{i=1}^{m^{\circ}}  x_i = n^{\circ}.\yesnum\label{eq:denoised_fair:cardinality_constraint}
      \end{align*}
    \end{mdframed}
    Where we used a superscript ``${}^{\circ}$'' on the variables of  \cite{MehrotraC21}, to differentiate between ours and \cite{MehrotraC21}'s variables.
    \cref{thm:hardness_results_main} follows from Theorem 5.2 of \cite{MehrotraC21} by observing that \prog{mc_prog} is a special case of \prog{prog:noisy_fair_boxed}, when:
    \begin{align*}
      &\text{$n = n^{\circ}$, $m = m^{\circ}$, $p = p^\circ$, $\gamma = 1_n$, $P=q^{\circ}$,}\\
      \forall k\in [n],\quad& \gamma_k = 1,\\
      \quad& U_{n\ell} = U^{\circ}_{\ell},\\
      \forall k\in [n]\setminus \inbrace{1},\quad& U_{k\ell} = n,\\
      \forall i\in [m], j\in [n],\quad& W_{ij} = w_{i}^{\circ}.
    \end{align*}
    Finally, we can choose any $c>1$.

  \subsection{Proof of \cref{lem:const_are_opt}}
  Given a non-empty subset $\cC\subseteq \cR$ denoting a constraint, let $R_{\cC}$ be the ranking with the highest utility in $\cC$, i.e.,
  $$R_{\cC}\coloneqq\argmax\nolimits_{R\in \cC} \inangle{R,W}.$$
  In other words, $R_{\cC}$ is the utility maximizing ranking subject to satisfying the ``constraint'' $\cC$.

  \begin{proposition}\label{lem:const_are_opt}

    Let $\cC^\star$ be the set of all rankings that satisfy $(\eps,\delta)$-constraint.
    For any subset $\cC\subseteq \cR$, such that $\cC\neq \cC^\star$, at least one of the following holds:

    \begin{itemize}
      \item there exists a matrix $W\in \R_{\geq 0}^{m\times n}$ such that, $R_{\cC}$ does not satisfy $(\eps,\delta)$-equal representation,
      \item there exists a matrix $W\in \R_{\geq 0}^{m\times n}$ such that, $\inangle{R_{\cC},W}\leq \inangle{R_{\cC^\star},W}\cdot \inparen{1-\frac1n}$.
    \end{itemize}
  \end{proposition}

  \noindent We will use the following lemma in the proof of \cref{lem:const_are_opt}.
  \begin{lemma}\label{lem:choose_util}
    For all rankings $R\in \cR$, there exists a matrix $W\in \R_{\geq 0}^{m\times n}$ such that for all other rankings $R'\in \cR$, $R\neq R'$, it holds that $\inangle{R',W}\leq \inangle{R,W}\cdot \inparen{1-\frac{1}{n}}.$
  \end{lemma}
  \begin{proof}
    Suppose $R$ ranks items $i_1,i_2,\dots,i_n$, in that order, in the first $n$ positions.
    Pick $W\in [0,1]^{n\times m}$ such that $W_{ij}=1$ if $i=i_j$ and 0 otherwise.
    $R$ has utility
    $$\inangle{W, R} = \sum_{j=1}^n \inparen{W}_{i_j j} = n.$$
    We claim that $\inangle{W, R'} \leq n-1$.
    If this is true, then the lemma follows.

    Since $R\neq R'$, there exists a position $k\in [n]$ such that $\inparen{x_\cC}_{i_k k}=0$.
    We can upper bound $\inangle{W, R'}$ as follows:
    \begin{align*}
      \inangle{W, R'}
      &= \sum_{j=1}^n\sum_{i=1}^m \mathds{I}[i=i_j] \inparen{R'}_{ij}\tag{By the choice of $W$}\\
      &= \sum_{j=1}^n \inparen{R'}_{i_j j}\\
      &= \sum_{j=1}^{k-1} \inparen{R'}_{i_j j} + 0 + \sum_{j=k+1}^{n} \inparen{R'}_{i_j j}
      \tag{Using that $\inparen{R'}_{i_k k}=0$}\\
      &\leq n-1.   \tag{Using that for all $i\in [m]$ and $j\in [n]$, $\inparen{W}_{ij}\leq 1$}
    \end{align*}

  \end{proof}

  \begin{proof}[Proof of \cref{lem:const_are_opt}]
    Since $\cC\neq \cC^\star$, at least one of the sets $\cC\setminus \cC^\star$ or $\cC^\star\setminus \cC$ is nonempty.
    We divide the proof into two cases.

    \paragraph{Case A ($\abs{\cC\setminus \cC^\star}\neq 0$):}
    In this case, there exists a ranking $R\in \cC$ such that $R\not\in \cC^\star$.
    Since $\cC^\star$ is the set of all rankings that satisfy $(\eps,\delta)$-constraint, it follows that $R$ does not satisfy $(\eps,\delta)$-constraint.
    Further, from \cref{lem:choose_util} it follows that there exists a matrix $W$ such that
    $R\coloneqq \argmax_{R'\in \cR}\inangle{R',W}$.
    Since $\cC\subseteq\cR$, it follows that $R_{\cC}=R$.
    Therefore, for this $W$, $R_{\cC}$ does not satisfy $(\eps,\delta)$-constraint.

    \paragraph{Case B ($\abs{\cC^\star\setminus \cC}\neq 0$):}
    In this case, there exists a ranking $R\in \cC^\star$ such that $R\not\in \cC$.
    From \cref{lem:choose_util} it follows that there exists a matrix $W$ such that, for rankings $R'$ different from $R$ (i.e., $R\neq R'$),
    $$\inangle{R',W}\leq \inangle{R,W}\cdot \inparen{1-\frac{1}{n}}.$$
    Thus, for this $W$, it follows that
    \begin{align*}
      \inangle{R_{\cC^{\star}},W}\cdot \inparen{1-\frac{1}{n}}
      \geq \inangle{R,W}\cdot \inparen{1-\frac{1}{n}}
      \geq \inangle{R',W}.
    \end{align*}
    In particular, for $R'=R_\cC$, we get $\inangle{R_{\cC^{\star}},W}\cdot \inparen{1-\frac{1}{n}} \geq \inangle{R',W}$.

  \end{proof}

  \subsection{Proof of \cref{lem:exp_const_only_suff}}

  Suppose there are two groups $G_1$ and $G_2$.
  Let $R_E$ be the optimal solution to \cref{eq:exp_const} and let $R^\star$ be the ranking with the highest utility subject to satisfying $(\gamma,\delta)$-equal representation constraints for the following $\gamma$:
  \begin{align*}
    \text{$\forall k\in [n]$},\quad
    \gamma_k\coloneqq \frac{1}{k} +  2\sqrt{\frac{6}{k}\cdot \log\inparen{\frac{2n}{\delta}}}.
    \yesnum
    \label{def:gamma:exx}
    \end{align*}
    \begin{lemma}\label{lem:exp_const_only_suff}
      There exists a matrices $P\in [0,1]^{m\times 2}$ and $W\in [0,1]^{m\times 2}$
      such that
      \begin{itemize}
        \item $R_{\rm E}$ satisfies $(\gamma,\delta)$-equal representation and has utility 0,
        \item $R^\star$ has utility 1.
      \end{itemize}
    \end{lemma}

    \begin{proof}
      Let $P$ be the matrix with $P_{i1}=P_{i2}=\frac12$ for all $i\in \inbrace{1,2,\dots,m-1}$ and $P_{m1}=1$ and $P_{m1}=0$.
      Let $W$ be the matrix whose first $m-1$ rows are 0, and the last row has is all 1s.
      Hence, only the last item, say $i_m$, has a nonzero contribution to the utility:
      If a ranking $R$ ranks $i_m$ in the first $n$ positions, then the utility of $R$ is 1, otherwise the utility of $R$ is 0.

      Our first claim will follow because the choice of $P$ ensures that any ranking which ranks $i_m$ in the first $n$ positions cannot satisfy \cref{eq:exp_const}.
      To see this, suppose $R$ ranks $i_m$ at the $k$-th position, then
      \begin{align*}
        \Ex\insquare{\sum\nolimits_{i\in G_1}\sum\nolimits_{j=1}^k R_{ij}}
        &= \sum\nolimits_{i\in [m]}\sum\nolimits_{j=1}^k P_{i1} R_{ij}\\
        &= 1 + \sum\nolimits_{i\in [m]\setminus\inbrace{i_m}}\sum\nolimits_{j=1}^{k-1}
        P_{i1} R_{ij}
        \tag{Using that $P_{i_m,1}=1$}\\
        &= \frac{k+1}{2}
        \tag{Using that $P_{i,1}=\frac12$ for all $i\neq i_m$}\\
        &> \frac{k+1}{2}.
      \end{align*}
      Hence, $R$ cannot satisfy \cref{eq:exp_const}.

      To prove our second claim, we will construct a ranking which has utility 1 and satisfies $(\gamma,\delta)$-equal representation .
      It suffices to choose any ranking $R$ which places $i_m$ in the first $n$ position satisfies constraint.
      By our earlier argument this ranking has a utility 1.
      Let $Z_j$ be the indicator random variable that the item in the $j$-th position in $R$ belongs to $G_1$.
      This implies that $\sum\nolimits_{i\in G_1}\sum\nolimits_{j=1}^k R_{ij}=  \sum\nolimits_{j=1}^k Z_j$ for all $k$.
      Further, by the choice of $P$, we have
      \begin{align*}
        \frac{k}{2}\leq \Ex\insquare{\sum\nolimits_{j=1}^k Z_j} \leq \frac{k+1}{2}.
        \yesnum\label{eq:2jjf}
      \end{align*}
      \noindent Further, by \cref{def:noise_model}, we have that $Z_j$ is independent of $Z_k$ for any $j\neq k$.
      Let $$\eps_k\coloneqq \sqrt{\frac{6}{k}\cdot \log\inparen{\frac{2n}{\delta}}}.$$
      Using the above, we have
      \begin{align*}
        \Pr\insquare{\sum\nolimits_{i\in G_1}\sum\nolimits_{j=1}^k R_{ij}\geq \frac{k+1}{2}\cdot (1+\eps_k)}
        &= \Pr\insquare{\sum\nolimits_{j=1}^k Z_{j}\geq \frac{k+1}{2}\cdot (1+\eps_k)}\\
        &\leq \Pr\insquare{\sum\nolimits_{j=1}^k Z_{j}\geq \Ex\insquare{\sum\nolimits_{j=1}^k Z_j}\cdot (1+\eps_k)}
        \tag{Using \cref{eq:2jjf}}\\
        &\leq \exp\inparen{-\frac{\eps_k^2}{3}\cdot \Ex\insquare{\sum\nolimits_{j=1}^k Z_j}}\tag{Using the Chernoff's bound, see \cite{motwani1995randomized}}\\
        &\leq \exp\inparen{-\frac{\eps_k^2k}{6}}
        \tag{Using \cref{eq:2jjf}}\\
        &\leq \frac{\delta}{2n}\tag{Using that $\eps_k\coloneqq \sqrt{\frac{6}{k}\cdot \log\inparen{\frac{2n}{\delta}}}$}.
      \end{align*}
      Further, as $\gamma_k\geq \frac{k+1}{k}\cdot (1+\eps_k)$, we get
      \begin{align*}
        \Pr\insquare{\sum\nolimits_{i\in G_1}\sum\nolimits_{j=1}^k R_{ij}\geq \frac{k}{2}\cdot (1+\gamma_k)}
        &\leq \Pr\insquare{\sum\nolimits_{i\in G_1}\sum\nolimits_{j=1}^k R_{ij}\geq \frac{k+1}{2}\cdot (1+\eps_k)}
        \\
        &\leq \frac{\delta}{2n}.
        \end{align*}
        Further, considering $1-Z_j$ and repeating a similar argument for $G_2$, we get
        \begin{align*}
          \Pr\insquare{\sum\nolimits_{i\in G_2}\sum\nolimits_{j=1}^k R_{ij}\geq \frac{k}{2}\cdot (1+\eps_k)}
          &= \Pr\insquare{\sum\nolimits_{j=1}^k (1-Z_{j})\geq \frac{k}{2}\cdot (1+\gamma_k)}\\
          &\leq \Pr\insquare{\sum\nolimits_{j=1}^k (1-Z_{j})\geq \Ex\insquare{\sum\nolimits_{j=1}^k (1-Z_j)}\cdot (1+\gamma_k)}
          \tag{Using \cref{eq:2jjf}}\\
          &\leq \exp\inparen{-\frac{\gamma_k^2}{3}\cdot \Ex\insquare{\sum\nolimits_{j=1}^k (1-Z_j)}}\tag{Using the Chernoff's bound, see \cite{motwani1995randomized}}\\
          &\leq \exp\inparen{-\frac{\gamma_k^2 (k-1)}{6}}
          \tag{Using \cref{eq:2jjf}}\\
          &\leq \frac{\delta}{2n}. \tag{Using \cref{def:gamma:exx}}
        \end{align*}
        By taking the union bound over all $k$, one can show that $R$ satisfies $(\gamma,\delta)$-equal representation.
      \end{proof}

      \noindent

      \subsection{Proof of \cref{lem:relaxation_nonconvex}}

      \begin{proposition}\label{lem:relaxation_nonconvex}
        There exist $p\in [0,1]^m$ such that \eqref{eq:exact_prob} is non-convex in $R$.
      \end{proposition}

      \begin{proof}
        It suffices to specify $n$, $m$, $p$, $\eps$, $\delta$, and two rankings $R_1$ and $R_2$ such that both $R_1$ and $R_2$ satisfy $(\eps,\delta)$-equal representation, but $\frac{R_1+R_2}{2}$ does not satisfy $(\eps,\delta)$-equal representation.

        Define $n\coloneqq 2$, $m\coloneqq 4$, and $\eps\coloneqq \insquare{\begin{smallmatrix}\frac13 & \frac13\end{smallmatrix}}^\top$.
        Fix any  $0<\delta<\frac12$.
        Define
        $$p\coloneqq \begin{bmatrix}1 & 0 & \delta & 1-\delta\end{bmatrix}^\top.$$
        Let $R_1$ be the ranking that places items 1 and 3 in the first and second position, and $R_2$ be the ranking that places items 2 and 4 in the first and second position, i.e.,
        \begin{align*}
          R_1\coloneqq
          \begin{bmatrix}
            1 & 0 & 0 & 0\\
            0 & 0 & 1 & 0
          \end{bmatrix}
          \quad\text{and}\quad
          R_2\coloneqq
          \begin{bmatrix}
            0 & 1 & 0 & 0\\
            0 & 0 & 0 & 1
            \end{bmatrix}.
          \end{align*}
          If $1\in G_1$ and $3\in G_2$, then $R_1$ places an equal number of items from $G_1$ and $G_2$ in the first two positions, and hence, satisfies equal representation.
          This event, happens with probability $p_1(1-p_3)=1-\delta$.
          Thus, $R_1$ satisfies $(0,\delta)$-equal representation, and hence, $(\eps,\delta)$-equal representation.
          Replace item 1 and 3 with 2 and 4 and swap $G_1$ and $G_2$ in the above argument, to get that $R_2$ also satisfies $(\eps,\delta)$-equal representation.

          However, we claim that $\frac{R_1+R_2}{2}$ does not satisfy $(\eps,\delta)$-equal representation.
          Note that with probability 1, $1\in G_1$ and $2\in G_2$.
          If $3,4\in G_1$ or $3,4\in G_2$, then $\frac{R_1+R_2}{2}$ violates the equal representation constraint on the top-2 positions by a multiplicative factor of $\frac32$.
          At least one of these events happens with probability $$p_3p_4+(1-p_3)(1-p_4)=2\delta(1-\delta)>\delta,$$ as $\delta<\frac12$.
          Thus, $\frac{R_1+R_2}{2}$ does not satisfy $(\eps,\delta)$-equal representation for the specified $\eps\coloneqq \insquare{\begin{smallmatrix}\frac13 & \frac13\end{smallmatrix}}^\top$ and $\delta<\frac12$.
        \end{proof}

        \subsection{Proof of \cref{thm:np_hardness_of_exact_const}}
        In this section, we prove the following theorem.
        \begin{theorem}\label{thm:np_hardness_of_exact_const}
          Given $p\in [0,1]^m$, $\delta\in (0,1]$, $W\in \R_{\geq 0}^{m\times n}$, $\eps\in [0,1]^n$, and $V\geq 0$
          it is \np{}-hard to decide if the value of \prog{eq:exact_prob} is at least $V$.
        \end{theorem}

        \noindent Recall that constraint~\eqref{eq:necc_and_suff_app} is necessary and sufficient to satisfy $(\eps,\delta)$-equal representation, and hence, the value of \eqref{prob:noisy_fair_ranking_2} is the maximum utility of a ranking subject to satisfying $(\eps,\delta)$-equal representation.
        \begin{align*}
          \max_{R\in \cR}& ~~ \inangle{R, W}\yesnum\customlabel{prob:noisy_fair_ranking_2}{\theequation}\\
          \st & ~\text{w.p. at least $1-\delta$ over draw of $G_1,G_2$,}\yesnum\label{eq:necc_and_suff_app}\\
          &\hspace{-4mm}
          \forall k\in [n],~\forall \ell\in [2],\quad \sum\nolimits_{i\in G_\ell}\sum_{j=1}^k R_{ij}\leq {\frac{k}2}\cdot (1+\eps_k).
        \end{align*}
        We will show that the decision version of  \eqref{prob:noisy_fair_ranking_2} is \np-hard:
        \begin{theorem}\label{thm:np_hard_noisy_fair}
          Given $L\geq 0$, $\delta\in [0,1]$, $\eps\in [0,1]^n$, $P\in [0,1]^{m\times p}$, and $W\in \R_{\geq 0}^{m\times n}$
          it is \np{}-hard to decide if the value of \eqref{prob:noisy_fair_ranking_2} is at least $L$.
        \end{theorem}
        \noindent The proof of \cref{thm:np_hard_noisy_fair} proceeds in two steps.
        In the first step, we reduce \eqref{prob:intermed} to \eqref{prob:noisy_fair_ranking_2}.
        In the second step, we prove that \eqref{prob:intermed} is \np-hard because the \np-complete product partition problem reduces to \eqref{prob:intermed}.
        Together, the two steps imply the hardness of \eqref{prob:noisy_fair_ranking_2}.
        The proof of the second step is inspired by the construction of \cite{product_knapsack} for the product knapsack problem, which is similar to \eqref{prob:intermed}.

        \paragraph{\textit{Step 1: Reduction from \eqref{prob:intermed} to \eqref{prob:noisy_fair_ranking_2}.}}
        \mbox{In this step, we will reduce the following problem to  \eqref{prob:noisy_fair_ranking_2}.}
        \begin{tcolorbox}[colback=white]
          {\em Input:} $L\geq 0$, $n\in [m]$, $\delta\in [0,1]$, $U\in \insquare{0,\frac{n}{2}}$ $v\in \R^m_{\geq 0}$, and $P\in [0,1]^{m\times p}$

          {\em Decision problem:}
          Is the value of \eqref{prob:intermed} at least $L$?
          \begin{align*}
            \max_{S\subseteq [m]\colon \abs{S}=n}& ~~~~\sum_{i\in S} v_i\yesnum\customlabel{prob:intermed}{\theequation}\\
            \st & ~~~\text{w.p. at least $1-\delta$ over draw of $G_1,G_2$,}\\
            &\hspace{-4mm}
            \quad ~~  \abs{S\cap G_1}\leq U+\frac{n}2
            \quad\text{and}\quad
            \abs{S\cap G_2}\leq U+\frac{n}2.
          \end{align*}
        \end{tcolorbox}
        \noindent {\bf Reduction.}
        Given an instance of \eqref{prob:intermed} we construct the following instance of \eqref{prob:noisy_fair_ranking_2}:
        \begin{align*}
          W &\coloneqq v1_n^\top,\yesnum\label{eq:cond_1}\\
          \eps_1=\eps_2=\dots=\eps_{n-1} &\coloneqq \frac{2n}{k} - 1,\yesnum\label{eq:cond_2}\\
          \eps_{n} &\coloneqq \frac{2U}{n} - 1,\yesnum\label{eq:cond_3}
        \end{align*}
        where $1_n\coloneqq (1,\dots,1)\in \R^n$.\footnote{To be precise, we consider $\eps_1=\eps_2=\dots=\eps_{n-1} \coloneqq \min\inbrace{1,\frac{2n}{k} - 1}$ and $\eps_{n} \coloneqq \min\inbrace{1, \frac{2U}{n} - 1}$.}
        The parameters $L$, $\delta$, and $P$ are the same as the instance of \eqref{prob:intermed}.

        The reduction from \eqref{prob:intermed} to \eqref{prob:noisy_fair_ranking_2} is as follows:
        First solve \eqref{prob:noisy_fair_ranking_2} to obtain a ranking $R$.
        Let $S$ be the set of items $R$ places in the top-$n$ positions.
        Output $S$.
        Clearly, this is a polynomial-time reduction.
        It remains to prove that it is sound and complete.

        In our construction, Condition \eqref{eq:cond_1} implies that the utility of a ranking only depends on the set of $n$ items it places in the top-$n$ positions, and hence, any two rankings that place the same set of items in the top-$n$ positions have the same utility.
        Condition \eqref{eq:cond_2} ensures that any ranking satisfies the constraints in the first $n-1$ positions with probability 1.
        This is because, for all $k\in [n-1]$, $\frac{k}2(1+\eps_k) = n > k$.
        Thus, a ranking $R$ is feasible for \eqref{prob:noisy_fair_ranking_2} if and only if it satisfies:
        With probability at least $1-\delta$ over draw of $G_1,G_2$,
        \begin{align*}
          \forall \ell\in [2],\quad \sum_{i\in G_\ell}\sum_{j=1}^k R_{ij}\leq \frac{n}2 \cdot (1+\eps_n)
          = U+\frac{n}2.
        \end{align*}

        \noindent {\bf Soundness and completeness.}
        Fix any $R\in \cR$.
        Let $S$ be the set of items $R$ places in the top-$n$ positions.
        It holds that
        \begin{align*}
          \inangle{R,W}
          \ \ &\Stackrel{\eqref{eq:cond_1}}{=}\ \  \sum_{i\in S} v_i.
        \end{align*}
        It remains to show that $R$ is feasible for \eqref{prob:noisy_fair_ranking_2} if and only if $S$ is feasible for \eqref{prob:intermed}.
        Due to conditions \eqref{eq:cond_2} and \eqref{eq:cond_3}, $R$ is feasible for \eqref{prob:noisy_fair_ranking_2} if and only if:
        With probability at least $1-\delta$ over draw of $G_1,G_2$,
        \begin{align*}
          \forall \ell\in [2],\quad \sum_{i\in G_\ell}\sum_{j=1}^k R_{ij}\leq U+\frac{n}2.
        \end{align*}
        Since by the definition of $S$, for all $T\subseteq[m]$, $\sum_{i\in T}\sum_{j=1}^n R_{ij}=\abs{S\cap T}$,
        it follows that with probability 1
        $$\sum_{i\in G_\ell }\sum_{j=1}^n R_{ij}=\abs{S\cap G_\ell}.$$ %
        Hence, $S$ is feasible for \eqref{prob:intermed} if and only if $R$ is feasible for \eqref{prob:noisy_fair_ranking_2}.
        Thus, the reduction is sound and complete.

        \newcommand{\ppp}{\textsc{PPP}}
        \newcommand{\cppp}{\textsc{CPPP}}

        \paragraph{\textit{Step 2: Reduction from Product Partition Problem To \eqref{prob:intermed}.}}
        We consider the following version of the product partition problem:
        \medskip
        \begin{tcolorbox}[colback=white,title={\em Cardinality constrained product partition problem (\cppp{})}]
          {\em Input:} $a_1,a_2,\dots,a_q\in \Z_{\geq 0}$ and $\ell\in \inbrace{0,1,\dots,q}$.

          {\em Decision problem:}
          Is there a set $S\subseteq [q]$ of size $\ell$ such that
          \begin{align*}
            \prod_{i\in S} a_i = \prod_{i\in[q] \setminus S} a_i?
          \end{align*}
        \end{tcolorbox}
        \medskip
        \noindent The usual product partition problem (\ppp{}) does not require $S$ to have size $\ell$ and is known to be \np-complete.
        \cppp{} is clearly in \np.
        To see that \cppp{} is \np-complete, one can reduce \ppp{} to \cppp{}:
        To see this, given an instance of \ppp{}, construct $q+1$ instances of \cppp{}, one for each value of $\ell\in \inbrace{0,1,\dots,q}$.
        Then, \ppp{} is a `Yes' instance if and only if at least one of the $q+1$ \cppp{} instances in a `Yes' instance.
        Thus, it follows that \cppp{} is also \np-complete.

        \paragraph{Assumptions on \cppp{} Instances without Loss of Generality.}
        The decision problem for \cppp{} is simple for instances with $\ell=0$, or with one or more of $a_1,\dots,a_q$ as 0.
        As all inputs are integral, without loss of generality, we assume that $\ell\geq 1$ and $a_1,\dots,a_q\geq 1$.
        Note that if in an \cppp{} $\sqrt{\prod_{i=1}^q a_i}$ is non-integral, then it is a `No' instance.
        This can be verified in polynomial time, and hence, without loss of generality, we assume that $\sqrt{\prod_{i=1}^q a_i}$ is integral.

        \paragraph{Reduction from \cppp{} to \eqref{prob:intermed}.}
        Given an instance of \cppp{},
        we construct an instance of \eqref{prob:intermed} with
        \begin{align*}
          n \coloneqq 2\ell,\quad
          m \coloneqq q+\ell,\quad
          U \coloneqq \ell-1,\quad \text{and}\quad
          \delta &\coloneqq \inparen{\frac{1}{a_{\max}}}^{\ell^2},
          \yesnum\label{eq:def_delta}
        \end{align*}
        where $a_{\max}\coloneqq \max_{i\in [q]}a_i$.
        Further, define constants
        \begin{align*}
          M\coloneqq (\ell+2)\cdot \sqrt{\prod_{i=1}^q a_i}
          \quad \text{and}\quad
          B\coloneqq q\ceil{M\log(a_{\max})}+1.
          \yesnum\label{eq:def_of_M}
        \end{align*}
        We choose $v$ so that the first $q$ items correspond to the $q$ numbers in the \cppp{} instance, and the next $\ell$ items have a  ``high'' value:
        \begin{align*}
          \forall i\in [q],\quad v_{i} &\coloneqq \ceil{M\log(a_{i})},\yesnum\label{eq:val_v_2}\\
          \forall i\in [\ell],\quad v_{i+q} &\coloneqq L.
          \yesnum\label{eq:val_v}
        \end{align*}
        Note that each of the last $\ell$ items has a value larger than the total value of the first $q$ items, i.e.,
        \begin{align*}
          \forall\ i\in [\ell],\quad v_{i+q} = B> \sum_{j\in [q]}v_j.
          \yesnum\label{eq:inequality_high_val_items}
        \end{align*}
        We choose $P$ so that for the first $q$ items $P_{i,1}\propto a_i^\ell$ and the next $\ell$ are in $G_1$ with probability 1:
        \begin{align*}
          \forall i\in [q],\quad &
          \quad P_{i,1} \coloneqq \inparen{\frac{a_i}{a_{\max}}}^\ell \cdot \frac{1}{\sqrt{\prod_{i=1}^q a_i}}
          \ \ \text{and} \quad \  \
          P_{i,2} = 1-P_{i,1},\yesnum\label{eq:prob_2}\\
          \forall i\in [\ell],\quad &P_{i+q,1} \coloneqq 1 \qquad\qquad\qquad\qquad\quad \ \
          \text{and}\ \ P_{i+q,2} = 1-P_{i+q,1}.\yesnum\label{eq:prob_1}
        \end{align*}
        Finally, let
        \begin{align*}
          L\coloneqq \ell B + \floor{\frac{M}2 \sum_{i=1}^q \log(a_{i})}.\yesnum\label{eq:def_of_A}
        \end{align*}
        The reduction from \cppp{} to \eqref{prob:intermed} is as follows:
        First solve the constructed instance of \eqref{prob:intermed} to get $S$.
        Then output $S\backslash Q$, where $$Q\coloneqq [\ell+q]\setminus [q]$$
        is the set of the last $\ell$ items.

        Let $C\in \Z$ be the bit complexity of the input for the given instance of \eqref{prob:intermed}.
        To show that the reduction is polynomial time, it suffices to show that $L$ and
        $\ceil{M\log(a_{1})},\dots,\ceil{M\log(a_{q})}$ can be computed in $\poly(C)$ time.
        Note that, $M\leq 2^{O(C)}$, and hence, to compute $\ceil{M\log(a_{i})}$ it suffices to compute $\log(a_{i})$ up to $O(C)$ bits,  which can be done in $\poly(C)$ time.
        Similarly, to compute $L$ it suffices to compute $\sum_{i=1}^q\log(a_{i})$ up to $O(C)$ bits,  which can be done in $\poly(C)$ time.
        Thus, the reduction is polynomial time.

        \medskip

        The choice of $L$ and $v$ ensures that the following fact holds.
        \begin{fact}\label{fact:struc_of_opt_set}
          If a set $S\subseteq[q]$ satisfies $\sum_{i\in S}v_i\geq L$ and $\abs{S}=n$, then $S\supseteq Q$.
        \end{fact}
        \begin{proof}
          Suppose toward a contradiction that satisfies $\sum_{i\in S}v_i\geq L$ and $\abs{S}=n$ but $S$ does not contain $Q$.
          Since $S=n=2\ell$
          Then,
          \begin{align*}
            \sum_{i\in S}v_i
            \quad
            &=\quad \sum_{i\in S\cap Q}v_i  +  \sum_{i\in S\setminus{} Q}v_i  \\
            &\Stackrel{}{\leq}\quad \abs{S\cap Q}\cdot \max_{i\in Q} v_i + \sum_{i\in [q]\setminus{} Q}v_i\tag{Using $S\subseteq [q]$ and $v_i\geq 0$}\\
            &\Stackrel{\eqref{eq:val_v},\ \eqref{eq:inequality_high_val_items}}{<}\quad \abs{S\cap Q}\cdot B + B\\
            &<\quad \abs{Q}\cdot B\tag{Using that $\abs{S\cap Q} \leq \abs{Q}-1$ and $B>0$}\\
            &\Stackrel{}{\leq }\quad L.
            \tag{Using \eqref{eq:def_of_A}, $\abs{Q}=\ell$, and $L\geq \ell B$}
          \end{align*}
        \end{proof}

        \noindent {\bf Soundness.}
        Suppose $S$ is feasible for \eqref{prob:intermed} and satisfies $\sum_{i\in S}v_i\geq L$.
        Due to \eqref{eq:prob_1}, with probability 1, $G_1\supseteq Q$.
        Hence, $G_2\cap Q=\emptyset$.
        Thus, %
        \begin{align*}
          \text{with probability 1},\quad
          \abs{S\cap G_2}
          &= \abs{\inparen{S\setminus Q}\cap G_2}
          \leq \abs{S\setminus Q}.
        \end{align*}
        Since $\sum_{i\in S}v_i\geq L$ and $\abs{S}=n$ (as $S$ is feasible for \eqref{prob:intermed}), \cref{fact:struc_of_opt_set} implies that %
        $S\supseteq  Q$, hence $\abs{S\setminus Q}=\abs{S}-\ell$.
        Combining this with the above equation, we get that %
        \begin{align*}
          \text{with probability 1},\quad
          \abs{S\cap G_2}
          \leq \abs{S}-\ell = \ell. \tag{Using that $\abs{S}=n=2\ell$}
        \end{align*}
        Since $U\geq 0$,
        \begin{align*}
          \text{with probability 1},\quad \abs{S\cap G_2} \leq U+\ell.\yesnum\label{eq:one_const_sat}
        \end{align*}
        $S$ is feasible for \eqref{prob:intermed} if and only if:
        \begin{align*}
          &\Pr_{G_1,G_2}\insquare{ \abs{S\cap G_1} \leq U+\ell\ \ \text{and}\ \  \abs{S\cap G_2} \leq U+\ell   }\geq 1-\delta\\
          &\qquad \Stackrel{\eqref{eq:one_const_sat}}{\iff}
          \Pr_{G_1,G_2}\insquare{ \abs{S\cap G_1} \leq U+\ell}\geq 1-\delta\\
          \quad &\qquad \Stackrel{}{\iff}\quad
          \Pr_{G_1,G_2}\insquare{ \abs{\inparen{S\setminus{}  Q}\cap G_1} \leq U+\ell}\geq 1-\delta \tag{Using that with probability 1, $S,G_1\supseteq Q$}\\
          &\qquad \Stackrel{}{\iff}\quad
          \Pr_{G_1,G_2}\insquare{ \abs{S'\cap G_1} \leq U}\geq 1-\delta\\
          &\qquad \Stackrel{}{\iff}\quad
          \Pr_{G_1,G_2}\insquare{ \abs{S'\cap G_1} > U}\leq \delta\\
          &\qquad \Stackrel{}{\iff}\quad
          \Pr_{G_1,G_2}\insquare{ \abs{S'\cap G_1} = n}\leq \delta
          \tag{Using that $U=n-1$ and $\abs{S'}=\ell$}\\
          &\qquad \Stackrel{}{\iff}\quad
          \prod_{i\in S'}{P_{i1}}\leq \delta\\
          &\qquad \Stackrel{\eqref{eq:prob_1},\eqref{eq:prob_2},\eqref{eq:def_delta}}{\iff}\quad
          a_{\max}^{-\ell\cdot \abs{S'}}\cdot \inparen{\prod_{i\in [q]} a_i}^{-\abs{S'}/2}\cdot \prod_{i\in S'}{a_{i}^\ell}\leq \inparen{\frac{1}{a_{\max}^{\ell}}}^\ell\\
          &\qquad \Stackrel{}{\iff}\quad
          \prod_{i\in S'}{a_{i}}\leq \sqrt{\prod_{i\in [q]} a_i}.
          \tagnum{Using that $\ell>0$, $a_1,\dots,a_q>0$, and $\abs{S'}=\ell$}
          \customlabel{eq:equation_11}{\theequation}
        \end{align*}
        Since $S$ is feasible for \eqref{prob:intermed}, it holds that
        \begin{align*}
          \prod_{i\in S'}{a_{i}}\leq \sqrt{\prod_{i\in [q]} a_i}.
        \end{align*}
        To show that $S'$ is feasible for \cppp{}, it remains to show that the above equation holds with equality.

        Suppose toward a contradiction that $\prod_{i\in S'}{a_{i}} < \sqrt{\prod_{i\in [q]} a_i}.$
        Then, because $\sqrt{\prod_{i\in [q]} a_i}$ and $a_1,\dots,a_q$ are integral $\prod_{i\in S'}{a_{i}} \leq \sqrt{\prod_{i\in [q]} a_i} - 1.$
        Because $M\geq 0$, taking the logarithm we get
        \begin{align*}
          M\sum_{i\in S'}{\log a_{i}} \leq M\log\inparen{\sqrt{\prod_{i\in [q]} a_i} - 1}.\yesnum\label{eq:first_ineq_reduction}
        \end{align*}
        To upper bound the RHS, we will use the following fact:
        \begin{fact}\label{fact:log_ineq}
          For all $x\geq 1$, $\log{x}-\log\inparen{x-1}\geq \frac1x$.
        \end{fact}
        \noindent Using \cref{fact:log_ineq} with $x=\sqrt{\prod_{i\in [q]} a_i}$ (as $a_1,\dots,a_q\geq 1$), we get
        \begin{align*}
          \log\inparen{\sqrt{\prod_{i\in [q]} a_i}} - \log\inparen{\sqrt{\prod_{i\in [q]} a_i} - 1}
          \geq \frac{1}{\sqrt{\prod_{i\in [q]} a_i}}.
        \end{align*}
        Hence, by \eqref{eq:def_of_M}
        \begin{align*}
          M = (\ell+2) \cdot \sqrt{\prod_{i\in [q]} a_i}
          \geq \frac{\ell+2}{\log\inparen{\sqrt{\prod_{i\in [q]} a_i}} - \log\inparen{\sqrt{\prod_{i\in [q]} a_i} - 1}}.
        \end{align*}
        On rearranging and using \eqref{eq:first_ineq_reduction}, we get
        \begin{align*}
          M\sum_{i\in S'}{\log a_{i}}\leq M\log\inparen{\sqrt{\prod_{i\in [q]} a_i} - 1} \leq M\log\inparen{\sqrt{\prod_{i\in [q]} a_i}} -\ell-2.
        \end{align*}
        Since for all $i\in S'$, $v_i\leq M\log\inparen{a_i}+1$, it follows that
        \begin{align*}
          \sum_{i\in S'}v_i
          \leq \frac{M}{2}\log\inparen{{\prod_{i\in [q]} a_i}} -2
          &< \floor{\frac{M}{2}\log\inparen{{\prod_{i\in [q]} a_i}}}.\yesnum\label{eq:contradiction_reduction}
        \end{align*}
        Thus,
        \begin{align*}
          \sum_{i\in S} v_i
          &=  \sum_{i\in S\cap  Q} v_i + \sum_{i\in S\setminus  Q} v_i\\
          &=  \ell B + \sum_{i\in S'} v_i
          \tag{Using that $S\supseteq  Q$ and $S'\coloneqq S\setminus Q$}\\
          &\Stackrel{\eqref{eq:contradiction_reduction}}{<}
          \ell B + \floor{M\log\inparen{\sqrt{\prod_{i\in [q]} a_i}}}\\
          &=L.
        \end{align*}
        This is a contradiction to $\sum_{i\in S} v_i\geq L$.

        \paragraph{Completeness.}
        It suffices to show that if  $S'$ is feasible for the given instance of \cppp{},
        then $S\coloneqq S'\cup  Q$ is feasible for \eqref{prob:intermed} and satisfies $\sum_{i\in S} v_i\geq A$.
        Due to \eqref{eq:prob_1}, with probability 1, $G_1\supseteq  Q$.
        Hence, $G_2\cap  Q=\emptyset$.
        Thus, %
        \begin{align*}
          \text{with probability 1},\quad
          \abs{S\cap G_2}
          &= \abs{\inparen{S\setminus Q}\cap G_2}
          \leq \abs{S\setminus Q}
          = \abs{S'}
          =\ell,
        \end{align*}
        where the last equality holds as $S'$ is feasible for the given instance of \cppp{}.
        This implies that \eqref{eq:one_const_sat} holds.
        Hence, by following the same arguments, \eqref{eq:equation_11} also holds.
        Thus, $S\coloneqq S'\cup  Q$ is feasible for \eqref{prob:intermed}
        It remains to show that $\sum_{i\in S} v_i\geq L$.
        \begin{align*}
          \sum_{i\in S} v_i\ \ \
          &= \ \  \ \sum_{i\in  Q} v_i+\sum_{i\in S'} v_i\tag{Using that $S\coloneqq S'\cup  Q$}\\
          &\Stackrel{\eqref{eq:val_v}, \eqref{eq:val_v_2}}{\geq} \ \ \  \ell B +\sum_{i\in S'} M \log{a_i}\\
          \qquad &=  \ \ \  \ell B +\frac{M}{2} \log\inparen{\prod_{i\in [q]}a_i}
          \tag{Using that $\prod_{i\in S'} a_i = \prod_{i\in[q]}a_i$}\\
          &\Stackrel{\eqref{eq:def_of_A}}{\geq}\ \ \  A.
        \end{align*}

\section{Extension of \cref{thm:ub} to Position-Weighted Constraints}\label{sec:extn_of_ub}

    In this section, we extend \cref{thm:ub} to position-weighted version of fairness constraints.
    In particular, given position-discounts
    $v_1\geq v_2\geq \dots\geq v_n$
    and a matrix $U\in \Z_{+}^{n\times p}$ the position-weighted fairness constraint requires a ranking $R$ to satisfy:
    $$\forall k\in [n],\ell\in [p],\quad {\sum\nolimits_{i\in G_\ell} \sum\nolimits_{j\in [k]}} {v_j R_{ij}} \leq  {U_{k\ell}}$$
    for all $k$ and $\ell$.
    For these constraints, we consider the following analogue of $(\eps,\delta)$-constraints:
    A ranking $R$ is said to satisfy $(\eps,\delta,v)$-constraint if with probability at least $1-\delta$ over the draw of $G_1,\dots,G_p$
    \begin{align*}
      \forall k\in [n]~\forall \ell\in [p],~~ \sum\nolimits_{i\in G_\ell}\sum\nolimits_{j=1}^k v_j R_{ij}\leq U_{k\ell}  (1+\eps_k).
      \yesnum\label{eq:eps_delta_const_gen:gen}
    \end{align*}

    \noindent For these position-dependent constraints, our framework largely remains the same and is stated in \prog{prog:noisy_fair_boxed:gen}.
    Compared to \prog{prog:noisy_fair_boxed}, the main difference is in the left-hand side of \prog{eq:def_gamma:gen}.
    We can prove the guarantees on the fairness and accuracy of the optimal solution of \prog{prog:noisy_fair_boxed:gen}, under the additional assumption that, for a constant $\psi>0$, $U$ satisfies that
    \begin{align*}
      \forall\ell\in[p],\forall k\in[n],\quad
      U_{k\ell} \geq \psi k.
      \yesnum\label{eq:uppckk}
    \end{align*}
    The parameter $\psi$ shows up in \cref{eq:def_gamma:gen}.

    \smallskip

    \begin{tcolorbox}[colback=white,title={Our Fair-Ranking Framework for Position-Dependent Constraints},left=3pt,right=3pt,top=2pt,bottom=0pt]
      {\em Input:} Matrices $P\in  {[0,1]}^{m\times p}$, $W\in {\R}_{\geq 0}^{m\times n}$, $U\in  {\R}^{n\times p}$

      {\em Parameters:} A constant $c>1$, a failure probability $\delta\in (0,1]$, and for each $k\in [n]$, a parameter
      \begin{align*}
        \gamma_k\coloneqq \frac{1}{\psi}\cdot \log\inparen{\frac{2np}\delta}\cdot  \max_{\ell\in [p]} \sqrt{\frac{1}{U_{k\ell }}}.
        \yesnum\label{eq:def_gamma:gen}
      \end{align*}
      \hrule

      \vspace{2mm}

      {\em Program:}

      \vspace{-8mm}

      \begin{align*}
        &\hspace{-5mm}\max\nolimits_{R\in \cR} \inangle{R,W},
        \yesnum\label{prog:noisy_fair_boxed:gen}\\
        \st &\hspace{-2mm} ~~~~ \forall {\ell \in [p]}~~\forall {k\in [n]},\quad \sum\nolimits_{i\in [m], j\in [k]} v_j P_{i\ell} R_{ij} \leq U_{k\ell} \inparen{1+\frac{2\sqrt{c}-1}{2\sqrt{c}}\cdot \gamma_k}.
        \yesnum\label{eq:const_of_noisy_fair_boxed:gen}
      \end{align*}
      \end{tcolorbox}

      \noindent We prove the following guarantees on the fairness and accuracy of the optimal solution of  \prog{prog:noisy_fair_boxed:gen}.

      \begin{theorem}\label{thm:ub_exnt_wt}
        Let $\gamma\in \R^n$ be as defined in \cref{eq:def_gamma:gen}.
        If the matrix $U\in \Z_{+}^{n\times p}$ satisfies that for all $\ell\in[p]$ and $k\in [n]$, $$U_{k\ell} \geq \psi k,$$ then is an optimization program \prog{prog:noisy_fair_boxed:gen},
        parameterized by a constant $c$ and failure probability $\delta$,
        such that for any
        \begin{align*}
            c>1\quad \text{and} \quad \delta\in \left(0,\frac{1}{2}\right]
        \end{align*}
        its optimal solution satisfies $(c\gamma,\delta,v)$-constraint and
        has a utility at least as large as the utility of any ranking satisfying $\inparen{(c-\sqrt{c})\gamma,\delta,v}$-constraint.
      \end{theorem}

      \noindent The proof of \cref{thm:ub_exnt_wt} is analogous to the proof of \cref{thm:ub}.
      Here, we highlight the differences.

      \paragraph{Notation.}
      Recall that for each item $i\in [m]$ and group $\ell\in [p]$,
      let $Z_{i\ell}\in \zo$ be indicator random variable that
      $$Z_i\coloneqq \mathds{I}[G_\ell\ni i].$$

      \noindent The first change is in the definition of $Z_{\#}(R,\ell,k)$. %
      In particular, we need to define
      $$Z_{\#}(R,\ell,k) = \sum\nolimits_{i\in G_\ell}\sum\nolimits_{j=1}^k v_j R_{ij}.$$
      For the new definition of $Z_{\#}$, we have following concentration result.
      \begin{lemma}\label{lem:conc_bound:gen}
        For any position $k\in [n]$, group $\ell\in[p]$, parameters $\eps\geq 0$ and $L,U\in \R$, and ranking $R\in \cR$,
        where $R$ is possibly a random variable independent of $\inbrace{Z_{i\ell}}_{i,\ell}$,
        if
        \begin{align*}
          P_{\#}(R,\ell,k) \leq U
          \quad\text{or}\quad
          P_{\#}(R,\ell,k) \geq L,
        \end{align*}
        then the following equations hold respectively
        \begin{align*}
          \Pr\insquare{Z_{\#}(R,\ell,k) < \inparen{1+\eps} U} &\geq 1-{e^{-\frac{2U^2\eps^2}{k}}},\\
          \Pr\insquare{Z_{\#}(R,\ell,k) > \inparen{1-\eps} L} &\geq 1-{e^{-\frac{2L^2\eps^2}{k}}}.
        \end{align*}
      \end{lemma}
      \noindent The proof of \cref{lem:conc_bound:gen} is identical to the proofs of \cref{lem:lowerbound_main_2,lem:upperbound_main_2}; the only change is the new definition of $Z_{\#}$.
      To prove \cref{thm:ub_exnt_wt}, it suffices to prove analogues of \cref{prop:xNF_is_fair:main_body,prop:xOPT_delta_is_feasible:main_body} for the new definition of $Z_{\#}$.
      Their proofs change as follows:

      \paragraph{Proof of \cref{prop:xNF_is_fair:main_body}}
      The parameters in \cref{eq:proof1:params_val:main_body} remain the same.
      Hence, following the same argument, \cref{eq:proof1:ub:main_body} holds.
      Now, we can prove \cref{eq:proof1:proved:main_body} as follows:
      \begin{align*}
        \Pr\insquare{Z_{\#}(R,\ell, k) \geq U_{\ell k}(1+\phi\gamma_k)}
        \quad\ \
        &\Stackrel{}{=}\quad\ \  \Pr\insquare{Z_{\#}(R,\ell,k) \geq U'
        (1+\zeta)}
        \tag{Using that $U'(1+\zeta) = U_{k\ell} (1+\phi\gamma_k)$}\\
        &\leq\quad\ \  \exp\inparen{-\frac{2\inparen{U^\prime}^2\zeta^2}{k}}
        \tag{Using \cref{lem:conc_bound:gen}}\\
        &=\quad\ \  \exp\inparen{-\frac{2(1-\phi)^2U_{\ell k}^2 \gamma_{k}^2}{k}}
        \tag{Using \cref{eq:proof1:params_val:main_body}}\\
        &\Stackrel{}{\leq}\quad\ \  \exp\inparen{-2\psi(1-\phi)^2U_{\ell k} \gamma_{k}^2}
        \tag{Using that $U_{k\ell}\geq \psi k$}\\
        &\Stackrel{}{\leq}\quad\ \ \frac{\delta}{2np}.
        \tagnum{Using \cref{eq:def_gamma:gen}}
        \customlabel{eq:delta_bound_on_upperbound:gen}{\theequation}
      \end{align*}
      \cref{prop:xNF_is_fair:main_body} follows by replacing \cref{eq:proof1:proved:main_body} by \Eqref{eq:delta_bound_on_upperbound:gen} in the rest of its proof.

      \paragraph{Proof of \cref{prop:xOPT_delta_is_feasible:main_body}}
      The parameters in \cref{eq:proof2:params_val2:main_body} remain the same.
      Now, we can prove $$\Pr\insquare{{Z_{\#}(R',k,\ell)\leq U_{k\ell}}}< 1 - \delta$$ as follows:
      \begin{align*}
        \Pr\insquare{Z_{\#}(R',k,\ell) \leq U_{k\ell}}\ \
        &=\ \Pr\insquare{Z_{\#}(R',k,\ell) \leq L^\prime\cdot (1-\zeta)}
        \tag{Using that $L'(1-\zeta) = U_{k\ell}(1+b\gamma_k)$}\\
        &\leq\ \exp\inparen{-\frac{2\inparen{L^\prime}^2\zeta^2}{k}}
        \tag{Using \cref{lem:conc_bound:gen}}\\
        &=\ \exp\inparen{-\frac{2(\phi-b)^2\gamma_k^2U_{k\ell}^2}{k}}
        \tag{Using \cref{eq:proof2:params_val2:main_body}}\\
        &\leq\ \exp\inparen{-2\psi(\phi-b)^2\gamma_k^2U_{k\ell}}
        \tag{Using that $U_{k\ell}\geq \psi k$}\\
        &\Stackrel{}{<}\ \frac{\delta}{2np}
        \tagnum{Using \cref{eq:def_gamma:gen} and \cref{eq:proof2:params_val2:main_body}}\\
        &\Stackrel{}{<}\ 1-\delta.
        \tagnum{Using that $\delta<\frac12$ and $n\geq 1$}
        \customlabel{eq:proof2:proved:alt:gen}{\theequation}
      \end{align*}
      The rest of the proof is identical.

\section{Limitations and Conclusion}  \label{sec:lim_conc}

Recent studies find that errors in socially-salient attributes can adversely affect fairness and utility of existing fair-ranking algorithms~\cite{GhoshDW21}.
We consider a model of random and independent errors in socially-salient attributes and present a framework that can output rankings with high fairness and utility in this model.
This framework works {for} a general class of fairness criteria, which involve multiple overlapping groups and upper bounds on the number of items that appear in the first $k$ positions from each group.
We also show near-tightness of the framework's fairness guarantee.
Empirically, on both synthetic and real-world datasets, we observe that, compared to baselines, our framework can achieve higher fairness-values and a {similar or better fairness-utility trade-off for standard metrics.}

\medskip

Compared to existing fair-ranking frameworks, our framework does not need accurate socially-salient attributes, but assumes that errors in attributes are random and independent.
When these assumptions do not hold, our framework may not satisfy its guarantees and a careful assessment of this on application-specific data would be important to avoid any (unintended) negative social impact.

\medskip

Our work only addresses one aspect of how bias may show up in rankings, and more generally, on the web.
{For instance, while we consider a large class of fairness constraints, it does not capture some important notions such as the qualitative representation of different groups \cite{KayMM15,Noble2018}.}
It is important to take an holistic approach to mitigate bias and incorporate our work as a part of such a broader effort.
Finally, our work adds to the line of works that develop fair decision-making algorithms robust to inaccuracies in data~\cite{LamyZ19,awasthi2020equalized,MozannarOS20,prob_fair_clustering,wang2020robust,wang2021label,MehrotraC21,celis2021fairclassification}.

\paragraph{Acknowledgments.}\  \
  This research was supported in part by  NSF Awards CCF-2112665 and IIS-2045951, and an AWS MLRA Award.

\addtocontents{toc}{\protect\setcounter{tocdepth}{2}}

\bibliography{bib-v1.bib}

\begin{thebibliography}{10}

\bibitem{cnn_model}
{OpenCV: Open Source Computer Vision Library}.
\newblock
  \url{https://github.com/opencv/opencv_3rdparty/raw/dnn_samples_face_detector_20170830/res10_300x300_ssd_iter_140000.caffemodel}.

\bibitem{microsoft_diverse}
Rakesh Agrawal, Sreenivas Gollapudi, Alan Halverson, and Samuel Ieong.
\newblock {Diversifying Search Results}.
\newblock In {\em Proceedings of the Second ACM International Conference on Web
  Search and Data Mining}, WSDM '09, page 5–14, New York, NY, USA, 2009.
  Association for Computing Machinery.

\bibitem{altingovde2008incremental}
Ismail~Sengor Altingovde, Engin Demir, Fazli Can, and {\"O}zg{\"u}r Ulusoy.
\newblock {Incremental Cluster-Based Retrieval Using Compressed
  Cluster-Skipping Inverted Files}.
\newblock {\em ACM Transactions on Information Systems (TOIS)}, 26(3):1--36,
  2008.

\bibitem{Andrus2021WhatWeCantMeasure}
McKane Andrus, Elena Spitzer, Jeffrey Brown, and Alice Xiang.
\newblock {What We Can't Measure, We Can't Understand: Challenges to
  Demographic Data Procurement in the Pursuit of Fairness}.
\newblock In {\em FAccT}, pages 249--260. {ACM}, 2021.

\bibitem{AngluinL87}
Dana Angluin and Philip~D. Laird.
\newblock {Learning From Noisy Examples}.
\newblock {\em Mach. Learn.}, 2(4):343--370, 1987.

\bibitem{awasthi2020equalized}
Pranjal Awasthi, Matth{\"a}us Kleindessner, and Jamie Morgenstern.
\newblock {Equalized Odds Postprocessing under Imperfect Group Information}.
\newblock In {\em International Conference on Artificial Intelligence and
  Statistics}, pages 1770--1780. PMLR, 2020.

\bibitem{bar2008random}
Ziv Bar-Yossef and Maxim Gurevich.
\newblock {Random Sampling from a Search Engine’s Index}.
\newblock {\em Journal of the ACM (JACM)}, 55(5):1--74, 2008.

\bibitem{googleLTR}
Michael Bendersky and Xuanhui Wang.
\newblock {Advances in TF-Ranking}, July 2021.
\newblock \url{https://ai.googleblog.com/2021/07/advances-in-tf-ranking.html}.

\bibitem{AmortizedFairnessBiega2018}
Asia~J. Biega, Krishna~P. Gummadi, and Gerhard Weikum.
\newblock {Equity of Attention: Amortizing Individual Fairness in Rankings}.
\newblock In {\em {SIGIR}}, pages 405--414. {ACM}, 2018.

\bibitem{BuolamwiniG18}
Joy Buolamwini and Timnit Gebru.
\newblock {Gender Shades: Intersectional Accuracy Disparities in Commercial
  Gender Classification}.
\newblock In {\em {FAT}}, volume~81 of {\em Proceedings of Machine Learning
  Research}, pages 77--91. {PMLR}, 2018.

\bibitem{burges2005learning}
Chris Burges, Tal Shaked, Erin Renshaw, Ari Lazier, Matt Deeds, Nicole
  Hamilton, and Greg Hullender.
\newblock {Learning to Rank Using Gradient Descent}.
\newblock In {\em Proceedings of the 22nd international conference on Machine
  learning}, pages 89--96, 2005.

\bibitem{burges2010ranknet}
Christopher~J.C. Burges.
\newblock {From RankNet to LambdaRank to LambdaMART: An Overview}.
\newblock {\em Learning}, 2010.

\bibitem{calders2010three}
Toon Calders and Sicco Verwer.
\newblock {Three Naive Bayes Approaches for Discrimination-Free
  Classification}.
\newblock {\em Data Min. Knowl. Discov.}, 21(2):277--292, 2010.

\bibitem{celis2021fairclassification}
L.~Elisa Celis, Lingxiao Huang, Vijay Keswani, and Nisheeth~K. Vishnoi.
\newblock {Fair Classification with Noisy Protected Attributes}.
\newblock In {\em {ICML}}, volume 120 of {\em Proceedings of Machine Learning
  Research}. {PMLR}, 2021.

\bibitem{celis2020cscw}
L.~Elisa Celis and Vijay Keswani.
\newblock {Implicit Diversity in Image Summarization}.
\newblock {\em Proc. {ACM} Hum. Comput. Interact.}, 4({CSCW2}):139:1--139:28,
  2020.

\bibitem{celis2020interventions}
L.~Elisa Celis, Anay Mehrotra, and Nisheeth~K. Vishnoi.
\newblock {Interventions for Ranking in the Presence of Implicit Bias}.
\newblock In {\em Proceedings of the 2020 Conference on Fairness,
  Accountability, and Transparency}, FAT* '20, page 369–380, New York, NY,
  USA, 2020. Association for Computing Machinery.

\bibitem{celis2021adversarial}
L.~Elisa Celis, Anay Mehrotra, and Nisheeth~K. Vishnoi.
\newblock {Fair Classification with Adversarial Perturbations}.
\newblock In A.~Beygelzimer, Y.~Dauphin, P.~Liang, and J.~Wortman Vaughan,
  editors, {\em Advances in Neural Information Processing Systems}, 2021.

\bibitem{celis2018ranking}
L.~Elisa Celis, Damian Straszak, and Nisheeth~K. Vishnoi.
\newblock {Ranking with Fairness Constraints}.
\newblock In {\em {ICALP}}, volume 107 of {\em LIPIcs}, pages 28:1--28:15.
  Schloss Dagstuhl - Leibniz-Zentrum fuer Informatik, 2018.

\bibitem{ChekuriVZ11}
Chandra Chekuri, Jan Vondr{\'{a}}k, and Rico Zenklusen.
\newblock {Multi-budgeted Matchings and Matroid Intersection via Dependent
  Rounding}.
\newblock In Dana Randall, editor, {\em Proceedings of the Twenty-Second Annual
  {ACM-SIAM} Symposium on Discrete Algorithms, {SODA} 2011, San Francisco,
  California, USA, January 23-25, 2011}, pages 1080--1097. {SIAM}, 2011.

\bibitem{ChenKMSU19}
Jiahao Chen, Nathan Kallus, Xiaojie Mao, Geoffry Svacha, and Madeleine Udell.
\newblock {Fairness Under Unawareness: Assessing Disparity When Protected Class
  Is Unobserved}.
\newblock In {\em {FAT}}, pages 339--348. {ACM}, 2019.

\bibitem{cleverdon1991significance}
Cyril~W Cleverdon.
\newblock {The Significance of the Cranfield Tests on Index Languages}.
\newblock In {\em Proceedings of the 14th annual international ACM SIGIR
  conference on Research and development in information retrieval}, pages
  3--12, 1991.

\bibitem{dave2003mining}
Kushal Dave, Steve Lawrence, and David~M Pennock.
\newblock {Mining the Peanut Gallery: Opinion Extraction and Semantic
  Classification of Product Reviews}.
\newblock In {\em Proceedings of the 12th international conference on World
  Wide Web}, pages 519--528, 2003.

\bibitem{elliott2009UsingCencusSurnameList}
Marc Elliott, Peter Morrison, Allen Fremont, Daniel Mccaffrey, Philip Pantoja,
  and Nicole Lurie.
\newblock {Using the Census Bureau's Surname List to Improve Estimates of
  Race/Ethnicity and Associated Disparities}.
\newblock {\em Health Services and Outcomes Research Methodology}, 9:252--253,
  06 2009.

\bibitem{Epstein2015}
Robert Epstein and Ronald~E Robertson.
\newblock {The Search Engine Manipulation Effect ({SEME}) And Its Possible
  Impact on the Outcomes of Elections}.
\newblock {\em Proceedings of the National Academy of Sciences},
  112(33):E4512--E4521, 2015.

\bibitem{prob_fair_clustering}
Seyed~A. Esmaeili, Brian Brubach, Leonidas Tsepenekas, and John Dickerson.
\newblock {Probabilistic Fair Clustering}.
\newblock In {\em NeurIPS}, 2020.

\bibitem{FrenayV14}
Beno{\^{\i}}t Fr{\'{e}}nay and Michel Verleysen.
\newblock {Classification in the Presence of Label Noise: A Survey}.
\newblock {\em {IEEE} Trans. Neural Networks Learn. Syst.}, 25(5):845--869,
  2014.

\bibitem{linkedin_ranking_paper}
Sahin~Cem Geyik, Stuart Ambler, and Krishnaram Kenthapadi.
\newblock {Fairness-Aware Ranking in Search {\&} Recommendation Systems with
  Application to LinkedIn Talent Search}.
\newblock In {\em {KDD}}, pages 2221--2231. {ACM}, 2019.

\bibitem{linkedin_recuiter_algorithm}
Sahin~Cem Geyik and Krishnaram Kenthapadi.
\newblock {Building Representative Talent Search at LinkedIn}.
\newblock {\em LinkedIn Engineering}, October 2018.
\newblock \url{http://bit.ly/2x65HDJ}.

\bibitem{GhoshDW21}
Avijit Ghosh, Ritam Dutt, and Christo Wilson.
\newblock {When Fair Ranking Meets Uncertain Inference}.
\newblock In {\em {SIGIR}}, pages 1033--1043. {ACM}, 2021.

\bibitem{GorantlaUnderranking21}
Sruthi Gorantla, Amit Deshpande, and Anand Louis.
\newblock {On the Problem of Underranking in Group-Fair Ranking}.
\newblock In {\em {ICML}}, volume 139 of {\em Proceedings of Machine Learning
  Research}, pages 3777--3787. {PMLR}, 2021.

\bibitem{grotschel2012geometric}
M.~Gr{\"o}tschel, L.~Lovasz, and A.~Schrijver.
\newblock {\em {Geometric Algorithms and Combinatorial Optimization}}.
\newblock Algorithms and Combinatorics. Springer Berlin Heidelberg, 2012.

\bibitem{hannak2017bias}
Anik{\'o} Hann{\'a}k, Claudia Wagner, David Garcia, Alan Mislove, Markus
  Strohmaier, and Christo Wilson.
\newblock {Bias in Online Freelance Marketplaces: Evidence from TaskRabbit and
  Fiverr}.
\newblock In {\em CSCW}, page 1914–1933, 2017.

\bibitem{DCG}
Kalervo J\"{a}rvelin and Jaana Kek\"{a}l\"{a}inen.
\newblock {Cumulated Gain-Based Evaluation of IR Techniques}.
\newblock {\em ACM Trans. Inf. Syst.}, 20(4):422–446, oct 2002.

\bibitem{jeh2003scaling}
Glen Jeh and Jennifer Widom.
\newblock {Scaling Personalized Web Search}.
\newblock In {\em Proceedings of the 12th international conference on World
  Wide Web}, pages 271--279. ACM, 2003.

\bibitem{jung2020multicalibration}
Christopher Jung, Changhwa Lee, Mallesh Pai, Aaron Roth, and Rakesh Vohra.
\newblock {Moment Multicalibration for Uncertainty Estimation}.
\newblock In Mikhail Belkin and Samory Kpotufe, editors, {\em Proceedings of
  Thirty Fourth Conference on Learning Theory}, volume 134 of {\em Proceedings
  of Machine Learning Research}, pages 2634--2678. PMLR, 15--19 Aug 2021.

\bibitem{KallusMZ20}
Nathan Kallus, Xiaojie Mao, and Angela Zhou.
\newblock {Assessing Algorithmic Fairness with Unobserved Protected Class Using
  Data Combination}.
\newblock In {\em FAT*}, page 110. {ACM}, 2020.

\bibitem{KasiviswanathanLNRS11}
Shiva~Prasad Kasiviswanathan, Homin~K. Lee, Kobbi Nissim, Sofya Raskhodnikova,
  and Adam~D. Smith.
\newblock {What Can We Learn Privately?}
\newblock {\em {SIAM} J. Comput.}, 40(3):793--826, 2011.

\bibitem{KayMM15}
Matthew Kay, Cynthia Matuszek, and Sean~A. Munson.
\newblock {Unequal Representation and Gender Stereotypes in Image Search
  Results for Occupations}.
\newblock In {\em {CHI}}, pages 3819--3828. {ACM}, 2015.

\bibitem{Kirnap0BECY21}
{\"{O}}mer Kirnap, Fernando Diaz, Asia Biega, Michael~D. Ekstrand, Ben
  Carterette, and Emine Yilmaz.
\newblock {Estimation of Fair Ranking Metrics with Incomplete Judgments}.
\newblock In {\em {WWW}}, pages 1065--1075. {ACM} / {IW3C2}, 2021.

\bibitem{KleinbergR18}
Jon~M. Kleinberg and Manish Raghavan.
\newblock {Selection Problems in the Presence of Implicit Bias}.
\newblock In {\em {ITCS}}, pages 33:1--33:17. Schloss Dagstuhl--Leibniz-Zentrum
  fuer Informatik, 2018.

\bibitem{konstantinov2021fairness}
Nikola Konstantinov and Christoph~H. Lampert.
\newblock {On the Impossibility of Fairness-Aware Learning from Corrupted
  Data}.
\newblock In Jessica Schrouff, Awa Dieng, Miriam Rateike, Kweku Kwegyir-Aggrey,
  and Golnoosh Farnadi, editors, {\em Proceedings of The Algorithmic Fairness
  through the Lens of Causality and Robustness}, volume 171 of {\em Proceedings
  of Machine Learning Research}, pages 59--83. PMLR, 13 Dec 2022.

\bibitem{LamyZ19}
Alexandre~Louis Lamy and Ziyuan Zhong.
\newblock {Noise-Tolerant Fair Classification}.
\newblock In {\em NeurIPS}, pages 294--305, 2019.

\bibitem{LiddyAutomatic05}
Elizabeth~D. Liddy.
\newblock {Automatic Document Retrieval}.
\newblock In {\em Encyclopedia of Language and Linguistics}. Elsevier, 2005.

\bibitem{liu2010personalized}
Jiahui Liu, Peter Dolan, and Elin~R{\o}nby Pedersen.
\newblock {Personalized News Recommendation Based on Click Behavior}.
\newblock In {\em Proceedings of the 15th international conference on
  Intelligent user interfaces}, pages 31--40. ACM, 2010.

\bibitem{liu2011learning}
Tie-Yan Liu.
\newblock {Learning to Rank for Information Retrieval}.
\newblock 3(3):225–331, mar 2009.

\bibitem{IRbook}
Christopher Manning, Prabhakar Raghavan, and Hinrich Sch{\"u}tze.
\newblock {Introduction to Information Retrieval}.
\newblock {\em Natural Language Engineering}, 16(1):100--103, 2010.

\bibitem{ManwaniS13}
Naresh Manwani and P.~S. Sastry.
\newblock {Noise Tolerance Under Risk Minimization}.
\newblock {\em {IEEE} Trans. Cybern.}, 43(3):1146--1151, 2013.

\bibitem{MehrotraC21}
Anay Mehrotra and L.~Elisa Celis.
\newblock {Mitigating Bias in Set Selection with Noisy Protected Attributes}.
\newblock In {\em FAccT}, pages 237--248. {ACM}, 2021.

\bibitem{robustFairLTR2020}
Omid Memarrast, Ashkan Rezaei, Rizal Fathony, and Brian~D. Ziebart.
\newblock Fairness for robust learning to rank.
\newblock {\em CoRR}, abs/2112.06288, 2021.

\bibitem{polarizationWSJ2020}
Christopher Mims.
\newblock {Why Social Media Is So Good at Polarizing Us}, October 2020.
\newblock
  \url{https://www.wsj.com/articles/why-social-media-is-so-good-at-polarizing-us-11603105204}.

\bibitem{MorikSHJ20}
Marco Morik, Ashudeep Singh, Jessica Hong, and Thorsten Joachims.
\newblock {Controlling Fairness and Bias in Dynamic Learning-to-Rank}.
\newblock In {\em {SIGIR}}, pages 429--438. {ACM}, 2020.

\bibitem{motwani1995randomized}
Rajeev Motwani and Prabhakar Raghavan.
\newblock {\em {Randomized Algorithms}}.
\newblock Cambridge university press, 1995.

\bibitem{mousavi2010tight}
Nima Mousavi.
\newblock {How Tight Is Chernoff Bound}, 2010.

\bibitem{MozannarOS20}
Hussein Mozannar, Mesrob~I. Ohannessian, and Nathan Srebro.
\newblock {Fair Learning with Private Demographic Data}.
\newblock In {\em {ICML}}, volume 119 of {\em Proceedings of Machine Learning
  Research}, pages 7066--7075. {PMLR}, 2020.

\bibitem{Noble2018}
Safiya~Umoja Noble.
\newblock {\em {Algorithms of Oppression: How Search Engines Reinforce
  Racism}}.
\newblock NYU Press, 2018.

\bibitem{criticalReviewFairRanking22}
Gourab~K. Patro, Lorenzo Porcaro, Laura Mitchell, Qiuyue Zhang, Meike Zehlike,
  and Nikhil Garg.
\newblock {Fair Ranking: A Critical Review, Challenges, and Future Directions}.
\newblock In {\em 2022 ACM Conference on Fairness, Accountability, and
  Transparency}, FAccT '22, page 1929–1942, New York, NY, USA, 2022.
  Association for Computing Machinery.

\bibitem{product_knapsack}
Ulrich Pferschy, Joachim Schauer, and Clemens Thielen.
\newblock {Approximating the Product Knapsack Problem}.
\newblock {\em Optimization Letters}, 15(8):2529--2540, 2021.

\bibitem{overviewFairRanking}
Evaggelia Pitoura, Kostas Stefanidis, and Georgia Koutrika.
\newblock {Fairness in Rankings and Recommendations: An Overview}.
\newblock {\em The VLDB Journal}, 2021.

\bibitem{imdb_wiki_code}
Rasmus Rothe, Radu Timofte, and Luc~Van Gool.
\newblock {Deep Expectation of Real and Apparent Age from a Single Image
  without Facial Landmarks}.
\newblock {\em International Journal of Computer Vision}, 126(2-4):144--157,
  2018.

\bibitem{Selbst:2019}
Andrew~D. Selbst, Danah Boyd, Sorelle~A. Friedler, Suresh Venkatasubramanian,
  and Janet Vertesi.
\newblock {Fairness and Abstraction in Sociotechnical Systems}.
\newblock In {\em Proceedings of the Conference on Fairness, Accountability,
  and Transparency}, FAT* '19, pages 59--68, New York, NY, USA, 2019. ACM.

\bibitem{fairExposureAshudeep}
Ashudeep Singh and Thorsten Joachims.
\newblock {Fairness of Exposure in Rankings}.
\newblock In {\em {KDD}}, pages 2219--2228. {ACM}, 2018.

\bibitem{policyLearningAshudeep}
Ashudeep Singh and Thorsten Joachims.
\newblock {Policy Learning for Fairness in Ranking}.
\newblock In {\em NeurIPS}, pages 5427--5437, 2019.

\bibitem{AshudeepUncertainty2021}
Ashudeep Singh, David Kempe, and Thorsten Joachims.
\newblock {Fairness in Ranking under Uncertainty}.
\newblock In A.~Beygelzimer, Y.~Dauphin, P.~Liang, and J.~Wortman Vaughan,
  editors, {\em Advances in Neural Information Processing Systems}, 2021.

\bibitem{singitham2004efficiency}
Pavan Kumar~C Singitham, Mahathi~S Mahabhashyam, and Prabhakar Raghavan.
\newblock Efficiency-quality tradeoffs for vector score aggregation.
\newblock In {\em Proceedings of the Thirtieth international conference on Very
  large data bases-Volume 30}, pages 624--635, 2004.

\bibitem{taylor2006optimisation}
Michael Taylor, Hugo Zaragoza, Nick Craswell, Stephen Robertson, and Chris
  Burges.
\newblock {Optimisation Methods for Ranking Functions with Multiple
  Parameters}.
\newblock In {\em Proceedings of the 15th ACM international conference on
  Information and knowledge management}, pages 585--593, 2006.

\bibitem{wang2021label}
Jialu Wang, Yang Liu, and Caleb Levy.
\newblock {Fair Classification with Group-Dependent Label Noise}.
\newblock In {\em FAccT}, pages 526--536. {ACM}, 2021.

\bibitem{wang2020robust}
Serena Wang, Wenshuo Guo, Harikrishna Narasimhan, Andrew Cotter, Maya~R. Gupta,
  and Michael~I. Jordan.
\newblock {Robust Optimization for Fairness with Noisy Protected Groups}.
\newblock In {\em NeurIPS}, 2020.

\bibitem{weston2010large}
Jason Weston, Samy Bengio, and Nicolas Usunier.
\newblock {Large Scale Image Annotation: Learning to Rank with Joint Word-Image
  Embeddings}.
\newblock {\em Machine learning}, 81(1):21--35, 2010.

\bibitem{YangZ18}
Grace~Hui Yang and Sicong Zhang.
\newblock {Differential Privacy for Information Retrieval}.
\newblock In {\em {WSDM}}, pages 777--778. {ACM}, 2018.

\bibitem{BalancedRankingYang2019}
Ke~Yang, Vasilis Gkatzelis, and Julia Stoyanovich.
\newblock {Balanced Ranking with Diversity Constraints}.
\newblock In {\em {IJCAI}}, pages 6035--6042. ijcai.org, 2019.

\bibitem{causal2021yang}
Ke~Yang, Joshua~R. Loftus, and Julia Stoyanovich.
\newblock {Causal Intersectionality and Fair Ranking}.
\newblock In {\em {FORC}}, volume 192 of {\em LIPIcs}, pages 7:1--7:20. Schloss
  Dagstuhl - Leibniz-Zentrum f{\"{u}}r Informatik, 2021.

\bibitem{YangS17}
Ke~Yang and Julia Stoyanovich.
\newblock {Measuring Fairness in Ranked Outputs}.
\newblock In {\em {SSDBM}}, pages 22:1--22:6. {ACM}, 2017.

\bibitem{ReducingDisparateExposureZehlike}
Meike Zehlike and Carlos Castillo.
\newblock Reducing disparate exposure in ranking: {A} learning to rank
  approach.
\newblock In {\em {WWW}}, pages 2849--2855. {ACM} / {IW3C2}, 2020.

\bibitem{fair_ranking_survey1}
Meike Zehlike, Ke~Yang, and Julia Stoyanovich.
\newblock {Fairness in Ranking, Part I: Score-Based Ranking}.
\newblock {\em ACM Comput. Surv.}, apr 2022.
\newblock Just Accepted.

\bibitem{fair_ranking_survey2}
Meike Zehlike, Ke~Yang, and Julia Stoyanovich.
\newblock {Fairness in Ranking, Part II: Learning-to-Rank and Recommender
  Systems}.
\newblock {\em ACM Comput. Surv.}, apr 2022.
\newblock Just Accepted.

\end{thebibliography}
\bibliographystyle{plain}

\newpage
\appendix

\setlength{\belowdisplayskip}{10pt plus 2.0pt minus 5pt}
\setlength{\abovedisplayskip}{10pt plus 2.0pt minus 5pt}
\setlength{\topsep}{4pt plus 1pt minus 2 pt}

\section{Further Discussion on Applicability of the Noise Model}\label{sec:dis_noise_model}
    The noise in \cref{def:noise_model}, arises in real-world settings where local differential privacy is ensured e.g., using the randomized response mechanism.

    \smallskip

    \begin{remark}[{\textbf{Model’s Assumptions Hold If Attributes Are Perturbed by Randomized Response}}]\label{rem:discussion_of_noise_model}
      The randomized response mechanism flips each item's protected attribute to an incorrect value with some (public) probability $0<\eta<\frac{1}{2}$, independent of all other items.
      Here, the independence assumption holds (by design) and $P$’s entries can be deduced from $\eta$.
      To see the latter concretely, consider two protected groups $G_1$ and $G_2$ ($p=2$), and their noisy versions $N_1$ and $N_2$ corresponding to the ``flipped'' attributes.
      For any item $i \in N_1$,
      \begin{align*}
          \text{$P_{i1} = (1-\eta) \cdot \frac{\abs{G_1}}{\abs{N_1}}$ \quad and\quad $P_{i2} = 1 - P_{i1}$.}
      \end{align*}
      For items in $N_2$, replace $P_{i1}$, $P_{i2}$, $G_1$, and $N_1$ with $P_{i2}$, $P_{i1}$, $G_2$, and $N_2$.
      When there are more than two groups ($p>2$), then the randomized response mechanism publicly specifies the probability $\eta_{a,b}$ with which it flips protected attribute value $\ell=a$ to another value $\ell=b$ (for any $a,b\in [p]$). As in the binary case above, P’s entries can be deduced from parameters $\inbrace{\eta_{a,b}\colon a,b \in [p]}$.
    \end{remark}

    Further, in other real-world settings such as image search and online recruiting, the entries of $P$ can be estimated using the confidence scores of classifiers or using auxiliary attributes. In more detail:
    \begin{itemize}[itemsep=0pt,leftmargin=\leftmarginINTERNAL]
        \item If the protected attribute is skin tone, then a classifier $C$ can be used to predict if image $i$ contains a person with a dark skin tone. If $C$ has a calibrated confidence score $0\leq c(i)\leq 1$ in this prediction, then $P_{i, {\rm dark skin-tone}} = c(i).$
        See \cref{fig:simulation_image} in \cref{sec:empirical_results} for results from a simulation that estimates $P$ in this fashion.
        \item If the protected attribute is race and individuals are uniformly drawn from the population, then for an individual $i$ with surname $S$ and zip-code $Z$, $P_{i, L}=f(Z, S),$ where $f(Z, S)$ is the fraction of individuals with surname $S$ in zip-code $Z$ who have the $L$-th race; which can be estimated using census data \cite{elliott2009UsingCencusSurnameList} (see \cref{fig:simulation_intersectional} in \cref{sec:empirical_results}).
    \end{itemize}

\paragraph{Discussion on the Noise Model with Disjoint Groups vs. Overlapping Groups.}
        For each item $i$ and group $G_\ell$ ($\ell\in [p]$), the noise model specifies the marginal probability that $i$ belongs to $G_\ell$: $$P_{i\ell}\coloneqq \Pr[G_\ell \ni i].$$
        For any $i$, the model allows for any joint probability distribution over the events $$(G_1\ni i), (G_2\ni i), \dots, (G_p\ni i)$$ that is consistent with the above marginal probabilities.
        This allows the model to capture the setting where all groups are disjoint -- by requiring the events $$(G_1\ni i), \dots, (G_p\ni i)$$ to be mutually exclusive.
        It also allows the model to capture the cases where all or only some of the groups can overlap.
        For instance, the case where $G_1$ can overlap with $G_2$ but both $G_1$ and $G_2$ are disjoint from $G_3$ can be captured by requiring the events $(G_3 \ni i)$ to be mutually exclusive of the events $(G_1 \ni i)$ and $(G_2 \ni i)$.
        Importantly, we do not need additional information to capture these settings--it suffices to know the marginal probabilities specified by $P$.

\section{Existing Fair-Ranking Algorithms with Rounding Is Insufficient}\label{sec:rounding_methods}

Since existing fair-ranking algorithms require access to protected attributes, one way to use them under the above model is to imputed groups $\hG_1,\dots,\hG_p$ using the specified probabilities.
Then run these algorithms w.r.t. the imputed groups.
To see an illustration, consider two groups $G_1$ and $G_2$.
A natural imputation strategy is to use the Bayes optimal classifier, which assigns item $i$ to $\hG_1$ if and only if $P_{i1}>0.5$ and has the lowest expected imputation error.
This may be reasonable when the imputation error is negligible.
However, on exploring this strategy with non-negligible imputation error, we find that the output rankings can violate equal representation significantly (see \cref{example:mlr}).
To gain some intuition consider an extreme case where all items in some set $S$, of size $n$, have $P_{i1}=0.51$.
The Bayes classifier assigns all items in $S$ to $\hG_1$, i.e., $$\sabs{S\cap \hG_1}=\abs{S}.$$
However, with high probability, $$\abs{S\cap G_1}\approx 0.51 \abs{S}.$$
Since $\abs{S\cap G_1}$ and $\sabs{S\cap \hG_1}$ are far, a ranking that selects $n$ items from $S$ and satisfies the constraints for  $\hG_1$ and $\hG_2$ but violate constraints with respect to the true groups.
\cref{example:mlr} gives an example where this occurs.

\medskip

Another imputation strategy, is independent rounding: it assigns each item $i$ to $\hG_1$ with probability $P_{i1}$ and otherwise to $\hG_1$.
This addresses the issue with Bayes imputation, because, it has property that for any set $T$ of size $n$, $\abs{T\cap G_1}$ are $\sabs{T\cap \hG_1}$ close with probability $1-e^{\Theta(n)}$.
However, when $m\gg n$, there are $$\binom{m}{n}\gg e^{n}$$ sets of size $n$, and hence, with high probability, there exists a set $S$ of size $n$ for which $\sabs{S\cap \wh{G}_1}$ and $\sabs{S\cap G_1}$ are arbitrarily far.
In this case also, existing fair-ranking algorithms can output rankings which violate equal representation significantly.
\cref{example:ir} gives an example where this occurs.

\begin{proposition}[\textbf{Imputing Protected Groups Using the Bayes Optimal Classifier Is Not Sufficient}]\label{example:mlr}
  Let $R$ be any optimal solution to \eqref{eq:equal_rep_with_known_groups} with protected groups imputed using the Bayes optimal classifier for given $p$.
  There exists a matrix $P\in [0,1]^{m\times 2}$ such that $R$ does not satisfy the $(\eps,\delta)$-equal representation constraint
  \begin{align*}
    \text{for any\quad $\delta<\frac12$\quad and\quad  $\eps$\quad $\st$\quad  $\eps_k<\frac1{20}$\quad for some \quad $k\geq 2$.}
    \end{align*}
  \end{proposition}

\begin{proposition}\label{example:ir}
Let $R$ be a random variable denoting the optimal solution to the fair-ranking problem (\prog{eq:equal_rep_with_known_groups}) for protected groups imputed using independent rounding with given $P\in [0,1]^{m\times 2}$.
For every $\beta>0$, there exists sufficiently large $n$ and $m$ and a matrix $P\in [0,1]^{m\times 2}$, such that, with probability at least $1-\beta$
$R$ does not satisfy the $(\eps,\delta)$-equal representation constraint
\begin{align*}
  \text{for any\quad $\delta<1-\beta$\quad and\quad  $\eps\in (0,1)^n$.}
\end{align*}

\end{proposition}

\medskip

\subsection{Proof of \cref{example:mlr}}
\begin{proof}[Proof of \cref{example:mlr}]
    Pick any even $n\in \N$.
    Let $m\coloneqq \frac{3n}2$.
    Let $\beta>0$ be a small constant that we will fix later.
    We will divide the items into the following three types:
    \begin{itemize}
      \item Type A: For each $1\leq i\leq \frac{n}2$ and $1\leq j\leq n$,
      \begin{align*}
        \text{$P_{i1}\coloneqq 0 = 1-P_{i2}$ and $W_{ij}\coloneqq 1$.}
      \end{align*}
      \item Type B: For each $\frac{n}2+1\leq i\leq n$ and $1\leq j\leq n$,
      \begin{align*}
        \text{$P_{i1}\coloneqq \frac12+\beta= 1-P_{i2}$ and $W_{ij}\coloneqq 1$.}
      \end{align*}
      \item Type C: For each $n+1\leq i\leq \frac{3n}2$ and $1\leq j\leq n$,
      \begin{align*}
        \text{$P_{i1}\coloneqq 1 = 1-P_{i2}$ and $W_{ij}\coloneqq 0$.}
        \end{align*}
      \end{itemize}
      Let $\wh{G}_1$ and $\wh{G}_2$ be the groups imputed using maximum likelihood rounding.
      By construction,
      $\wh{G}_1$ contains all items of Types A and B and no items of Type C, whereas $\wh{G}_2$ contains all items of Type C and no items of Types A and B.

      Let $R$ be an optimal solution of \prog{eq:equal_rep_with_known_groups} with parameters $G_1=\wh{G}_1$ and $G_2=\wh{G}_2$.
      Since $W_{ij}\leq 1$ for all $i\in [m], j\in[n]$, $$\inangle{R,W}\leq n.$$
      Because $R$ satisfies the equal representation constraints for two disjoint groups,
      for any even $k\in [n]$, $R$ places exactly $\frac{k}2$ items of Type A and $\frac{k}2$ items of Type B in the top $k$ positions.
      From $\wh{G}_1$, $R$ only places items of Type A:
      If $R$ picks no items of Type C, then $\inangle{R,W}=n$, whereas, if $R$ picks one or more items of Type C, then $\inangle{R,W}\leq n-1$, which is a contradiction since there is a ranking with utility $n$ that satisfies equal representation constraints (e.g., a ranking which places items of Type A and B in alternate positions).

      Since all items of Type A are (always) in $\wh{G}_2$, $R$ places at least $\frac{k}2$ items from $\wh{G}_2$ in the first $k$ positions.
      We will show that with probability larger than $\frac12$, at least $\frac{k}{20}$ of the $\frac{k}2$ items of Type B are in $\wh{G}_2$.
      Thus, with probability larger than $\frac12$,
      $R$ places more than $\frac{k}{2}\cdot \frac{11}{10}$ items from $\wh{G}_2$ in the top-$k$ positions, and hence, $R$ does not satisfy the $(\eps,\delta)$-equal representation constraint for any $\delta<\frac12$ and  $\eps\in \inparen{0,\frac1{10}}^n$.

      It remains to prove our claim.
      Select any $k \in \inbrace{2,4,\dots,n}$.
      Let $i_1,i_2,\dots,i_{k/2}\in [m]$ be the $n$ items of Type B that $R$ places in the first $k$ positions.
      Let $Z_{i_j}\in \zo$ be the indicator random variable that $i_j\in \wh{G_2}$.
      Thus, $Z_{i_1},\dots,Z_{i_{k/2}}$ are independent random variables, such that, for $j\in [k]$, $\Pr[Z_{i_j}]=1-P_{i_j}=\frac12-\beta.$
      It follows that $\Ex[\sum_{j=1}^{k/2}Z_{i_j}]=\frac{k}{2}\inparen{\frac12-\beta}$ and ${\rm Var}[\sum_{j=1}^{k/2} Z_{i_j}]=\frac{k}{2}\inparen{\frac14-\beta^2}.$
      Thus, using the Chebyshev's inequality on $\sum_{j=1}^{k/2} Z_{i_j}$,
      \begin{align*}
        \Pr\insquare{\abs{ \sum_{j=1}^{k/2} Z_{i_j} -\frac{k}{4}\inparen{1-2\beta} } > \frac{k}{8}\inparen{1-4\beta^2}\cdot \sqrt{2+\beta} } \leq \frac{1}{2+\beta}.
      \end{align*}
      Thus,
      \begin{align*}
        \Pr\insquare{\sum_{j=1}^{k/2} Z_{i_j} < \frac{k}{4}\inparen{1-2\beta} - \frac{k}{8}\inparen{1-4\beta^2}\cdot \sqrt{2+\beta} } \leq \frac{1}{2+\beta}.
      \end{align*}
      Since $\frac{k}{4}\inparen{1-2\beta} - \frac{k}{8}\inparen{1-4\beta^2}\cdot \sqrt{2+\beta}=k\inparen{\frac14-\frac{\sqrt{2}}{8}}+k\cdot O(\beta)$, for a sufficiently small $\beta>0$,
      $$\frac{k}{4}\inparen{1-2\beta} - \frac{k}{8}\inparen{1-4\beta^2}\cdot \sqrt{2+\beta}>\frac{k}{20}.$$
      Hence,
      \begin{align*}
        \Pr\insquare{\sum_{j=1}^{k/2} Z_{i_j} < \frac{k}{20}} \leq \frac{1}{2+\beta} \quad\Stackrel{(\beta>0)}{<}\quad \frac12.
        \yesnum\label{eq:chebyshev_bound}
      \end{align*}

    \end{proof}

\subsection{Proof of \cref{example:ir}}
\begin{proof}[Proof of \cref{example:ir}]
  Let $\phi>0$ be a small constant that we will fix later.
  We will divide the items into the following two types:
  \begin{itemize}[leftmargin=\leftmarginINTERNAL]
    \item Type A: For each item $i$ of Type A
    \begin{align*}
      \text{$P_{i1}\coloneqq \phi$, $P_{i2}\coloneqq 1-\phi$ and $W_{ij}\coloneqq 1$ for all $j\in[n]$.}
    \end{align*}
    \item Type B: For each item $i$ of Type B
    \begin{align*}
      \text{$P_{i1}\coloneqq 1$, $P_{i2}\coloneqq 0$ and $W_{ij}\coloneqq 0$ for all $j\in[n]$.}
    \end{align*}
    \item Type B: For each item $i$ of Type C
    \begin{align*}
      \text{$P_{i1}\coloneqq 0$, $P_{i2}\coloneqq 1$ and $W_{ij}\coloneqq 0$ for all $j\in[n]$.}
    \end{align*}
  \end{itemize}

  \noindent Let there be $m_A\coloneqq O\inparen{\log\inparen{\frac{n}{\beta} }\cdot \frac{n}{\log\inparen{\frac{1}{1-\phi}}}}$ items of Type A, $m_B\coloneqq n$ items of Type B, and $m_C\coloneqq n$ items of Type C.

  Note that a ranking which ranks items of Type B and Type C alternately, satisfies the equal representation constraints with probability 1.
  So in this instance, there exists a ranking which satisfies $(\delta,\eps)$-equal representation.
  However, we will show that $R$ does not satisfy $(\delta,\eps)$-equal representation with probability at least $1-\beta$.

  Let $\hG_1$ and $\hG_2$ be the groups imputed by independent rounding.
  Let $\evE$ be the event that $\hG_1$ contains at least $n$ items of Type A and $\evF$ be the event that  $\hG_2$ contains at least $n$ items of Type A.
  Both $\evE$ and $\evF$ occur with probability at most $O(\beta)$.
  To see this, divide the items of Type A into $n$ groups of equal size.
  From each group, at least one item is selected in $\hG_1$ and $\hG_2$ with probabilities at least $1-(1-\phi)^{\frac{m_A}{n}}$ and $1-(\phi)^{\frac{m_A}{n}}$ respectively.
  Taking a union bound over all groups and substituting $m_A$, we get
  \begin{align*}
    \Pr[\evE]\geq 1-\beta \quad \text{and}\Pr[\evF]\geq 1-\beta.
  \end{align*}
  Since only items of Type A have a nonzero contribution to the utility of a ranking and because there are at least $n$ items of Type A in each imputed group, it follows that $R$ only selects items of Type A.
  Now, the claim follows because, for small $\phi$, most items of Type A belong to $G_1$.

  Suppose $\evE$ and $\evF$ happen and, hence, $R$ only selects items of Type A.
  Let $Z_j$ be the indicator random variable that the item in the $j$-th position of $R$ is in $G_1$.
  We have that $\Pr[Z_j]=\phi$.
  Therefore, ${\rm Var}[\sum_{j=1}^n Z_j] = n\phi(1-\phi)$.
  Thus, using the Chebyshev's inequality we have
  \begin{align*}
    \Pr\insquare{\abs{\sum_{j=1}^n Z_j - n\phi} \geq \frac{n\eps_n}{4}}
    \leq \frac{4n\phi(1-\phi)}{n^2\eps_n^2}.
  \end{align*}
  Hence, for $\phi=\Theta(\eps_n^2\beta)$, we have that
  \begin{align*}
    \Pr\insquare{\sum_{j=1}^n Z_j\leq \frac{n\eps_n}{2}} \geq 1-\beta.
  \end{align*}
  The result follows since whenever $\sum_{j=1}^n Z_j\leq \frac{n\eps}{2}$, $R$ violates the equal representation constraint at the $n$-th position by a multiplicative factor larger than $1+\eps_n$.

\end{proof}

\section{Implementation Details and Additional Empirical Results}\label{sec:empirical_results_extended}
    In this section, we present the implementation details of our simulations (\cref{sec:implement,sec:preprocess}), give additional plots for the simulation in \cref{sec:empirical_results} (\cref{sec:add_plot}), and additional simulations that use weighted-selection risk as the fairness metric or vary the amount of noise in the data (\cref{sec:selection_lift,sec:vary_noise})

    \paragraph{Code.}  The code for all simulations is available at the following link \url{https://github.com/AnayMehrotra/FairRankingWithNoisyAttributes}.

    \subsection{{Implementation Details}}\label{sec:implement}

    In this section, we give implementation details of our algorithm and baselines.
    \begin{itemize}[itemsep=0pt,leftmargin=\leftmarginINTERNAL]
        \item {\ouralgo{}:}
        We implement \ouralgo{} in Python 3 and use the Gurobi optimization library to solve the linear program in Step 1 of \cref{algo}.
        We state complete pesudocode of \ouralgo{}'s implementation as \cref{algo:full_algo}.
        \item \sj{}: This is \cite{fairExposureAshudeep}'s algorithm.
        \sj{} {(1) solves a linear program whose objective encodes the utility of the ranking and whose constraints capture the fairness constraints, and (2) decomposes the solution as a convex combination of the rankings, and uses this convex combination to generate rankings (see \cite[Section 3.4]{fairExposureAshudeep}).
        \begin{itemize}
            \item More precisely, \cite{fairExposureAshudeep}'s approach works for any linear constraint on the ranking (see the last equation in \cite[Section 3.3]{fairExposureAshudeep}).
            For instance, as noted in \cite[Section 3.3]{fairExposureAshudeep},  their approach can satisfy multiple constraints of the form: Given any vectors $f\in \R^m$, $g\in \R^n$, and $h\in \R$, require the ranking $R\in \zo^{m\times n}$ to satisfy
            $f^\top R g = h$.
            By introducing a class variable $s$ with constraint $s\geq 0$, their approach extends to constraints of the form $$f^\top R g \leq h.$$
            These are sufficient to encode the constraint in Definition 2.2:
            For any $k$ and $\ell$, define
            \begin{align*}
                \forall i\in [m],\quad f_i &=\mathds{I}[i\in G_\ell],\\
                \forall j\in [n],\quad g_j &=\mathds{I}[j\leq k],\\
                h &=U_{k\ell}.
            \end{align*}
            The constraint $f^\top R g\leq h$ with the above values is equivalent to the upper bound specified by $U_{k\ell}$ in \cref{def:noise_model}.
            Repeating this construction for each $k$ and $\ell$, we get a set of constraints that capture the fairness constraints specified by $U$.
        \end{itemize}
        \cite{fairExposureAshudeep} do not provide an implementation of \sj{} and we implement \sj{} in Python3:
        We (1) construct an optimization program as defined above,
        (2)} use the Gurobi optimization library to solve the linear program constructed by \cite{fairExposureAshudeep}, and (3) use the code available at \url{https://github.com/jfinkels/birkhoff} to compute the Birkhoff-von Neumann decomposition of the solution (\cite{fairExposureAshudeep} also use the same code to compute the decomposition, see \cite[Section 3.4]{fairExposureAshudeep}).
        \item \csv{}:
        This is the greedy algorithm from \cite[Theorem 3.3]{celis2018ranking}.
        \cite{celis2018ranking} do not provide an implementation of \csv{}, we implement their algorithm in Python3 with NumPy.
        \item \detgreedy{}:
        This is the Det-Greedy algorithm of \cite{linkedin_ranking_paper}.
        \cite{linkedin_ranking_paper} do not provide an implementation of \detgreedy{}, we implement \detgreedy{} in Python3 with NumPy.
        \item \mc{} :
        This first uses the algorithm of \cite{MehrotraC21} to compute a subset $S$ and then selects a ranking of these items that maximize the utility (in the simulations this amounts to sorting items by $w_i$).
        We used the implementation of \cite{MehrotraC21}'s algorithm available at \url{https://github.com/AnayMehrotra/Noisy-Fair-Subset-Selection} and use Python3's in-built sorting function to generate the ranking.
        \cite{MehrotraC21}'s algorithm takes $P$ and parameters $U$ specifying upper bound constraints as input.
        \item \uncons{}:
        This is the baseline that outputs the ranking with the maximum utility.
        In the simulation, this amounts to sorting all items in decreasing order of $w_i$ and outputting the ranking with the first $n$ items (in that order).
        We implement \uncons{} in Python3 with NumPy.

    \end{itemize}

    \paragraph{Computational Resources Used.} All simulations were run on a \texttt{t3.xlarge} instance with 4 vCPUs and 16Gb RAM, on Amazon's Elastic Compute Cloud (EC2).

    \begin{algorithm}[h!] %
        \caption{Pseudo-code of \cref{algo}'s implementation} \label{algo:full_algo}
        \begin{algorithmic}[1]
            \Require {Matrices $P\in  {[0,1]}^{m\times p}$, $W\in {\R}_{\geq 0}^{m\times n}$, $U\in  {\R}^{n\times p}$}
            \Ensure A ranking $R\in \cR$
            \item[] \white{.} \hspace{-12.75mm} {\bf Parameters:} {Constant $c>1$, failure probability $\delta\in (0,1]$, and for each $k\in [n]$, relaxation \white{.} parameter $\gamma_k>0$}\vspace{2mm}
            \State Compute a solution $R_F$ to the standard linear programming relaxation of \prog{prog:noisy_fair_boxed}
            \item[] \Comment{{In the implementation, we use the Gurobi optimization library in Python 3 to compute $R_F$}}
            \vspace{2mm}
            \State Compute rankings $R_1,R_2,\dots,R_\Delta\in \cR$ and coefficients $\alpha_1\geq \alpha_2\geq \dots\geq \alpha_\Delta\in [0,1]$ such that
            $$R_F = \sum_{i=1}^{\Delta} \alpha_iR_i.$$
            \item [] \hspace{-10mm} \Comment{{In the implementation, we use the code from \url{https://github.com/jfinkels/birkhoff} to compute this decomposition.
            This code implements an algorithm to compute the Birkhoff-von Neumann decomposition.
            The value of $\Delta$ does not need to be specified: It is the number of rankings output by the algorithm to compute the Birkhoff-von Neumann decomposition.}
            }\vspace{2mm}
            \State Construct matchings $M_1,\dots,M_\Delta$ corresponding to each ranking $R_1,\dots,R_\Delta$ such that, for each $t\in [\Delta]$, $M_t$ has an edge between item $i$ and position $j$ if $i$ appears in position $j$ in $R_t$
            \vspace{2mm}
            \State Initialize $N_1=M_i$
            \For{$t=1,2,\dot,\Delta-1$}
              \State $N_{t+1}=$ {\bf Merge}($\alpha_{t+1}, M_{t+1}, \sum_{\ell=1}^t\alpha_\ell, N_t$)
            \EndFor
            \vspace{2mm}
            \State Construct a ranking $R$ corresponding to $N_\Delta$:
            Item $i$ appears at position $j$ in $R$, if and only if, $i$ and $j$ are matched in $N_{\Delta}$
            \State \Return $R$
            \vspace{3mm}
        \end{algorithmic}
    \end{algorithm}

    \begin{algorithm}[h] %
        \caption{Merge procedure (used in \cref{algo:full_algo})} \label{algo:merge}
        \begin{algorithmic}[1]
            \Require {Numbers $0<\alpha,\beta\leq 1$ and matchings $M$  and $N$}
            \Ensure A matching $K$
            \item[] \white{.} \hspace{-11mm} {\bf Parameters:} {A constant $t$ \Comment{We set to $t\coloneqq 100$ in the implementation}}\vspace{2mm}
            \While{$M\neq N$}
              \State $P=$ {\bf getPaths($M, N, t$)}
              \State $P'=$ {\bf getPaths($M, N, t$)}
              \State Set $\rho\coloneqq \frac{t-1}{\abs{P}}$ and $\sigma\coloneqq \frac{t}{\abs{P'}}$, and $p=\frac{\beta\sigma}{a\rho+\beta\sigma}$
              \vspace{2mm}
              \State Draw variables $v,u$ u.a.r. from $[0,1]$
              \If{$u\leq \frac{\beta\sigma}{\alpha\rho+\beta\sigma}$}
                \State Draw $i$ u.a.r. from $[\abs{P}]$ and set $M=M\Delta P_i$
              \Else
                \State Draw $i$ u.a.r. from $[\abs{P'}]$ and set $N=N\Delta P_i'$
              \EndIf
              \vspace{2mm}
            \EndWhile{}
            \State \Return $K\coloneqq M$
            \vspace{3mm}
        \end{algorithmic}
    \end{algorithm}

    \begin{algorithm}[h] %
        \caption{\textbf{getPaths} procedure (used in \cref{algo:merge})} \label{algo:get_path}
        \begin{algorithmic}[1]
            \Require {Matchings $M$ and $N$ and a parameter $t\geq 1$}
            \Ensure A set of paths $P$
              \State Set $P=\emptyset$
              \If{$\abs{M\Delta N} \leq 2t$}
                \State Construct $t$ paths $p_1,\dots,p_t$, where $p_i=M\Delta N$ for each $i$
                \State Let $N\backslash M\coloneqq \inbrace{v_1, \dots, v_n}$
                \State For each $i\in [t]$, remove $v_i$ from $p_i$
                \State Set $P\coloneqq\inbrace{p_1,\dots,p_t}$
              \ElsIf{$M\Delta N$ is a path}
                \State Let the path formed by $M\Delta N$ be $\inparen{v_1, \dots, v_n}$
                \For{$j=1,2,\dots,t+1$}
                  \State If $v_1\in N$ set $\ell=1$ else set $\ell=0$
                  \State Set $D\coloneqq \inbrace{v_{\ell+2tk}\colon k\in \N, \ \ell+2tk\leq \abs{M\Delta N}}$
                  \State Set $P=P\cup \inbrace{(M\Delta N)\backslash D}$
                \EndFor
              \Else
                \Comment{Here, $M\Delta N$ is a cycle}
                \State Let the cycle formed by $M\Delta N$ be $\inparen{v_1, \dots, v_n}$
                \For{$i=1,2,\dots,\abs{M\Delta N}$}
                  \State If $v_i\in M$: {\bf continue}
                  \State $S\coloneqq \inbrace{v_{(i+j)\%\abs{M\Delta N}}\colon j=0,1,\dots,2t-1}$
                  \State Set $P=P\cup S$
                \EndFor
              \EndIf
            \State \Return $K\coloneqq M$
            \vspace{3mm}
        \end{algorithmic}
    \end{algorithm}

    \subsubsection{Pre-processing Details of the Simulation with Image Data}\label{sec:prepros_image_data}\label{sec:preprocess}
    In this section, we present details of the preprocessing using while estimating $\hP$ in the simulation with the Occupations dataset presented in \cref{sec:empirical_results}.

    \paragraph{Estimating $\hP$.}
    We begin by removing all images with gender label NA; this leaves 5,825 images (out of 9600).
    On the remaining images, we use an off-the-shelf face-detector \cite{cnn_model}  to extract the faces of the people from the images and remove all images where the face-detector did not detect a face; this leaves 4,494 the images.
    We use a CNN-based gender classifier \cite{imdb_wiki_code} on the detected faces to predict the apparent gender of the depicted individuals.
    For each image $i$, the classifier outputs a gender (coded as male and female) and an uncalibrated confidence score $c_i\in [0,1]$.
    We take the set of uncalibrated confidence scores $\inbrace{c_i\in [0,1]}_{i}$ and calibrate them by first binning them, then computing the distribution of gender labels (provided in the dataset) for each bin.
    For each image $i$, we set ${\hP_{i1}}$ (respectively ${\hP_{i2}}$) equal to the fraction of images in the same bin as $i$ whose gender label is female (respectively male).
    We perform this calibration once and on all occupations and, then, use it for a subset of occupations.

    \renewcommand{\algorithmicrequire}{\textbf{Input:}}
    \renewcommand{\algorithmicensure}{\textbf{Output:}}

    \subsection{Further Discussion of Simulations in \cref{sec:empirical_results}}\label{sec:add_plot}

    \paragraph{Illustrating the fairness vs. utility trade-off.}
      In our empirical results, we use fairness metrics such as weighted risk-difference (\cref{sec:empirical_results}) and weighted selection-lift (\cref{sec:selection_lift}) to measure the algorithms' {\em achieved} fairness.
      We do not use the parameter $\phi$ to measure fairness because the output of algorithms may have lower fairness than specified by $\phi$.
      \cref{fig:simulation_image,fig:fig1bskjkj,fig:fig6} plot utility vs. weighted risk-difference and \cref{7b,8b,9b} plot utility vs. weighted selection-lift (\selift{}) for the simulations in \cref{sec:empirical_results}.
      They show that \ouralgo{} better or similar (up to standard errors) achieved fairness vs utility trade-off compared to baselines.
      For example, in \cref{8b}, to achieve \selift{}$=0.55$ use \cref{8a} to choose $\phi=1.19$ for \ouralgo{} and $\phi=1.15$ for \csv{} or \sj{}.
      For these values of $\phi$, \ouralgo{} has 2\% higher utility than \csv{} and \sj{}.

    \newcommand{\clean}{{\bf Clean-Fair}}

    \paragraph{Comparison to Baseline Which Has Access to \textit{Accurate} Protected Attributes.}
      Let \clean{} be the algorithm that, given utilities and accurate protected attributes, outputs the ranking with the maximum utility subject to satisfying equal representation constraint.
      Note that \clean{} can only be run in the ideal scenario where one has access to accurate protected attributes.
      We repeated the simulations in \cref{sec:empirical_results} and, for each of them, also measured the utility and fairness of \clean{}.
      We observe that the rankings output by \clean{} have a \rd{} close to 1 ($>$0.99), this is expected because \clean{} has access to the clean protected attributes.
      We observe that the ranking output by \ouralgo{} (for any parameter $0\leq \phi\leq 1$, specifying the fairness constraints for \ouralgo{}) has a utility that is at most 2\%, 10\%, and 4\% smaller than that the ranking output by \clean{}.

    \paragraph{Plots with a Small Number of Iterations.}
        \cref{fig:small_iter_syn,fig:small_iter_image,fig:small_iter_name} present results from simulations in \cref{sec:empirical_results} with 25, 50, and 100 iterations; compared to 500 or 1000 iterations in \cref{fig:different_fdr,fig:simulation_image,fig:simulation_intersectional}.
        We observe that:
        \begin{itemize}[itemsep=0pt,leftmargin=\leftmarginINTERNAL]
            \item the error bars for both utility and fairness (w.r.t. \rd{}) are a larger (up to 0.025 compared to at most 0.0125 with 500/1000 iterations). %
            \item the mean utilities and fairness (w.r.t. \rd{}) of all algorithms at all values of $\phi$ are additively within 0.05 of their corresponding values in \cref{fig:different_fdr,fig:simulation_image,fig:simulation_intersectional}. %
        \end{itemize}
        Moreover, the relative order of the algorithms with respect to both their fairness (w.r.t. \rd{}) and utility is the same as in  \cref{fig:different_fdr,fig:simulation_image,fig:simulation_intersectional} for all $\phi$.

    \paragraph{Plots with Different Values of $n$.}\label{sec:vary_n}
        \cref{fig:vary_n_syn,fig:vary_n_image,fig:vary_n_name} plot the $\rd{}$ and utilities (NDCG) with $n\in \inbrace{10, 30, 50}$ in the simulations from \cref{sec:empirical_results}; compared to $n=25$ in \cref{fig:different_fdr,fig:simulation_image,fig:simulation_intersectional}. %
            We observe that the best \rd{} attained by \ouralgo{} increases with $n$:
                Increasing $n$ from 10 to 30, increases $\rd$ from  $0.76$ to 0.85 with the synthetic data, from 0.75 to 0.84 with the real-world image data, and from 0.61 to 0.71 with the real-world name data (see \cref{fig:vary_n_syn,fig:vary_n_image,fig:vary_n_name}).
            Further, in all simulations, \ouralgo{}'s maximum \rd{} is 2\% to 8\% higher than that of the baselines (see \cref{fig:vary_n_syn,fig:vary_n_image,fig:vary_n_name}).
                One exception is the simulation with real-world image data and $n=10$.
                In this simulation, \ouralgo{}'s best \rd{} is equal to \detgreedy{}'s best \rd{}.
                Both of them have $>6\%$ higher best $\rd$ than any other algorithm. (See \cref{fig:vary_n_image}.)

    \renewcommand{\folder}{./figs/vary-n}
    \begin{figure}[t!]
      \centering
      \subfigure[$n=10$]{
        \includegraphics[width=0.40\linewidth, trim={0cm 0cm 0cm 0cm},clip]{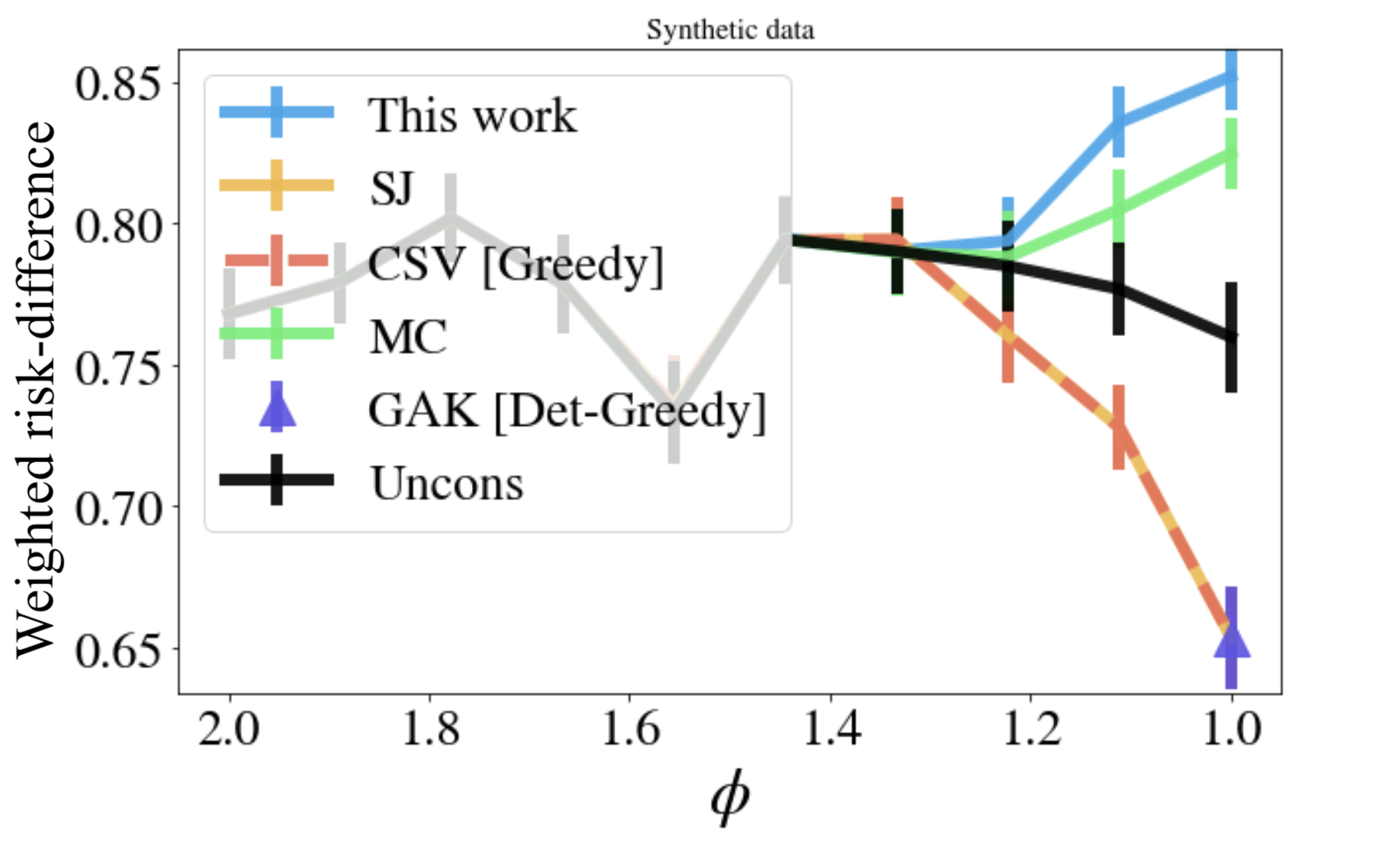}
      }
      \subfigure[$n=30$]{
        \includegraphics[width=0.40\linewidth, trim={0cm 0cm 0cm 0cm},clip]{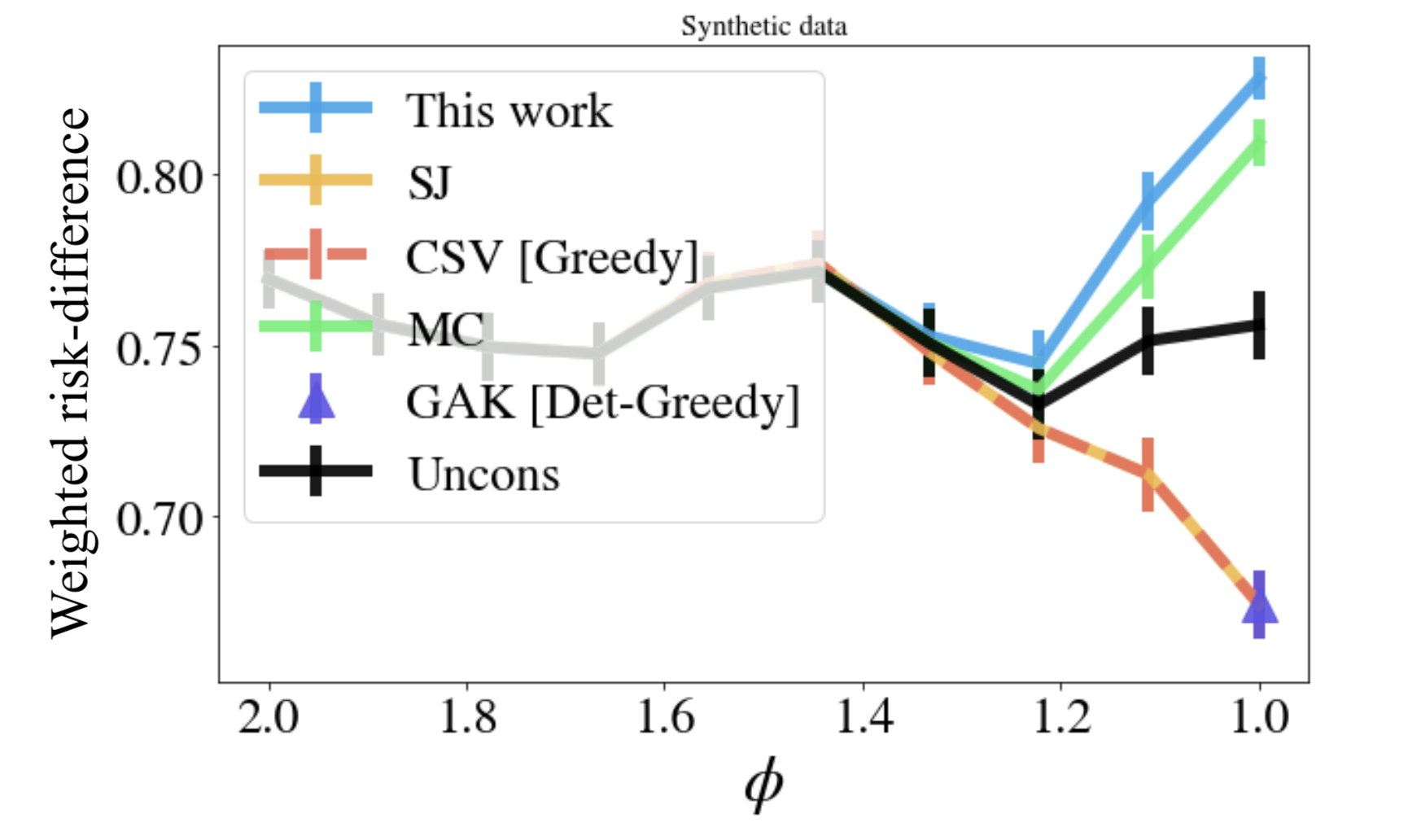}
      }
      \subfigure[$n=50$]{
        \includegraphics[width=0.40\linewidth, trim={0cm 0cm 0cm 0cm},clip]{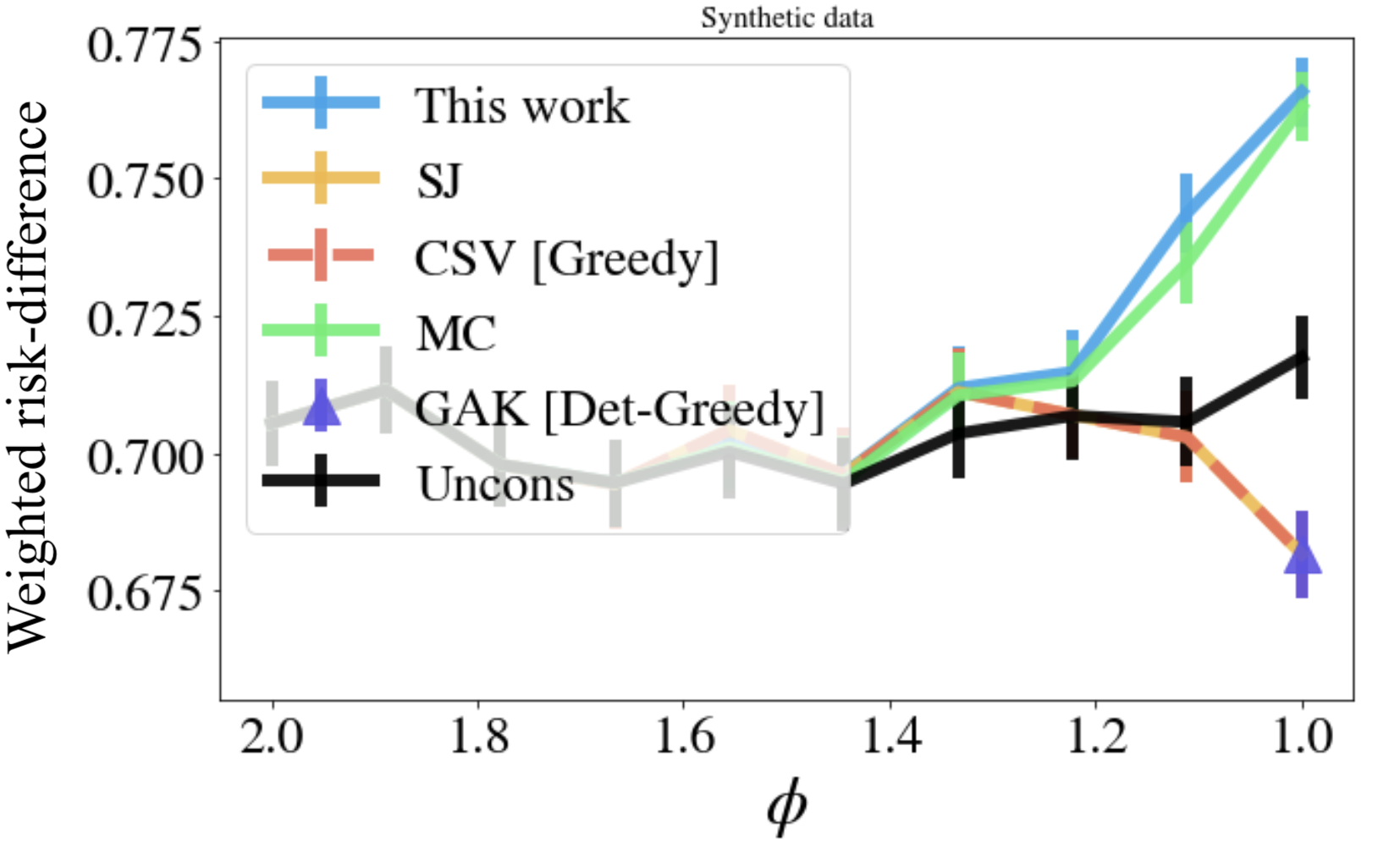}
      }
      \caption{
      Simulation on synthetic data with different values of $n$.
      The details appear in \cref{sec:vary_n}.
      }
      \label{fig:vary_n_syn}
    \end{figure}

    \begin{figure}[h!]
      \centering
      \subfigure[$n=10$]{
        \includegraphics[width=0.40\linewidth, trim={0cm 0cm 0cm 0cm},clip]{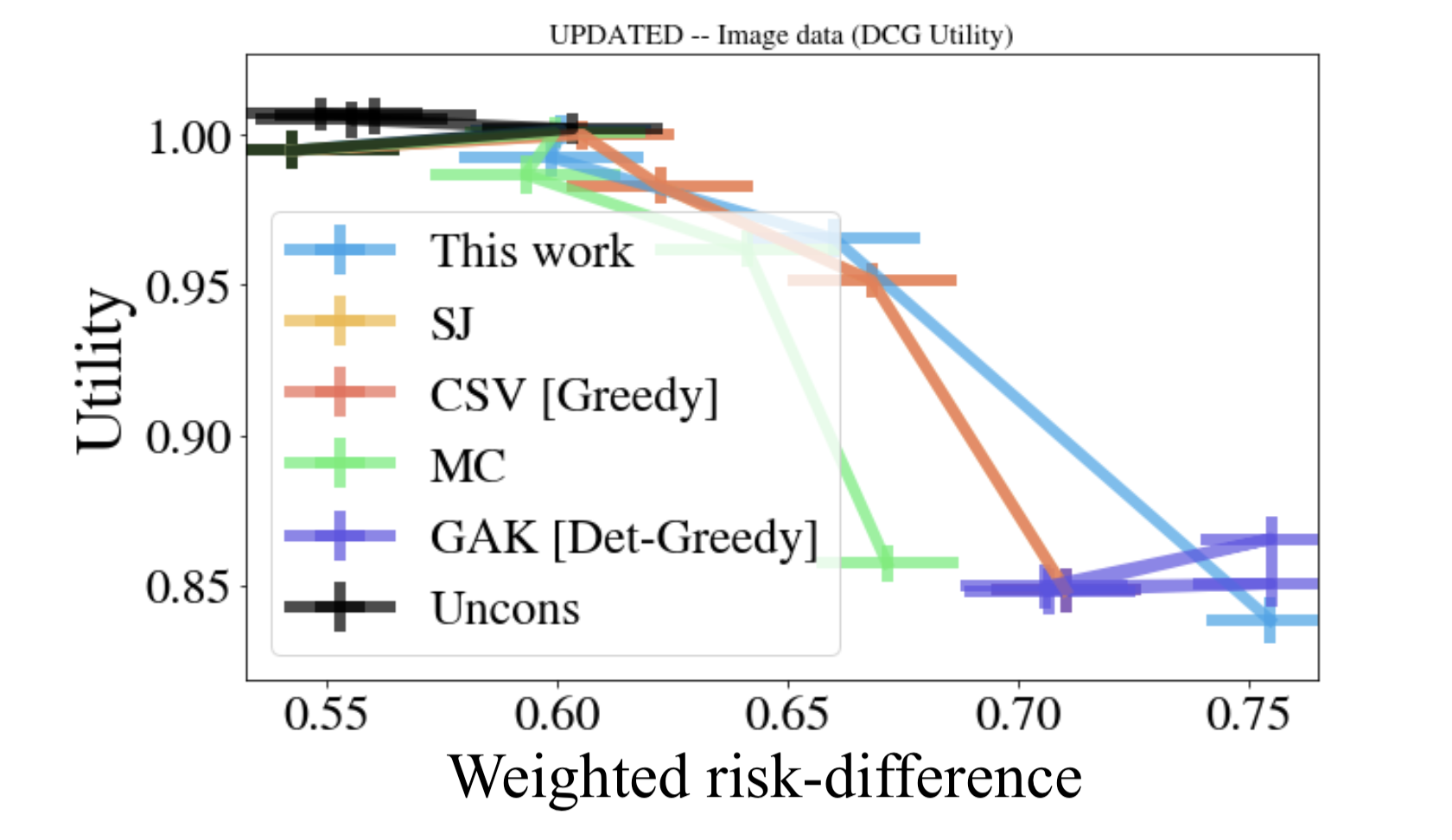}
      }
      \subfigure[$n=30$]{
        \includegraphics[width=0.40\linewidth, trim={0cm 0cm 0cm 0cm},clip]{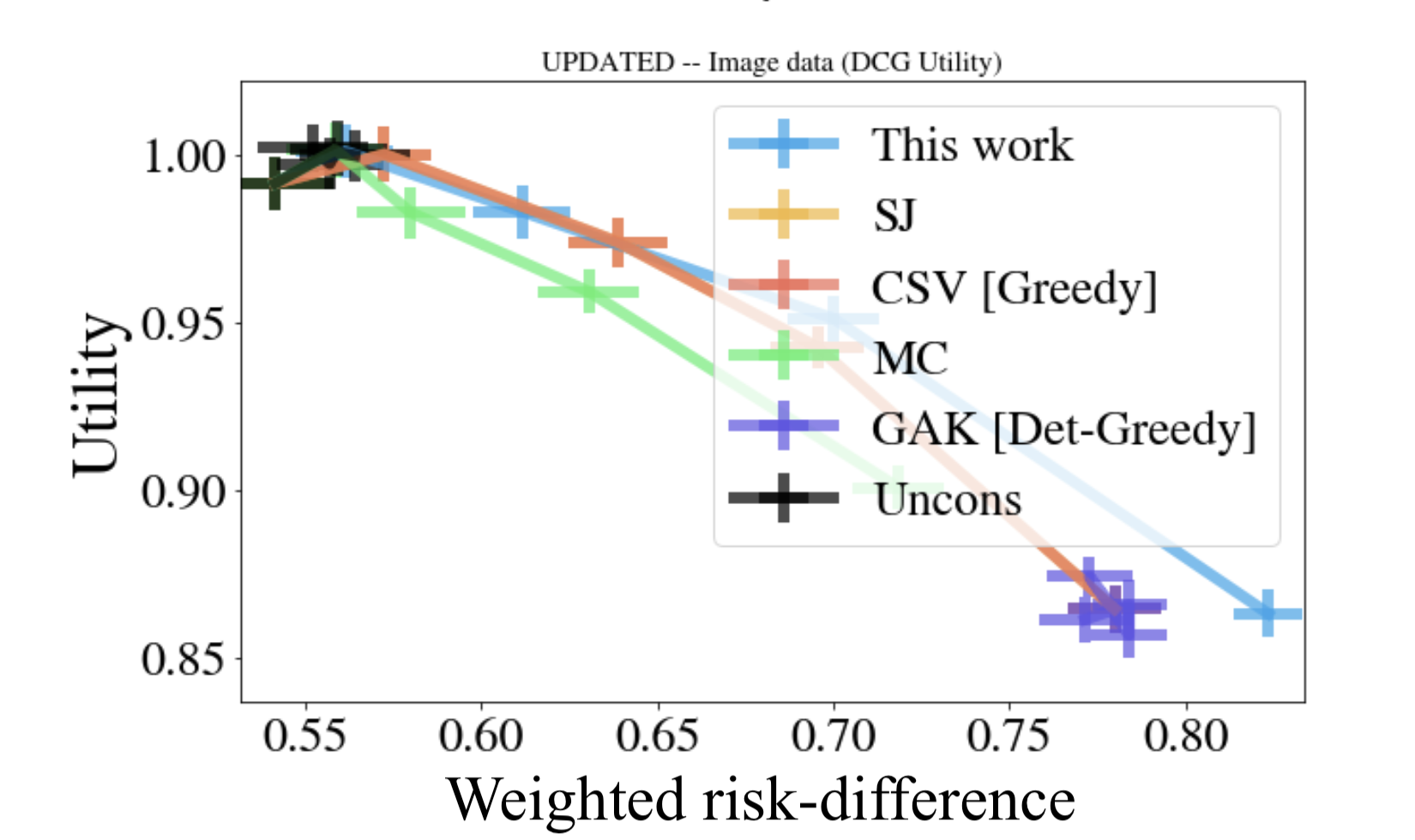}
      }
      \subfigure[$n=50$]{
        \includegraphics[width=0.40\linewidth, trim={0cm 0cm 0cm 0cm},clip]{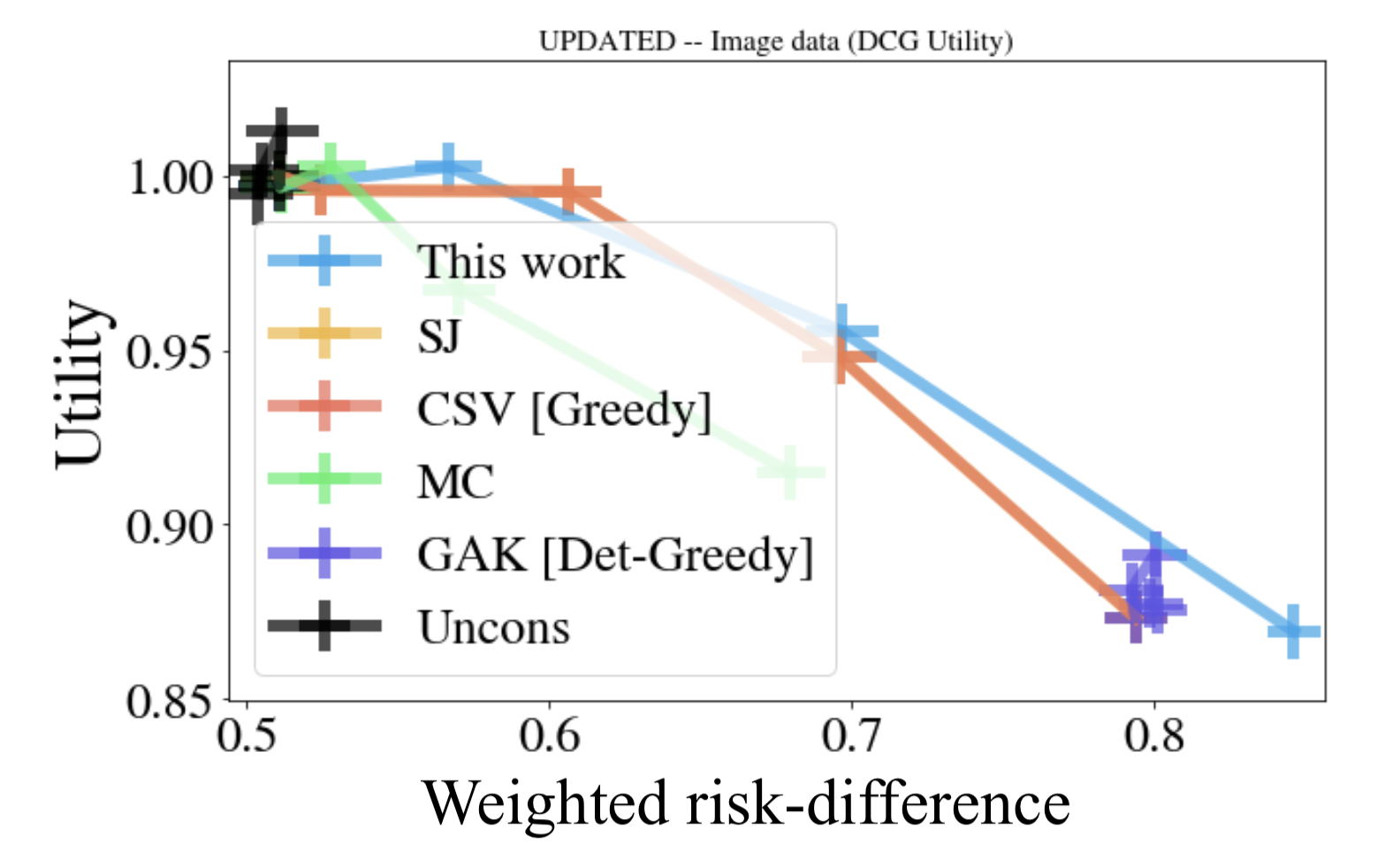}
      }
      \caption{
      Simulation on image data with different values of $n$.
      The details appear in \cref{sec:vary_n}.
      }
      \label{fig:vary_n_image}
    \end{figure}

    \begin{figure}[h!]
      \centering
      \subfigure[$n=10$]{
        \begin{tikzpicture}
          \node (image) at (0,-0.17) {\includegraphics[width=0.40\linewidth, trim={0cm 0cm 0cm 0cm},clip]{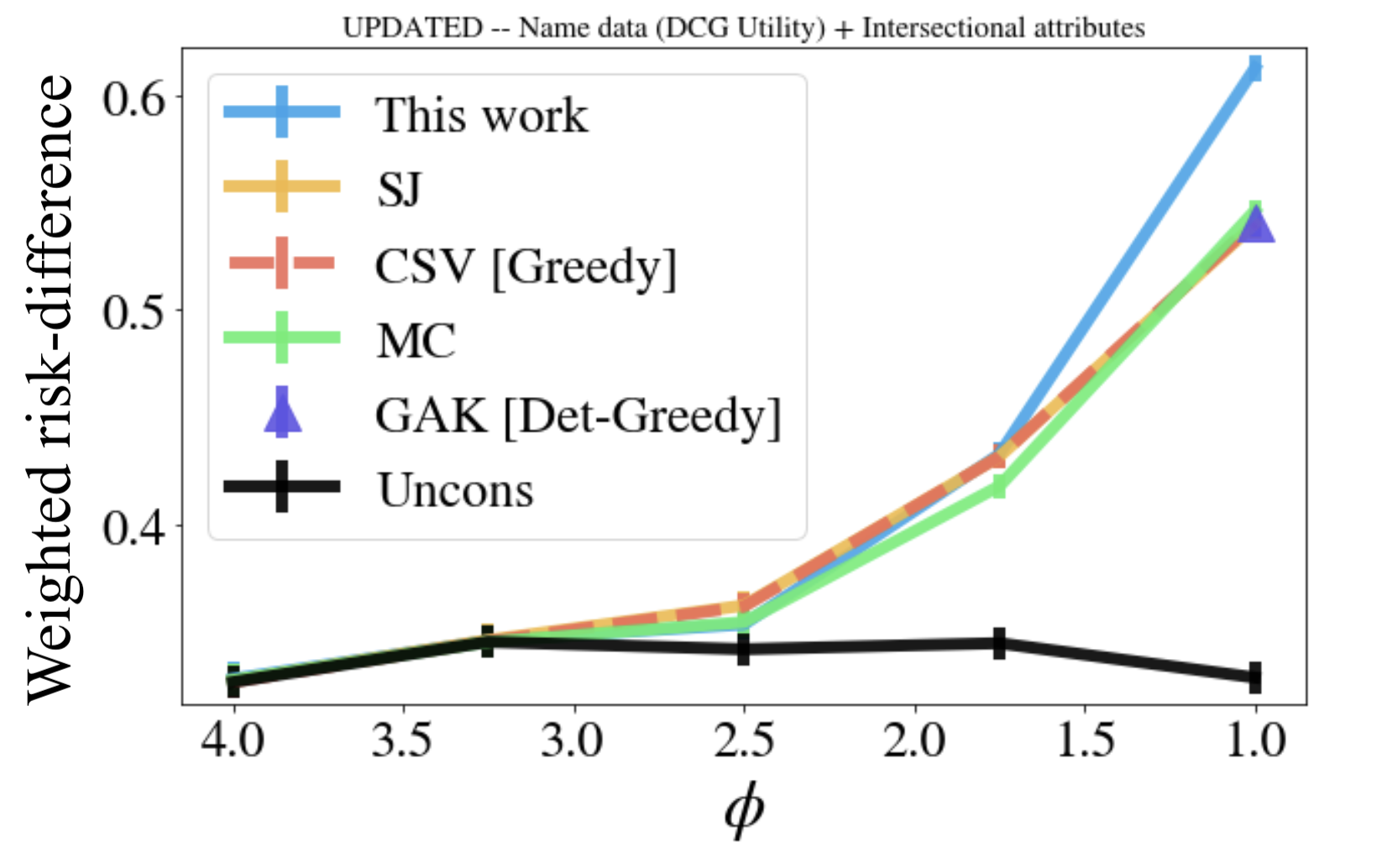}};
          \draw[draw=white, fill=white] (-2.15,1.875) rectangle ++(5,0.55);
        \end{tikzpicture}
      }
      \subfigure[$n=30$]{
        \begin{tikzpicture}
          \node (image) at (0,-0.17) {\includegraphics[width=0.40\linewidth, trim={0cm 0cm 0cm 0cm},clip]{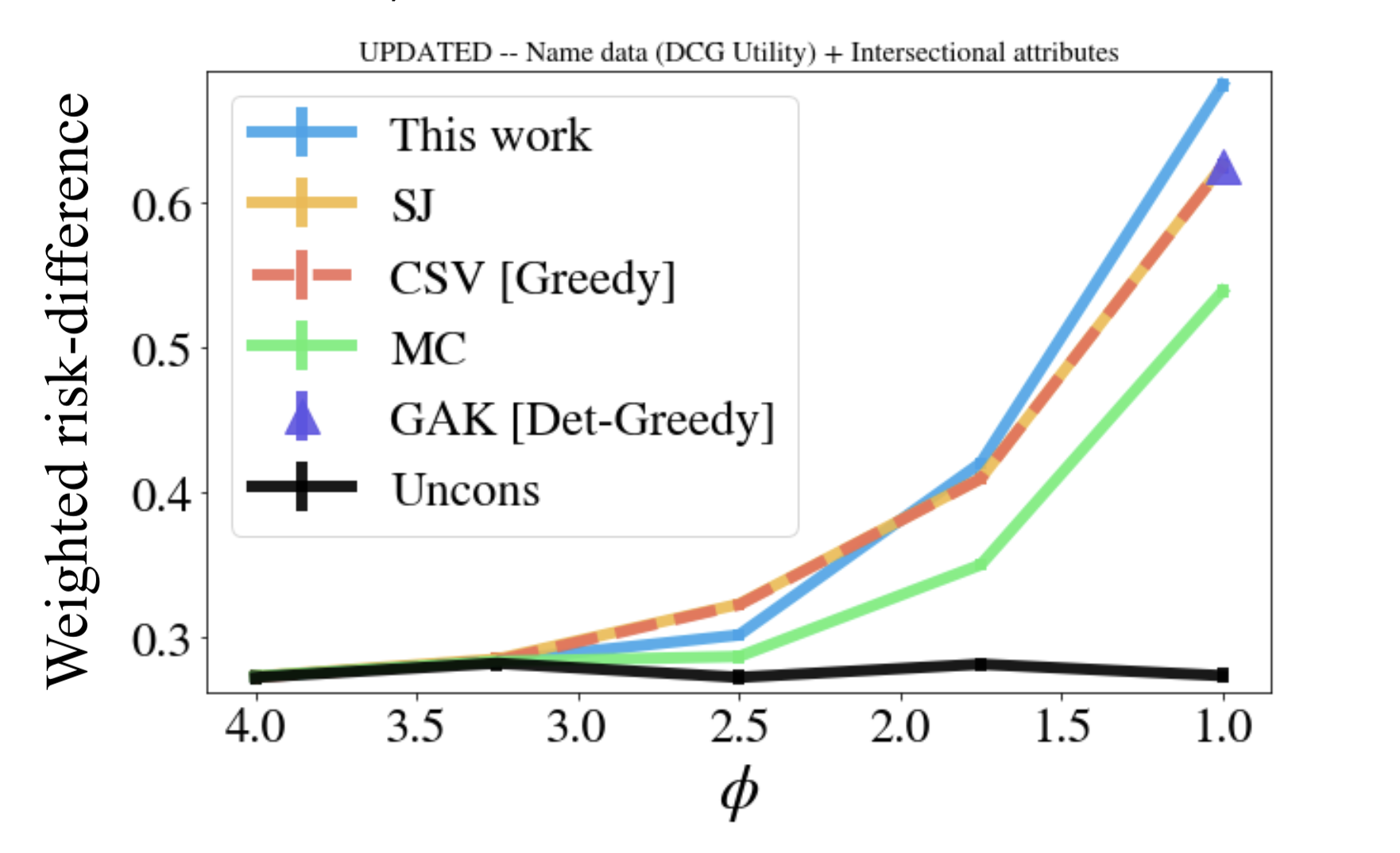}};
          \draw[draw=white, fill=white] (-2.15,1.875) rectangle ++(5,0.55);
        \end{tikzpicture}
      }
      \subfigure[$n=50$]{
        \begin{tikzpicture}
          \node (image) at (0,-0.17) {\includegraphics[width=0.40\linewidth, trim={0cm 0cm 0cm 0cm},clip]{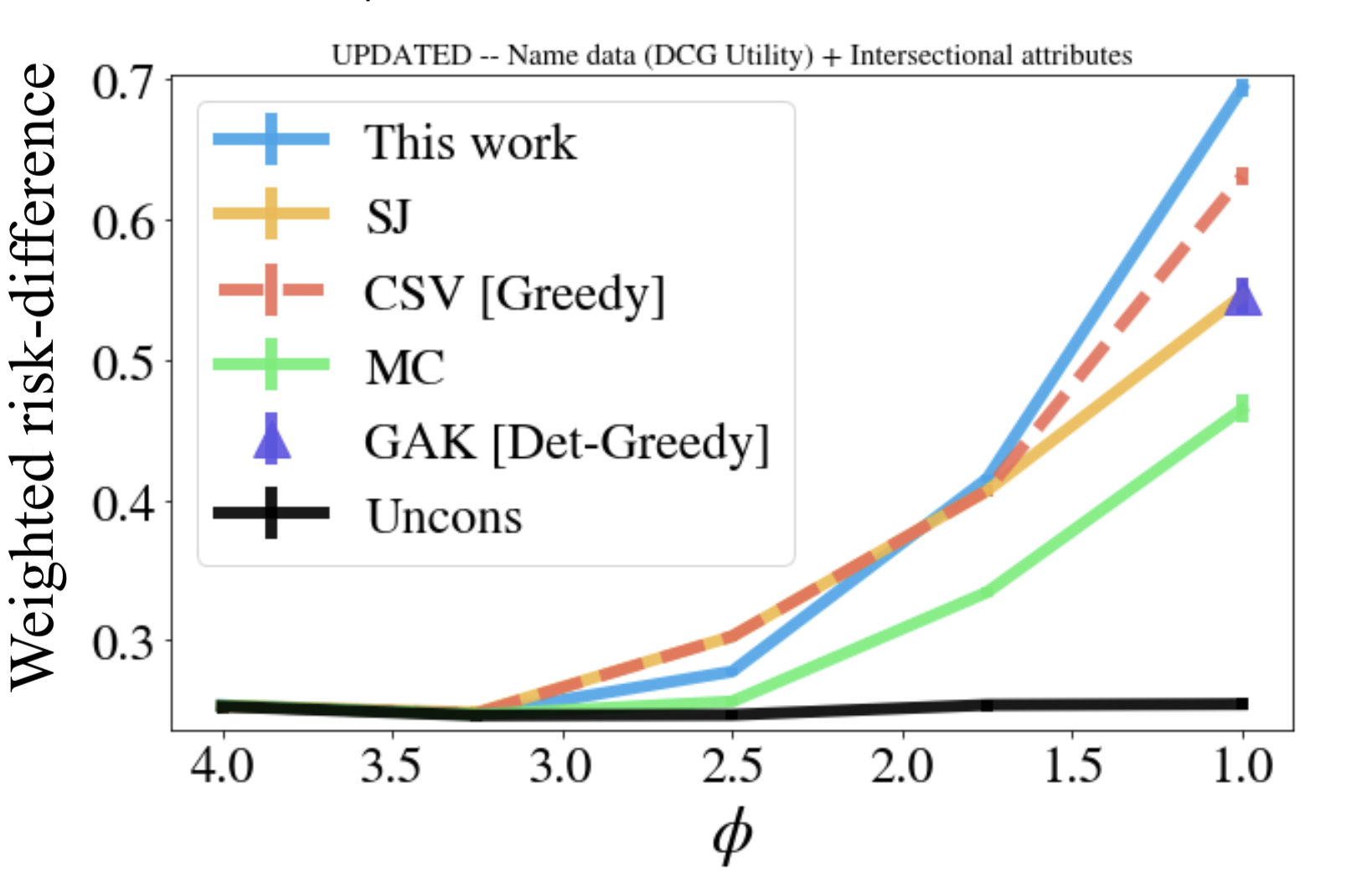}};
          \draw[draw=white, fill=white] (-2.15,1.875) rectangle ++(5,0.55);
        \end{tikzpicture}
      }
      \caption{
      Simulation on real-world name data with different values of $n$.
      The details appear in \cref{sec:vary_n}.
      }
        \label{fig:vary_n_name}
    \end{figure}

    \paragraph{Empirical Results with Real-World Name Data and Overlapping Groups.}\label{sec:overlapping_grps}
        We present a variant of the simulation in \cref{fig:simulation_intersectional} that considers four overlapping groups: The sets of all women players, all male players, all non-White players, and all White players.
        In contrast, the simulation in \cref{fig:simulation_intersectional} uses four disjoint groups: The sets of non-White non-men players, White non-men players, non-White men players, and White men players.

        \noindent {\em Setup.}
        The same setup as the simulation in \cref{fig:simulation_intersectional}.
        The only difference is in estimating $\hP$:
        For each $i$, we estimate $\hP$ as:
        \begin{align*}
            \hP_{i, \text{women}} &= p_{\text{women}}(i),
            \quad & \hP_{i, \text{men}} &= 1-p_{\text{women}}(i),\\
            \hP_{i, \text{non-white}} &= p_{\text{non-white}}(i),
            \quad & \hP_{i, \text{white}} &= 1-p_{\text{non-white}}(i).
        \end{align*}
        Where $p_{\text{women}}(i)$ and $p_{\text{non-white}}(i)$ are values output by Genderize API and EthniColr Library that estimate the probability that player $i$ is labeled as a women and non-white respectively.
        (\csv{} and \detgreedy{} require protected groups to be disjoint and, hence, are not applicable to this simulation.)

        \smallskip
        \noindent {\em Observations.}
        \cref{fig:overlapping} plots RD and utilities (NDCG) averaged over 200 iterations.
        The results are similar to the corresponding simulation on the same data with disjoint groups.
        In particular, compared to other baselines, \ouralgo{} achieves the highest \rd{}. The maximum \rd{} of \ouralgo{} in this simulation is 0.64 compared to 0.67 in \cref{fig:simulation_intersectional}.
        \sj{} achieves the next highest \rd{} followed by \mc{} as in \cref{fig:simulation_intersectional}.

      \renewcommand{\folder}{./figs/overlapping-grps}
      \begin{figure}[t!]
      \centering
      \subfigure[Weighted risk-difference vs. $\phi$]{
        \begin{tikzpicture}
          \node (image) at (0,-0.17) {\includegraphics[width=0.40\linewidth, trim={0cm 0cm 0cm 0cm},clip]{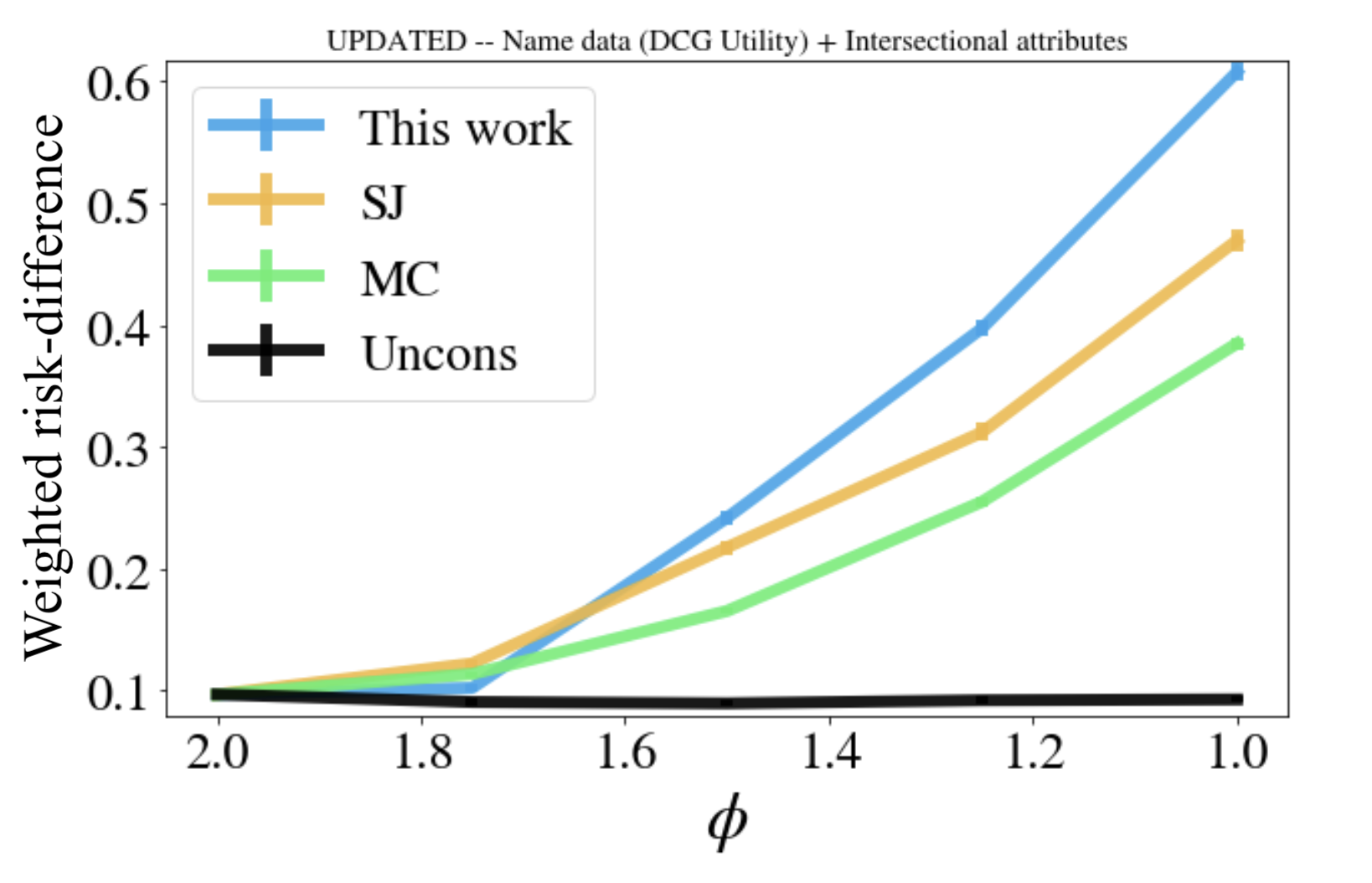}};
          \draw[draw=white, fill=white] (-2.15,1.675) rectangle ++(5,0.25);
        \end{tikzpicture}
      }
      \subfigure[Utility vs. $\phi$]{
        \begin{tikzpicture}
          \node (image) at (0,-0.17) {\includegraphics[width=0.40\linewidth, trim={0cm 0cm 0cm 0cm},clip]{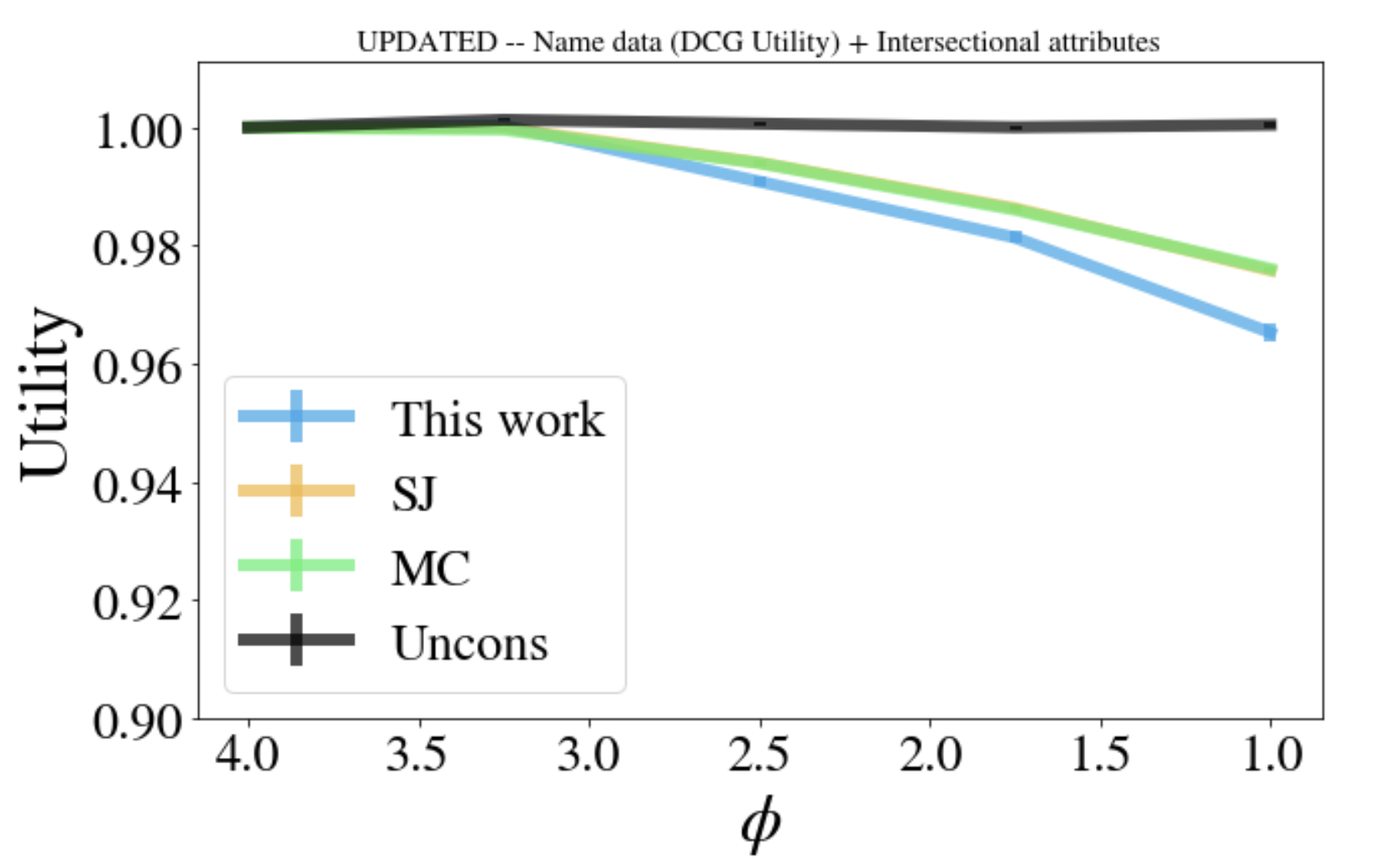}};
          \draw[draw=white, fill=white] (-2.15,1.6) rectangle ++(5,0.25);
        \end{tikzpicture}
      }
      \subfigure[Weighted risk-difference vs. $\phi$]{
        \begin{tikzpicture}
          \node (image) at (0,-0.17) {\includegraphics[width=0.40\linewidth, trim={0cm 0cm 0cm 0cm},clip]{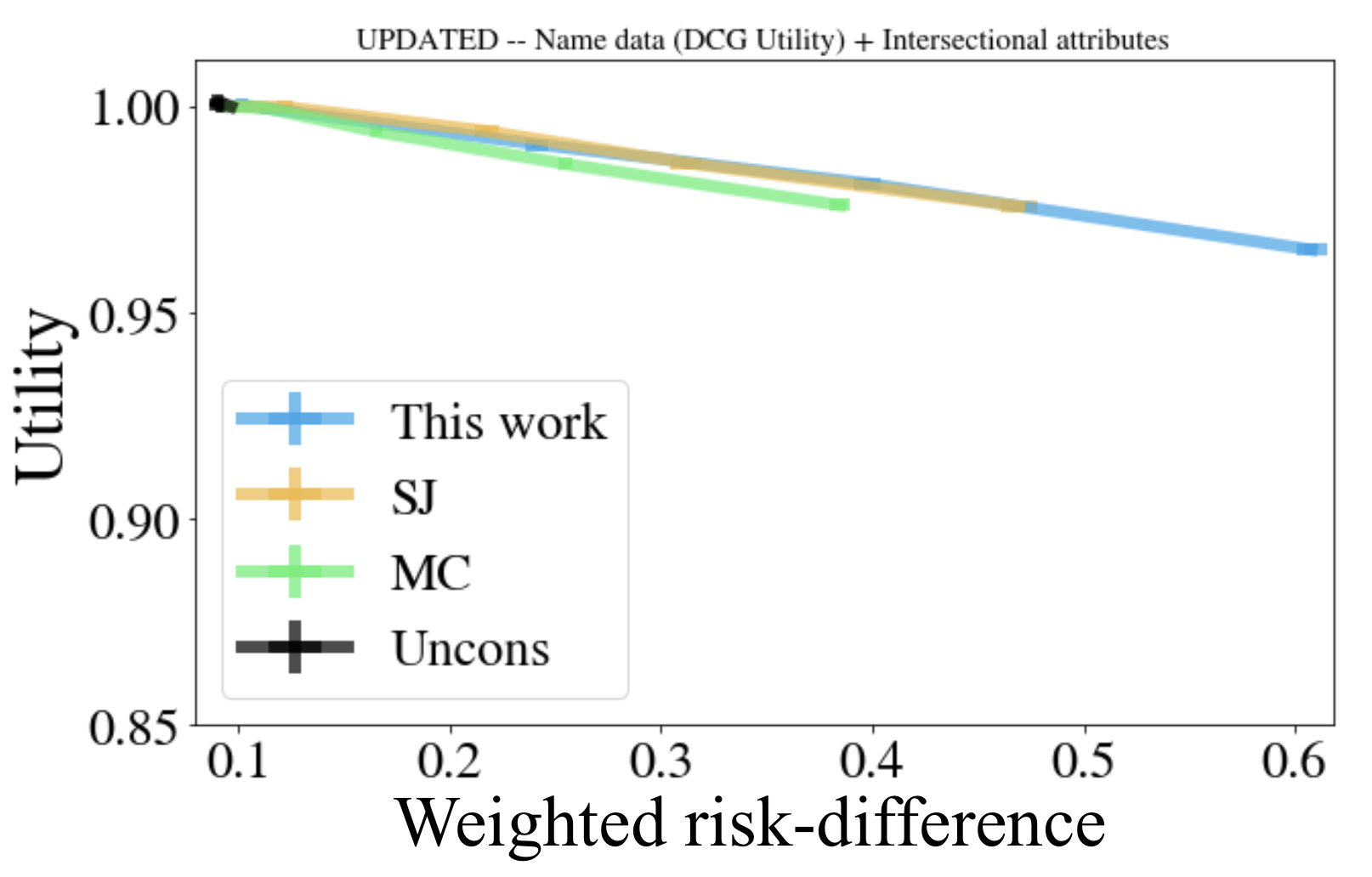}};
          \draw[draw=white, fill=white] (-2.15,1.65) rectangle ++(5,0.25);
        \end{tikzpicture}
      }
      \caption{
      Simulation with the real-world name data and overlapping groups.
      The details appear in \cref{sec:overlapping_grps}.
      }
      \label{fig:overlapping}
    \end{figure}

    \renewcommand{\folder}{./figs/syn-rd}
    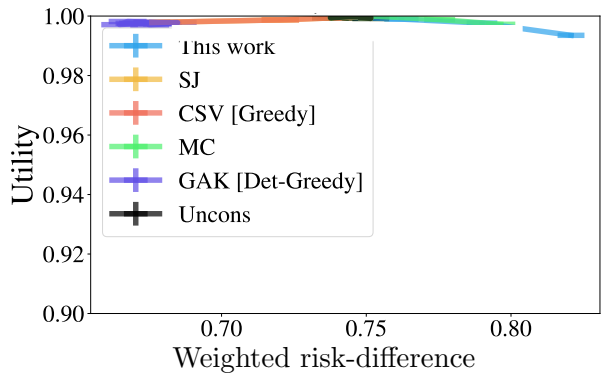
\begin{figure}[h!]
      \centering
      \begin{tikzpicture}[scale=1.5]
        \tikzmath{\s = 0.9;}
        \tikzmath{\mvx = 0;}
        \tikzmath{\mvy = 0;}
        \node (image) at (0,-0.17*\s) {\includegraphics[width=0.5\linewidth, trim={0cm 0cm 1.4cm 0.9cm},clip]{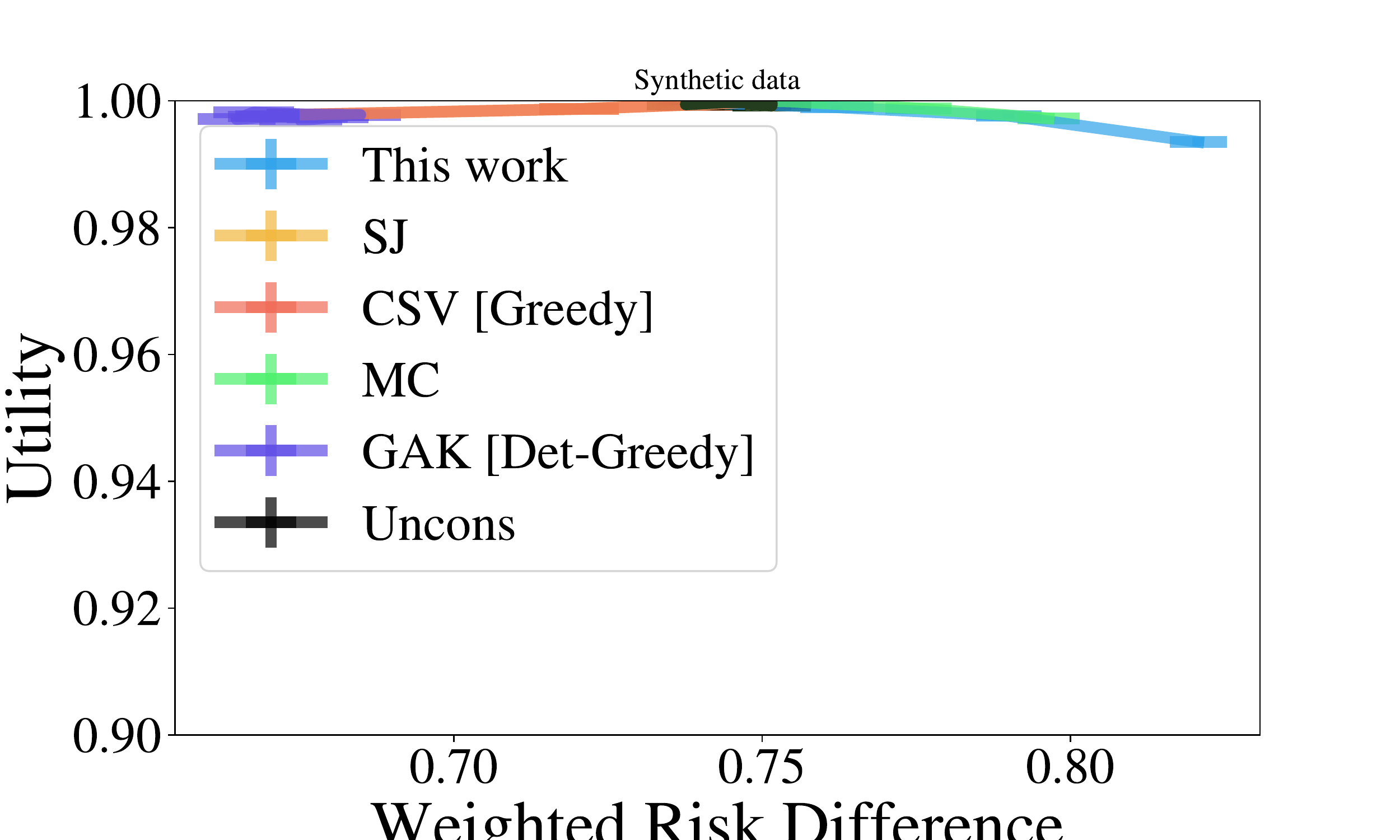}};
        \draw[draw=white, fill=white] (-2*\s+0.5, 3.3*5/8*\s-0.66*\s-0.17*\s+0.35-0.4) rectangle ++(4*\s-0.5,0.15*\s);
        \draw[draw=white, fill=white] (+\mvx-2*\s+0.35,1.2325*\s+\mvy+0.2) rectangle ++(4*\s,0.15*\s+0.14);
        \node[rotate=0, fill=white] at (+\mvx+0.05*\s, -3.3*5/8*\s+0.05*\s-0.255+\mvy-0.2+0.5)  {\white{||||}Weighted risk-difference\white{||||}};
      \end{tikzpicture}
      \caption{
          {\em {Synthetic Data: Nonuniform Error Rate.}}
          This simulation considers synthetic data where imputed socially-salient attributes have a higher false-discovery rate for one group compared to the other.
          We vary the fairness constraint from $\phi$ from $2$ (less fair) to $1$ (more fair) and observe the weighted risk-difference (weighted risk-difference) of different algorithms.
          The $y$-axis plots utility and $x$-axis shows weighted risk-difference ({\em Note that the values decrease toward the right}).
          Error-bars denote the error of the mean.
      }
      \label{fig:fig1bskjkj}
    \end{figure}

    \renewcommand{\folder}{./figs/image-rd}
    \begin{figure}[h!]
      \centering
      \begin{tikzpicture}[scale=1.5]
        \tikzmath{\s = 0.9;}
        \tikzmath{\mvx = 0;}
        \tikzmath{\mvy = 0;}
        \node (image) at (0,-0.17*\s) {\includegraphics[width=0.5\linewidth, trim={0cm 0cm 1.4cm 0.9cm},clip]{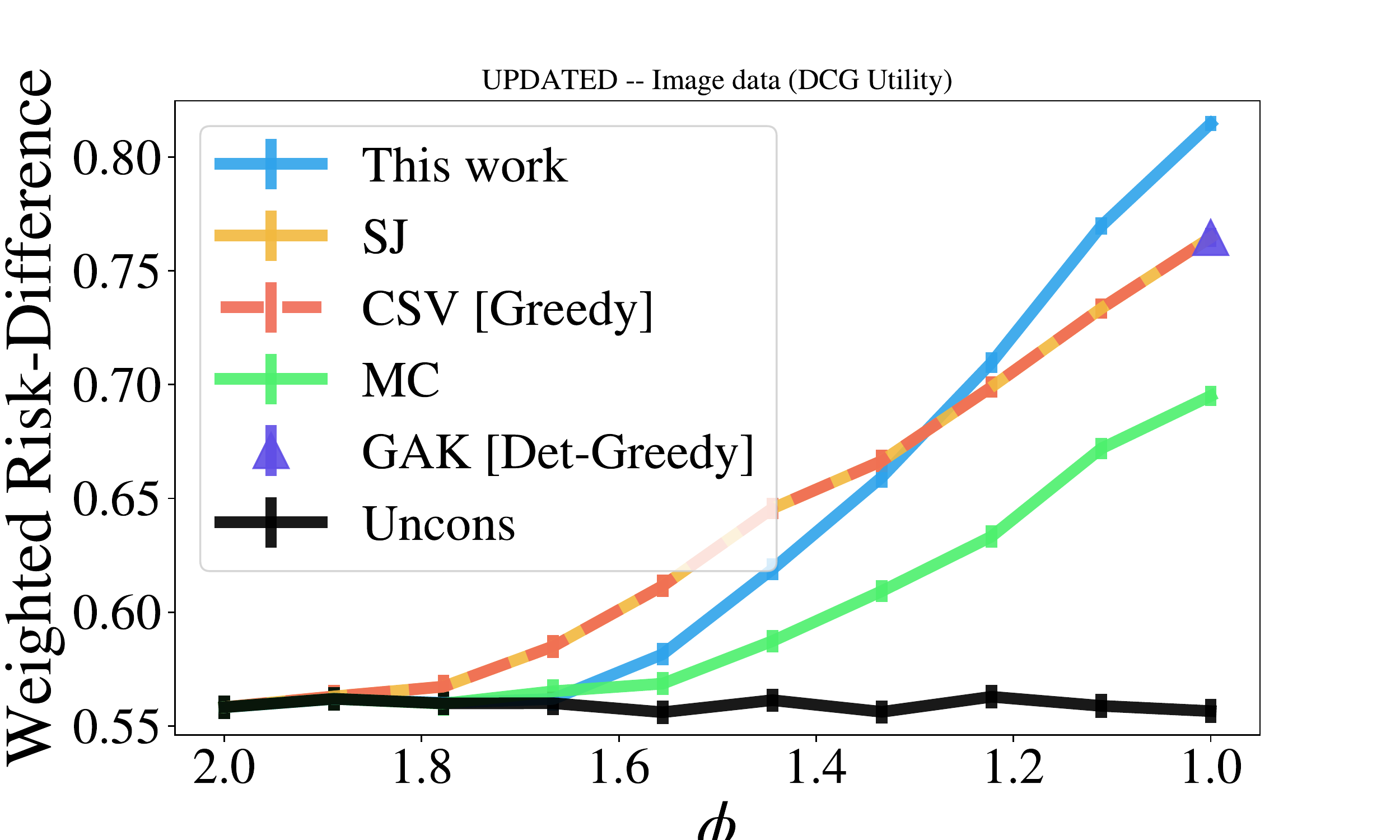}};
        \draw[draw=white, fill=white] (-2*\s+0.5, 3.3*5/8*\s-0.66*\s-0.17*\s+0.35-0.38) rectangle ++(4*\s-0.5,0.15*\s);
        \draw[draw=white, fill=white] (-2*\s+0.35,1.2325*\s+0.2) rectangle ++(4*\s,0.15*\s+0.14);
        \node[rotate=0, fill=white] at (0.2+0.05*\s, 0.15-3.3*5/8*\s-0.1*\s+0.15*\s-0.16)  {\white{|||}$\phi$ \white{|||}};
      \end{tikzpicture}
      \caption{
          {\em Real-World Image Data.}
          This simulation considers images-search results which are known to overrepresent the stereotypical gender~\cite{KayMM15}.
          Given relevant {\em non-gender labeled} images and their utilities, our goal is to generate a high-utility gender-balanced ranking.
          We estimate $P$ using an off-the-shelf ML-classifier and vary $\phi$ from $p=2$ (less fair) to $1$ (more fair).
          In the first subfigure, the $y$-axis plots weighted risk-difference and $x$-axis shows $\phi$. ({\em Note that the values decrease toward the right}.)
          Error bars show the error of the mean.
      }
      \label{fig:fig2b}
    \end{figure}

    \renewcommand{\folder}{./figs/chess-rd}
    \begin{figure}[h!]
      \centering
      \begin{tikzpicture}[scale=1.5]
        \tikzmath{\s = 0.9;}
        \node (image) at (0,-0.17*\s) {\includegraphics[width=0.5\linewidth, trim={0cm 0cm 1.5cm 0.6cm},clip]{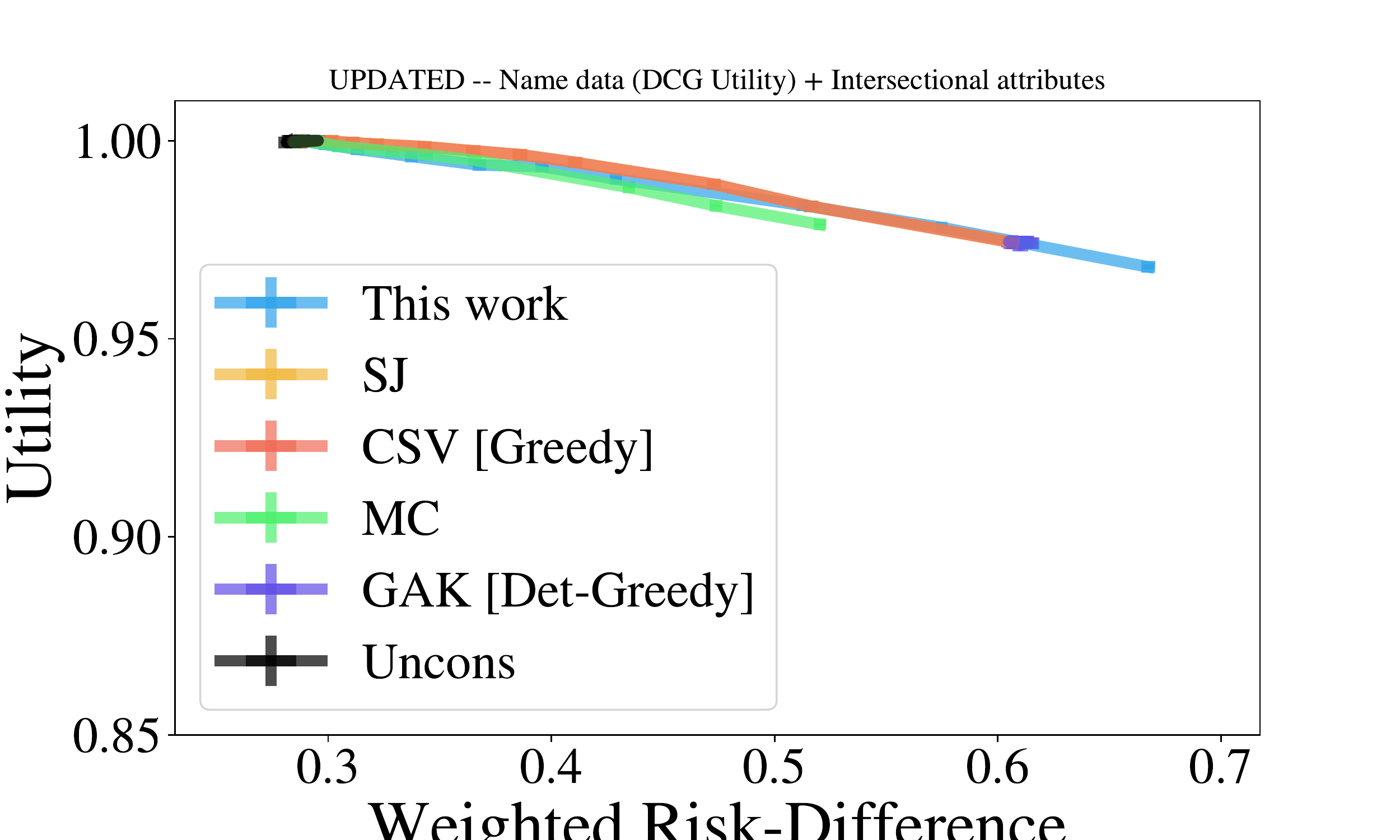}};
        \node[rotate=0, fill=white] at (0.05*\s+0.30, -3.3*5/8*\s+0.05*\s-0.05)  {\white{||||}Weighted risk-difference\white{||||}};
        \draw[draw=white, fill=white] (-2*\s+0.35,1.2325*\s+0.15) rectangle ++(4*\s,0.15*\s+0.14);
      \end{tikzpicture}
      \caption{
      {\em {Real-World Name Data: Intersectional Attributes.}}
      This simulation considers two socially-salient attributes, gender and race, and our goal is to ensure equal representation across the four {\em intersectional} socially-salient groups (non-White non-men, White non-men, non-White men, and White men).
      We estimate $P$ from the full names using public APIs and libraries.
      We vary $\phi$ from $p=4$ (less fair) to $1$ (more fair) and observe weighted risk-difference of all algorithms.
      The $y$-axis plots utility and $x$-axis shows weighted risk-difference. %
      ({\em Note that the values decrease toward the right}.)
      Error bars represent the error of the mean.
      }
      \label{fig:fig6}
    \end{figure}

    \renewcommand{\folder}{./figs/small-iters}
    \begin{figure}[h!]
      \centering
      \subfigure[25 iterations]{
        \begin{tikzpicture}
          \node (image) at (0,-0.17) {{\includegraphics[width=0.40\linewidth, trim={0cm 0cm 0cm 0cm},clip]{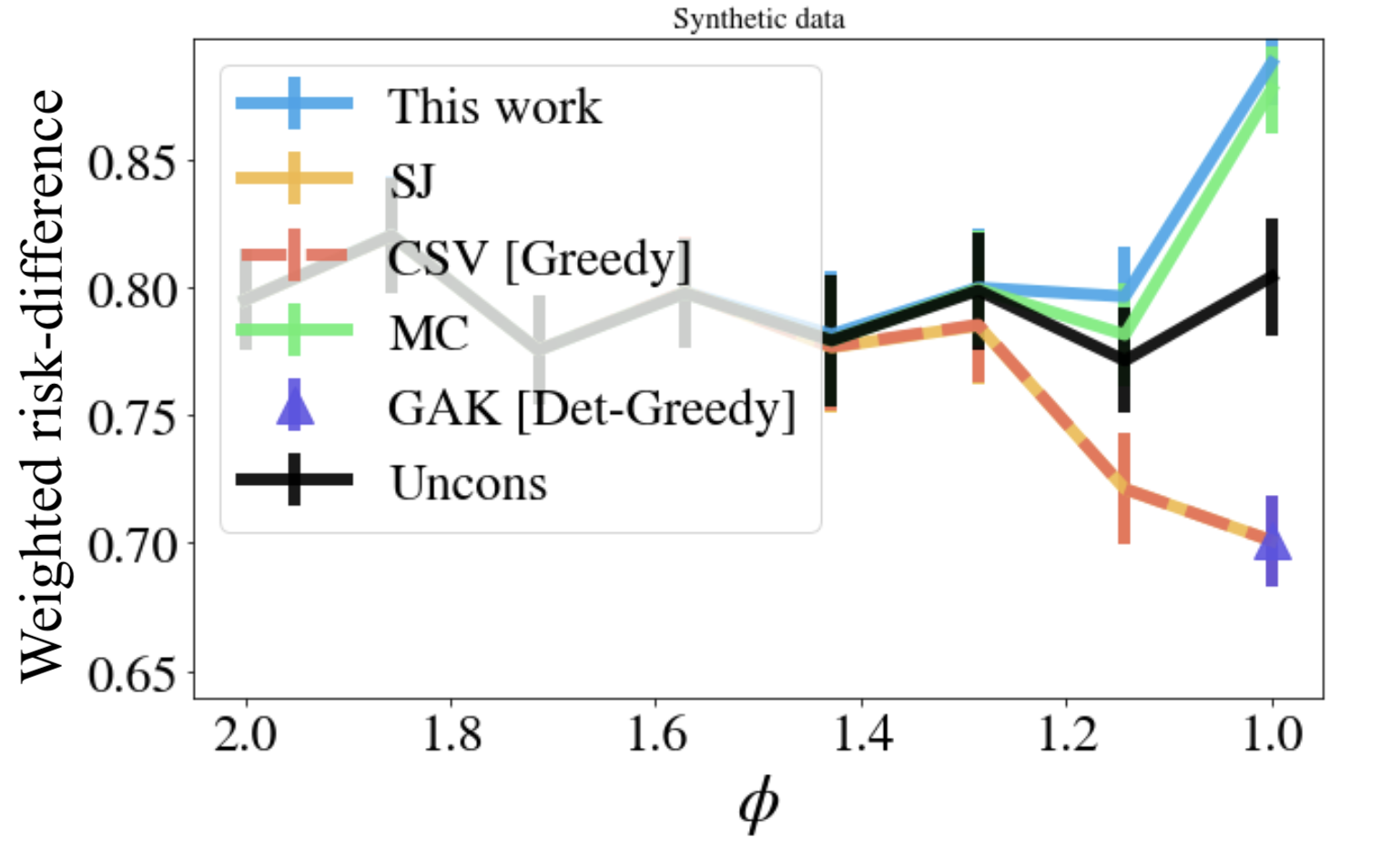}}};
          \draw[draw=white, fill=white] (-2.15,1.875) rectangle ++(5,0.55);
        \end{tikzpicture}
      }
      \subfigure[50 iterations]{
        \begin{tikzpicture}
          \node (image) at (0,-0.17) {\includegraphics[width=0.40\linewidth, trim={0cm 0cm 0cm 0cm},clip]{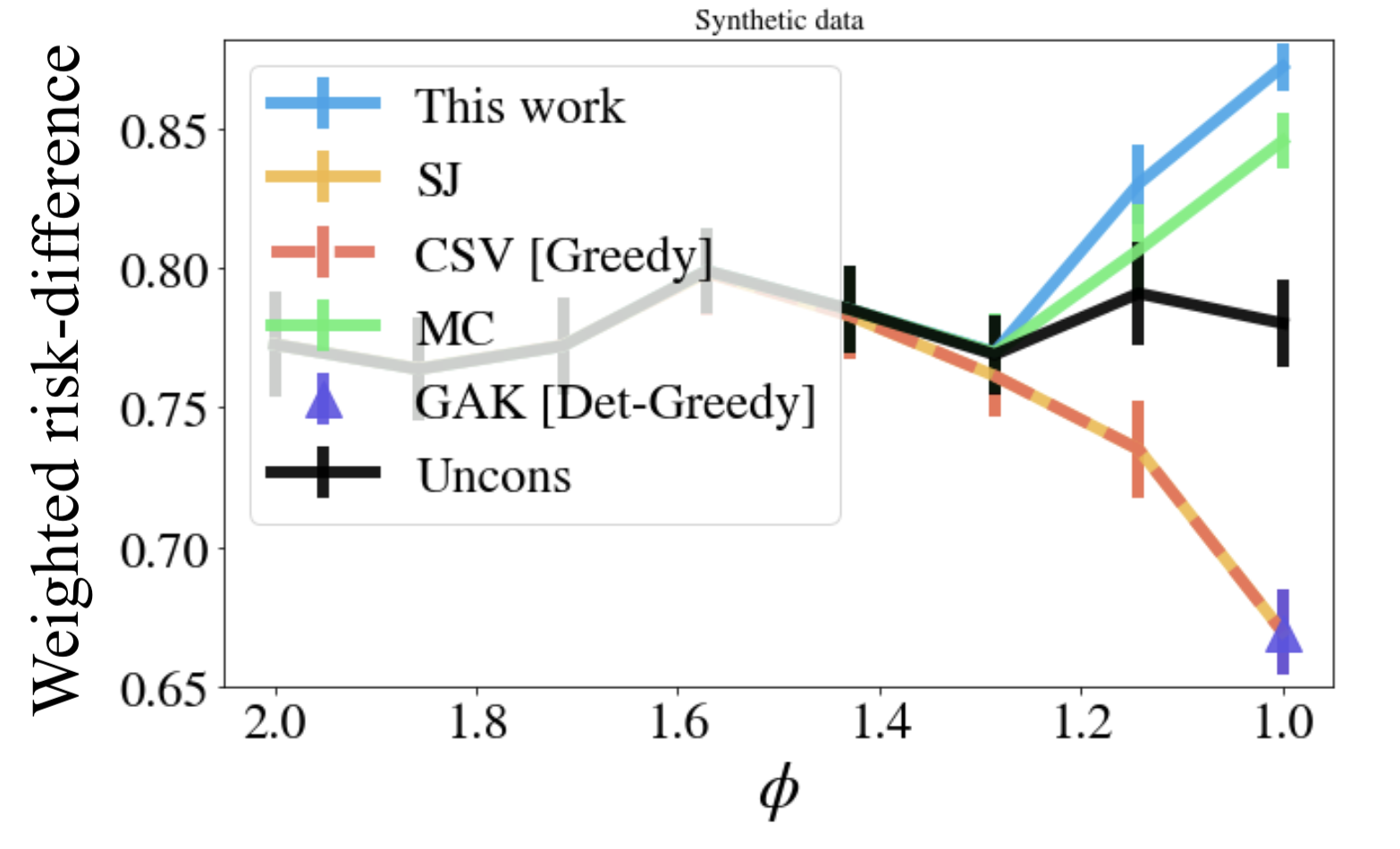}};
          \draw[draw=white, fill=white] (-2.15,1.875) rectangle ++(5,0.55);
        \end{tikzpicture}
      }
      \subfigure[100 iterations]{
        \begin{tikzpicture}
          \node (image) at (0,-0.17) {\includegraphics[width=0.40\linewidth, trim={0cm 0cm 0cm 0cm},clip]{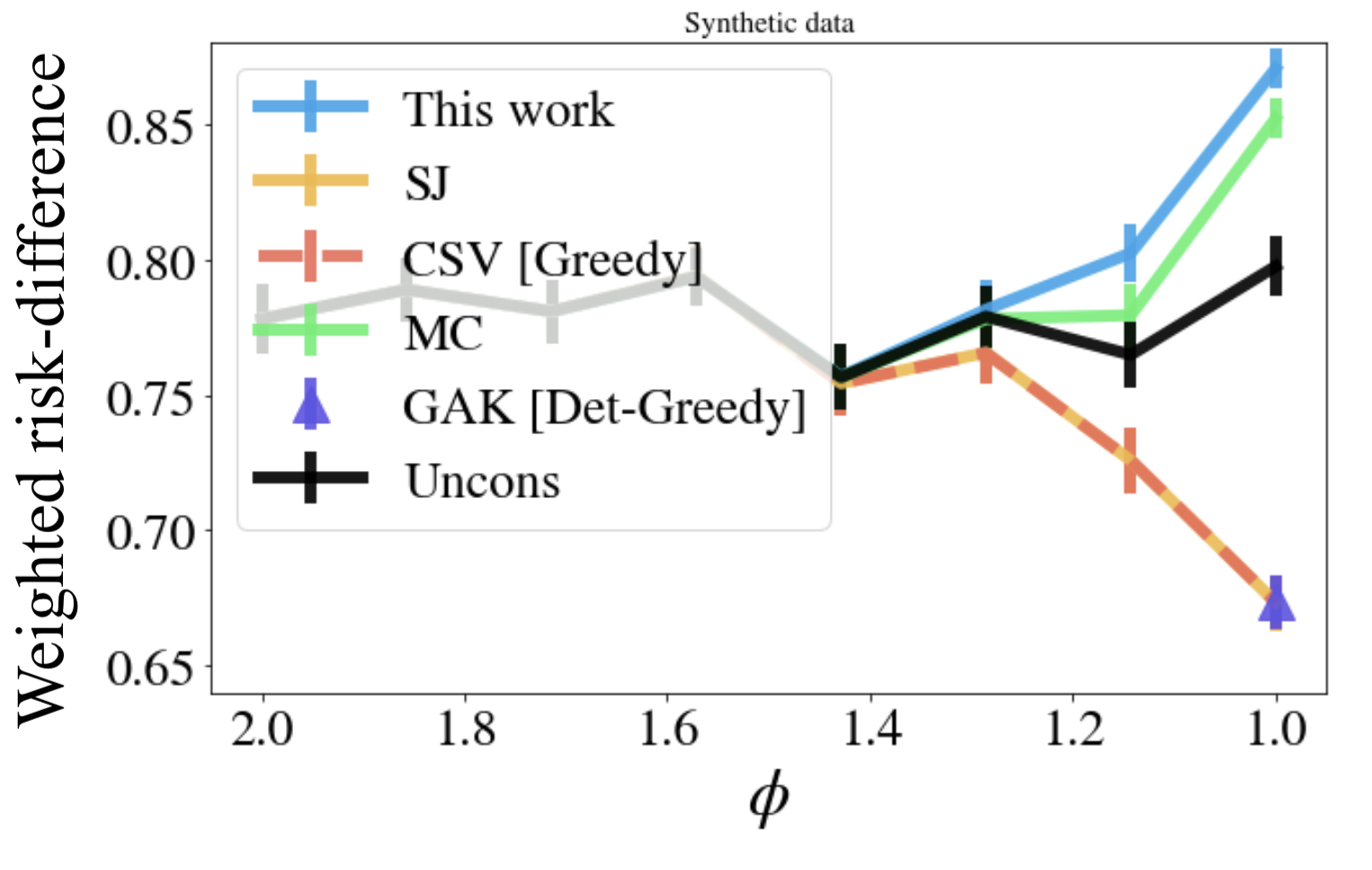}};
          \draw[draw=white, fill=white] (-2.15,1.875) rectangle ++(5,0.55);
        \end{tikzpicture}
      }
      \caption{
      Simulations on synthetic data from \cref{sec:empirical_results} with 25, 50, and 100 iterations.
      The details appear in \cref{sec:add_plot}.
      }
      \label{fig:small_iter_syn}
    \end{figure}
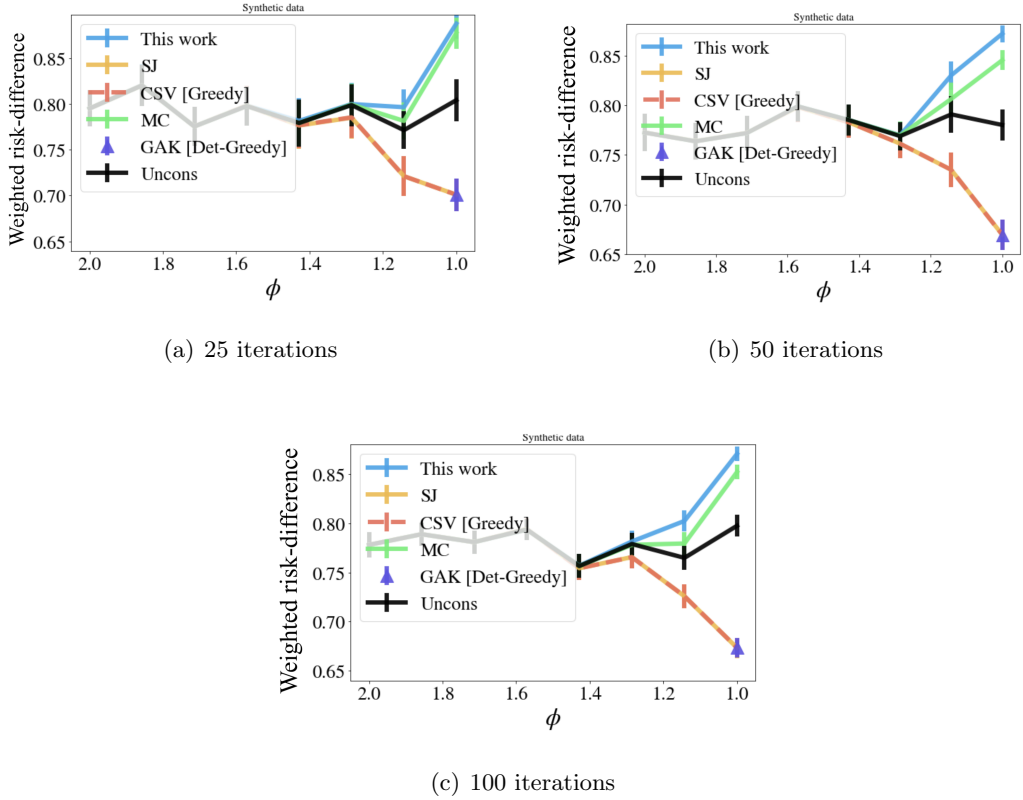

    \begin{figure}[h!]
      \centering
      \vspace{-15.0mm}
      \subfigure[25 iterations]{
        \begin{tikzpicture}
          \node (image) at (0,-0.17) {\includegraphics[width=0.40\linewidth, trim={0.5cm 0cm 2.5cm 0.5cm},clip]{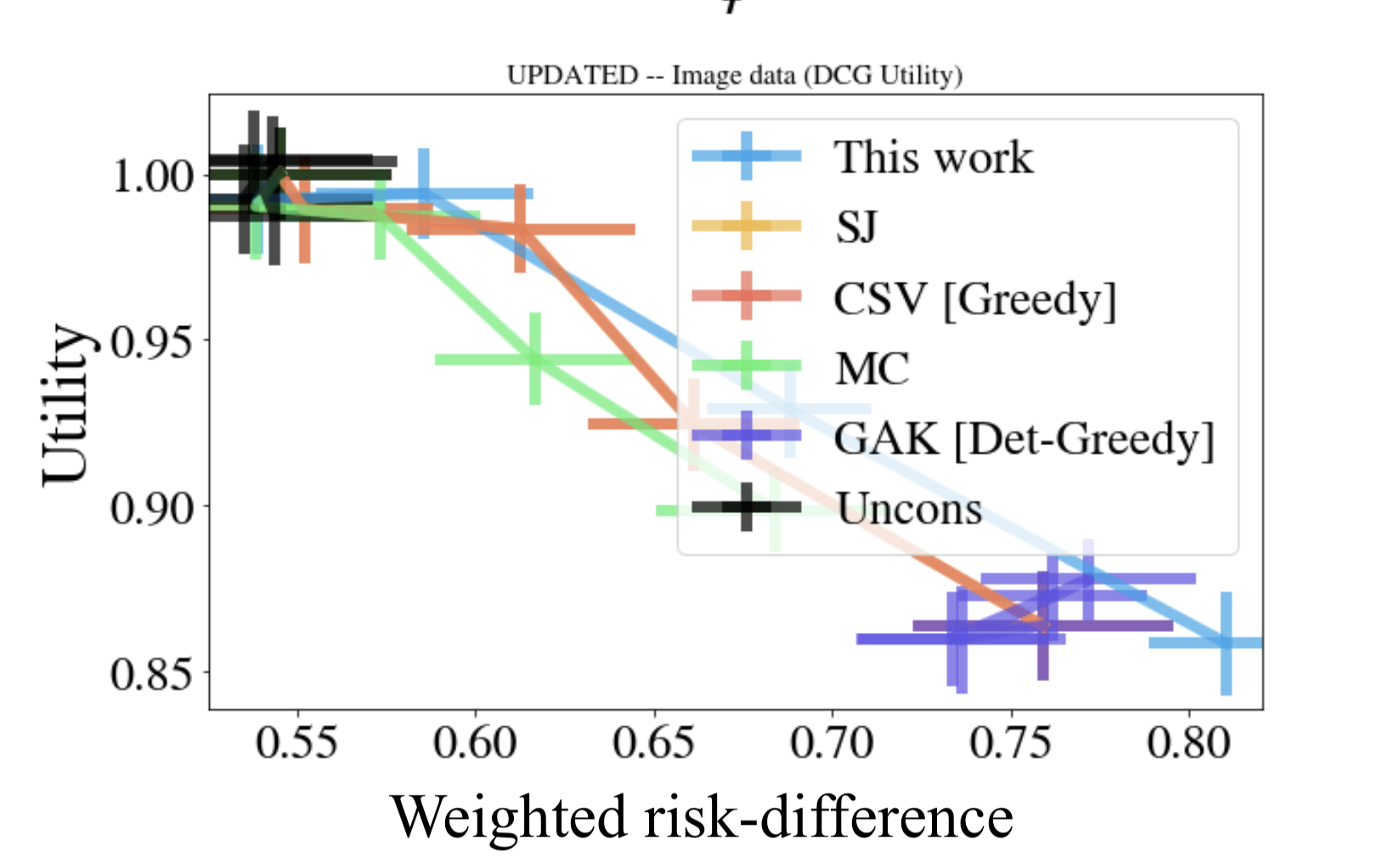}};
          \draw[draw=white, fill=white] (-2.15,1.67) rectangle ++(5,0.3);
        \end{tikzpicture}
      }
      \subfigure[50 iterations]{
        \begin{tikzpicture}
          \node (image) at (0,-0.17) {\includegraphics[width=0.40\linewidth, trim={0.5cm 0cm 2.2cm 0.5cm},clip]{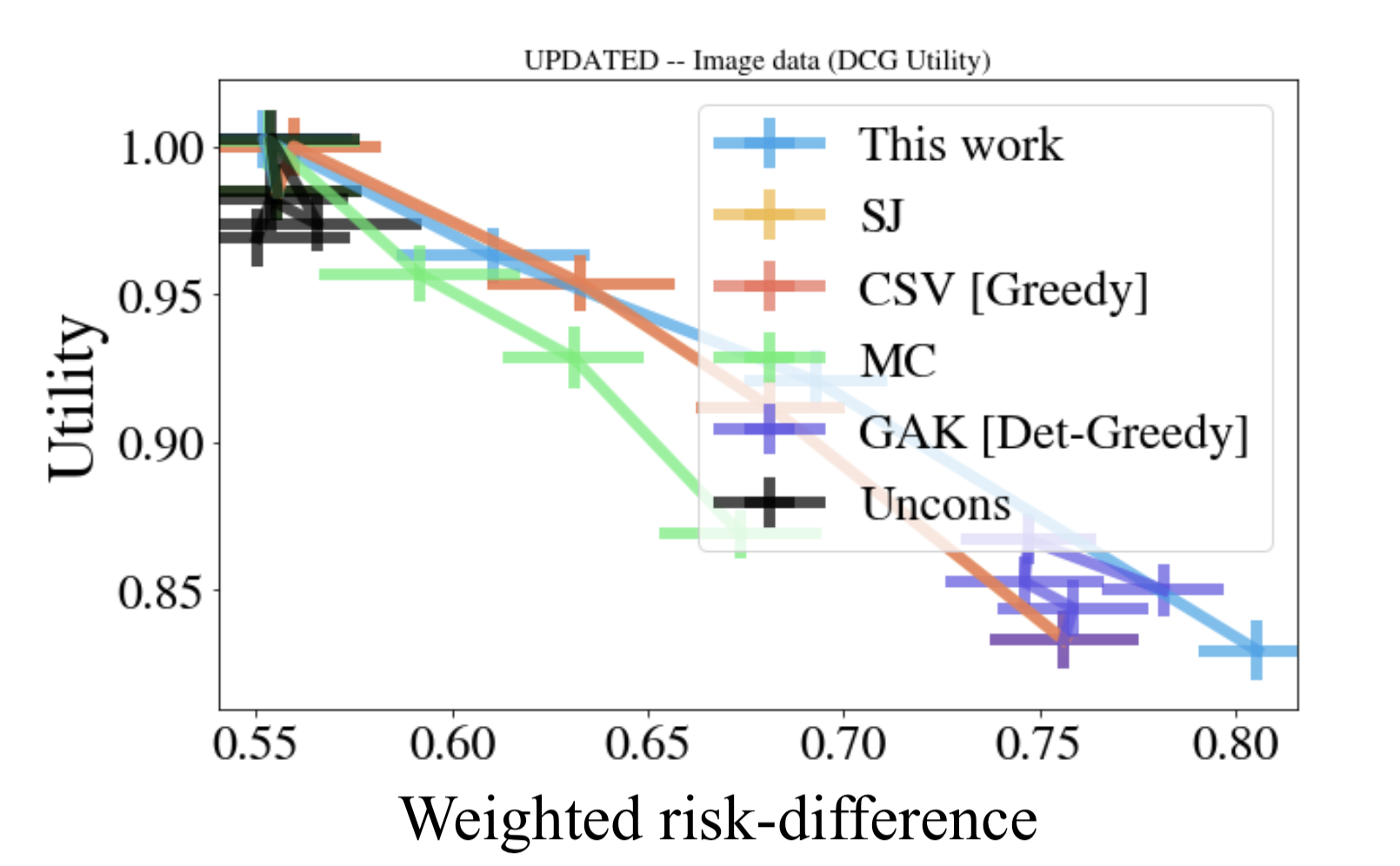}};
          \draw[draw=white, fill=white] (-2.15,1.75) rectangle ++(5,0.3);
        \end{tikzpicture}
      }
      \par \vspace{-5mm}
      \subfigure[100 iterations]{
        \begin{tikzpicture}
          \node (image) at (0,-0.17) {\includegraphics[width=0.42\linewidth, trim={0.5cm 0cm 2.2cm 0cm},clip]{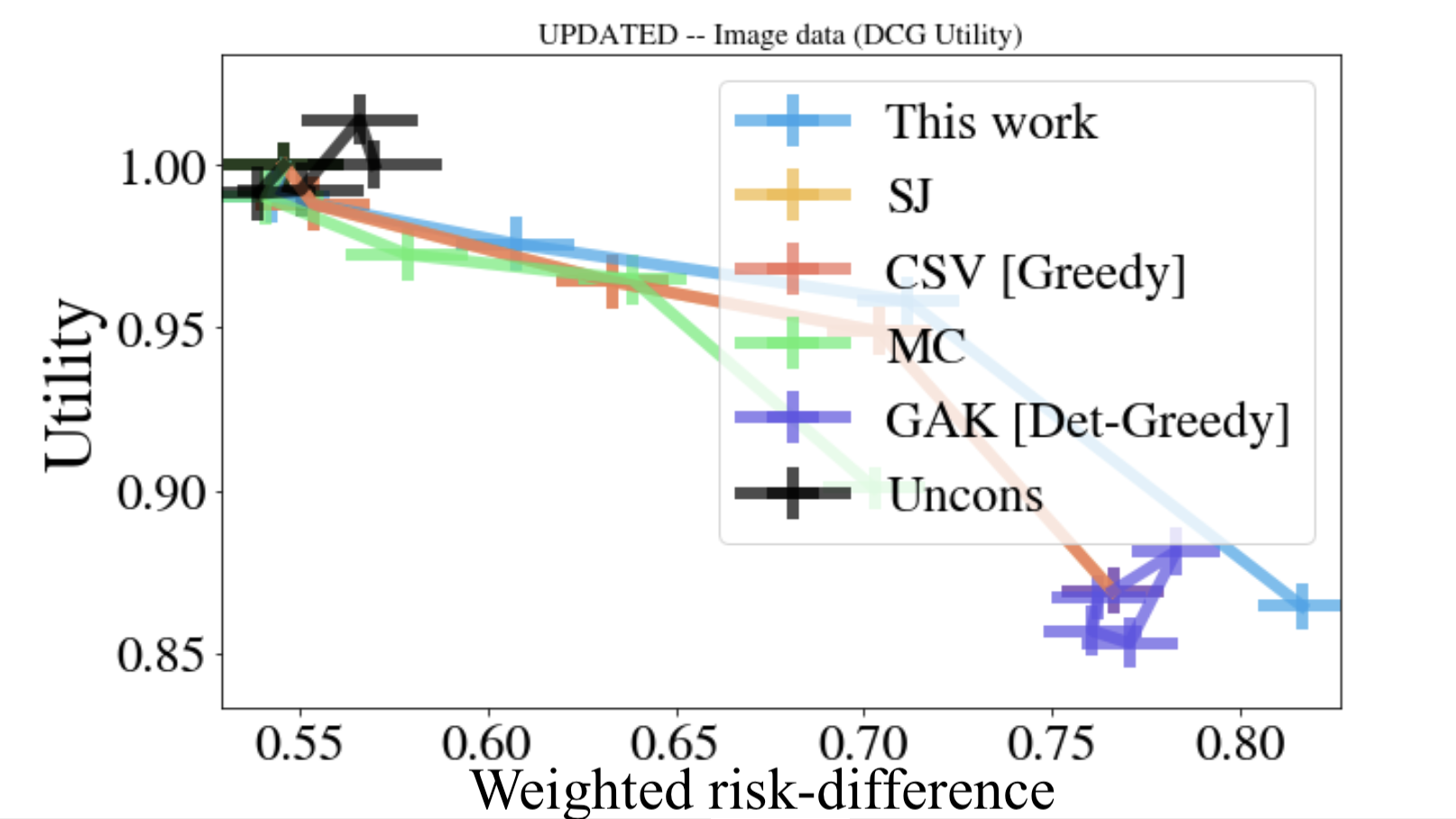}};
          \draw[draw=white, fill=white] (-2.15,1.72) rectangle ++(5,0.3);
        \end{tikzpicture}
      }
      \caption{
      Simulations on image data from \cref{sec:empirical_results} with 25, 50, and 100 iterations.
      The details appear in \cref{sec:add_plot}.
      }
        \label{fig:small_iter_image}
      \vspace{-2mm}
    \end{figure}
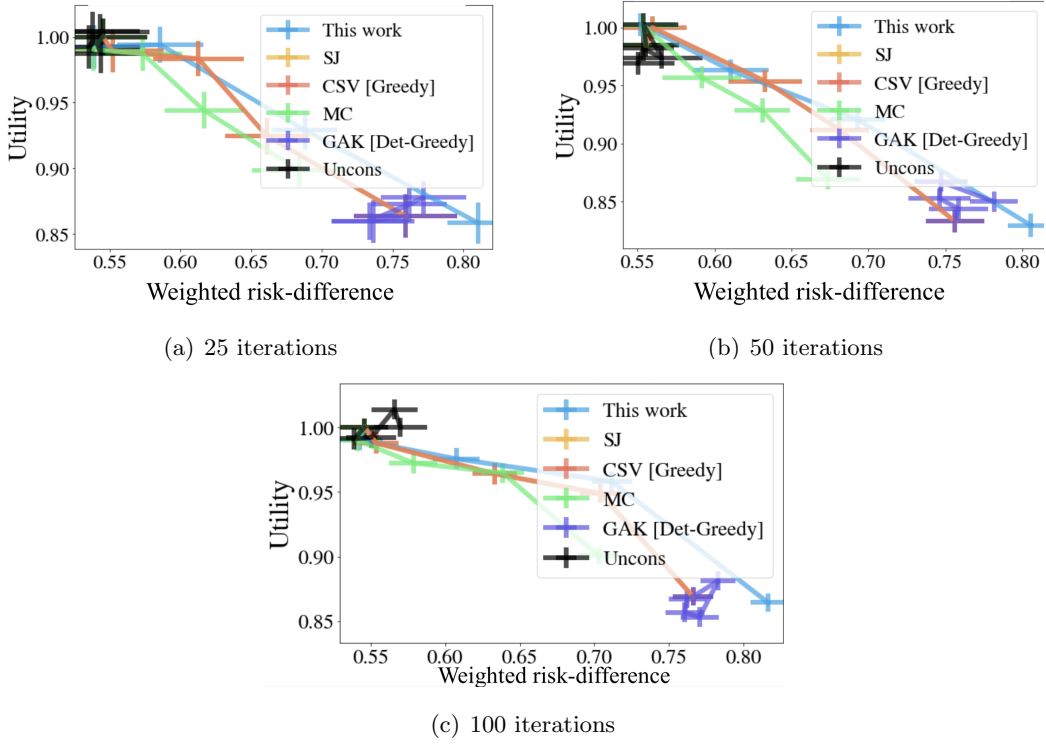

    \begin{figure}[h!]
      \centering
      \vspace{-2mm}
      \subfigure[25 iterations]{
        \begin{tikzpicture};
          \node (image) at (0,-0.17) {\includegraphics[width=0.40\linewidth, trim={0cm 0cm 1.85cm 0cm},clip]{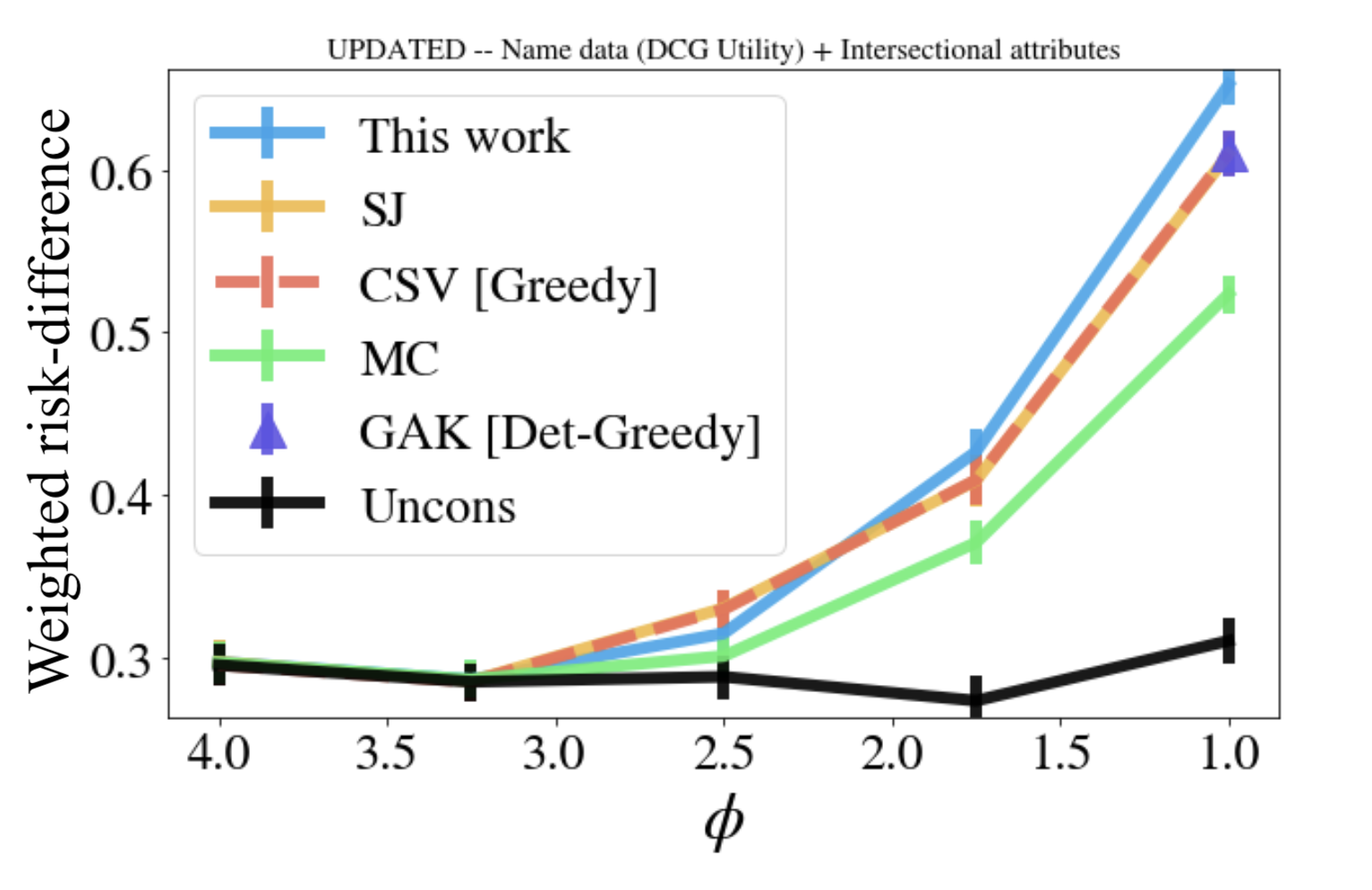}};
          \draw[draw=white, fill=white] (-2.15,1.7) rectangle ++(5,0.3);
        \end{tikzpicture}
        \vspace{-5mm}
      }
      \subfigure[50 iterations]{
        \begin{tikzpicture};
          \node (image) at (0,-0.17) {\includegraphics[width=0.40\linewidth, trim={0cm 0cm 1.2cm 0cm},clip]{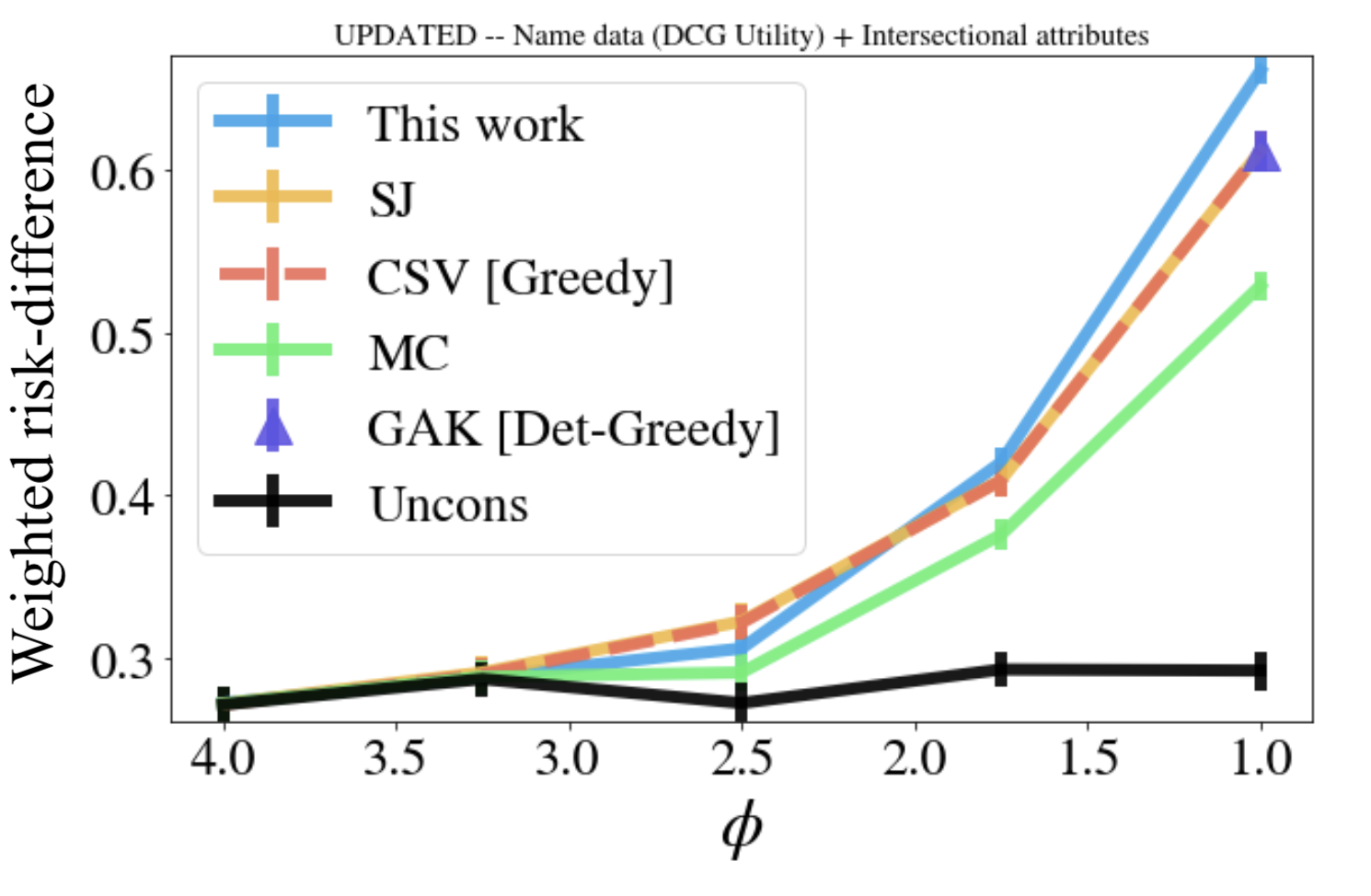}};
          \draw[draw=white, fill=white] (-2.15,1.85) rectangle ++(5,0.3);
        \end{tikzpicture}
        \vspace{-5mm}
      }
      \par \vspace{-3mm}
      \subfigure[100 iterations]{
          \begin{tikzpicture}
            \node (image) at (0,-0.17) {\includegraphics[width=0.40\linewidth, trim={0cm 0cm 1.3cm 0cm},clip]{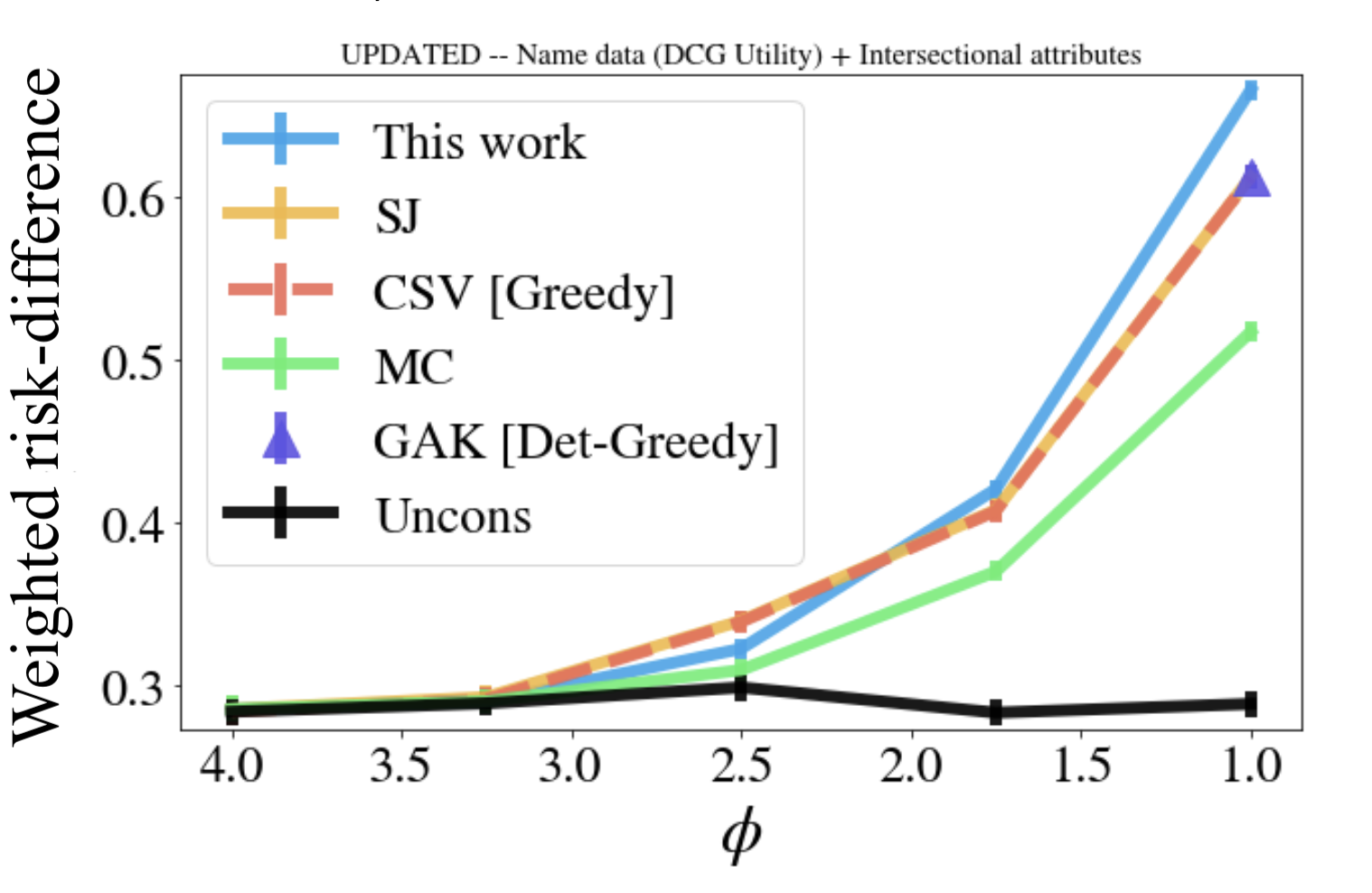}};
            \draw[draw=white, fill=white] (-2.15,1.72) rectangle ++(5,0.3);
          \end{tikzpicture}
          \vspace{-5mm}
      }
      \vspace{-5mm}
      \caption{
        Simulations on real-world name data from \cref{sec:empirical_results} with 25, 50, and 100 iterations.
        The details appear in \cref{sec:add_plot}.
      }
      \vspace{-5mm}
      \label{fig:small_iter_name}
    \end{figure}

    \subsection{Additional Empirical Results}\label{sec:selection_lift}
    \subsubsection{Empirical Results with Weighted Selection-Lift}

    In this section, we present empirical results with the weighted selection-lift fairness metric (\cref{fig:fig7,fig:fig8,fig:fig9}).
    Weighted selection-lift is a position-weighted version of the standard selection-difference metric.
    Like weighted risk-difference, it also measures the extent to which a ranking violates equal representation.
    The weighted selection-lift of a ranking $R$ is:
    \begin{align*}
      { \frac1Z\sum_{\scriptsize k=5,10,\dots} \frac{1}{\log{k}}
      \min\limits_{\ell,q\in [p]}\abs{
      \frac{\sum_{\substack{i\in G_\ell,\ j\in [k]\white{,}}} \hspace{-1mm}  R_{ij}}
      {\sum_{\substack{i\in G_q,\  j\in [k]\white{,}}} \hspace{-1mm}  R_{ij} }
      },}
    \end{align*}
    Where $G$ denotes the ground-truth protected groups and $Z$ is a constant so that $\rd$ has range $[0,1]$.
    Here, a value of $1$ is most fair and $0$ is least fair.

    \renewcommand{\folder}{./figs/syn-sl}
    \begin{figure}[h!]
      \centering
      \subfigure[\label{7a}]{
      {\begin{tikzpicture}[scale=1.5]
        \tikzmath{\s = 0.9;}
        \node (image) at (0,-0.17*\s) {\includegraphics[width=0.46\linewidth, trim={0cm 0cm 2cm 0.9cm},clip]{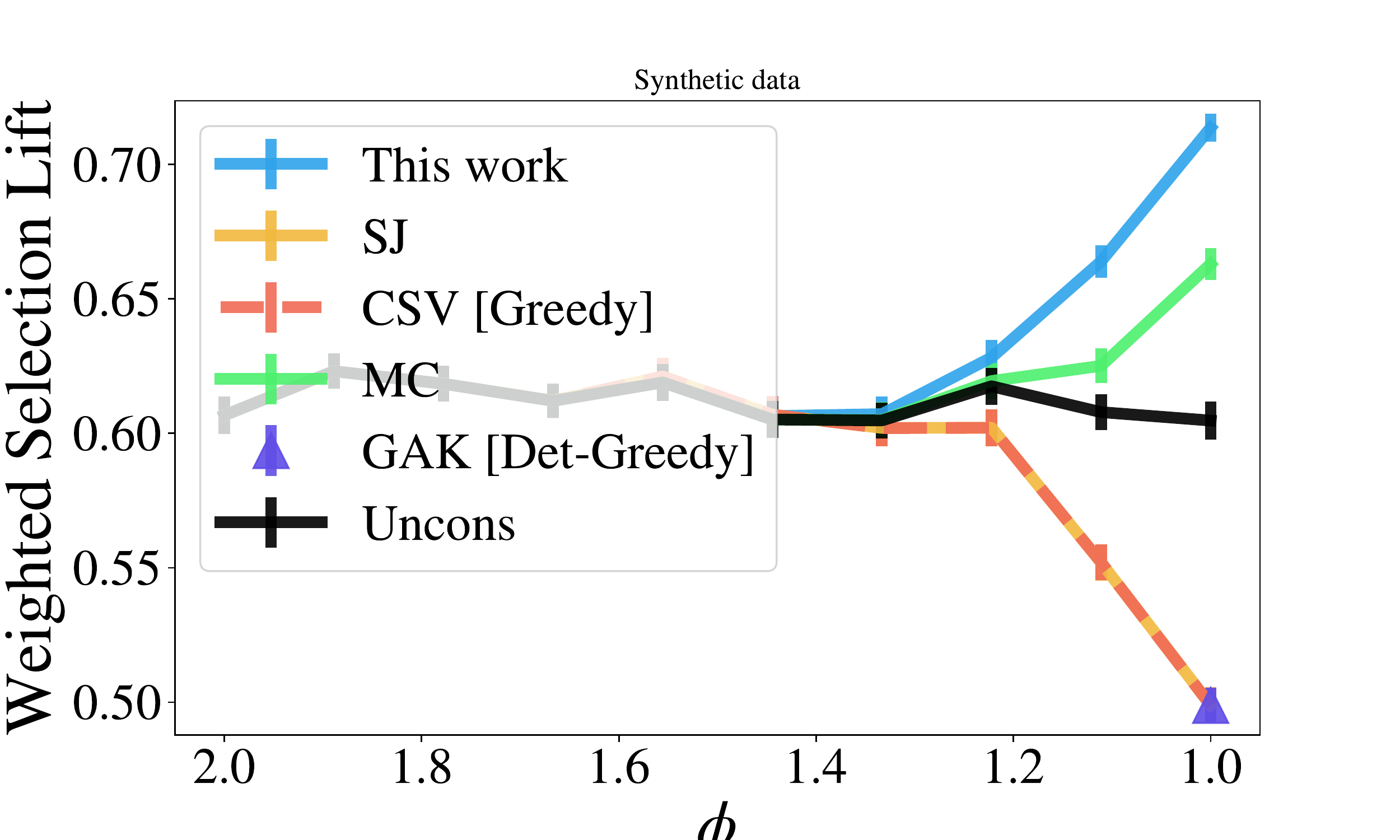}};
        \draw[draw=white, fill=white] (-2*\s+0.5, 3.3*5/8*\s-0.66*\s-0.17*\s+0.35-0.44) rectangle ++(4*\s-0.5,0.15*\s);
        \node[rotate=0, fill=white] at (0.2+0.05*\s, 0.15-3.3*5/8*\s-0.1*\s+0.15*\s-0.08)  {\white{|||}$\phi$ \white{|||}};
      \end{tikzpicture}}
      }
      \hspace{-6mm}
      \subfigure[\label{7b}]{
      {\begin{tikzpicture}[scale=1.5]
        \tikzmath{\s = 0.9;}
        \node (image) at (0,-0.17*\s) {\includegraphics[width=0.46\linewidth, trim={0cm 0cm 2cm 0.7cm},clip]{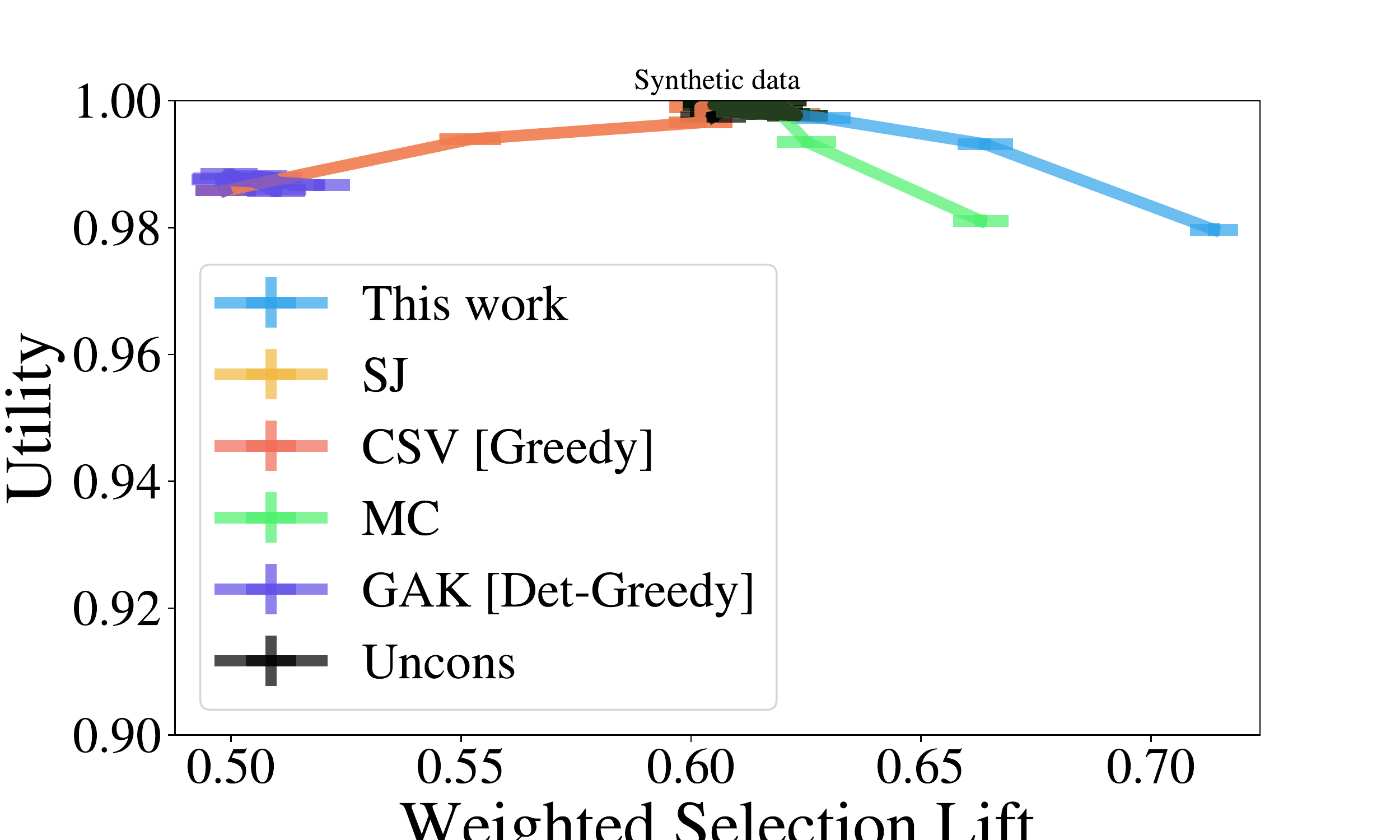}};
        \draw[draw=white, fill=white] (-2*\s+0.5, 3.3*5/8*\s-0.66*\s-0.17*\s+0.35-0.38-0.08) rectangle ++(4*\s-0.6,0.15*\s);
        \node[rotate=0, fill=white] at (0.2+0.05*\s, 0.15-3.3*5/8*\s-0.1*\s+0.15*\s-0.1)  {\white{|||}Weighted Selection Lift\white{|||}};
      \end{tikzpicture}}
      }
      \caption{
      {\em {Synthetic Data (Weighted Selection Lift): Nonuniform Error Rate.}}
      This simulation considers synthetic data where imputed socially-salient attributes have a higher false-discovery rate for one group compared to the other.
      We vary the fairness constraint from $\phi$ from $2$ (less fair) to $1$ (more fair) and observe the weighted risk-difference (weighted risk-difference) of different algorithms.
      In the first sub-figure, the $y$-axis plots weighted selection-lift and $x$-axis shows $\phi$.
      In the second sub-figure, the $y$-axis plots utility and $x$-axis shows weighted selection-lift.
      Error bars represent the error of the mean.
      }
      \label{fig:fig7}
    \end{figure}
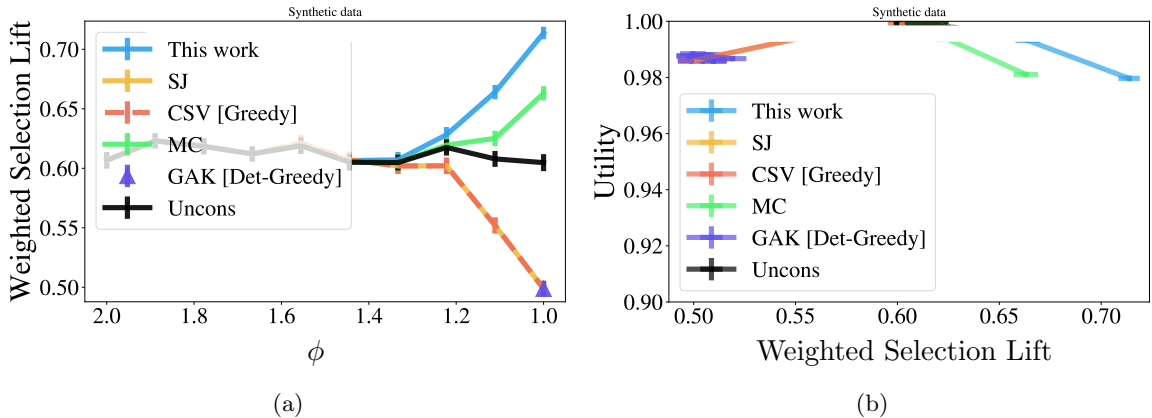

    \renewcommand{\folder}{./figs/image-sl}
    \begin{figure}[h!]
      \centering
      \subfigure[\label{8a}]{
      {\begin{tikzpicture}[scale=1.5]
        \tikzmath{\s = 0.9;}
        \node (image) at (0,-0.17*\s) {\includegraphics[width=0.46\linewidth, trim={0cm 0cm 2cm 0.9cm},clip]{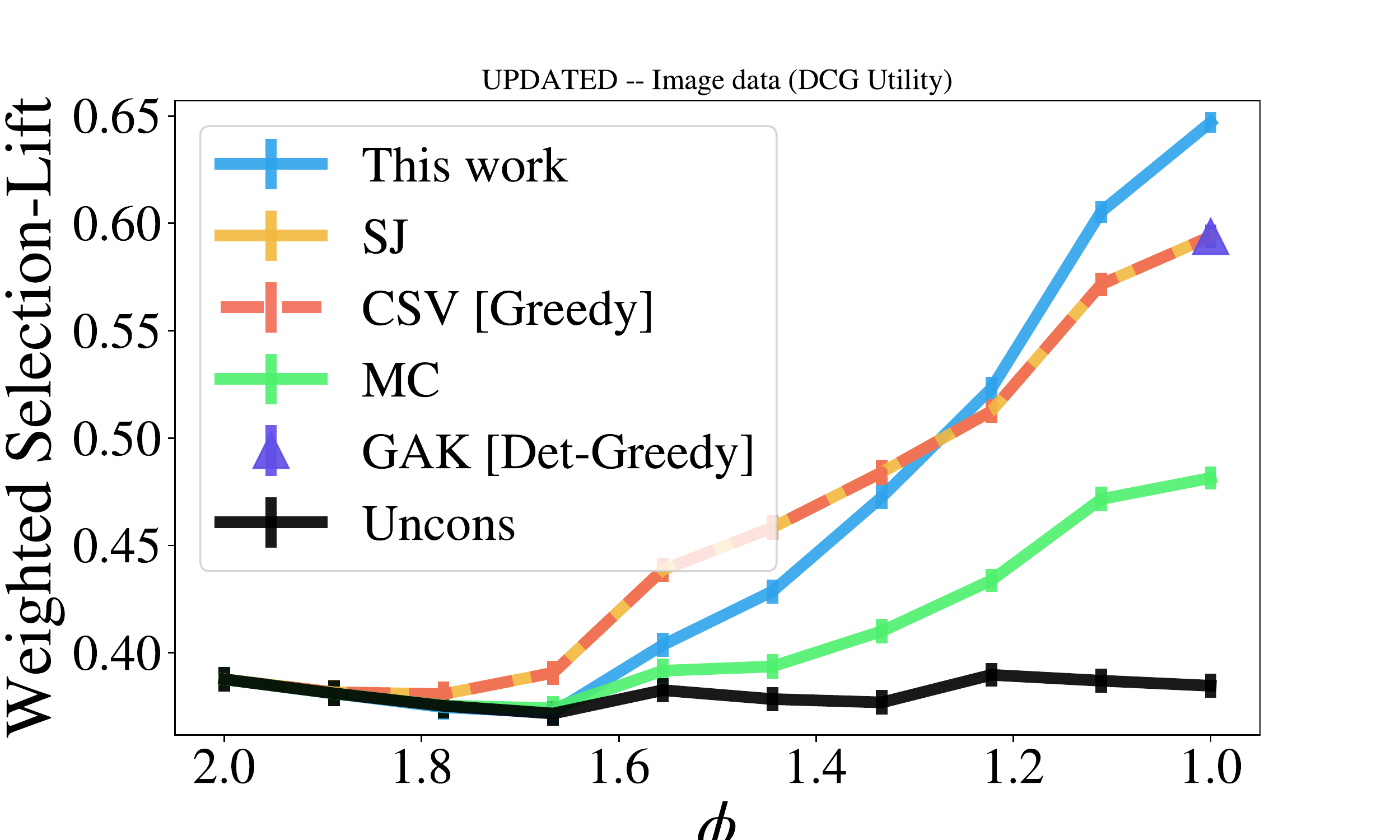}};
        \draw[draw=white, fill=white] (-2*\s+0.5, 3.3*5/8*\s-0.66*\s-0.17*\s+0.35-0.24) rectangle ++(4*\s-0.5,0.15*\s);
        \node[rotate=0, fill=white] at (0.2+0.05*\s, 0.15-3.3*5/8*\s-0.1*\s+0.15*\s-0.08)  {\white{|||}$\phi$ \white{|||}};
      \end{tikzpicture}}
      }
      \hspace{-6mm}
      \subfigure[\label{8b}]{
      {\begin{tikzpicture}[scale=1.5]
        \tikzmath{\s = 0.9;}
        \node (image) at (0,-0.17*\s) {\includegraphics[width=0.46\linewidth, trim={0cm 0cm 2cm 0.7cm},clip]{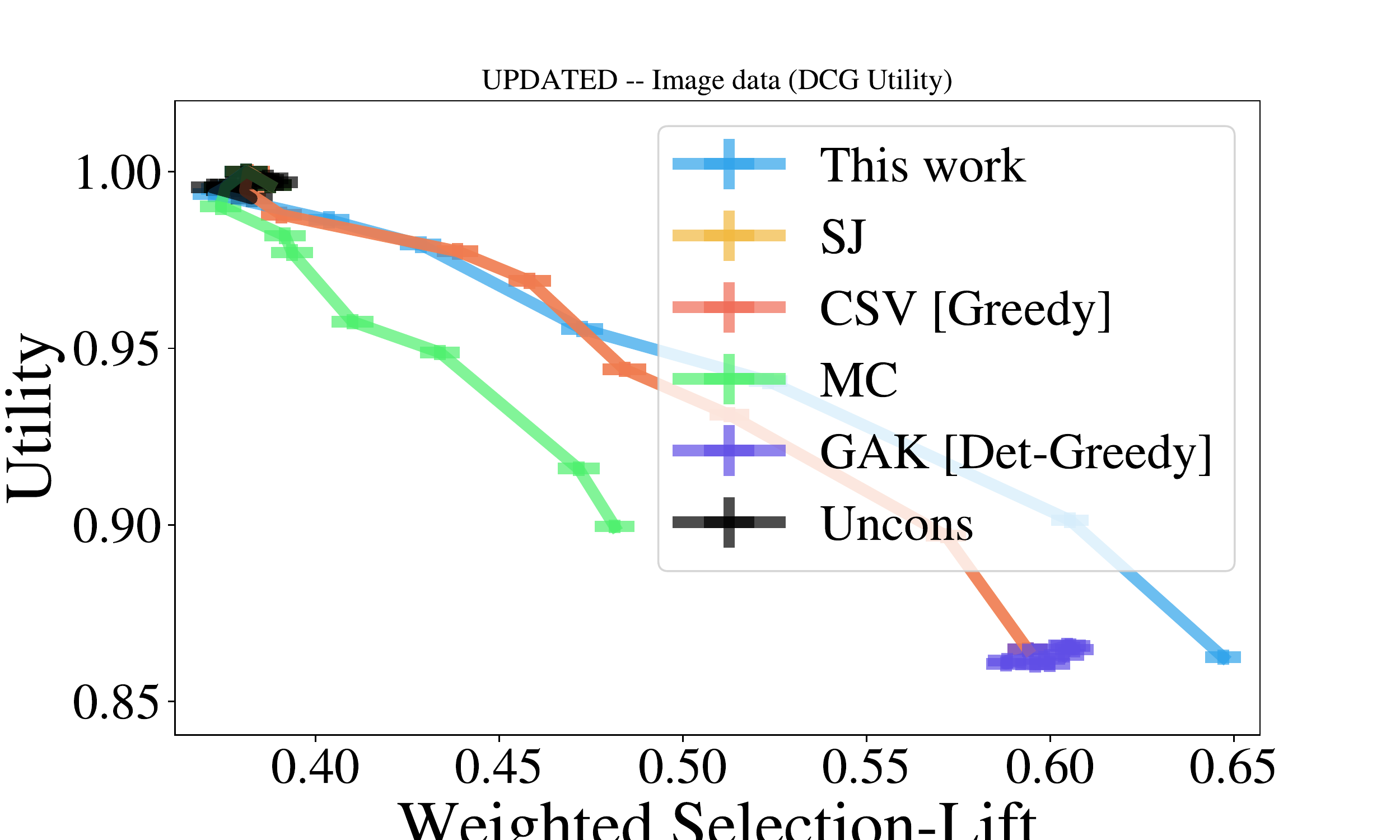}};
        \draw[draw=white, fill=white] (-2*\s+0.5, 3.3*5/8*\s-0.66*\s-0.17*\s+0.1) rectangle ++(4*\s-0.6,0.15*\s);
        \node[rotate=0, fill=white] at (0.2+0.05*\s, 0.15-3.3*5/8*\s-0.1*\s+0.15*\s-0.1)  {\white{|||}Weighted Selection Lift\white{|||}};
      \end{tikzpicture}}
      }
      \caption{
      {\em Real-World Image Data.}
      This simulation considers images-search results which are known to overrepresent the stereotypical gender~\cite{KayMM15}.
      Given relevant {\em non-gender labeled} images and their utilities, our goal is to generate a high-utility gender-balanced ranking.
      We estimate $P$ using an off-the-shelf ML-classifier and vary $\phi$ from $p=2$ (less fair) to $1$ (more fair).
      In the first sub-figure, the $y$-axis plots weighted selection-lift and $x$-axis shows $\phi$.
      In the second sub-figure, the $y$-axis plots utility and $x$-axis shows weighted selection-lift.
      Error bars represent the error of the mean.
      }
      \label{fig:fig8}
    \end{figure}

    \renewcommand{\folder}{./figs/chess-sl}
    \begin{figure}[h!]
      \centering
      \subfigure[\label{9a}]{
      {\begin{tikzpicture}[scale=1.5]
        \tikzmath{\s = 0.9;}
        \tikzmath{\mvx = 0;}
        \tikzmath{\mvy = 0;}
        \node (image) at (0,-0.17*\s) {\includegraphics[width=0.46\linewidth, trim={0cm 0cm 2cm 0.9cm},clip]{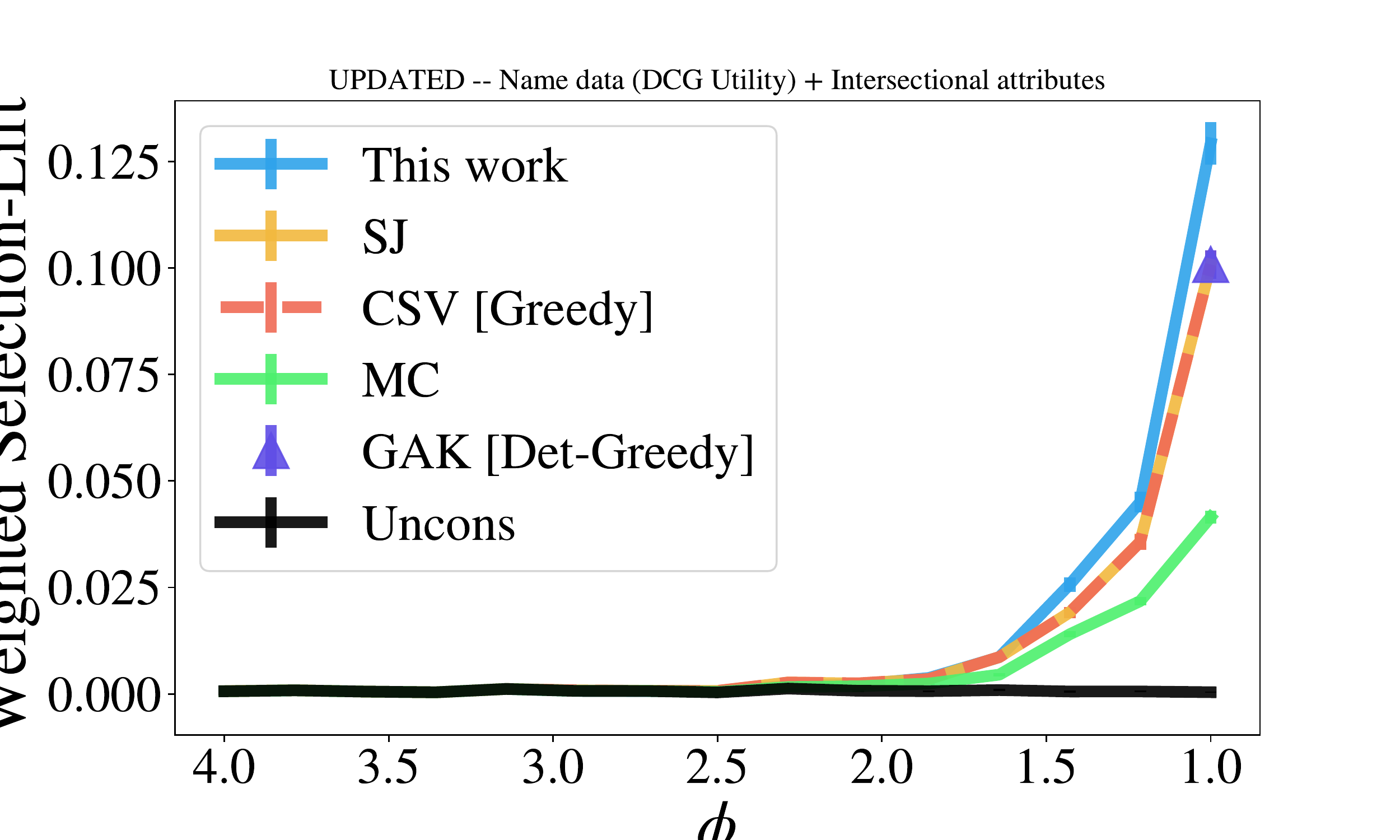}};
        \draw[draw=white, fill=white] (-2*\s+0.5, 3.3*5/8*\s-0.66*\s-0.17*\s+0.35-0.38+0.14) rectangle ++(4*\s-0.4,0.15*\s);
        \node[rotate=0, fill=white] at (0.2+0.05*\s, 0.15-3.3*5/8*\s-0.1*\s+0.15*\s-0.05)  {\white{|||}$\phi$ \white{|||}};
      \end{tikzpicture}}
      }
      \hspace{-6mm}
      \subfigure[\label{9b}]{
      {\begin{tikzpicture}[scale=1.5]
        \tikzmath{\s = 0.9;}
        \tikzmath{\mvx = 0;}
        \tikzmath{\mvy = 0;}
        \node (image) at (0,-0.17*\s) {\includegraphics[width=0.46\linewidth, trim={0cm 0cm 2cm 0.7cm},clip]{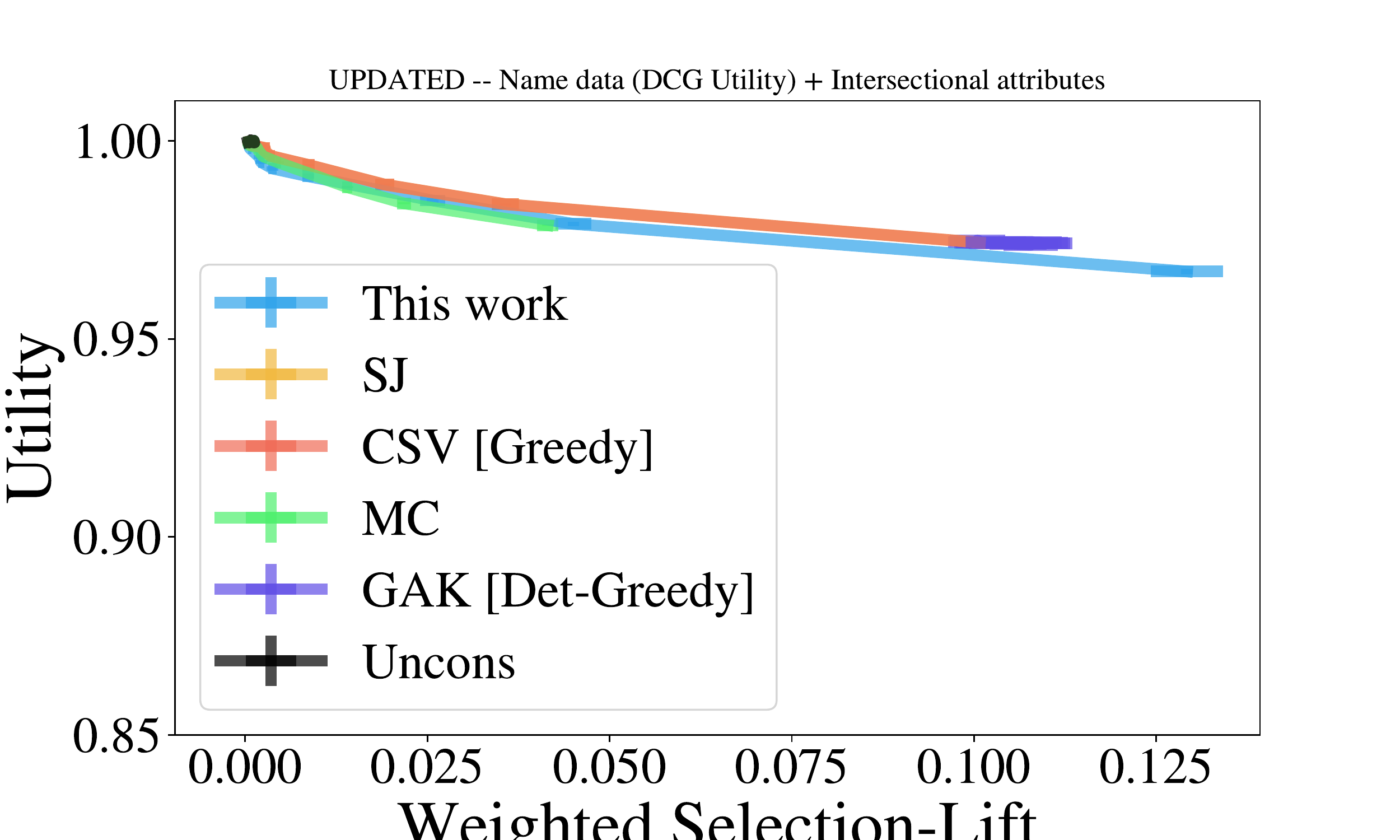}};
        \draw[draw=white, fill=white] (-2*\s+0.5, 3.3*5/8*\s-0.66*\s-0.17*\s+0.35-0.38+0.12) rectangle ++(4*\s-0.4,0.15*\s);
        \node[rotate=0, fill=white] at (0.2+0.05*\s, 0.15-3.3*5/8*\s-0.1*\s+0.15*\s-0.1)  {\white{|||}Weighted Selection Lift \white{|||}};
      \end{tikzpicture}}
      }
      \caption{
      {\em Real-World Name Data: Intersectional Attributes.}
        This simulation considers two socially-salient attributes, gender and race, and our goal is to ensure equal representation across the four {\em intersectional} socially-salient groups (non-White non-men, White non-men, non-White men, and White men).
        We estimate $P$ from the full names using public APIs and libraries.
        We vary $\phi$ from $p=4$ (less fair) to $1$ (more fair) and observe RD of all algorithms.
        In the first sub-figure, the $y$-axis plots weighted selection-lift and $x$-axis shows $\phi$.
        In the second sub-figure, the $y$-axis plots utility and $x$-axis shows weighted selection-lift.
        Error bars represent the error of the mean.
      }
      \label{fig:fig9}
    \end{figure}

    \subsubsection{Empirical Results with Varying Amount of Noise}\label{sec:vary_noise}
        In this section, we present a simulation which uses the randomized response mechanism to generate noisy protected attributes and compares the performance of algorithms at varying noise levels.

        \paragraph{\bf Data.}
            We use the Occupation images data \cite{celis2020cscw}.
            We refer the reader to \cref{sec:empirical_results} for a discussion of the data.

        \paragraph{\bf Setup.}
            We fix equal representation constraints ($\phi=1$) and consider the same protected groups as the simulation with the same data in \cref{sec:empirical_results}.
            We vary the noise level $0\leq \eta\leq \frac{1}{2}$.
            For each $\eta$, we construct noisy attributes by mislabeling true protected attribute with probability $\eta$.
            Here, $P$ is specified by $\eta$ as explained in \cref{rem:discussion_of_noise_model}.
            Specifically, if $N_1$ and $N_2$ be the noisy versions of true protected groups $G_1$ and $G_2$
            (corresponding to the ``flipped'' protected attributes), then we set:
            For each item $i \in N_1$,
            \begin{align*}
              \text{$\hP_{i1} = (1-\eta) \cdot \frac{\abs{G_1}}{\abs{N_1}}$ \quad and\quad $\hP_{i2} = 1 - \hP_{i1}$.}
          \end{align*}
          For items in $N_2$, replace $\hP_{i1}$, $\hP_{i2}$, $G_1$, and $N_1$ with $\hP_{i2}$, $\hP_{i1}$, $G_2$, and $N_2$.
          We do not have access to $G_1$ (and, hence, $\abs{G_1}$), and in the above expression we estimate $\abs{G_1}$ by $$\alpha_1\coloneqq \frac{(1-\eta) }{1-2\eta}\cdot \inparen{(1-\eta)\abs{N_1}-\eta\abs{N_2}}.$$
          This is because $\alpha_1$ can be shown to be concentrated around $\abs{G_1}$.

          Like the simulations in \cref{sec:empirical_results}, \csv{}, \detgreedy{}, and \sj{} are given the noisy attributes (as they require) and \ouralgo{} and \mc{} are given $\hP$ (computed above).

        \paragraph{\bf Observations.}
            See \cref{fig:vary_noise} for RD and utilities (NDCG) averaged over 100 iterations.
            We observe that for each $\eta\geq 0.1$, \ouralgo{}'s RD is $>$6.8\% better than any baseline (\cref{fig:vary_noise1}) and its utility is $<$3\% smaller than the baseline (\csv{}) with best RD (\cref{fig:vary_noise2}).
            At $\eta=0$, \ouralgo{} 3.3\% lower \rd{} than \csv{}, \detgreedy{}, and \sj{} and the same utility as them.

            Note that in \cref{fig:vary_noise1,fig:vary_noise2} the plots of \csv{}, \detgreedy{}, and \sj{} overlap.
            This is consistent with the other simulations where \csv{}, \detgreedy{}, and \sj{} have the same \rd{} and utility at $\phi=1$.

    \renewcommand{\folder}{./figs/vary-noise/}
    \begin{figure}[t!]
      \centering
      \hspace{-12mm}\subfigure[\label{fig:vary_noise1}]{
          \begin{tikzpicture}
            \tikzmath{\s = 1.1;}
            \tikzmath{\mvx = 0;}
            \tikzmath{\mvy = 1.9;}
            \node (image) at (0,-0.17*\s+\mvy) {{\includegraphics[width=0.50\linewidth, trim={0cm 0cm 2.5cm 0.9cm},clip]{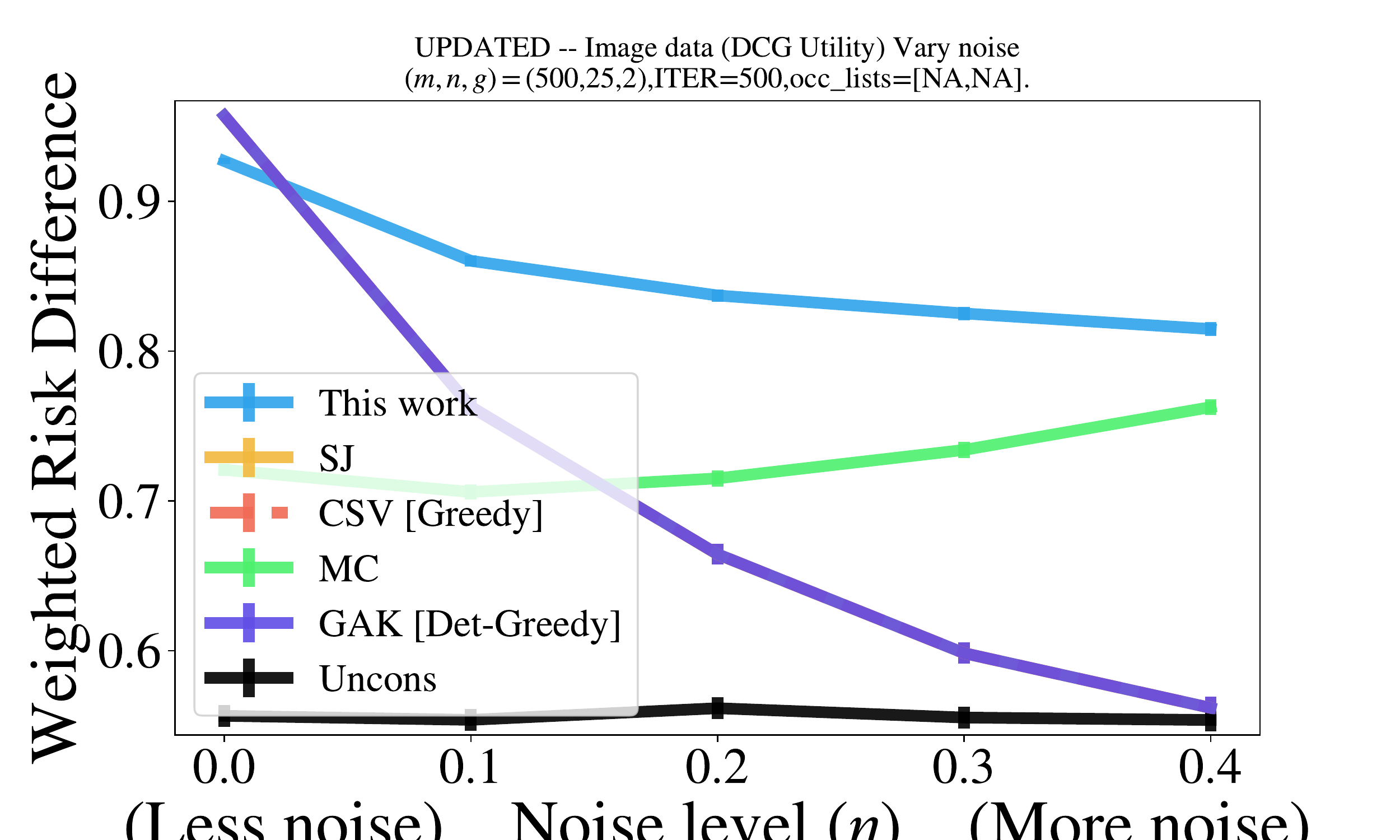}}};
            \draw[draw=white, fill=white] (+\mvx+-2*\s+0.35, 3.3*5/8*\s-0.66*\s-0.17*\s+0.45+\mvy+0.08-0.1+0.3) rectangle ++(4*\s,0.15*\s+0.2);
            \node[rotate=0,fill=white] at (0.5,0.15-3.3*5/8*\s-0.1*\s+0.15*\s-0.075-0.1+\mvy-0.13-0.4) {\textit{Noise Level ($\eta$)}};
            \node[rotate=0,fill=white] at (+\mvx+4*5/8*\s-0.15*\s+0.2+0.5,0.15-3.3*5/8*\s-0.1*\s+0.15*\s-0.075-0.1+\mvy-0.13-0.4) {\textit{(more noise)}};
            \node[rotate=0,fill=white] at (+\mvx+-3.75*5/8*\s+0.2*\s+0.6-0.5, 0.15*\s-3.3*5/8*\s-0.1*\s+0.15*\s-0.075-0.2+\mvy-0.05-0.4) {\textit{(less noise)}};
          \end{tikzpicture}
      }
      \hspace{-6mm}
      \subfigure[\label{fig:vary_noise2}]{
          \begin{tikzpicture}
            \tikzmath{\s = 1.1;}
            \tikzmath{\mvx = 0;}
            \tikzmath{\mvy = 1.9;}
            \node (image) at (0,-0.17*\s+\mvy) {{\includegraphics[width=0.50\linewidth, trim={0cm 0cm 2.5cm 0.9cm},clip]{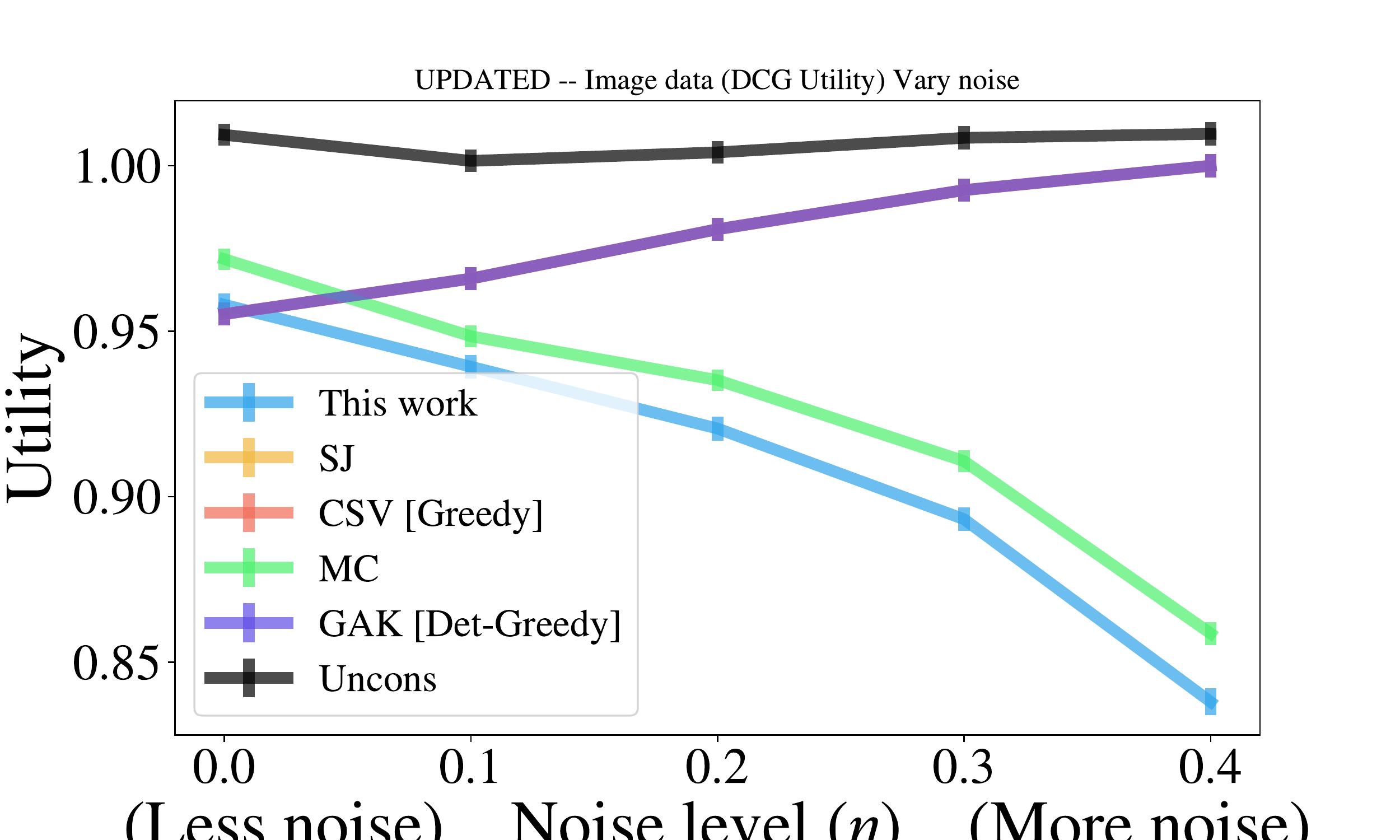}}};
            \draw[draw=white, fill=white] (+\mvx+-2*\s+0.35, 3.3*5/8*\s-0.66*\s-0.17*\s+0.45+\mvy+0.08-0.1+0.3) rectangle ++(4*\s,0.15*\s+0.2);
            \node[rotate=0,fill=white] at (0.5,0.15-3.3*5/8*\s-0.1*\s+0.15*\s-0.075-0.1+\mvy-0.13-0.4) {\textit{Noise Level ($\eta$)}};
            \node[rotate=0,fill=white] at (+\mvx+4*5/8*\s-0.15*\s+0.2+0.5,0.15-3.3*5/8*\s-0.1*\s+0.15*\s-0.075-0.1+\mvy-0.13-0.4) {\textit{(more noise)}};
            \node[rotate=0,fill=white] at (+\mvx+-3.75*5/8*\s+0.2*\s+0.6-0.5, 0.15*\s-3.3*5/8*\s-0.1*\s+0.15*\s-0.075-0.2+\mvy-0.05-0.4) {\textit{(less noise)}};
          \end{tikzpicture}
      }
      \hspace{-12mm}
      \protect\rule{0ex}{5ex}
      \caption{
      \protect\rule{0ex}{5ex}
      {\em Simulation varying the amount of noise.}
      In this simulation, we use the  Occupation's images data \cite{celis2020cscw} and generate noisy protected attributes using the randomized response mechanism, with parameter $\eta$.
        We vary the amount of noise added from $\eta=0$ (no noise) to $\eta=0.4$ (large noise) and compare the performance of different algorithms.
        The $y$-axis plots RD and $x$-axis plots $\eta$.
        We present the key observations in the paragraph above the figure.
        Error-bars denote the error of the mean.
      }
      \label{fig:vary_noise}
    \end{figure}

    \subsubsection{Empirical Results with Proportional Representation Constraints}\label{sec:prop_repr}
        In this section, we present variants of the simulations in \cref{fig:different_fdr,fig:simulation_image,fig:simulation_intersectional} that use proportional representation fairness constraints.
        To measure the deviation of a ranking from proportional representation, we consider an  adaptation of weighted risk-difference metric, \prd{}.
        \prd{} of a ranking $R$ is
        \begin{align*}
           { 1-\frac1Z\sum_{ k=5,10,\dots} \frac{1}{\log{k}}
          \max\nolimits_{\ell,q\in [p]}\abs{{ \frac{n}{\abs{G_\ell}}\cdot \sum_{\substack{i\in G_\ell, j\in [k]\white{,}}}   R_{ij}  - \frac{n}{\abs{G_q}}\cdot\sum_{\substack{i\in G_q, j\in [k]\white{,}}}   R_{ij}}}.}
          \yesnum
        \end{align*}
        Where $G$ denotes the ground-truth protected groups and $Z$ is a constant so that $\rd$ has range $[0,1]$.
        Here, \prd{}$=1$ is most fair and \prd{}$=0$ is least fair.

        \paragraph{Setup.} The setup of the simulations is identical to the simulations in \cref{fig:different_fdr,fig:simulation_image,fig:simulation_intersectional} except that, given $\phi\geq 1$, the upper bounds are set to $U_{k\ell} \coloneqq \phi\cdot \smash{\frac{\abs{G_\ell}}{n}}\cdot k$ for each $k\in [n]$ and $\ell\in [p]$.

        \paragraph{Observations.}
        \cref{fig:prop_repr_syn} presents the values of \prd{} averaged over 50 iterations.
        We observe that,  relative to the baselines, \ouralgo{}'s performance is similar to \cref{fig:different_fdr,fig:simulation_image,fig:simulation_intersectional}.
        In particular, in all simulations, \ouralgo{} achieves a higher value of the fairness metric than any baselines, as in \cref{fig:different_fdr,fig:simulation_image,fig:simulation_intersectional}.
        Further, in the simulation with the real-world image data, \ouralgo{} has a better fairness-utility trade-off than all baselines, as in \cref{fig:simulation_image}.
        One difference is that, with the synthetic data, the value of the fairness metric achieved by \ouralgo{} can be non-monotonous in $\phi$, whereas it is increasing in $\phi$ in \cref{fig:different_fdr} (see \cref{fig:prop_repr_syn}).

    \renewcommand{\folder}{./figs/prop-repr}
    \begin{figure}[h!]
      \centering
      \subfigure[Synthetic data]{
        \begin{tikzpicture}[scale=1.5]
          \tikzmath{\s = 0.9;}
          \node[rotate=0] (image) at (0,-0.17*\s) {\includegraphics[width=0.39\linewidth, trim={0cm 0cm 0cm 0cm},clip]{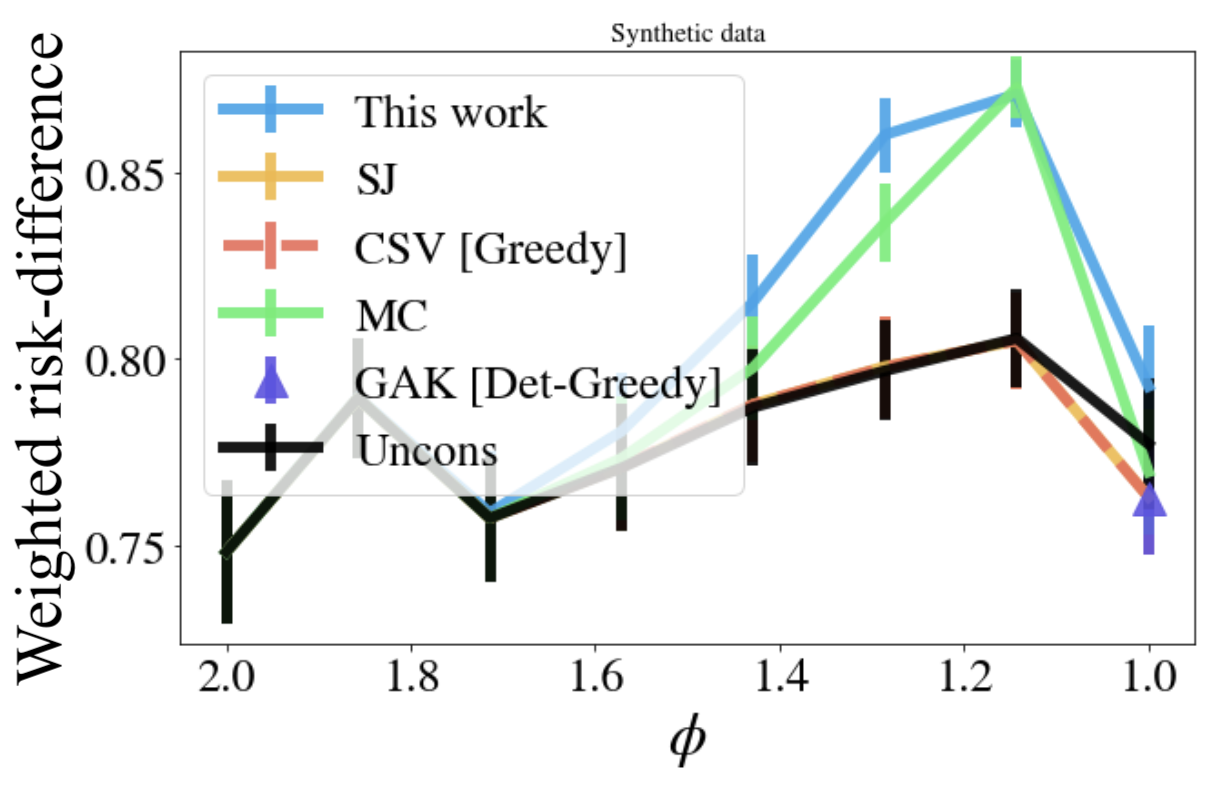}};
          \draw[draw=white, fill=white] (-1.6,1.05) rectangle ++(3,0.2);
        \end{tikzpicture}
      }
      \subfigure[Image data]{
         \begin{tikzpicture};
           \node (image) at (0,-0.17) {\includegraphics[width=0.41\linewidth, trim={0cm 0cm 0cm 0cm},clip]{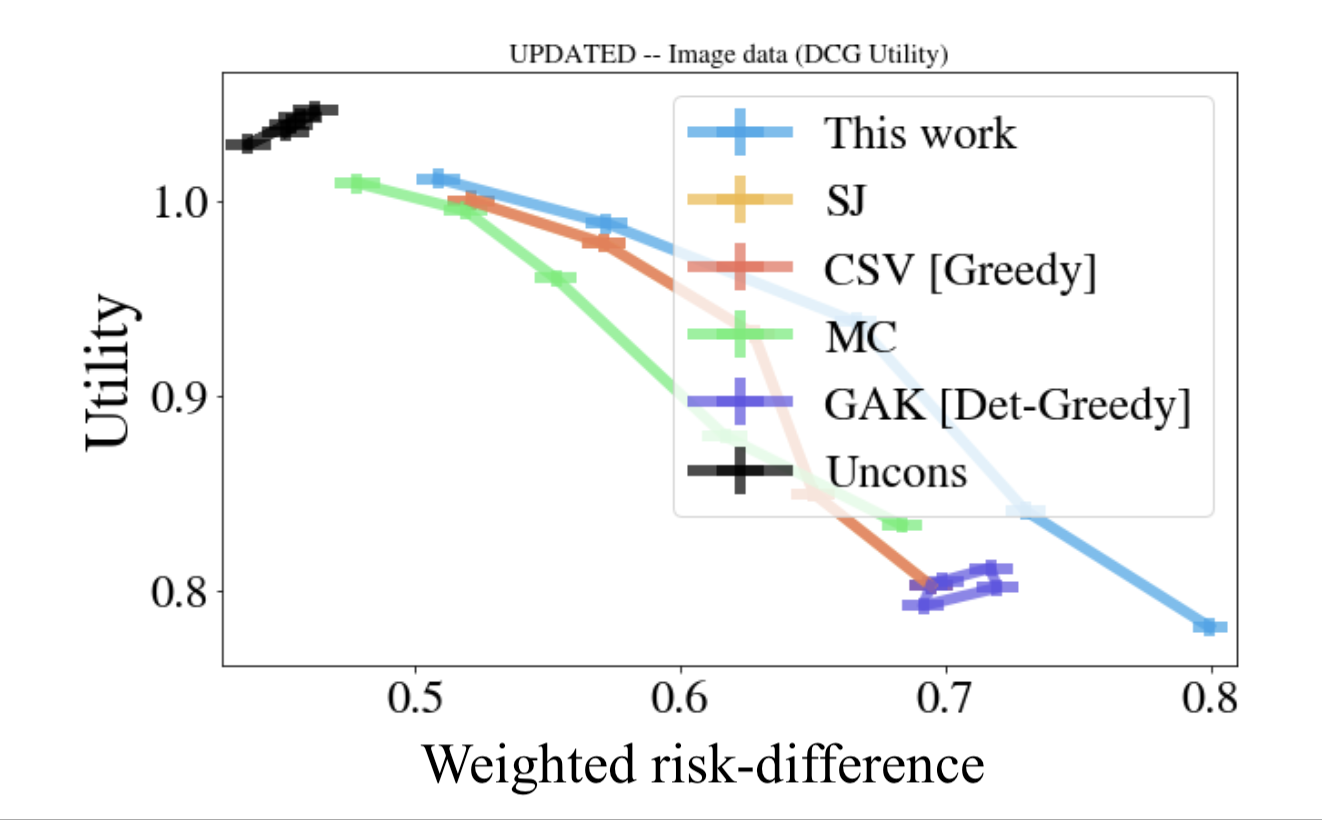}};
           \draw[draw=white, fill=white] (-2.15,1.6) rectangle ++(5,0.55);
         \end{tikzpicture}
      }
      \subfigure[Name data]{
        \begin{tikzpicture};
          \node (image) at (0,-0.17) {\includegraphics[width=0.42\linewidth, trim={0cm 0cm 0cm 0cm},clip]{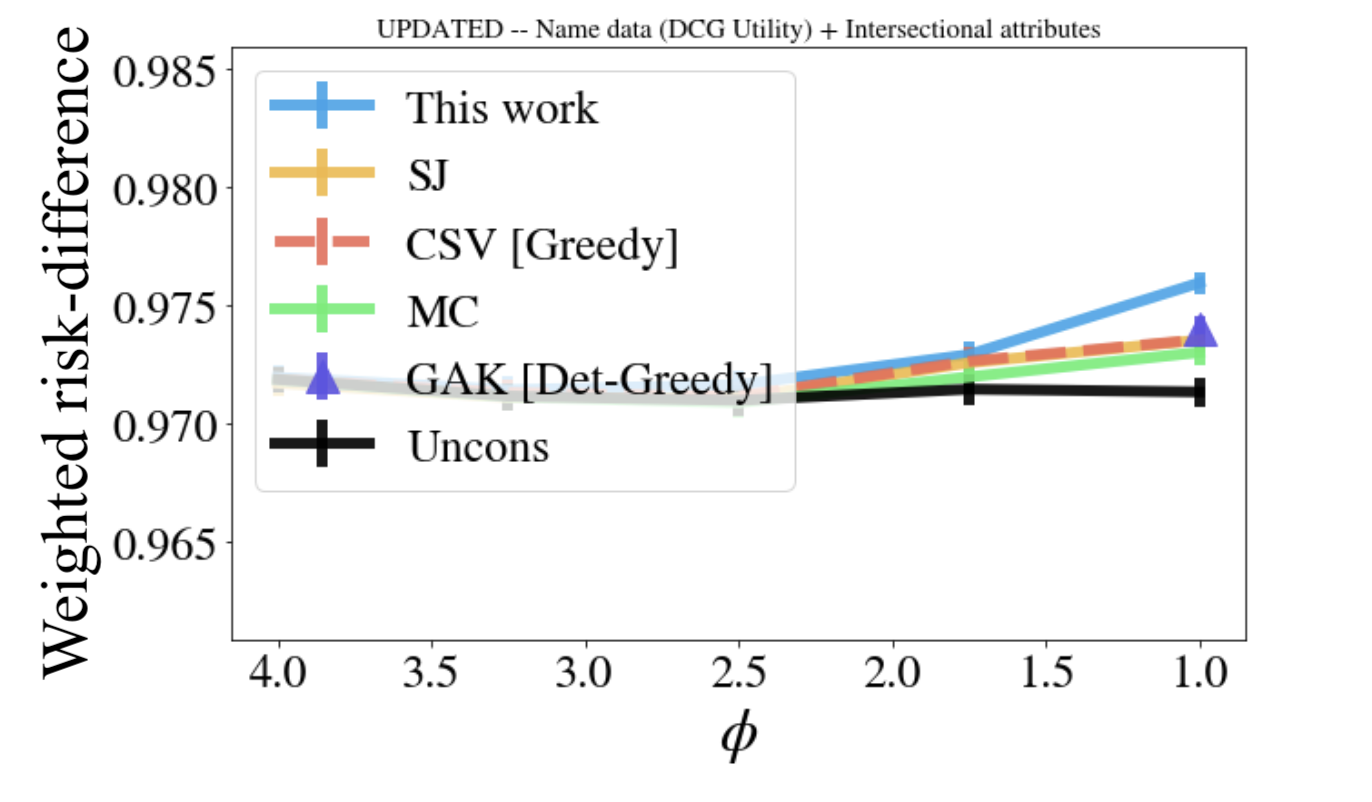}};
          \draw[draw=white, fill=white] (-2.15,1.65) rectangle ++(5,0.55);
        \end{tikzpicture}
      }
      \caption{
      Simulations with proportional representation constraints and variant of \rd{} for proportional representation constraints.
      The details appear in \cref{sec:prop_repr}.
      }
         \label{fig:prop_repr_syn}
    \end{figure}
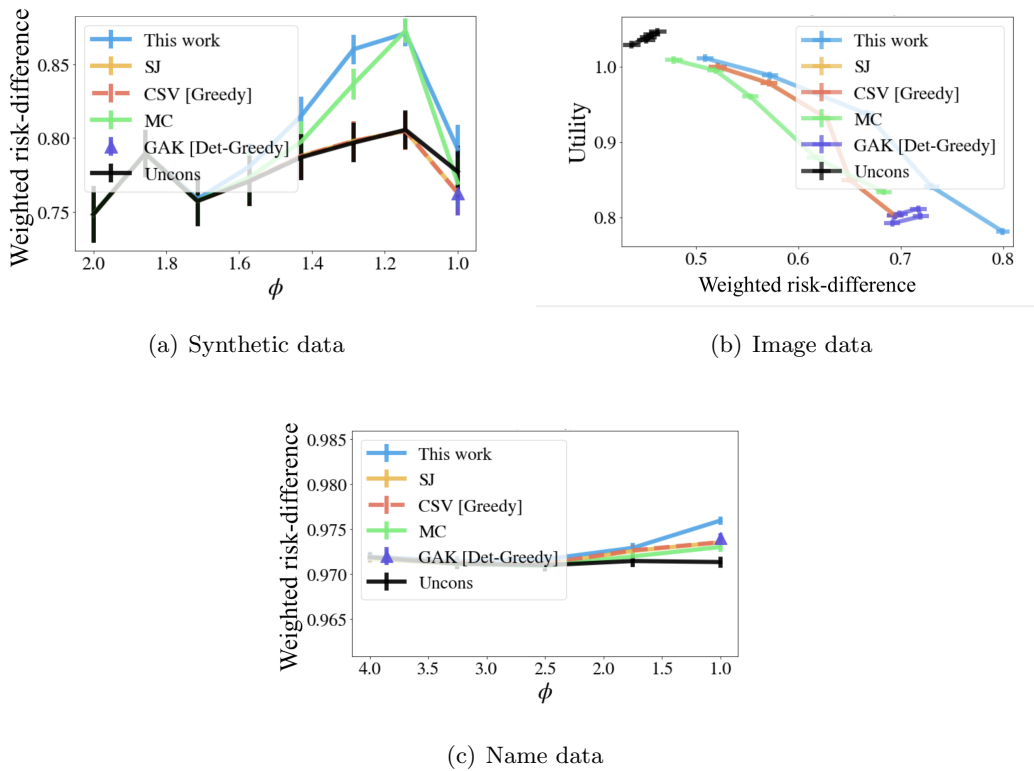

    \subsubsection{Empirical Results with Varying False-Discovery Rates}\label{sec:vary_fdr}
        In this section, we present a variant of the simulation in \cref{fig:different_fdr}.
        This simulation varies the difference in false-discovery rates (FDRs) of the attributes inferred from $\hP$ for the groups.

        \renewcommand{\folder}{./figs/vary-FDR}
        \begin{figure}[h!]
          \centering
          \vspace{-5mm}
          \subfigure[Minority group's FDR is 5\% smaller than the majority's FDR]{
             \begin{tikzpicture};
               \node (image) at (0,-0.17) {\includegraphics[width=0.4\linewidth, trim={0cm 0cm 0cm 0cm},clip]{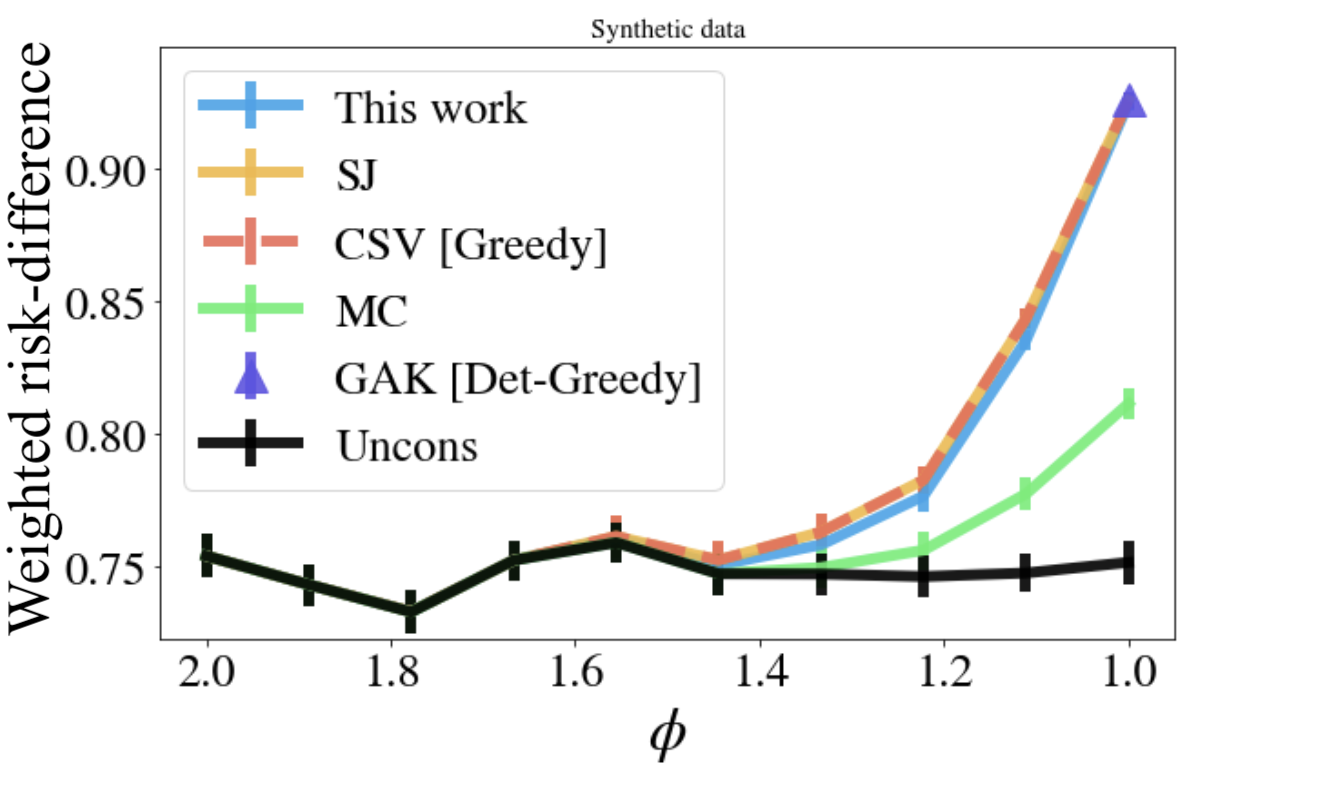}};
               \draw[draw=white, fill=white] (-2.15,1.8) rectangle ++(5,0.55);
             \end{tikzpicture}
          }
          \subfigure[Minority group's FDR is 10\% smaller than the majority's FDR]{
             \begin{tikzpicture};
               \node (image) at (0,-0.17) {\includegraphics[width=0.4\linewidth, trim={0cm 0cm 0cm 0cm},clip]{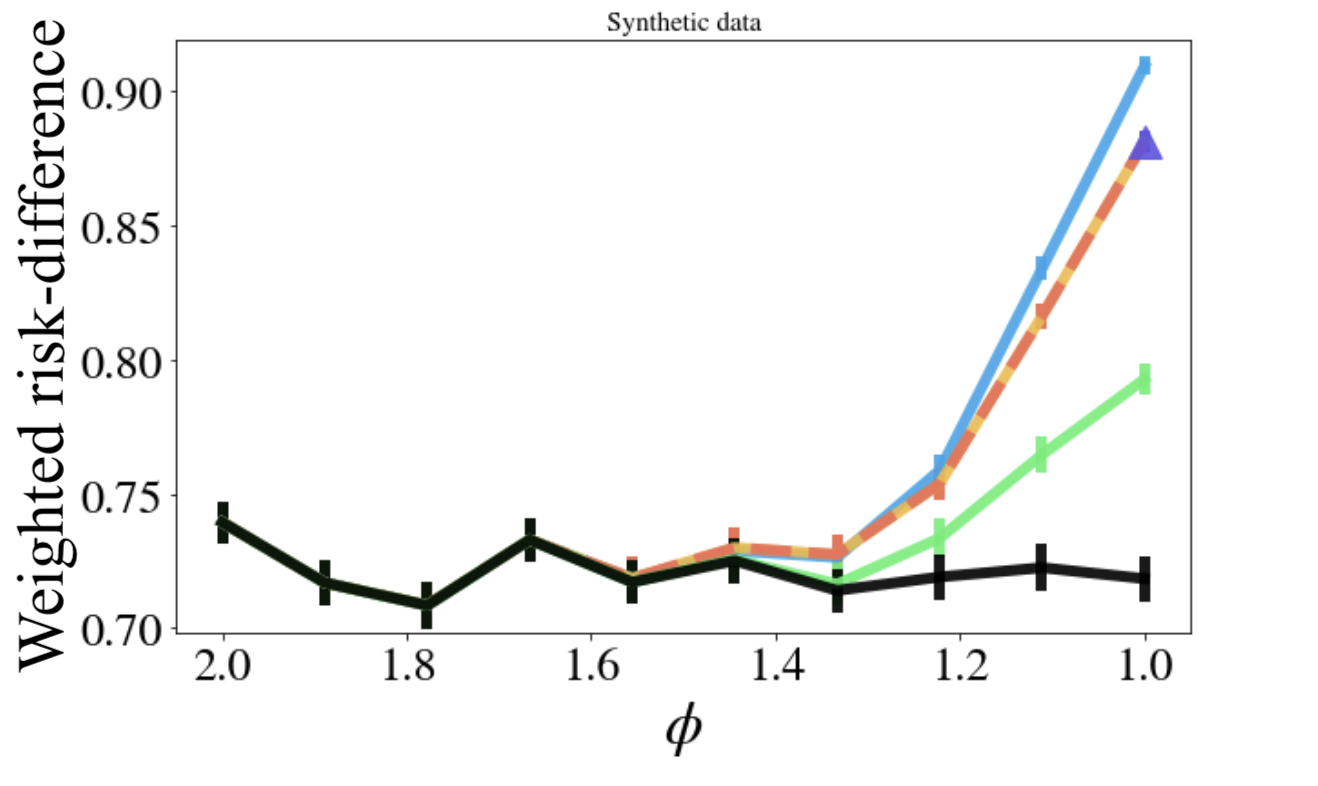}};
               \draw[draw=white, fill=white] (-2.15,1.85) rectangle ++(5,0.55);
             \end{tikzpicture}
          }
          \par \vspace{-4mm}
          \subfigure[Minority group's FDR is 20\% smaller than the majority's FDR]{
             \begin{tikzpicture};
               \node (image) at (0,-0.17) {\includegraphics[width=0.4\linewidth, trim={0cm 0cm 0cm 0cm},clip]{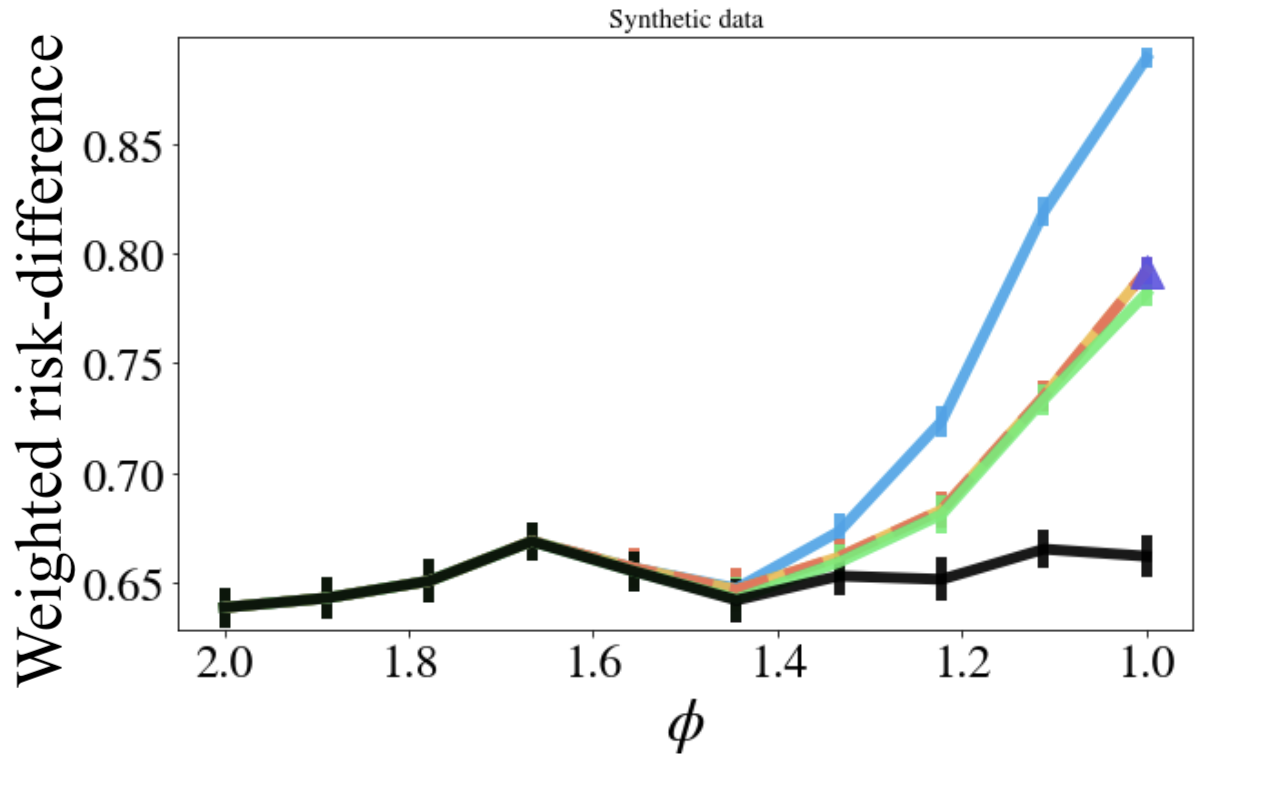}};
               \draw[draw=white, fill=white] (-2.15,1.95) rectangle ++(5,0.55);
             \end{tikzpicture}
             \vspace{-4mm}
          }
          \vspace{-2mm}
          \caption{
            Simulation on synthetic data where the minority group's FDR is $\Delta=5\%,10\%,20\%$ smaller than the majority's FDR.
            The details appear in \cref{sec:vary_fdr}.
          }
          \label{fig:vary_FDR}
        \end{figure}
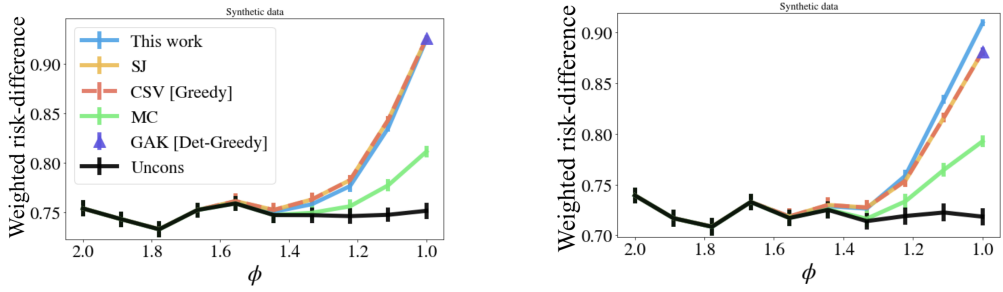
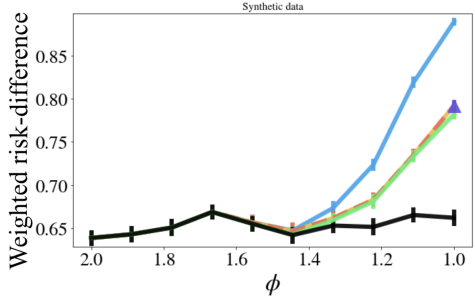

        \paragraph{Synthetic Data.}
            We generate utilities $w_1,w_2,\dots,w_m$ by drawing $w_i$ is independently from the uniform distribution over $[0,1]$ for each $1\leq i\leq m$.
            Fix $\mu_1\coloneqq 1-\frac{1}{20}$, $\mu_2\coloneqq \frac{1}{2}-\frac{1}{20}$, $\sigma_1\coloneqq \frac{1}{50}$, and $\sigma_2\coloneqq \frac{1}{10}$.
            Given a parameter $0\leq \tau\leq 1$, controlling the FDRs of the two groups, we construct $P$ as follows:
            For each $i$, with probability 0.6, $P_{i1}$ is drawn from $$\mathcal{N}\inparen{(1-\tau)\cdot \mu_1+\tau, (1-\tau)\cdot \sigma_1}$$ and otherwise $P_{i1}$ is drawn $$\mathcal{N}\inparen{(1-\tau)\cdot \mu_2+\tau\cdot 0, (1-\tau)\cdot \sigma_2}.$$
            We set $P_{i2}\coloneqq 1-P_{i1}$ for each $i$.
            This ensures that, with high probability,
            \begin{align*}
                \abs{G_1}=0.6 n\pm o_n(1)
                \quad\text{and}\quad
                \abs{G_2}=0.4 n\pm o_n(1)
            \end{align*}
            Let $\text{FDR}_1(\tau)$ and $\text{FDR}_2(\tau)$ be the false-positive rates of attributes inferred from $P$ on groups $G_1$ and $G_2$ for a given $\tau$.
            Let $$\Delta(\tau)\coloneqq \text{FDR}_2(\tau)-\text{FDR}_2(\tau).$$
            We have $\Delta(0)=0.4$, $\Delta$ decreases with $\tau$, and $\Delta(1)=0$.

        \paragraph{Setup.}
            The setup is identical the simulation in \cref{fig:different_fdr} except that we use the above synthetic data.
            We consider three values $\tau_1,\tau_2,$ and $\tau_3$ of $\tau$ such that $\Delta(\tau_1)=20\%$, $\Delta(\tau_2)=20\%$, and $\Delta(\tau_3)=20\%$.
            (For comparison, the FDRs of the two groups differ by 30\% for the simulation in \cref{fig:different_fdr}.)

        \paragraph{Observations.}
            See \cref{fig:vary_FDR} for $\rd$ averaged over 50 iterations.
            We observe that the difference between the best $\rd{}$ of \ouralgo{} those of $\sj$ and \csv{} decreases with $\Delta$: At $\Delta=20\%, 10\%,5\%$, \ouralgo{}'s \rd{} is 12\%, 4\%, and 0\% higher than $\sj$'s and \csv{}'s \rd{} respectively.

    \subsubsection{Empirical Results with a Large Number of Groups}\label{sec:many_grps}
        In this section, we present simulations on synthetic datasets with 4, 6, 8, and 10 protected groups.

        For simplicity, all groups have equal sizes.
        In particular, we construct variants of the synthetic dataset in \cref{sec:empirical_results} so that the false-discovery rates of the attributes inferred from the matrix $\hP$ on the groups are spread at equal intervals in the interval $[10\%, 40\%]$.
        For instance, for $p=4$, the FDRs of the four groups are 10\%, 20\%, 30\%, and 40\% respectively.

        \paragraph{Synthetic Data.}
            We generate utilities $w_1,w_2,\dots,w_m$ by drawing $w_i$ is independently from the uniform distribution over $[0,1]$ for each $1\leq i\leq m$.
            Fix $\mu_1\coloneqq 1-\frac{1}{20}$, $\mu_2\coloneqq \frac{1}{2}+\frac{1}{20}$, $\sigma_1\coloneqq \frac{1}{50}$, and $\sigma_2\coloneqq \frac{1}{10}$.
            For each group $G_\ell$, there is a parameter $0\leq \tau_\ell\leq 1$, that controls the corresponding FDR.
            We construct $P$ as follows:
            For each $\ell$ and $i\in G_\ell$,
            \begin{itemize}
                \item $P_{i1}$ is iid from $\mathcal{N}\inparen{(1-\tau)\cdot \mu_1+\tau\cdot \mu_2, (1-\tau)\cdot \sigma_1 + \tau\cdot \sigma_2}.$
                \item $P_{iz}\coloneqq 1-P_{i1}$ where $z$ is drawn uniformly at random from $[p]\backslash \inbrace{\ell}$
                \item $P_{ij}=0$ for each $j\in [p]\backslash\inbrace{\ell,z}$.
            \end{itemize}
            Let $\Delta(\tau)$ be the FDR of $G_\ell$ when $\tau_\ell=\tau$.
            (By construction, this function is independent of $\ell$.)

        \paragraph{Setup.}
            The setup is identical the simulation in \cref{fig:different_fdr} except that we use the above synthetic data to generate $w$ and $P$.
            We vary $p\in \inbrace{4,6,8,10}$.
            For each $p$, we fix $\tau_\ell$ such that, for each $\ell\in [p]$, $\Delta(\tau_\ell)\coloneqq 10\%+\frac{\ell-1}{p-1}\cdot 30\%.$

       \paragraph{Observation.}
            \cref{fig:many_grps} plots $\rd$ averaged over 50 iterations.
            We observe that \ouralgo{} has a better or similar (within 1\%) utility and \rd{} compared to the best performing baseline at all values of $\phi.$

    \renewcommand{\folder}{./figs/multiple-groups}
    \begin{figure}[h!]
      \centering
      \subfigure[Simulation with 4 groups]{
          \includegraphics[width=0.4\linewidth, trim={0cm 0cm 1cm 0.65cm},clip]{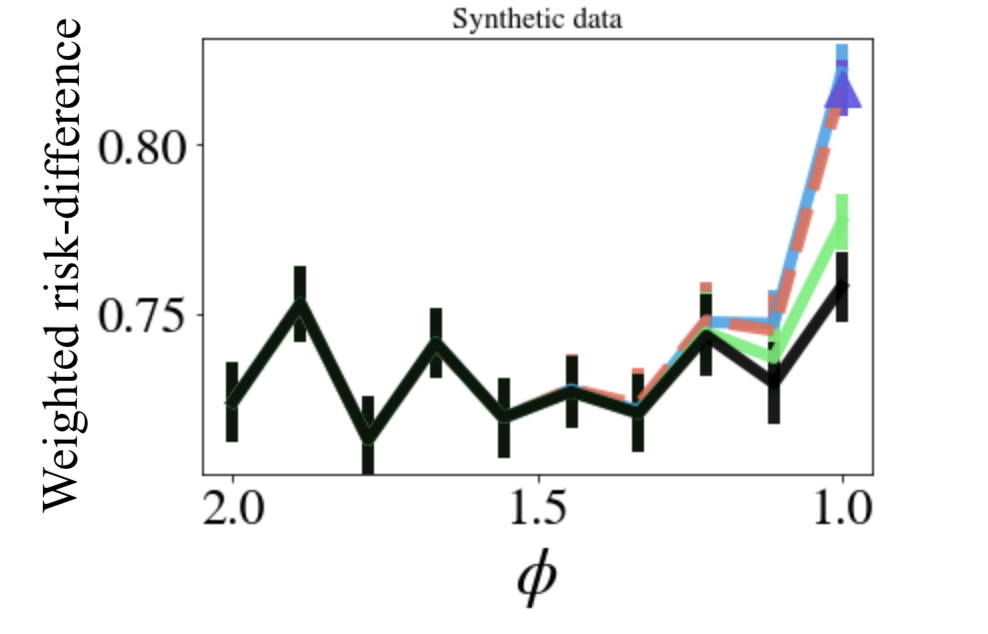}
      }
      \subfigure[Simulation with 6 groups]{
          \includegraphics[width=0.4\linewidth, trim={0cm 0cm 1cm 0.9cm},clip]{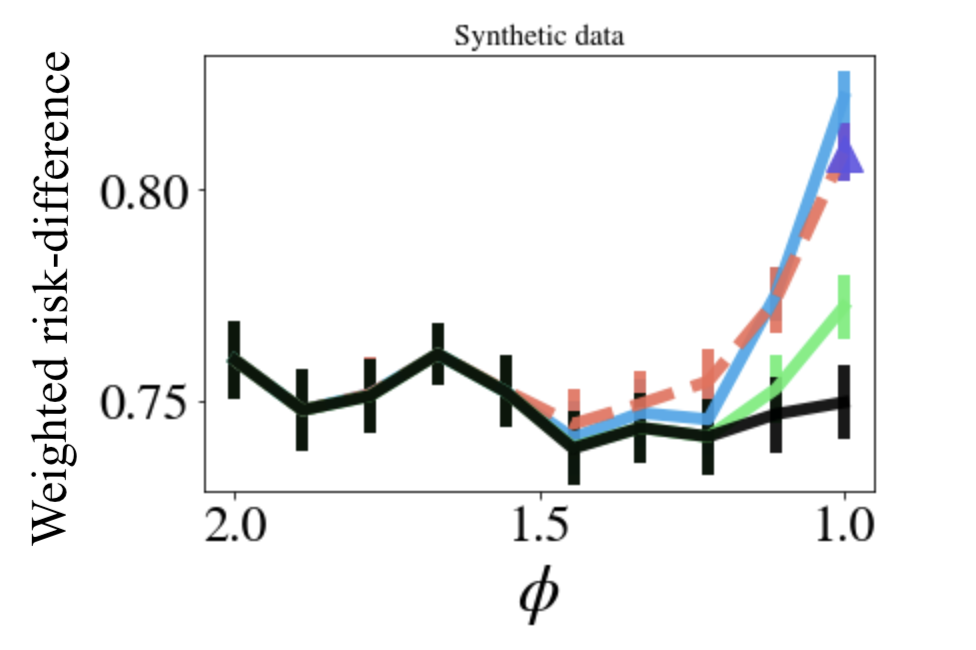}
      }\par
      \subfigure[Simulation with 8 groups]{
          \includegraphics[width=0.4\linewidth, trim={0cm 0cm 1cm 0.75cm},clip]{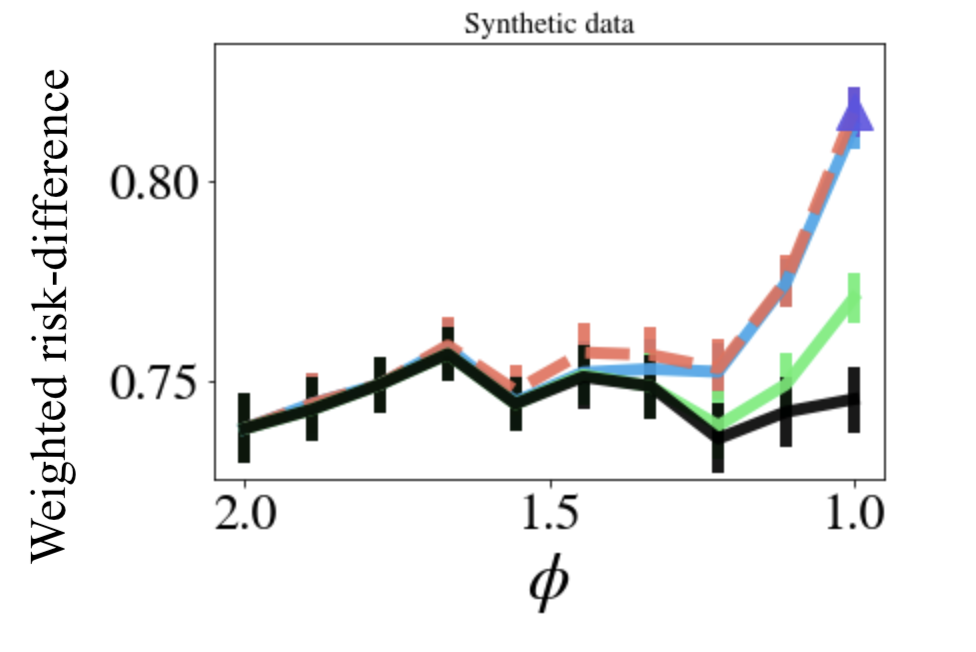}
      }
      \subfigure[Simulation with 10 groups]{
        \includegraphics[width=0.4\linewidth, trim={0cm 0cm 1cm 0.75cm},clip]{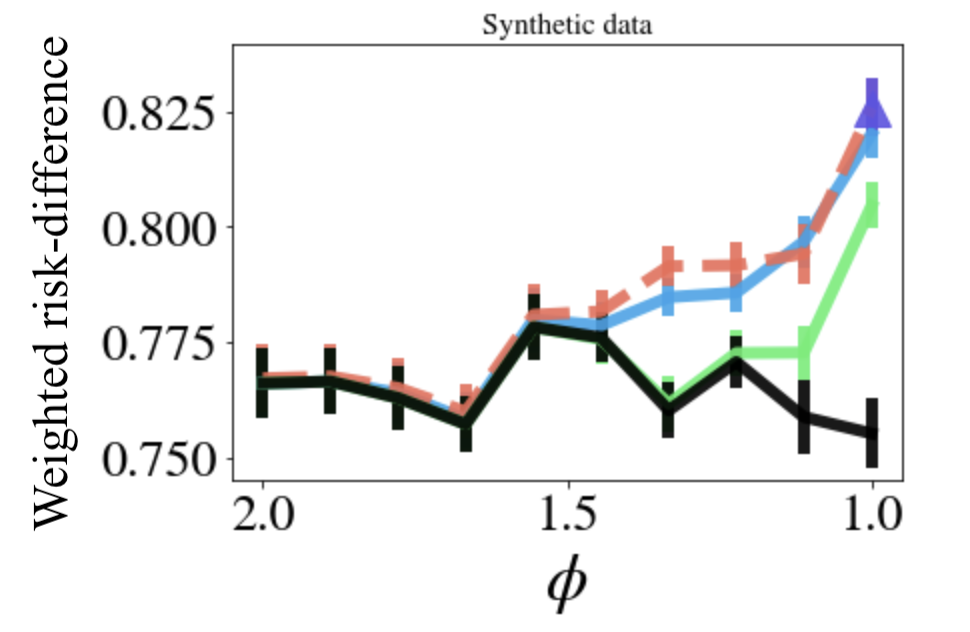}
      }
      \caption{
      Simulation on synthetic data with four, six, eight, and ten groups.
      The details appear in \cref{sec:many_grps}.
      }
      \label{fig:many_grps}
    \end{figure}

\end{document}